\documentclass[preprint]{imsart}
\usepackage[margin=1.1in,a4paper]{geometry}\RequirePackage{etoolbox}

\usepackage[utf8]{inputenc}
\usepackage{amsmath}
\usepackage{amsfonts}
\usepackage{amssymb}
\usepackage{graphicx}
\usepackage{amsthm}
\usepackage{yfonts}
\usepackage{accents}
\usepackage{enumerate}
\usepackage[colorlinks,citecolor=blue,urlcolor=blue,filecolor=blue,linkcolor=blue,backref=page]{hyperref}

\allowdisplaybreaks

\usepackage{tikzsymbols}
\usetikzlibrary{positioning}
\usetikzlibrary{%
	calc,%
	decorations.pathmorphing,%
	fadings,%
	shadings,%
	hobby,
	arrows.meta,
	shapes.arrows,
	graphs,
	quotes,
	decorations.markings,
	patterns
}

\usepackage[ruled, linesnumbered]{algorithm2e}
\usepackage{multicol}
\usepackage{caption, subcaption}
\usepackage{float}
\usepackage{dsfont}
\usepackage{accents}
\usepackage{natbib}
\numberwithin{equation}{section}

\newtheorem{theorem}{Theorem}

\newtheorem{conj}{Conjecture}

\newtheorem{lem}{Lemma}[section]

\newtheorem{prop}{Proposition}

\newtheorem{remark}{Remark}

\newtheorem{proposition}{Proposition}

\newtheorem{definition}{Definition}

\newtheorem{assumption}{Assumption}

\newtheorem{prob}{Problem}

\DeclareMathOperator*{\argmin}{arg\,min}
\DeclareMathOperator*{\argmax}{arg\,max}

\newcommand{\R}{\mathbb{R}}
\newcommand{\Sp}{\mathbb{S}}
\newcommand{\iid}{\stackrel{iid}{\sim}}

\let\oldnl\nl
\newcommand{\nonl}{\renewcommand{\nl}{\let\nl\oldnl}}

\begin{document}
	
\begin{frontmatter}
\title{Nonparametric posterior learning for emission tomography with multimodal data}
\runtitle{}

\begin{aug}
	\author{\fnms{Fedor} \snm{Goncharov}\thanksref{addr1}\ead[label=e0]{fedor.goncharov@cea.fr}},
	\author{\fnms{\'Eric} \snm{Barat}\thanksref{addr1}\ead[label=e1]{eric.barat@cea.fr}}
	\and
	\author{\fnms{Thomas} \snm{Dautremer}\thanksref{addr1}\ead[label=e2]{thomas.dautremer@cea.fr}}
	\runauthor{}
	
	\address[addr1]{Université Paris-Saclay, CEA, List, F-91120, Palaiseau, France.\\ fedor.goncharov@cea.fr, eric.barat@cea.fr, thomas.dautremer@cea.fr}	


	\begin{abstract}
		We continue studies of the uncertainty quantification problem 
		in emission tomographies such as PET or SPECT when additional multimodal data (e.g., anatomical MRI images)
		are available. To solve the aforementioned problem we 
		adapt the recently proposed nonparametric posterior learning technique to the 
		context of Poisson-type data in emission tomography. 
		Using this approach we derive sampling algorithms which are trivially parallelizable, scalable and very easy to implement. In addition, we prove  conditional consistency and tightness for the distribution of produced samples in the small noise limit (i.e., when the acquisition time tends to infinity) and derive new geometrical and necessary condition on how MRI images must be used. This condition arises naturally in the context of identifiability problem for misspecified generalized Poisson models.
		We also contrast our approach with Bayesian Markov Chain Monte Carlo sampling based on one  data augmentation scheme which is very popular in the context of Expectation-Maximization algorithms for PET or SPECT. We show theoretically and also numerically that such data augmentation significantly increases mixing times for the Markov chain. In view of this, our algorithms seem to give a reasonable trade-off between design complexity, scalability, numerical load and assessment for the uncertainty.
	\end{abstract}
\end{aug}
\end{frontmatter}

	\section{Introduction}
	Emission tomographies (further referred as ET) such as Positron Emission Tomography (PET) or Single Photon Emission Computed Tomography (SPECT) are  functional imaging modalities of nuclear medicine which are used to image activity processes and, in particular, metabolism in soft tissues via the uptake of certain injected biomarkers. The level of metabolism provides critical information for diagnostics and treatment of cancers; see e.g.,   \citet{weber2005pet}, \citet{loredana2018tumor} and references therein. 
	
	In this work we continue studies on the two following problems:
	
	\begin{prob}
		\label{prob:intro:quant-uncertainty}
		Quantify the uncertainty of reconstructions in ET. 
	\end{prob}
	
	\begin{prob}
		\label{prob:intro:multimodal-data}
		Regularize the inverse problem using the multimodal data (e.g., images from CT or MRI).
	\end{prob}

	Problem~\ref{prob:intro:quant-uncertainty} is not new and several approaches have been established already which in turn can be grouped according to the  statistical view of the problem: frequentist (\citet{fessler1996meanvar}, \citet{barrett1994noise}, \citet{li2011noise}), Bayesian (\citet{higdon1997fully}, \citet{weir1997bayesian}, \citet{marco2007multiscale}, \citet{sitek2012data}, 
	\citet{bochkina2014}, \citet{filipovic2018pet}) and bootstrap  (\citet{haynor1989resampling}, \citet{dahlbom2001estimation}, \citet{lartizien2010comparison}, \citet{filipovic2021reconstruction}). The list of given references is far from being complete and it should also include references therein.
	
	Problem~\ref{prob:intro:multimodal-data} can be splitted further depending on which type of exterior data are used - CT or MRI. The most common use of both modalities consists in extracting boundaries of anatomical features on side images and embedding them into regularization schemes via special penalties and/or non-invariant filters; see e.g., \citet{fessler1992regularized}, \citet{chun2013post}, \citet{hero1999minimax},
	\citet{comtat2001clinically}, 
	\citet{vunckx2011pet}.
	Main reasons to use multimodal data in ET are the ill-posedness of corresponding inverse problems (in PET/SPECT forward operators are ill-conditioned; see e.g., \citet{thorsten2016poisson}) and very low signal-to-noise ratio in the raw measured data. All this together results in loss of resolution in reconstructed images and consequently in oversmoothing, e.g., when applying spatially invariant filters for post-smoothing.
	In our work as multimodal data we use series of presegmented anatomical MRI images. 
	Problem~\ref{prob:intro:multimodal-data} for additional MRI data is now of particular interest due to appearance of commercially available models of PET-MRI scanners \citet{luna2013functional}, \citet{judenhofer2008simultaneous} which allow simultaneous registrations of both signals, thus significantly reducing motion effects.	
	Moreover, in the experiment on tumor imaging in \citet{bowsher2004mri} correlations between PET and MRI signals were observed, therefore, potentially MRI data can be used to regularize accurately the inverse problem.
	In Section~\ref{sect:prelim} we explain in detail how we use MRI data and compare our approach with previous works.
	

	For Problem~\ref{prob:intro:quant-uncertainty} already the definition of uncertainty for reconstructions in ET is not obvious: during time interval $(0, t)$ raw data $Y^t$ (sinogram) is generated from unknown distribution $P^t$ (typically it is assumed to be from the generalized Poisson model with unknown intensity parameter $\lambda_*$ and known design $A$, i.e., $P^t = P^t_{A, \lambda} = \mathrm{Po}(t A\lambda_*)$), so any reconstruction $\widehat{\lambda}^t$ would be also a function of observed data, that is $\widehat{\lambda}^t=\widehat{\lambda}^t(Y^t)$ and uncertainty propagates directly from $Y^t$. 
	This is known as frequentist approach, and for ET it often leads to estimation of confidence intervals for the maximum likelihood estimator (MLEM) or for penalized maximum log-likelihood estimator (pMLEM or MAP) (both are $M$-estimators \citet{vaart2000asymptotic}); see e.g., \citet{fessler1996meanvar}. In particular, frequentist approach has an advantage of being relatively robust to model misspecification (i.e., when $P^t\not= P^t_{A,\lambda}$ for any $A$ and $\lambda$). In this case for large $t$ estimate $\widehat{\lambda}^t$ will tend to a projection of $P^t$ onto  $P^t_{A, \lambda}$ with respect to some chosen distance between probability distributions (e.g., for Kullback-Liebler divergence). Under additional assumptions on $P^t$ even in misspecified case it is still possible to establish asymptotic distribution of $\widehat{\lambda}^t$ (e.g., via asymptotic normality), from which, for example, the  asymptotic confidence intervals can be retrieved. However, use of asymptotic results for ET practice seems doubtful since very little data are available in a single scan.
	
	Bayesian approach is also used for uncertainty quantification in ET. In this case the initial uncertainty on the parameter of interest (e.g., anatomical information from side images, assumptions on support and smoothness) is  encoded in some prior measure $\pi_{\mathcal{M}}(\lambda)$ which is updated using model family $P^t_{A,\lambda}$ and data $Y^t$ to define posterior distribution via the well-known Bayes' formula; see e.g., \citet{bochkina2014}. Sampling from such posteriors is done via Markov Chain Monte Carlo (MCMC) techniques \citet{weir1997bayesian}, \citet{higdon1997fully}, \citet{marco2007multiscale}, \citet{filipovic2018pet}. Common bottlenecks here are: complicated design of the algorithm and  its implementation, high numerical load per iteration, lack of scalability and most importantly -- poor mixing in constructed chains; see e.g.,~\citet{vandyk2001art},~\citet{duan2018scaling}. Additional issue is the misspecification of the model which cannot be included in the classical Bayesian framework and for robust inference it leads to the recently proposed general Bayesian updating and bootstrap-type sampling; see \citet{pompe2021introducing}, Section~1.
	
	As noted above bootstrap is another attractive technique to assess the uncertainty which can be also seen as some probabilistic sensitivity analysis or as approximate/exact sampling via (nonparametric) Bayesian posteriors; see e.g., \citet{newton1994wbb}, \citet{lyddon2018npl}, \citet{fong2019scalable}. Nontrivial questions for ET are the following ones: (1) how to define a bootstrap procedure for Poisson-type raw data in ET and also include side information (multimodal images) (2) provide theoretical guarantees on the coverage by asymptotic credible intervals. A common approach to answer question (1) is to use resampling in list-mode data; see e.g., \citet{haynor1989resampling}, \citet{dahlbom2001estimation}. Such approach targets to resample photon counts and then propagate the uncertainty by using some reconstruction algorithm (e.g., FBP (Filtered backprojection), MLEM or MAP (maximum a posteriori)). 
	In this sense our approach is similar to bootstrap as it will be explained further. Question (2) is often resolved by demonstrating asymptotic equivalence between bootstrap, Bayesian and frequentist approaches via Bernstein von-Mises type theorems; see e.g., \citet{vaart2000asymptotic}, \citet{lyddon2018npl}, \citet{ng2020random} or equivalence of Edgeworth's expansions for higher orders; see \citet{pompe2021introducing}. 
	
	In view of the above discussion, we note that for practice it seems that it is not of great  importance which kind of uncertainty model is used -- frequentist, Bayesian or bootstrap. Most important is to make usable the resulting framework and algorithms by practitioners, hence, it should be simple, tractable and numerically feasible.
	
	Being inspired with nonparametric posterior learning (further referred as NPL) originating from \citet{lyddon2018npl}, \citet{fong2019scalable}, we propose sampling algorithms for ET with and without MRI data at hand. Therefore, our main contribution is that we extend the NPL originally proposed for regular statistical models and i.i.d data to the non-regular generalized Poisson model of ET (see \citet{bochkina2014}), where the raw data are not i.i.d but a realization from a point process. The initial motivation for this work was the problem of poor mixing for the Gibbs-type sampler in \citet{filipovic2018pet} which was designed for posterior sampling in the PET-MRI context. Below we give a detailed analysis of this phenomenon and give few empirical advises on design of MCMC-samplers for ill-posed inverse problems such as PET or SPECT.
	Our new algorithms solve the above problem since sampled images are automatically i.i.d, moreover, the scheme is trivially parallelizable, scalable and very easy to implement because it relies on the well-known EM-type reconstruction methods from  \citet{shepp1982mlem}, \citet{fessler1995sagepet}. Because of the aforementioned non-regularity of the model we conduct a separate theoretical study of our algorithms for when large dataset is available (for ET this is equivalent to $t\rightarrow +\infty$) and establish consistency and tightness of the posterior for almost any trajectory $Y^t$, $t\in (0, +\infty)$. Establishing further the asymptotic normality requires existence of a strongly consistent estimator which has specific contraction rates in the span of the design and for components activated by positivity constraints. Existence of such estimator is left conjectured, however, we propose one candidate and explain the intuition behind which makes the requirement quite natural.
	
	Though our main theoretical results rely on the assumption of well-specified model, at the end we study the identification problem for the KL-criterion in the misspecified case with wrong design. If a certain geometrical condition on design matrix and observed asymptotic sinogram are satisfied, then the identification problem has positive answer and  negative otherwise. In particular, the latter result gives a clue to extend our theoretical results to fully misspecified scenario for the model of ET when design matrix is incorrect. The latter case is meaningful in practice since the design in ET is always computed very approximately and it does not reflect very complicated photon-matter interactions inside the human  body.

	This paper is organized as follows. In Section~\ref{sect:prelim} we give notations and all necessary preliminaries on statistical models of ET and on  use of multimodal data. In Section~\ref{sect:motivating-example-mcmc} we give a very informative example for the problem of poor mixing for MCMC in ET. In Section~\ref{sect:new-algo} we adapt nonparametric posterior learning for ET context and derive our sampling algorithms.
	In Section~\ref{sect:new-thero-results} we study theoretically the asymptotic properties of our algorithms.
	In Section~\ref{sect:conclusion} we discuss our results and possibilities for future work.

	\section{Preliminaries}
	\label{sect:prelim}
	\subsection{Notations} By $\mathbb{N}_0$ we denote the set of non-negative all integers, $\mathbb{R}^n_+$ denotes the nonnegative cone of $\R^n$, by $x \succeq  y$, $x\in \R^n, \, y\in \R^n$, we denote the property that $x_j \geq y_j$ for all $j=1, \dots, n$, $x\succ y$ denotes the same but with strict inequalities,  $\langle x, y \rangle$ stands for the scalar product $x^Ty$  (we will use both notations), $R_+(A)$ denotes the image of positive cone $\R^p_+$ under action of operator $A \in \mathrm{Mat}(d, p)$, 
	by $X\sim F$ we denote the property that random variable $X$ has distribution $F$, $\mathrm{Po}(\lambda)$ denotes the Poisson distribution with intensity $\lambda, \, \lambda \geq 0$, by $\Gamma(\alpha, \beta)$ we denote the gamma distribution with shape parameter $\alpha$, and scale $\beta$ ($\xi \sim \Gamma(\alpha, \beta)$, $\mathbb{E}\xi = \alpha \beta^{-1}$, $\mathrm{var}(\xi) = \alpha\beta^{-2}$).  
	Let $A\in \mathrm{Mat}(d,p)$, $I\subset \{1, \dots, d\}$, then 
	$\mathrm{cond}(A)$ denotes the condition number of $A$, \
	$A_I$ denotes the submatrix of $A$ with rows indexed by elements in $I$, $\mathrm{Span}(A^T)$ denotes the span of the rows of $A$ being considered as vectors in $\R^p$.
	Let $Z$ be a complete separable metric space equipped with metric $\rho_Z(\cdot, \cdot)$ and boundedly finite non-negative measure $dz$, $B(Z)$ denotes the sigma algebra of borel sets in~$Z$. By $\mathcal{PP}^t$ we denote a point process on $Z$ defined for each $t\in \R_+$ and $\mathcal{PP}^t_{\Lambda}$ denotes the Poisson point process on $Z$ with intensity $t\Lambda$, where $\Lambda$ is the nonnegative function $\Lambda = \Lambda(z), \, z\in Z$, $\Lambda$ is integrable with respect to $dz$. Weighted gamma process on $Z$ is denoted by $GP(\alpha, \beta) = G_{\alpha, \beta}$, where $\alpha$ is the shape measure on $Z$ and $\beta$ is the scale which is a non-negative function $Z$ and also $\alpha$-integrable; see, e.g., \citet{albertlo1982bayesiannon}. Finally, by $\mathcal{KL}(P, Q)$ we denote the standard Kullback-Leibler divergence between probability distributions $P$, $Q$.

	\subsection{Mathematical model for ET} Raw data in ET are described by vector $Y^t = (Y_1^t, \dots, Y_d^t) \in (\mathbb{N}_0)^d$ called sinogram which stands for the  photon counts recorded during exposure time $t$ along $d$ lines of response (LORs).
	It is assumed that 
	\begin{align}\label{eq:poisson-model-pet}
	\begin{split}
	&Y^t_i \sim\mathrm{Po}(t \Lambda_i), \, \Lambda_i = a_i^T\lambda, \\
	&Y_i^t\text{ are mutually independent for } i\in \{1, \dots, d\},
	\end{split}
	\end{align}
	where $\lambda \in \R^p_+$ is the parameter of interest on which we aim to perform inference. In practice, vector $\lambda$ denotes the spatial  emission concentration of the isotope (or tracer uptake) measured in [Bq/mm$^3$], that is  $\lambda_j$ is the concentration at pixel $j\in \{1, \dots, \,  p\}$. 
	Vector $\Lambda = (\Lambda_1, \dots, \, \Lambda_d)$ denotes the observed photon intensities along LORs $\{1, \dots, \, d\}$, respectively. To separate the LORs with strictly positive intensities from those ones with zeros we introduce the following notations:
	\begin{align}\label{eq:ind-lors-pos-zeros}
	I_0(\Lambda) = \{i : \Lambda_i = 0\}, \, I_1(\Lambda) = \{i : \Lambda_i  > 0\}, \, I_0 \sqcup I_1 = \{1, \dots, d\}.
	\end{align}

	Collection of $a_i\in \R^p$ in \eqref{eq:poisson-model-pet} constitute  matrix $A = [a_1^T, \dots, a_d^T]^T$, $A\in \mathrm{Mat}(d, p)$  which is called by projector or system matrix in applied literature on ET and by design or design matrix in statistical literature. Each element $a_{ij}$ in $A$ denotes the probability to observe a pair of photons along LOR  $i\in \{1, \dots, d\}$ if both they were emitted from pixel $j\in \{1, \dots, p\}$. In view of such interpretation, for design $A$ we assume the following:
	\begin{align}\label{eq:design-matrix-positivity-restr-1}
	&a_{ij} \geq 0 \text{ for all pairs } (i, j), \\
	\label{eq:design-matrix-positivity-restr-2}
	&A_j = \sum\limits_{i=1}^{d} a_{ij}, \, 0 < A_j \leq 1 \text{ for all }j\in \{1, \dots, p\}, \\
	\label{eq:design-matrix-positivity-restr-3}
	&\sum_{j=1}^{p} a_{ij} > 0 \text{ for all } i \in \{1, \dots, d\}.
	\end{align}
	
	If any of formulas \eqref{eq:design-matrix-positivity-restr-2}, \eqref{eq:design-matrix-positivity-restr-3} would not be satisfied, then, in practice it would mean that either some pixel is not detectable at all (hence it can be completely removed from the model) or some detector pair is broken and cannot detect any of incoming photons. These scenarios are outside of our scope.
	
	It is well-known that the inverse problems for PET and SPECT are mildly ill-posed (see e.g., \citet{thorsten2016poisson}, \citet{natterer2001mathematics}), which in practice means that 
	\begin{equation}\label{eq:design-matrix-non-empt-ker}
	\ker A \neq \{0\}.
	\end{equation}

	\begin{remark}\label{prelim:rem:ill-posedness}
		Numerically $A$ represents a discretized version of weighted Radon transform operator $R_a$ for ET with complete data (see e.g., \citet{natterer2001mathematics}). Since $A$ approximates $R_a$ in strong operator norm (e.g., for $R_a : L_0^2(D) \rightarrow L_0^2([-1,1]\times \Sp^1)$, $D$ is the centered unit ball in $\R^2$) we know that 
		\begin{equation}\label{eq:prelim:ill-posedness:deg}
		\sigma_k \asymp k^{-1/2}, \, k = 1, \dots, p,
		\end{equation}
		where $\sigma_k$ are the singular values of $A$. In particular, even if $A$ is injective for $p$ large enough, due to \eqref{eq:prelim:ill-posedness:deg}, it may happen that $\mathrm{cond}(A) > \varepsilon^{-1}_F$, where $\varepsilon_F$ is the floating-point precision.
		In the latter case, due to the cancelling effect singular values of $A$ numerically will be equivalent to machine zeros which means then exactly the existence of a nontrivial kernel for $A$.
	\end{remark}

	Likelihood and negative log-likelihood functions for model in  \eqref{eq:poisson-model-pet} are given by the formulas:
	\begin{align}\label{eq:prelim:poiss-prob-model}
	P_{A, \lambda}^t(Y^t) &= \mathrm{pr}(Y^t \mid A, \lambda, t) = \prod_{i=1}^d \dfrac{(ta_i^T\lambda)^{Y_i^t}}{Y_i^t!}
	e^{-t a_i^T\lambda}, \, \lambda \in \R^p_+, \, t \geq 0,  \\
	\label{eq:prelim:poiss-log-likelihood-model}
	L( \lambda \mid Y^t, A, t )& = \sum_{i=1}^{d}	- Y_i^t \log(t\Lambda_i) + t\Lambda_i, \, \Lambda_i = a_i^T\lambda.
	\end{align}
	
	For $A$ satisfying \eqref{eq:design-matrix-non-empt-ker} and for any $Y^t$ function  $L(\lambda \mid Y^t, A, t)$ is not strictly convex even at the point of the global minima since $L(\lambda + u \mid Y^t, A, t) = L(\lambda \mid Y^t, A, t)$ for any $\lambda\in \R^p_+$ and $u\in \ker A$. To avoid numerical instabilities due to this phenomenon a convex penalty $\varphi(\lambda)$ is added to $L(\lambda \mid Y^t, A, t)$, so we also consider the  penalized negative log-likelihood:
	\begin{equation}\label{eq:prelim:poiss-log-likelihood-penalized}
	L_{p}( \lambda \mid Y^t, A, t, \beta^t) = L(\lambda \mid Y^t, A, t) + \beta^t \varphi(\lambda), \, \lambda\in \R^p_+,
	\end{equation}
	where  $\beta^t \geq 0$ is the regularization coefficient. We assume that  $\beta^t$ may increase with time $t$ at a certain rate which is important for practice in order to increase the signal-to-noise ratio in reconstructed images.

	\subsection{Regularization penalty}
	
	The role of regularization penalty $\varphi(\lambda)$ in  \eqref{eq:prelim:poiss-log-likelihood-penalized} is to decrease the numerical instability in the underlying inverse problem and to make function  $L_{p}(\lambda \mid Y^t, A, t, \beta^t)$ more convex, especially in directions close to $\ker A$. 
	
	In view of this we assume that 
	\begin{align}
	\label{eq:prelim:penalty-cond-convex}
	&\varphi  \text{ is continous and convex on $\R^p$}, \\
	\label{eq:prelim:penalty-cond-strict-conv}
	&g_u(w)= 
	\varphi(u + w) \text{ is strictly convex in } w\in \ker A \text{ for any } u\in \mathrm{Span}(A^T).
	\end{align}
	In Subsection~\ref{subsect:theory:consistency} and in our proofs we use extensively the following technical result.
	
	\begin{lem}\label{lem:consistency:lem-kernel-continuity}
		Let $\varphi(\lambda)$ be the function satisfying  \eqref{eq:prelim:penalty-cond-convex}, \eqref{eq:prelim:penalty-cond-strict-conv}, $A$ satisfies conditions in \eqref{eq:design-matrix-positivity-restr-1}-\eqref{eq:design-matrix-positivity-restr-3}.
		Let $\lambda \in \R^p_+$ and $U\subset \mathrm{Span}(A^T)$ be a compact such that 
		\begin{equation}
		\{ w : \lambda + u + w\succeq 0, \, w\in \ker A\} \text{ is non-empty for any } u\in U.
		\end{equation}
		Then, mapping defined by the formula
		\begin{align}\label{eq:proofs:lemma-existence-minimizer-ker-a}
		w_{A,\lambda}(u) = \argmin_{\substack{w:\lambda + u + w\succeq 0, \\ w\in\ker A}}
		\varphi(\lambda + u + w), \, u\in U
		\end{align}
		is one-to-one. Moreover, $w_{A,\lambda}(u)$ is continuous on $U$.
	\end{lem}



	\subsection{Multimodal data for ET} 
	\label{subsect:multimodal}
	From \eqref{eq:poisson-model-pet} one can see that recorded signal $Y^t$ is essentially the Poisson noise
	for which its signal-to-noise ratio (SNR) is proportional to $\sqrt{t\Lambda}$ and is quite low in practice (e.g., because of low injected dose and moderate~$t$ in standard medical protocols). In order to increase the SNR in reconstructed images and not to loose a lot in resolution it is proposed to regularize the inverse problem using  multimodal data -- images from CT or MRI. We choose MRI since it provides anatomical information with high contrast in soft tissues in comparison to CT (see Figures~\ref{fig:mri-ct-comparison} (a), (b)).
	
\begin{figure}[H]
	\centering
 	\subcaptionbox{CT}{\includegraphics[height=30mm]{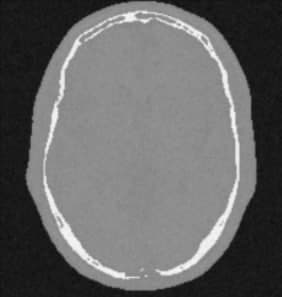}} \hspace{4em}%
 	\subcaptionbox{MRI}{\includegraphics[height=30mm]{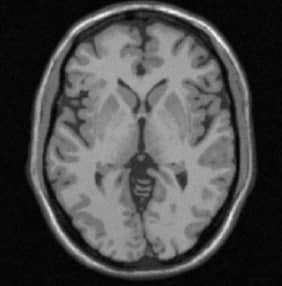}} \hspace{4em}%
	\subcaptionbox{$M\in \mathcal{M}$}{\includegraphics[height=30mm]{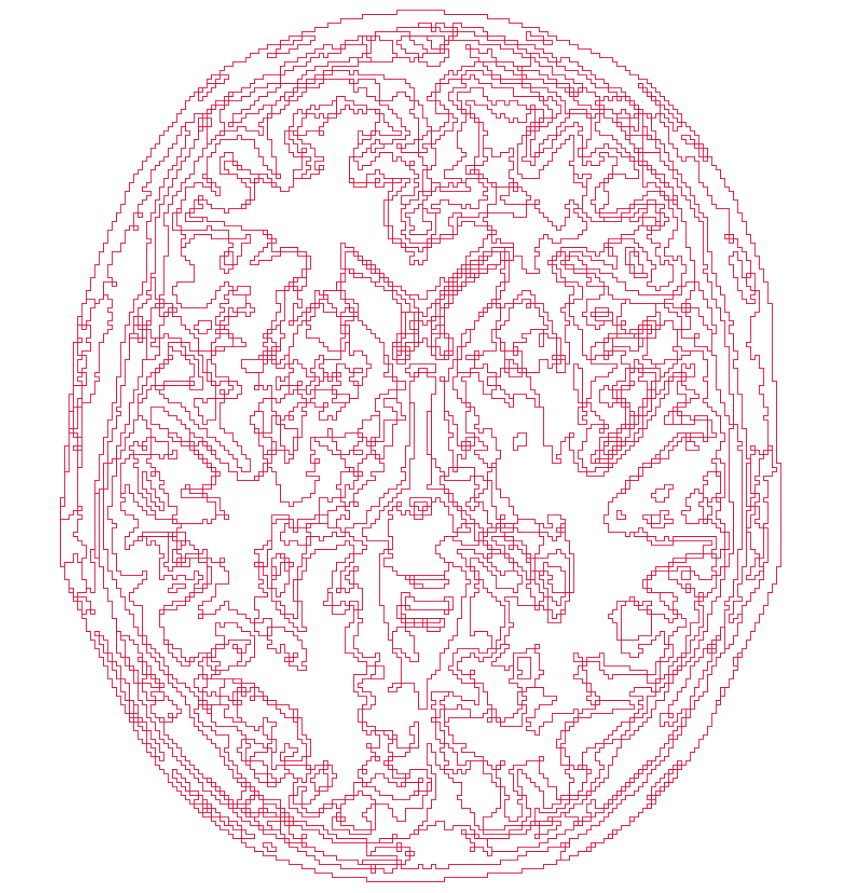}}
	\caption{Multimodal data for ET of the brain}
	\label{fig:mri-ct-comparison}
\end{figure}

	We assume that our exterior data consists of $r$ presegmented MRI images $\mathcal{M} = \{M_1, \dots, M_r\}$ (see Figure~\ref{fig:mri-ct-comparison} (c)) (segmentations of MRI images are precomputed using the ddCRP algorithm from \citet{Ghosh2011ddcrp}). In fact, MRI-guided reconstructions in PET is an active topic of research (see the discussion in \citet{filipovic2021reconstruction} and also references therein) and still a lot of work is needed to describe precisely correlations between ET and MRI signals (especially from biological point of view); see e.g., \citet{bowsher2004mri}. Current use of MRI data is purely image-based: spatially regularizing penalties are constructed using MRI data in \citet{bowsher1996bayesian}, \citet{bowsher2004mri}, \citet{vunckx2011pet}, models built upon MRI-segmented data for  locally-constant tracer distribution are used in \citet{filipovic2018pet} and also in our work. Our approach is ideologically different from ones in \citet{vunckx2011pet} because we use $\mathcal{M}$ to construct models of tracer distributions and then sample ``pseudo-sinogram'' to mix it with real observed data $Y^t$. That is MRI data are used only in observation space for ET. This has a practical feature of interpretability for our main  calibration parameter which reflects the ratio between number of real detected photons $N^t = \sum Y^t$ and the number of ``pseudo-photons'' generated from the MRI-based models.

	\section{A motivating example for NPL in ET}
\label{sect:motivating-example-mcmc}

Recently a Gibbs-type sampler was proposed in \citet{filipovic2018pet} for  Bayesian inference for PET-MRI. Despite a number of  positive practical features (spatial regularization, use of multimodal data) the problem of slow mixing for the corresponding Markov chain was observed. Below we consider its  simplified version which shares the same mixing problem and explain the phenomenon numerically and theoretically.

\par In algorithms for ETs it is common  to introduce data augmentation (latent variables) $n^t = \{n_{ij}^{t}\}$, where 
$n_{ij}^t$ is the number of photons emitted from pixel $j$ and detected in LOR~$i$, 
$n^t_{ij} \sim \mathrm{Po}(t a_{ij} \lambda_j)$, $n^t_{ij}$ are mutually independent for all $(i,j)$; see e.g., \citet{shepp1982mlem}.

\par In view of this physical interpretation of $n^t$,  for variable $(n^t, Y^t)$ the following coherence condition must be satisfied:
\begin{equation}\label{eq:example:mcmc-latent:consistency-cond}
\sum\limits_{j=1}^{p}n^{t}_{ij} = Y_i^t \text{ for all }i\in \{1, \dots, p\}.
\end{equation}
From \eqref{eq:example:mcmc-latent:consistency-cond} it follows that $Y^t$ is a function of $n^t$, so $(Y^t, n^t)$ is indeed a data augmentation of~$Y^t$. Note that $n^t$ are not observed in a real experiment but $n^t$ greatly simplifies the design of samplers (see e.g.,  \citet{james2003gwprocesses}, \citet{filipovic2018pet}), because conditional distributions $p(n^t \mid Y^t, A, \lambda, t)$, $p(\lambda \mid n^t, A, t)$ admit very simple analytical forms even for nontrivial priors involving multimodal data. For our example below we use only a simple pixel-wise positivity gamma-prior:
\begin{equation}\label{eq:example:pixel-gamma-prior}
\pi (\lambda) = \prod_{j=1}^{p}\pi_j(\lambda_j), \, \pi_j = \Gamma(\alpha, \beta^{-1}), \, \alpha > 0, \, \beta > 0,
\end{equation}
where $\alpha$, $\beta$ are some fixed constants. For the prior in \eqref{eq:example:pixel-gamma-prior} and model \eqref{eq:poisson-model-pet}
conditional distributions $p(n^t \mid Y^t, A, \lambda, t)$, $p(\lambda \mid n^t, A, t)$ are as follows:

\begin{align}
\begin{split}\label{eq:example:latent-multinomial-distr}
p(n_{ij}^t \mid Y^t, A, \lambda, t) &= \mathrm{Multinomial}(Y_i^t, p_{i1}(\lambda), \dots, p_{ip}(\lambda)),\\
p_{ij}(\lambda) &= \dfrac{a_{ij}\lambda_j}{\sum_k a_{ik}\lambda_k}, \, i \in \{ 1, \dots, d\},
\end{split}\\
\label{eq:example:latent-gamma-pixel-distr}
p(\lambda \mid n^t, Y^t, A, t) &= \Gamma\left(
\sum\limits_{i=1}^{d} n_{ij}^t + \alpha, (tA_j + \beta)^{-1}
\right), 
\end{align}
where $A_j$ is defined in   \eqref{eq:design-matrix-positivity-restr-2}. 

Using \eqref{eq:example:latent-multinomial-distr}, \eqref{eq:example:latent-gamma-pixel-distr} the construction a Gibbs sampler for Bayesian posterior sampling from $p(\lambda \mid Y^t, A, t)$ is straightforward.\\


\begin{center}
	\begin{minipage}{0.87\textwidth}
		\begin{algorithm}[H]
			\KwData{sinogram $Y^t$}
			\KwIn{initial point $\lambda_0\in \R^p_+$, 
				parameters $(\alpha, \beta)$ for prior $\pi(\lambda_j) \sim \Gamma(\alpha, \beta^{-1})$, $A$, $B$ -- number of samples}
			\For{ $k=1$ \KwTo $B$ }{
				Sample $n^t_k \sim p(n^t \mid Y^t, A, \lambda_{k-1}, t)$ 
				Sample $\lambda_k^t \sim p(\lambda \mid n^t_k, Y^t, A, t)$ 
			}
			\KwOut{samples $\{\lambda_k^t\}_{k=1}^B$}
			\KwResult{empirical distribution of $\{\lambda_k^t\}_{k=1}^B$ approximates posterior $p(\lambda \mid Y^t, A, t)$}
			\caption{Gibbs sampler for $p(\lambda \mid Y^t, A, t)$}
			\label{alg:example:pet-gibbs}
		\end{algorithm}
	\end{minipage}
\end{center}








\begin{remark}
	One may argue that prior in \eqref{eq:example:pixel-gamma-prior} is a very bad choice from practical point of view, especially in view of ill-posedness of the inverse problem since it does not bring any regularization. However, we consider the mixing rate for the Markov chain in Algorithm~\ref{alg:example:pet-gibbs} in the small noise limit, i.e., when $t\rightarrow +\infty$, and for the latter it is known from the Bernstein von-Mises theorem (see \citet{bochkina2014}) that asymptotically for $t\rightarrow +\infty$ any prior effect will disappear no matter the choice of $\pi(\lambda)$.
\end{remark}


We choose $h(\lambda)$ to be linear, i.e., $h(\lambda) = h^T\lambda$, for some $h\in \R^p$, and  
consider the correlations between values of $h(\lambda)$ for subsequent samples from the Markov chain in Algorithm~\ref{alg:example:pet-gibbs}
\begin{equation}\label{eq:example:corr-value}
\gamma^t(h) = \mathrm{corr}(h(\lambda^t_{k+1}), h(\lambda^t_k) \mid Y^t, t).
\end{equation}
In formula \eqref{eq:example:corr-value} we assumed that the chain is in stationary state, i.e. $k$ can be any. 

Markov chain for the sampler in Algorithm~\ref{alg:example:pet-gibbs} coincides with data augmentation schemes from \citet{liu1994fraction}, \citet{liu1994covariance}, where the latter are exactly Gibbs samplers with only one layer of latent variables. In Bayesian context $\gamma^t(h)$ is known as fraction of missing information; see  \citet{liu1994fraction}. In particular, in \citet{liu1994fraction} authors gave an exact formula for $\gamma^t(h)$ which can be written for our example as follows:
\begin{align}\label{eq:example:formula-fraction}
\gamma^t(h) = 1-\dfrac{\mathbb{E}[\mathrm{var}(h(\lambda) \mid n^t, Y^t, t) \mid Y^t, t]}{\mathrm{var}(h(\lambda) \mid Y^t, t)}.
\end{align}
For simplicity assume that 
\begin{equation}\label{eq:example:positivity-true-point-assump}
\lambda_{*j} > 0 \text{ for all }j\in \{1, \dots, p\}.
\end{equation}

Exact formulas for the nominator and the denominator in \eqref{eq:example:formula-fraction} for arbitrary $t$ seem difficult (if possible) to obtain, however, in the asymptotic regime $t\rightarrow +\infty$ one can apply the Bernstein von-Mises type theorem from \citet{bochkina2014} and arrive to the following simple expression:
\begin{align}\label{eq:example:asymptotic-fraction-missinfo-def}
\gamma(h) = \lim_{t\rightarrow +\infty} \gamma^t(h) = 1-\dfrac{h^TF_{aug}^{-1}(\lambda_*)h}{h^TF_{obs}^{-1}(\lambda_*)h}, \, h\in \R^p, \, \text{ a.s. } Y^t, t\in (0, +\infty).
\end{align}
where 
\begin{align}
&\lambda_*\in \R^p_+ \text{ is the true parameter},\\
\label{eq:example:f-obs-expression}
&F_{obs}(\lambda_*) = \sum\limits_{i=1}^{d} \dfrac{a_ia_i^T}{\Lambda_i^*} = A^TD_{\Lambda^*}^{-1}A, \, D_{\Lambda^*} = \mathrm{diag}(\dots, \Lambda_i^*, \dots), \, \Lambda^*_i = a_i^T\lambda_*,\\
&F_{aug}(\lambda_*) = \mathrm{diag}(\dots, c_j, \dots), \, 
c_j = A_j / \lambda_{*j}.
\end{align}

From \eqref{eq:design-matrix-positivity-restr-3}, \eqref{eq:example:positivity-true-point-assump} it follows that $\Lambda_i^* > 0$ for all $i$, therefore division by $\Lambda^*_i$ in \eqref{eq:example:f-obs-expression} is well-defined. Matrices $F_{obs}(\lambda_*)$, $F_{aug}(\lambda_*)$ are the Fisher information matrices at $\lambda_*$ for Poisson models with observables $Y^t$, $n^t$, respectively. Note also that $F_{obs}$ is not invertible in the usual sense, so in \eqref{eq:example:asymptotic-fraction-missinfo-def} its pseudo-inversion in the sense of Moore-Penrose is considered.

\begin{remark}\label{rem:gibbs-pet-example:positivity-exaplained}
	Assumption in \eqref{eq:example:positivity-true-point-assump} is not practical and a precise analytic formula which extends \eqref{eq:example:asymptotic-fraction-missinfo-def} for $\lambda_*\in \partial\R^p_+$ can be established using the results from \citet{bochkina2014}. 
	The point is that model \eqref{eq:poisson-model-pet} is non-regular since the parameter of interest belongs to a domain with a boundary, so a separate result for Bernstein von-Mises phenomenon is needed in this case. 
	For our purposes it is sufficient to consider the case in \eqref{eq:example:positivity-true-point-assump} since we are mostly interested in mixing times of the Markov chain in areas with positive tracer concentration.
\end{remark}

Let $h_1, \dots, h_p$ be the orthonormal basis of eigenvectors of  $F_{obs}(\lambda_*)$ being ordered according to their corresponding  eigenvalues $s_1 \geq s_2 \dots \geq s_p \geq 0$.  Intuitively, vectors $\{h_m\}_{m=1}^{p}$ constitute a basis in space of reconstructed images where higher indices $m$ correspond to higher frequencies on images (see Figure~\ref{fig:example:orederd-projectors}).

\begin{figure}[H]
	\centering
	\begin{subfigure}{0.24\textwidth}
		\centering
		\includegraphics[width=30mm, height=30mm, trim={3.0cm 0.5cm 3.5cm 1cm}, clip]{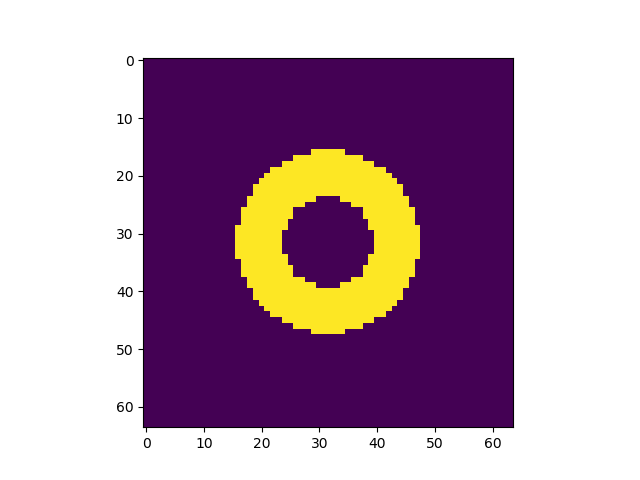}
		\caption{$\lambda_*$}
	\end{subfigure}
	\hspace{-1.2em}
	\begin{subfigure}{0.24\textwidth}
		\centering
		\includegraphics[width=30mm, height=30mm]{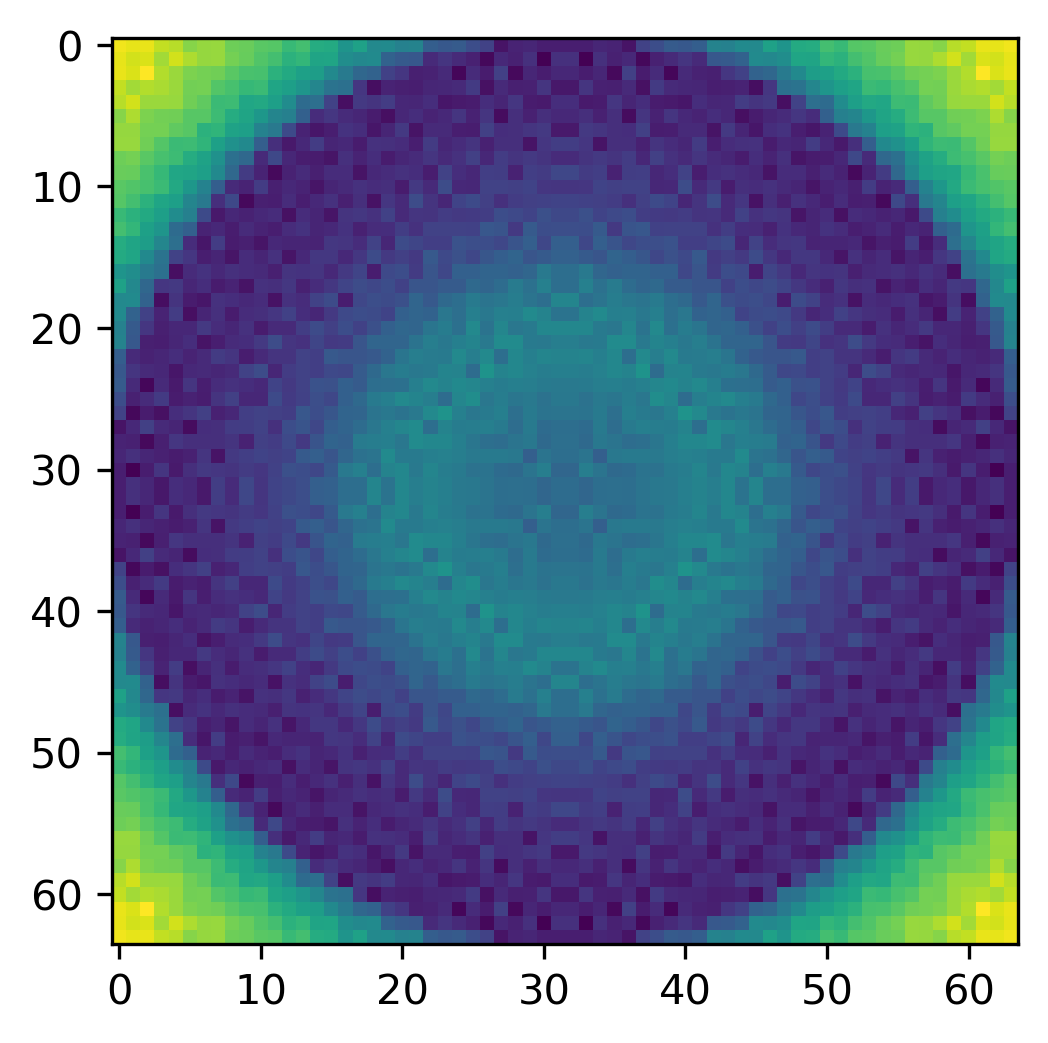}
		\caption{$h_1$}
	\end{subfigure}
	\hspace{-0.8em}
	\begin{subfigure}{0.24\textwidth}
		\includegraphics[width=30mm, height=30mm]{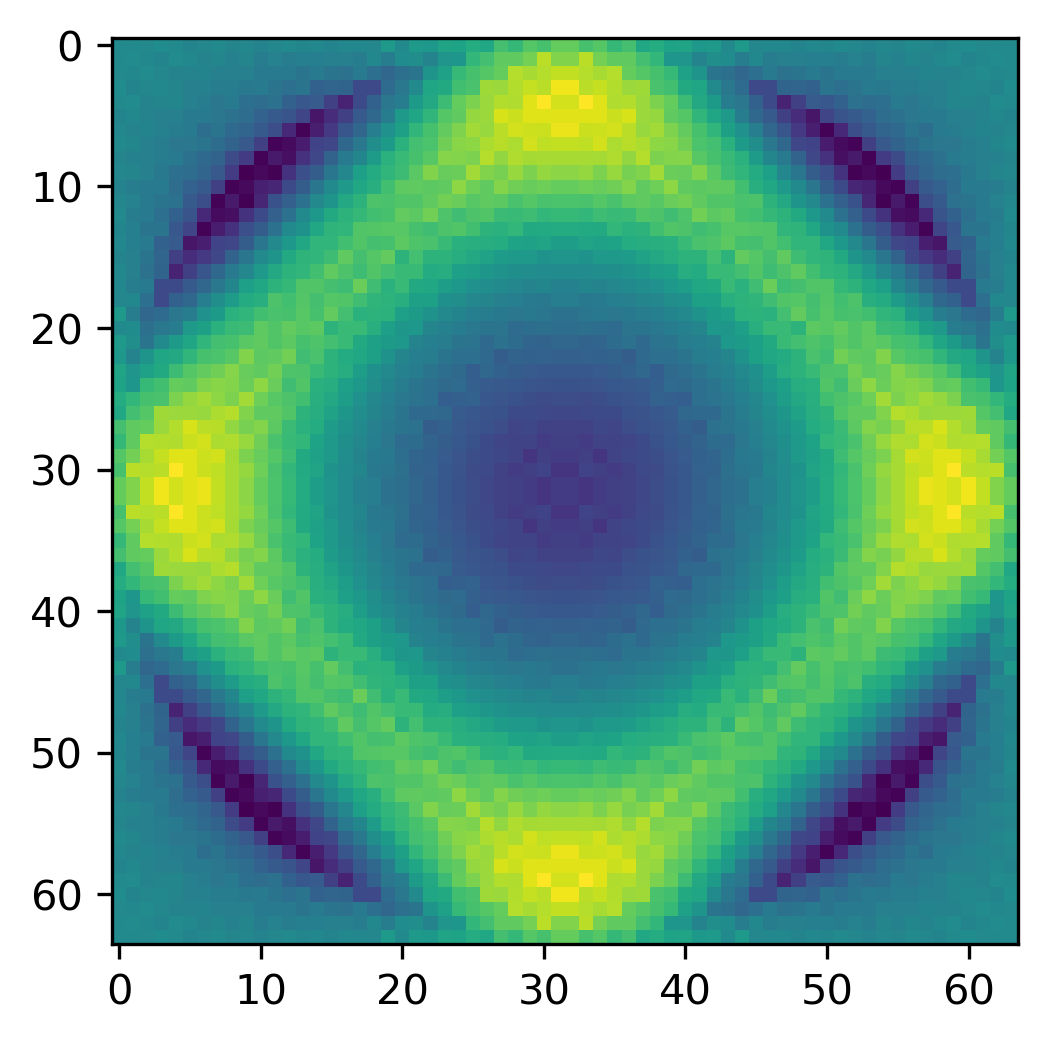}
		\caption{$h_2$}
	\end{subfigure}
	\hspace{-1.5em}
	\begin{subfigure}{0.24\textwidth}
		\includegraphics[width=30mm, height=30mm]{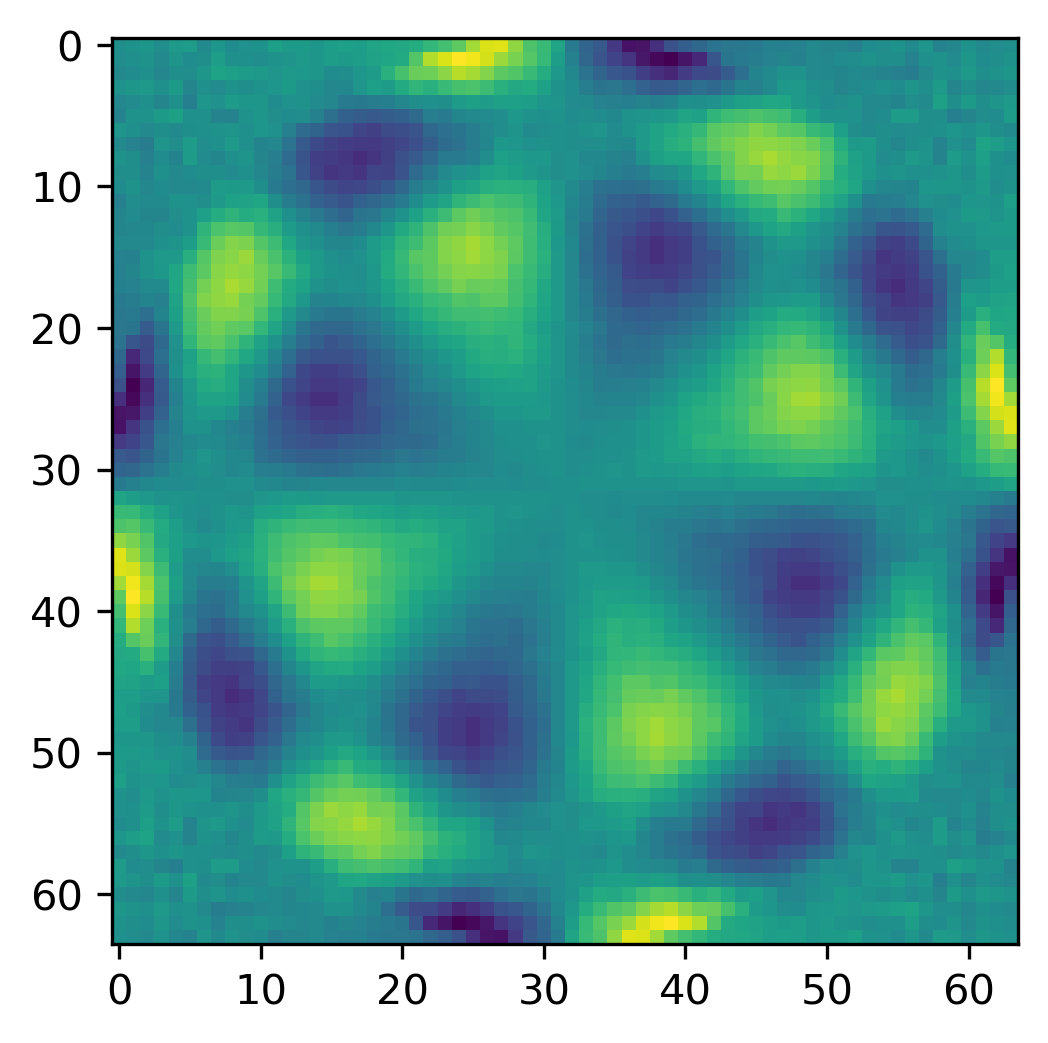}
		\caption{$h_{50}$}
	\end{subfigure}
	\caption{eigenvectors $h_m$ for $F_{obs}(\lambda_*)$}
	\label{fig:example:orederd-projectors}
\end{figure}

From  \eqref{eq:example:asymptotic-fraction-missinfo-def} it follows that 
\begin{align}\label{eq:example:fraction-eigenvect-projections}
\gamma(h_m) = 1 - s_m h_m^TF_{aug}^{-1}h_m.
\end{align}

Matrix $F_{aug}(\lambda_*)$ is well-conditioned, continuously invertible and the quadratic term in  \eqref{eq:example:fraction-eigenvect-projections} admits the following bound:
\begin{equation}\label{eq:example:augmented-fisher-bound}
F^{-1}_{aug}(\lambda_*) = \mathrm{diag}(\dots, \frac{\lambda_{*j}}{A_j}, \dots) \Rightarrow  h^T_mF^{-1}_{aug}(\lambda_*)h_m \leq 
\frac{\max_j(\lambda_{*j})}{\min_j (A_j)}.
\end{equation}
Regular behavior of $F_{aug}^{-1}$ is not surprising because this is the Fisher information matrix for latent variables $n^t$ for which the inverse problem is not ill-posed at all. From \eqref{eq:example:f-obs-expression} and the ill-conditioning nature of $A$ it follows that $F_{obs}(\lambda_*)$ is ill-conditioned\footnote{In practice the  ill-conditioning of $F_{obs}(\lambda_*)$ is commonly observed in ET practice in form of very slow convergence of non-penalized EM-algorithms; see \citet{green1990bayesian}.}, moreover, $s_m \approx 0$ for large~$m$. From this and \eqref{eq:example:fraction-eigenvect-projections}, \eqref{eq:example:augmented-fisher-bound} we conclude that 
\begin{equation}\label{eq:example:correlation-one}
\gamma(h_m) \approx 1 \text{ for large }m.
\end{equation}
Formulas \eqref{eq:example:corr-value}, \eqref{eq:example:correlation-one} constitute a clear evidence of poor mixing in the Markov chain in Algorithm~\ref{alg:example:pet-gibbs}. Though  \eqref{eq:example:asymptotic-fraction-missinfo-def}-\eqref{eq:example:correlation-one} were derived for $t\rightarrow +\infty$, they reflect well the behavior of the chain for moderate $t$ which is seen from the numerical experiment below (see Supplementary Materials, Section~\ref{app:numerical-mixing-example} for details).

\begin{figure}[H]
	\begin{subfigure}{0.49\textwidth}
		\centering
		\includegraphics[width=73mm, height=55mm]{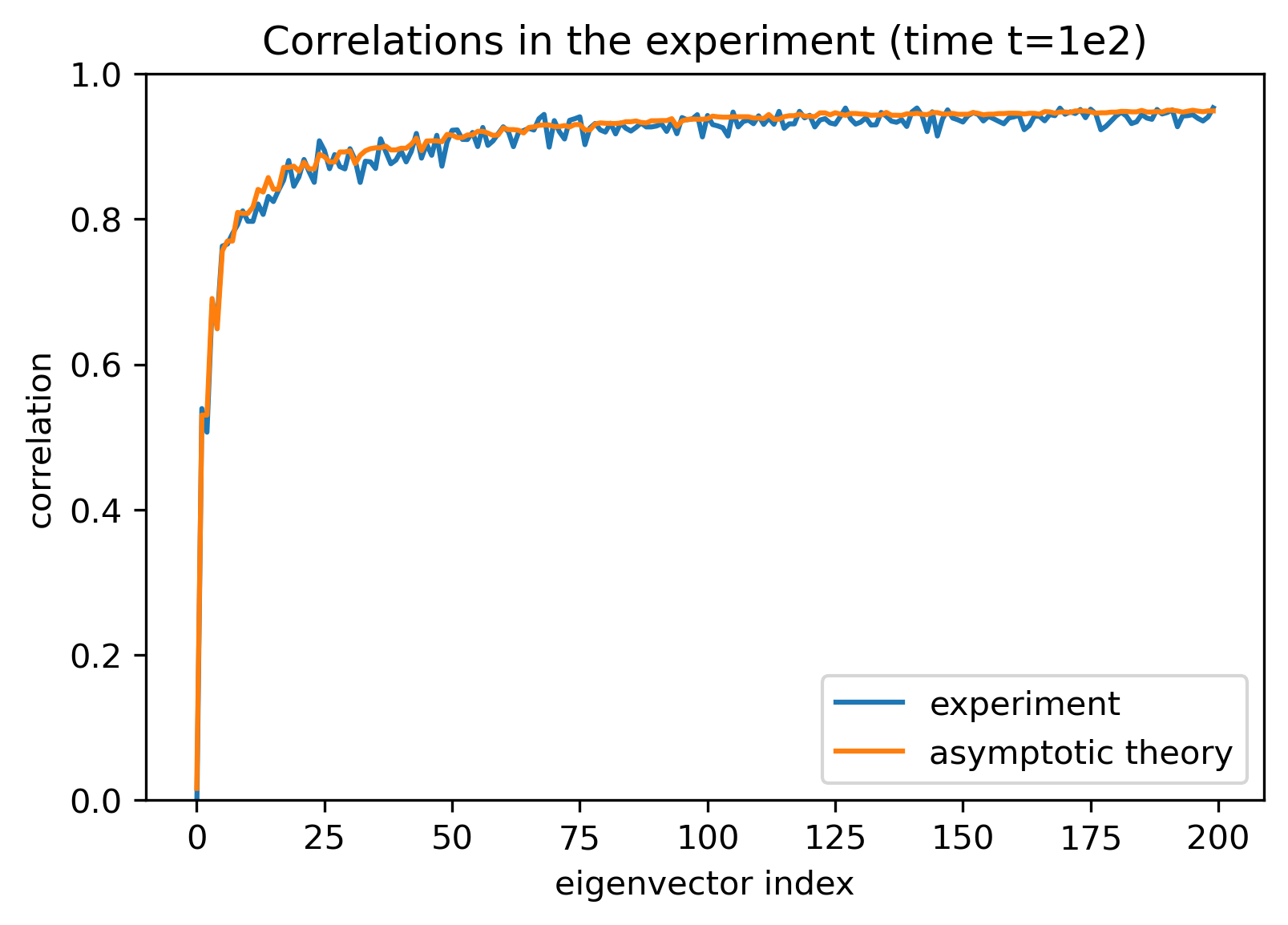}
	\end{subfigure}
	\begin{subfigure}{0.5\textwidth}
		\includegraphics[width=73mm, height=55mm]{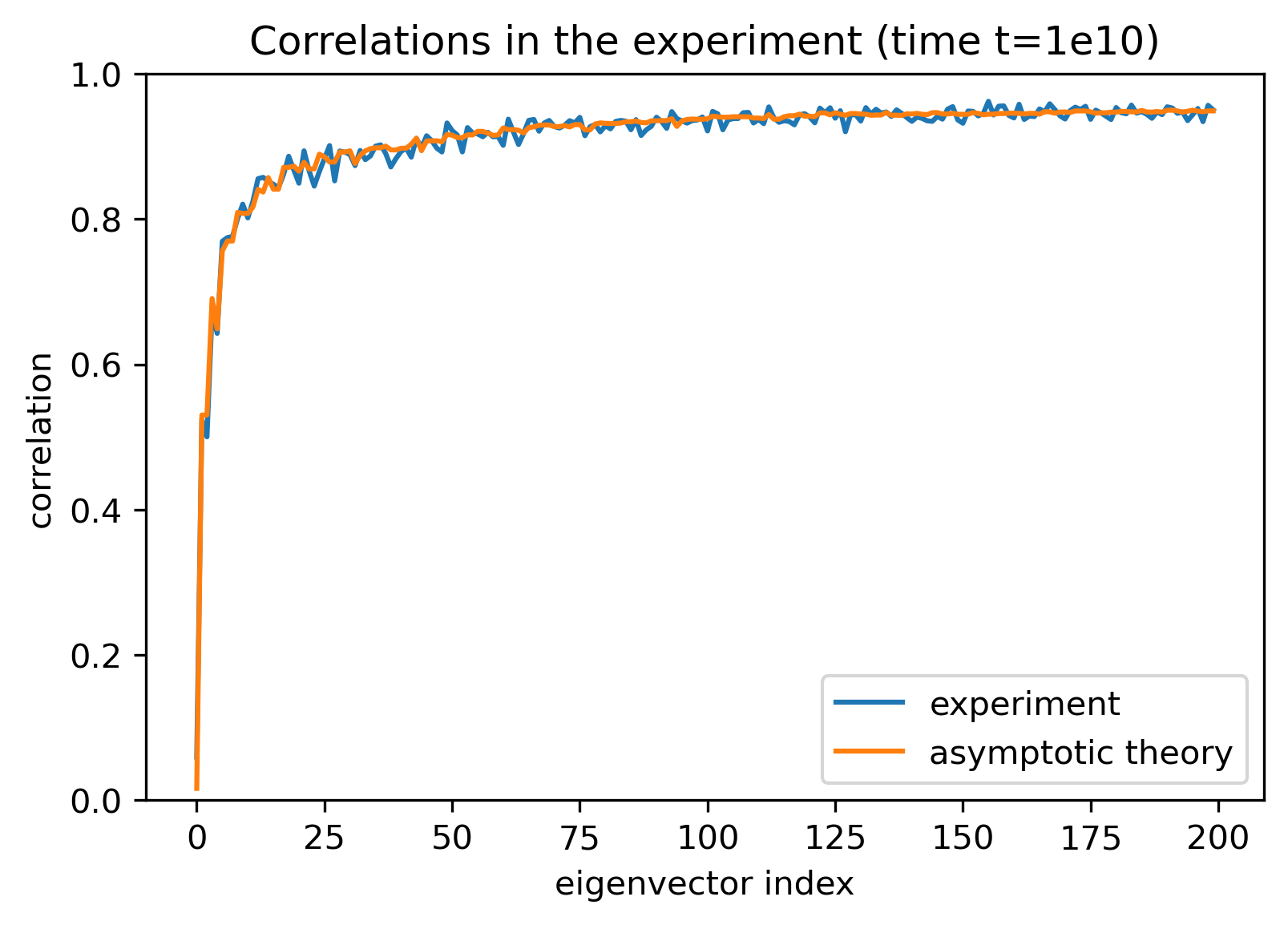}
	\end{subfigure}
	\caption{$\mathrm{corr}(h^T\lambda_k^t, h^T\lambda^t_{k+1} \mid Y^t)$ for $t=10^2, 10^{10}$ for $h=h_m$; blue curve -- empirical correlations computed from $2000$ samples, orange curve -- values for $\gamma(h_m)$ for $m=1,\dots, 200$ by formula~\eqref{eq:example:asymptotic-fraction-missinfo-def}.}
	\label{fig:example:corrs-gibbs}
\end{figure}
In Figure~\ref{fig:example:corrs-gibbs} values given by formula \eqref{eq:example:fraction-eigenvect-projections} are in full correspondence with our numerical results demonstrating that $\mathrm{corr}(h_m^T\lambda^t_k, h_{m}^T\lambda^t_{k+1} \mid Y^t, t)$ increase fast with $m$. Here one concludes that mixing is much slower for high-frequency parts of images. Therefore, to estimate reliably, say mean $h^T\lambda$ ($h\in \R^p$ may be a domain mask), one needs almost infinite number of samples if $h$ contains  a high-frequency component in terms of~$\{h_m\}_{m=1}^p$ (see Supplementary Material, Section~\ref{app:mixing-example} for details). This also can be seen as a recommendation for choosing $h$ in practice: $h$ should belong to  $\mathrm{Span}(A^T)$ and projections $h^Th_m$ should be as small as possible for large $m$.

Note that such behavior of the sampler is not due to the choice of pixel-wise prior but due to sampling of $n^t_{ij}$ which correspond to observations for the well-posed inverse problem. In this situation a practical advice would be to avoid sampling of missing data in the Markov chain or to use a strong smoothing prior/regularizer (for example by  greatly increasing regularization coefficients so that asymptotic arguments in \eqref{eq:example:asymptotic-fraction-missinfo-def} will no longer hold but the posterior consistency is still preserved). The latter approach will accelerate mixing at cost of oversmoothing in sampled images. 

By this negative but informative example we support the message in \citet{vandyk2001art} saying that design of a data augmentation scheme while preserving good mixing in the Markov chain is an ``Art'', especially in the case of ill-posed inverse problems. In view of poor mixing, complexity of the design and implementation, lack of scalability and high numerical load while using MCMC (\citet{weir1997bayesian}, \citet{higdon1997fully}, \citet{marco2007multiscale}, \citet{duan2018scaling}, \citet{filipovic2018pet}) we turn to NPL as a practical relaxation of Bayesian sampling for the problem of ETs.

	\section{Nonparametric posterior learning for emission tomography}
	\label{sect:new-algo}
	

	
	
	\subsection{Nonparametric model for emission tomography}
	
	\label{subsect:new-algo:npl-poisson}
	Nonparametric framework for ET can be seen as a classical scanning scenario with a machine having infinite number of infinitely small detectors.
	Let $Z$ be the manifold of all detector positions in the acquisition geometry of a scanner (e.g., $Z=(\R\times \Sp^1)/\mathbb{Z}_2$, i.e., all non-oriented straight lines in $\R^2$) for full angle acquisition in a single plane slice. For completeness we assume that $Z$ is equipped with a boundedly-finite measure $dz$ (which reflects the sensitivity of detectors for various lines) and with a metric $\rho_Z$ describing distances between the lines (e.g., $\rho_Z$ could be a geodesic distance on cylinder $\R\times \Sp^1 / \mathbb{Z}_2$).
	
	For exposure time $t$  the raw data are given by random measure $Z^t$ generated by a counting point process:
	\begin{align}
	Z^t = \sum_{j=1}^{N^t} \delta_{(t_j,z_j)}, \, (t_j,z_j)\in (0, \infty)\times Z, \, 
	t_j \leq t_{j+1}, \, t_j \leq t, 
	\end{align}
	where 
	\begin{align}
	&N^t \text{ is total number of registered photons}, \\
	&\{z_j\}_{j=1}^{N^t}, \, \{t_j\}_{j=1}^{N^t} \text{ are the LORs and arrival times of registrations, respectively}.
	\end{align}
	In practical literature on PET/SPECT sample $Z^t$ is known as list-mode data, whereas sinogram $Y^t$ is the version of $Z^t$ binned to finite spatial resolution and integrated withing interval $[0,t)$. Under the assumption of temporal stationarity $Y^t$ contains the same amount of information as $Z^t$ since the first one is then a sufficient statistic.
	
	For statistical model of $Z^t$ one takes the family of temporal stationary Poisson point processes $\mathcal{P}\mathcal{P}^t_{A\lambda}$ on $Z$, where $A$, $\lambda$ stand for the nonparametric versions of the projector and vector denoting the tracer concentration, respectively; see Section~\ref{sect:prelim}. For example, in such model the intensity parameter in LOR $z\in Z$ during time interval $[0, t)$ is $t\Lambda(z)dz = t[A\lambda](z)dz$.
	
	The negative log-likelihood for $\mathcal{PP}^t_{A\lambda}$ and observation $Z^t$ is defined via the following formula (see, e.g., \citet{thorsten2016poisson}, Section 2):
	\begin{align}\label{eq:new-algo:npl-poisson:neg-likl-pp}
	\begin{split}
	L( \lambda \mid Z^t, A, \lambda, t) &= -\sum_{j=1}^{N^t}\log(\Lambda(z_j)) + 
	\int_{Z\times [0,t)}\hspace{-0.8cm}\Lambda(z)\, dz dt \\
	& = -\int_{Z\times [0,t)}\hspace{-0.8cm}\log(\Lambda) dZ^t + t \int_{Z}\Lambda(z)\, dz, \, \Lambda(z) = A\lambda(z).
	\end{split}
	\end{align}
	
	\subsection{Misspecification and the KL-projection} 
	\label{subsect:misspecification-kl-projection}
	In reality our model assumption on $Z^t$ is always incorrect (i.e., $\mathcal{P}\mathcal{P}^t_{A\lambda}$ being  misspecified) and $Z^t\sim \mathcal{PP}^t$ for some point temporal stationary process $\mathcal{P}\mathcal{P}^t$, $\mathcal{PP}^t\not = \mathcal{P}\mathcal{P}^t_{A\lambda}$ for any $\lambda\succeq 0$. Since the (penalized) maximum log-likelihood estimates are the most popular in ET, we say that the best one can hope to reconstruct using family  $\mathcal{P}\mathcal{P}^t_{A\lambda}$ is the projection of $\mathcal{P}\mathcal{P}^t$ onto $\mathcal{PP}^t_{A\lambda}$ in the sense of Kullback-Leibler divergence:
	\begin{equation}\label{eq:new-algo:npl-poisson:opt-kullback-projection}
	\lambda_{*}(\mathcal{PP}^t) = \argmin_{\lambda\succeq 0} \mathcal{KL}(\mathcal{P}\mathcal{P}^t, \mathcal{P}\mathcal{P}^t_{A\lambda}).
	\end{equation} 
	Note that due to temporal stationarity of $\mathcal{PP}^t$, $\mathcal{PP}^t_{A\lambda}$, parameter  $\lambda_*$ in \eqref{eq:new-algo:npl-poisson:opt-kullback-projection} is independent of $t$ ($t$ being the proportionality factor in \eqref{eq:new-algo:npl-poisson:opt-kullback-projection} so it has no effect on $\lambda_*$).
	Since  $A$ is  ill-conditioned (see formula~\eqref{eq:design-matrix-non-empt-ker}), in general, $\lambda_*$ in \eqref{eq:new-algo:npl-poisson:opt-kullback-projection} may not be defined uniquely. For this we consider the penalized KL-projection defined by the formula:
	\begin{align}\label{eq:new-algo:npl-poisson:opt-kullback-projection-pen}
	\lambda_{*}(\mathcal{PP}^t, \beta^t) = \argmin_{\lambda\succeq 0}[ \mathcal{KL}(\mathcal{P}\mathcal{P}^t, \mathcal{P}\mathcal{P}^t_{A\lambda}) + \beta^t \varphi(\lambda)], 
	\end{align}
	where $\beta^t$ is the regularization coefficient and $\varphi(\lambda)$ is a nonparametric extension of penalty from Section~\ref{sect:prelim}. From \eqref{eq:new-algo:npl-poisson:neg-likl-pp} and the definition of Kullback-Leibler divergence it follows that 
	\begin{align}\label{eq:new-algo:npl-poisson:kullback-opened}
	\mathcal{KL}(\mathcal{P}\mathcal{P}^t, \mathcal{PP}^t_{A\lambda}) = -\int_{Z\times [0,t)}\hspace{-0.7cm}\log(\Lambda)\mathbb{E}_{\mathcal{P}\mathcal{P}^t}[dZ^t] + t\int_Z \Lambda(z)dz, 
	\end{align}
	where $\mathbb{E}_{\mathcal{P}\mathcal{P}^t}$ is the expectation with respect to $\mathcal{P}\mathcal{P}^t$. Putting together \eqref{eq:new-algo:npl-poisson:opt-kullback-projection-pen},  \eqref{eq:new-algo:npl-poisson:kullback-opened}, for the penalized KL-projection we get the following formulas:
	\begin{align}\label{eq:new-algo:npl-poisson:kullback-projection}
	\lambda_* &= \argmin_{\lambda\succeq 0} \mathbb{L}_p(\lambda\mid \mathcal{PP}^t,\, A, \, t, \, \beta^t),\\
	\label{eq:new-algo:npl-poisson:kullback-functional}
	\begin{split}
	\mathbb{L}_p(\lambda\mid \mathcal{PP}^t, \, A, \, t, \, \beta^t)&= -\int_{Z\times [0,t)}\log(\Lambda) \mathbb{E}_{\mathcal{P}\mathcal{P}^t}[dZ^t] + t\int_Z \Lambda(z)dz + \beta^t\varphi(\lambda), \\
	\Lambda(z) &= A\lambda(z).
	\end{split}
	\end{align} 
	
	\subsection{Propagation of uncertainty and the generic algorithm}
	\label{subsect:propagation-uncertainty} 
	Following the idea from \citet{lyddon2018npl}, we say that uncertainty on $\lambda$ propagates from the one on $\mathcal{PP}^t$ via  \eqref{eq:new-algo:npl-poisson:kullback-projection}, \eqref{eq:new-algo:npl-poisson:kullback-functional}.
	Let $\pi_{\mathcal{M}}$ be a prior in which we encode our beliefs over a set of possible $\mathcal{PP}^t$'s, that is $\pi_{\mathcal{M}}$ is a nonparametric prior on spatio-temporal point processes on $(0, \infty)\times Z$. In particular,  $\pi_{\mathcal{M}}$ is constructed using multimodal  data $\mathcal{M}$.
	Let data be the list-mode $Z^t$ (or the sinogram $Y^t$), then our prior beliefs on $\mathcal{PP}^t$ can be updated in form of posterior distribution  $\pi_\mathcal{M}(\cdot \mid Z^t \vee Y^t, t)$.
	In this case the definition of NPL for ET with multimodal data is straightforward as shown below.\\
	

	\begin{center}
		\begin{minipage}{0.87\textwidth}
			\begin{algorithm}[H]
				\vspace{0.1cm}
				\KwData{list-mode $Z^t$ or sinogram $Y^t$,  $\mathcal{M}$}
				\vspace{0.2cm}
				\KwIn{$B$ -- number of samples, $A$, $\beta^t$,  $\varphi(\lambda)$}
				\vspace{0.2cm}
				\For{ $b=1$ \KwTo $B$ }{
					\vspace{0.2cm}
					Draw point process $\widetilde{\mathcal{PP}}^t\sim \pi_{\mathcal{M}}(\cdot \mid Z^t \vee Y^t, t)$\;
					Compute $\widetilde{\lambda}_b^t = \argmin\limits_{\lambda\succeq 0} \mathbb{L}_p(\lambda\mid \widetilde{\mathcal{PP}}^t, A, t, \beta^t)$ for $\mathbb{L}_p(\cdot)$ defined in \eqref{eq:new-algo:npl-poisson:kullback-functional}\;
					
				}
				\vspace{0.2cm}
				\KwOut{$\{\widetilde{\lambda}^t_b\}_{b=1}^B$}
				\caption{NPL for ET with multimodal data}
				\label{alg:npl-posterior-sampling}
			\end{algorithm}
		\end{minipage}
	\end{center}
	
	As it has already been outlined in \citet{lyddon2018npl}, \citet{fong2019scalable}, the above scheme generates i.i.d samples and is trivially parallelizable which is a strong advantage in front of MCMC sampling from pure Bayesian posteriors (see  Section~\ref{sect:motivating-example-mcmc}).
	
	\subsection{Construction of $\pi_{\mathcal{M}}(\cdot )$ and of posterior $\pi_{\mathcal{M}}(\cdot \mid Z^t \vee Y^t, t)$} Sample $Z^t\sim\mathcal{PP}^t$ is a purely atomic random measure on $(0, \infty) \times Z$ which stands for photon registration events along various lines $z\in Z$ during period $[0, t)$. It is intuitive to assume mutual independence of emission events inside the patient, which is then translated as follows:
	\begin{align}
	\nonumber
	&\text{for any finite family of mutually disjoint bounded Borel sets $\{B_i\}_{i=1}^{N}$}, \, B_i\in B(Z),\\
	\label{eq:new-algo:npl-poisson:gt-mutual-independence}
	&\text{$Z^t(B_i\times [0,t)) = \int_{B_i\times [0,t)}\hspace{-0.8cm}dZ^t$, $i = 1, \dots, N$,} 
	\text{ are mutually independent}.
	\end{align}
	\par Measure $Z^t$ which satisfies \eqref{eq:new-algo:npl-poisson:gt-mutual-independence} is known as completely random measure; see \citet{daley2007pointproc2}, Chapter 10. In particular, under the additional and intuitive assumption that $Z^t$ contains no fixed atoms (i.e., $Z^t$~is purely atomic but locations and  registration times differ from sample to sample) the representation theorem of Kingman says that $\mathcal{PP}^t$ is characterized uniquely by a Poisson point process with some intensity measure $\mu$ on $Z\times [0, +\infty)$; see \citet{daley2007pointproc2}, Section 10.1, Theorem 10.1.III. Therefore, any prior on $\mathcal{PP}^t$ must be a prior on~$\mu$. 
	
	In view of the above discussion and temporal stationarity of $\mathcal{PP}^t$ we assume that 
	\begin{align}\label{eq:new-algo:npl-poisson:assumption-prior-ppt}
	&\mathcal{PP}^t = \mathcal{PP}^t_{\Lambda}, \text{ for some intensity }
	\Lambda \text{ on }Z, \text{ that is if } Z^t\sim \mathcal{PP}^t, \text{ then}\\ \nonumber
	&Z^t(B\times [0,t)) \sim \mathrm{Po}(t \Lambda(B)), \, \Lambda(B) = \int_{B}\Lambda(z) \, dz, 
	\text{ for any } B\in B(Z).
	\end{align}
	Note that $Y^t(B) = Z^t(B\times [0,t))$, where $Y^t$ are the sinogram data. The above assumption can also be interpreted that we do not rely completely on design $A$ when inferring on $\mathcal{PP}^t$ (moreover, $A$ is known only approximately in practice).
	
	Hence, to build $\pi_{\mathcal{M}}$ we construct a prior on $\Lambda$ using $\mathcal{M}$. For the prior on $\Lambda$ we use the mixture of gamma processes (further denoted by MGP) which can be written as follows:
	\begin{align} 
	\begin{split}\label{eq:new-algo:npl-poisson:mgp-intro}
	&\Lambda_{\mathcal{M}}\sim P_{\mathcal{M}}(\cdot ), \, \Lambda \mid \Lambda_{\mathcal{M}} \sim GP(\theta^t \Lambda_{\mathcal{M}}, (\theta^t)^{-1} \mathds{1}_Z), 
	\end{split}
	\end{align}
	where $\Lambda_{\mathcal{M}}$ is the mixing parameter, $P_{\mathcal{M}}(\cdot)$ is the mixing distribution (hyperprior),
	$\theta^t$ is a positive scalar, $\mathds{1}_Z$ is the identity function on $Z$, $GP(\alpha, \beta) = G_{\alpha, \beta}$ is the weighted gamma process on $Z$ (shape $\alpha$ and scale $\beta$).

	In short, we will use the following notation 
	\begin{align}\label{eq:new-algo:npl-poisson:mgp-short}
		\pi_{\mathcal{M}}(\cdot) = \mathrm{MGP}(P_{\mathcal{M}}, t, \theta^t \Lambda_{\mathcal{M}}, (\theta^t)^{-1}).
	\end{align}

	Note that the scale parameter in the gamma process in \eqref{eq:new-algo:npl-poisson:mgp-intro} is constant for all $Z$ and is equal  to  $(\theta^t)^{-1}$. Such choice allows to center gamma process $\Lambda$ on $\Lambda_{\mathcal{M}}$, so  $\theta^t$ controls only the spread (e.g., $\theta^t =  0$ corresponds to improper uniform distribution on $Z$, $\theta^t = +\infty$ corresponds to prior $\mathcal{PP}^t_{\Lambda_{\mathcal{M}}}$, where $\Lambda_{\mathcal{M}}\sim P_{\mathcal{M}}(\cdot)$).
	
	The key to compute the posterior for MGP in \eqref{eq:new-algo:npl-poisson:mgp-short} is the following theorem which is an adaptation of Theorem~3.1 from~\citet{albertlo1982bayesiannon}.
	
	\begin{theorem}\label{thm:non-parametric-poisson-gamma-posterior}
		Let $Y^t\sim \mathcal{PP}^t_{\Lambda}$ and $G_{\alpha, \beta}$ be the prior on~$\Lambda$. Then, the posterior distribution of $\Lambda$ is a weighted gamma process $G_{\alpha + Y^t, \frac{\beta}{1 + t\beta}}$.
	\end{theorem}
	
	From the result of Theorem~\ref{thm:non-parametric-poisson-gamma-posterior} it follows that posterior for MGP in \eqref{eq:new-algo:npl-poisson:mgp-short} is also an MGP:
	\begin{equation}\label{eq:new-algo:npl-poisson:mgp-posterior-short}
	\pi_{\mathcal{M}}(\cdot \mid Z^t, t) = \mathrm{MGP}(P_{\mathcal{M}}(\widetilde{\Lambda}^t_{\mathcal{M}} \mid Z^t\vee Y^t, t), t,  Y^t + \theta^t \widetilde{\Lambda}^t_{\mathcal{M}}, (\theta^t + t)^{-1}),
	\end{equation}
	where $P_{\mathcal{M}}(\widetilde{\Lambda}^t_{\mathcal{M}} \mid Z^t \vee Y^t, t)$ is posterior for the mixing parameter. From \eqref{eq:new-algo:npl-poisson:mgp-intro}-\eqref{eq:new-algo:npl-poisson:mgp-posterior-short} one can see that samples from the MGP posterior are normalized random pseudo-sinograms $\widetilde{\Lambda}_{\mathcal{M}}^t$ in the MRI-based model linearly combined with observed data $Z^t$. Therefore, regularizing effect of MRI originally takes place in the observation space through pseudo-observations.
	
	
	\begin{remark}
		MGP prior in \eqref{eq:new-algo:npl-poisson:mgp-short} and the posterior in \eqref{eq:new-algo:npl-poisson:mgp-posterior-short} are direct analogs of MDP (Mixture of Dirichlet processes) prior and posterior from \citet{lyddon2018npl}, respectively. Weighted gamma processes as priors were also considered in \citet{james2003gwprocesses} for various semiparametric intensity models including very elaborate Poisson model for PET (temporal non-stationarity, detector transition kernels). In particular, in \citet{james2003gwprocesses} a weighted gamma prior was used in the image space (i.e., as a prior on $\lambda$) but not in observation space and the sampling from posteriors was based on data augmentation schemes similar to the one in Section~\ref{sect:motivating-example-mcmc} for which MCMC is difficult. 
		In our approach most of complexity is moved to construction of a ``good'' prior in observation space which should be initially centered at the true (KL-optimal) intensity map built from MRI data which also puts zero (or small) mass on $\Lambda \not \in R_+(A)$ 
		(see also formula \eqref{eq:new-algo:npl-poisson:mgp-intro}).
	\end{remark}
	
	\subsection{Binning to parametric models and algorithms}
	\label{subsect:new-algo:binning}
	Each detector has a screen of finite size which detects incoming photons from a family of lines in $Z$. Let the machine detect photons along $d$ LORs. Mathematically it means that 
	$
	Z = \left(\bigsqcup_{i=1}^{d}Z_i\right) \bigsqcup \overline{Z},
	$
	where each set $Z_i\in B(Z)$ corresponds to set of lines which are visible in LOR  $i$, $\overline{Z}$ are the lines which are not visible at all. For each $i$ we define binning of the data and the corresponding intensities by the formulas:
	\begin{align}\label{eq:new-algo:binning:binning-operator}
	&\left(\int_{Z_i\times [0, t)} \hspace{-0.5cm}dZ^t, \int_{Z_i}\Lambda(z)\, dz\right) = (Y_i^t, \Lambda_i), \\
	&Y_i^t \text{ are mutually independent and } Y_i^t\sim \mathrm{Po}(t \Lambda_i), i\in\{1,\dots, d\}.
	\end{align}
	Nonparametric weighted gamma prior and its posterior in \eqref{eq:new-algo:npl-poisson:mgp-short}, \eqref{eq:new-algo:npl-poisson:mgp-posterior-short}, penalized negative log-likelihood in \eqref{eq:new-algo:npl-poisson:kullback-functional} are also binned in a similar way with \eqref{eq:new-algo:binning:binning-operator}, so the finite-dimensional version of Algorithm~\ref{alg:npl-posterior-sampling} can be written as follows\\
	
	
	\begin{center}
		\begin{minipage}{0.87\textwidth}
			\begin{algorithm}[H]
				\vspace{0.1cm}
				\KwData{sinogram $Y^t$, $\mathcal{M}$}
				\vspace{0.2cm}
				\KwIn{$B$ -- number of samples, $\theta^t$, $A$, $\beta^t$, $\varphi(\lambda)$}
				\vspace{0.2cm}
				\For{ $b=1$ \KwTo $B$ }{
					\vspace{0.2cm}
					Draw $\widetilde{\Lambda}^t_{\mathcal{M}} = (\widetilde{\Lambda}^t_{\mathcal{M}, 1}, \dots, \widetilde{\Lambda}^t_{\mathcal{M}, d})$ from $P_{\mathcal{M}}(\widetilde{\Lambda}^t_{\mathcal{M}} \mid Y^t, t)$\;
					\label{eq:alg:npl-posterior-sampling:binned:mixing-param}
					\vspace{0.2cm}
					Draw $\widetilde{\Lambda}_{b,i}^t \sim \Gamma(Y_i^t + \theta^t \widetilde{\Lambda}^t_{\mathcal{M}, i}, (\theta^t + t)^{-1})$ independently for each $i$\;
					\label{eq:alg:npl-posterior-sampling:binned:post-intens}
					\vspace{0.2cm}
					Compute $\widetilde{\lambda}_b^t = \argmin\limits_{\lambda\succeq 0} L_p( \lambda\mid\,   \widetilde{\Lambda}_b^t, A, t, \beta^t/t)$ for $L_p(\cdot)$ defined in~\eqref{eq:prelim:poiss-log-likelihood-penalized}\;
				}
				\vspace{0.2cm}
				\KwOut{$\{\widetilde{\lambda}^t_b\}_{b=1}^B$}
				\caption{Binned NPL for ET with multimodal data}
				\label{alg:npl-posterior-sampling:binned}
			\end{algorithm}
		\end{minipage}
	\end{center}
	\begin{remark}
		\label{rem:npl-posterior-sampling:normalization}
		In steps~1,~2 intensities $\widetilde{\Lambda}_{b,i}^t$ are sampled from the binned MGP posterior in  \eqref{eq:new-algo:npl-poisson:mgp-posterior-short}.
		In step~3 we have used the fact that binned version of $\mathbb{L}_p(\cdot)$ from \eqref{eq:new-algo:npl-poisson:kullback-functional} coincides with $L_p(\cdot)$ from \eqref{eq:prelim:poiss-log-likelihood-penalized}.
		In addition, from formula \eqref{eq:prelim:poiss-log-likelihood-penalized} it follows that 
		\begin{equation}
		L_p(\lambda \mid t\widetilde{\Lambda}_b^t, A, t, \beta^t) = t 
		L_p(\lambda \mid \widetilde{\Lambda}_b^t, A, 1, \beta^t/t) + R,
		\end{equation}
		where $R$ is a function which is independent of $\lambda$. Therefore, minimization in  step~3 is directly applied to $L_p(\lambda \mid \widetilde{\Lambda}_b^t, A, 1, \beta^t/t)$ instead of $L_p(\lambda \mid t\widetilde{\Lambda}_b^t, A, t, \beta^t)$.	
	 	If the complexity of sampling in step~1 is controlled by our choice of $P_{\mathcal{M}}(\cdot)$, step~3 is inevitable, hence, it must be numerically feasible via some scalable optimization algorithm.
		 This is the case for us in view of the well-known in ET the Generalized Expectation-Maximization(GEM)-type algorithm from~\citet{fessler1995sagepet} which  is specially designed for Poisson-type log-likelihood $L_p(\cdot)$, where $\varphi(\cdot)$ must be a convex pairwise difference penalty, for example, as one in our numerical experiment (see Supplementary Material~\ref{eq:prelim:numerical:log-cosh-prior-def}).
	\end{remark}

	\subsection{Final algorithm}
	\label{subsect:final-algo}
		First, we explain the intuition behind sampling from $P_{\mathcal{M}}(\widetilde{\Lambda}_{\mathcal{M}}^t \, | \, Y^t, t)$ in step~1 in Algorithm~\ref{alg:npl-posterior-sampling:binned}, then, we present the formal and complete procedure.
		
		Using $\mathcal{M}$ (see Figure~\ref{fig:mri-ct-comparison}(c)) we construct a model of type \eqref{eq:poisson-model-pet} for which we assume that the isotope's concentration is constant in each segment and has uniform (improper) prior distribution on $\R_+$. If $\lambda_{\mathcal{M}}\in \R^{p_{\mathcal{M}}}_+$ be the corresponding random vector ($p_{\mathcal{M}}$ being the number of segments), then a sample from the prior $P_{\mathcal{M}}(\cdot)$ is defined as $\Lambda_{\mathcal{M}} = A_{\mathcal{M}}\lambda_{\mathcal{M}}$, where $A_{\mathcal{M}} \in \mathrm{Mat}(d, p_{\mathcal{M}})$ is the design for segment-like model of ET computed directly from $A$ (see formulas \eqref{eq:wbb-mri:npl-presentation:reduced-design-elem}, \eqref{eq:wbb-mri:npl-presentation:concat-design-def}). The point is that $p_{\mathcal{M}} \ll p$, so $A_{\mathcal{M}}$ is of moderate size (hence, can be stored in memory), is also injective and  well-conditioned. Posterior $P(\widetilde{\Lambda}^t_{\mathcal{M}}\,  \mid\, Y^t, t)$ is defined via classical Bayes' formula for model $P(Y^t \mid A_{\mathcal{M}}, \lambda_{\mathcal{M}}, t) = \mathrm{Po}(t\Lambda_{\mathcal{M}})$ and the aforementioned prior on~$\Lambda_{\mathcal{M}}$. 
		Formal constructions of $P_{\mathcal{M}}(\cdot)$, $P_{\mathcal{M}}(\cdot \mid Y^t, t)$ are given in  Supplementary Material, Section~\ref{app:new-algo:wbb-algorithm-with-mri}. In practice, for the sake of simplicity we sample from $P_{\mathcal{M}}(\widetilde{\Lambda}_{\mathcal{M}}^t \, | \, Y^t, t)$ using the weighted log-likelihood bootstrap (WLB) adapted for ET.\\
		
		
		\begin{center}
		\begin{minipage}{0.87\textwidth}
			\begin{algorithm}[H]
				\KwData{sinogram $Y^t$, $\mathcal{M}$}
				\vspace{0.2cm}
				\KwIn{$A_\mathcal{M}\in \mathrm{Mat}(d, p_{\mathcal{M}})$ from \eqref{eq:wbb-mri:npl-presentation:reduced-design-elem}, \eqref{eq:wbb-mri:npl-presentation:concat-design-def} (well-conditioned)}
				\vspace{0.2cm}
				Draw $\widetilde{\Lambda}^{t}_i \, {\sim} \, \Gamma(Y_i^t, t^{-1})$ independently for each $i\in \{1, \dots, d\}$\;
				\vspace{0.3cm}
				\label{eq:alg:wbb-pet-boostrap:mri:mixing-param-perturbation}
				Compute $\widetilde{\lambda}^{t}_{\mathcal{M}} = \argmin\limits_{\lambda_{\mathcal{M}}\succeq 0}L(\lambda_{\mathcal{M}} \mid \widetilde{\Lambda}^{t}, A_{\mathcal{M}}, 1)$, $L(\cdot)$ being defined in \eqref{eq:prelim:poiss-log-likelihood-model}\;
				\label{eq:alg:wbb-pet-bootstrap:mri:mixing-param-minimizer}
				
				Compute $\widetilde{\Lambda}_{\mathcal{M}}^t =  A_{\mathcal{M}}\widetilde{\lambda}_{\mathcal{M}}^{t}$\;
				\KwOut{$\widetilde{\Lambda}^t_{\mathcal{M}}$ is sampled approximately from $P_{\mathcal{M}}(\widetilde{\Lambda}_{\mathcal{M}}^t \mid Y^t, t)$}
				\caption{Approximate sampling from $P_{\mathcal{M}}(\widetilde{\Lambda}^t_{\mathcal{M}} \mid Y^t, t)$ (via WLB)}
				\label{alg:wbb-pet-bootstrap:mri:posterior-mixing-param}
			\end{algorithm}
		\end{minipage}
		\end{center}
	
	\begin{remark}		
		Since we assume that $A_{\mathcal{M}}$ is well-conditioned, minimizer $\widetilde{\lambda}^{t}_{\mathcal{M}}$ in step~2 of Algorithm~\ref{alg:wbb-pet-bootstrap:mri:posterior-mixing-param} can be efficiently  computed via the classical EM-algorithm from  \citet{shepp1982mlem}.\\
	\end{remark}

	
	\begin{center}
		\begin{minipage}{0.87\textwidth}
			\begin{algorithm}[H]
				\vspace{0.1cm}
				\KwData{sinogram $Y^t$, $\mathcal{M}$}
				\vspace{0.2cm}
				\KwIn{$B$ -- number of samples, $\theta^t$,  $A_{\mathcal{M}}$, $A$, $\beta^t$, $\varphi(\lambda)$}
				\vspace{0.2cm}
				\For{ $b=1$ \KwTo $B$ }{
					\vspace{0.2cm}
					Draw $\widetilde{\Lambda}^t_{\mathcal{M}} = (\widetilde{\Lambda}^t_{\mathcal{M}, 1}, \dots, \widetilde{\Lambda}^t_{\mathcal{M}, d})$ from $P_{\mathcal{M}}(\widetilde{\Lambda}^t_{\mathcal{M}} \mid Y^t, t)$ via Algorithm~\ref{alg:wbb-pet-bootstrap:mri:posterior-mixing-param}\;
					\label{eq:alg:npl-posterior-sampling:binned:mixing-param:final}
					\vspace{0.2cm}
					Draw $\widetilde{\Lambda}_{b,i}^t \sim \Gamma(Y_i^t + \theta^t \widetilde{\Lambda}^t_{\mathcal{M}, i}, (\theta^t + t)^{-1})$ independently for each $i$\;
					\label{eq:alg:npl-posterior-sampling:binned:post-intens:final}
					\vspace{0.2cm}
					Compute $\widetilde{\lambda}_b^t = \argmin\limits_{\lambda\succeq 0} L_p( \lambda\mid\,   \widetilde{\Lambda}_b^t, A, 1, \beta^t/t)$ for $L_p(\cdot)$ defined in \eqref{eq:prelim:poiss-log-likelihood-penalized} 
					\nonl using the GEM-type algorithm from \citet{fessler1995sagepet}\;
					\label{eq:alg:npl-posterior-sampling:binned:argmin:final}
					
				}
				\vspace{0.2cm}
				\KwOut{$\{\widetilde{\lambda}^t_b\}_{b=1}^B$}
				\caption{Binned NPL for ET with MRI data}
				\label{alg:npl-posterior-sampling:mri:binned}
			\end{algorithm}
		\end{minipage}
	\end{center}
	
	\begin{remark}\label{rem:parameter-theta-interpretation}
		Parameter $\theta^t$ in Algorithm~\ref{alg:npl-posterior-sampling:mri:binned} has the following  physical meaning: it is exactly the rate of creation of ``pseudo-photons'' in the poisson model constructed from MRI data. More precisely, by choosing $\theta^t = \rho t, \, \rho \geq 0$ in step~2 we sum up sinograms $Y^t$ and $t\widetilde{\Lambda}^t_{\mathcal{M}}$ in  proportions $1 / (1 + \rho)$ and $\rho / (1+\rho)$, respectively.
		For $\theta^t = 0$ we see Algorithm~\ref{alg:npl-posterior-sampling:mri:binned} as a version of WLB from \citet{newton1994wbb} being adapted for the ET context; see also \citet{lyddon2018npl}, \citet{fong2019scalable}, \citet{pompe2021introducing} for connections between classical WLB and NPL.
	\end{remark}

	Numerical tests of Algorithm~\ref{alg:npl-posterior-sampling:mri:binned} are given in the Supplementary Material, Section~\ref{app:gem}.

\section{Asymptotic analysis of the new algorithm}
\label{sect:new-thero-results}
	
	Statistical model \eqref{eq:poisson-model-pet} is non-regular because the domain for parameter $\lambda$ is not open, contains boundary $\partial \R^p_+ = \{\lambda\in \R^p_+ : \exists j \text{ s.t. } \lambda_j = 0\}$ and, in general, $\lambda_*\in \partial\R^p_+$. This model was investigated in the small noise limit (i.e., when $t\rightarrow +\infty$) in pure Bayesian framework in \citet{bochkina2014} for large class of priors for the well-specified case (i.e., $Y^t\sim P^t_{A,\lambda_*}$ for some $\lambda_* \in \R^p_+$) and for design $A$ of the full rank though also ill-conditioned. It was shown that the posterior is consistent at $\lambda_*$, the asymptotic distribution is centered around the MLE estimate for the quadratic approximation of $L(\lambda \mid Y^t, A, t)$ and the non-regularity results in splitting of the posterior in three modes: multivariate exponential (for coordinates which are related to pixels intersected by LORs with zero photon intensities) contracting to zeros with the fastest rate (scaled with $t$), 
	Gaussian (for pixels where $\lambda_{*, j} > 0$) and half-Gaussian (for pixels with $\lambda_{*,j} = 0$ and pixels being intersected only by LORs with positive intensities) contracting with standard rate (scaled with~$\sqrt{t})$.
	
	Our results for consistency and conditional distribution are similar to ones from \citet{bochkina2014}, however, there are several major and minor differences. Asymptotic consistency at $\lambda_*$ and a very similar splitting are also present in NPL, with the asymptotic distribution being  tight around a strongly consistent estimator $\widehat{\lambda}^t_{sc}$ satisfying additional properties in observation space. The assumptions we put on $\widehat{\lambda}_{sc}^t$ for conditional tightness seem very natural and we discuss them thoroughly in the text. 
	Interestingly, the splitting of the posterior into different modes depends not on $\lambda_*$ (as it was in \citet{bochkina2014}) but again on $\widehat{\lambda}_{sc}^t$ because of which yet we fail to demonstrate the asymptotic normality since it requires additional results on behavior of strongly consistent estimators with constraints on the domain. Intuitively, the asymptotic distribution should be similar to the frequentist distribution of MAP estimates from \citet{bochkina2014}: atom at zero for the exponential part, Gaussian -- for the Gaussian part, and sum of atom at zero and half-Gaussian for the half-Gaussian part (see \citet{geyer1994asymptotics}). We address this investigation for future and conjecture that classical MLE or penalized MLE (i.e., MAP) from \citet{bochkina2014} are the right candidates for $\widehat{\lambda}_{sc}^t$. 
	
	A minor remark would be that, in pure Bayesian framework there is only one free parameter that is controlled by a specialist -- the prior distribution, whereas in Algorithm~\ref{alg:npl-posterior-sampling:mri:binned} we have several free parameters: $\theta^t,  \, \beta^t$, $A_{\mathcal{M}}$. Therefore, our theoretical results also contain restrictions on the above parameters. At the end, we address the problem of model misspecification for the generalized Poisson model with wrong design which arises twice our setting: first, in Algorithm~\ref{alg:wbb-pet-bootstrap:mri:posterior-mixing-param} when sampling $\widetilde{\Lambda}_{\mathcal{M}}^t$ (because we use $Y^t$ with incorrect design $A_{\mathcal{M}}$) and, second, when assume that model \eqref{eq:poisson-model-pet} is wrong, in general.

	\subsection{Convergence for conditional probabilities.}
	\label{subsect:theory:notations}
	
	Let $(\Omega, \mathcal{F}, P)$ be the common probability space on which process $Y^t$, $t\in (0,+\infty)$ and MGP prior in \eqref{eq:new-algo:npl-poisson:mgp-short} are defined (see Supplementary Material, Section~\ref{app:prob-space} for details). Let 
	\begin{equation}
	\mathcal{F}^t = \sigma(Y^\tau, \, \tau\in (0, t)) \subset \mathcal{F},
	\end{equation}
	where $\sigma(\cdot)$ denotes the sigma-algebra generated by a family of random variables.
	
	\begin{definition}
		We say that $U^t$ converges in conditional probability to $U$ almost surely $Y^t$, $t\in (0, +\infty)$ if for every $\varepsilon > 0$ the following holds:
		\begin{equation}
		\label{eq:convergence-cond-prob-as-def}
		P(\|U^t - U\| > \varepsilon\,  \mid \, \mathcal{F}^t) \rightarrow 0 \text{ when }
		t\rightarrow +\infty, \text{ a.s. } Y^t, t\in (0, +\infty).
		\end{equation}
		This type of convergence will be denoted as follows:
		\begin{equation}
		U^t \xrightarrow{c.p.} U \text{ when } t\rightarrow +\infty, \text{ a.s. } Y^t, t\in (0, +\infty).
		\end{equation}
		\qed
	\end{definition}
	In our proofs for $U^t \xrightarrow{c.p.} 0$ we also write
	\begin{equation}
	U^t = o_{cp}(1).
	\end{equation}

	\begin{definition}
		We say that $U^t$ is conditionally tight almost surely $Y^t$, $t\in (0, +\infty)$ if for any $\varepsilon > 0$ and almost any trajectory $Y^t$, $t\in (0, +\infty)$ there exists $M = M(\varepsilon, \{Y^t\}_{t\in (0, +\infty)})$ such that 
		\begin{equation}\label{eq:convergence-cond-tightness-def}
		\sup_{t\in (0, +\infty)} P(\|U^t\| > M \, \mid \, \mathcal{F}^t) < \varepsilon.
		\end{equation}
	\end{definition}
	
	In short, in the definition above almost surely $Y^t$, $t\in (0, +\infty)$ means that statements in  \eqref{eq:convergence-cond-prob-as-def},~\eqref{eq:convergence-cond-tightness-def} hold for almost every trajectory $Y^t$, $t\in (0,+\infty)$.

	\subsection{Consistency}
	\label{subsect:theory:consistency}
	
	\begin{assumption}\label{assump:theory:consistency:well-spec}
		Model \eqref{eq:poisson-model-pet} is well-specified, that is
		\begin{equation}\label{eq:theory:assumption-well-specified-lstar}
		Y^t \sim P^t_{A, \lambda_*}, \text{ for some } \lambda_* \in \R^p_+ \text{ and all } t\in (0, +\infty),
		\end{equation}
		where $A$ satisfies \eqref{eq:design-matrix-positivity-restr-1}-\eqref{eq:design-matrix-non-empt-ker}, $P^t_{A,\lambda}$ is defined in \eqref{eq:prelim:poiss-prob-model}.
	\end{assumption}
	
	\begin{theorem}
		\label{thm:wbb-algorithm-consistency}
		Let Assumption~\ref{assump:theory:consistency:well-spec} and conditions \eqref{eq:prelim:penalty-cond-convex}, \eqref{eq:prelim:penalty-cond-strict-conv} for $\varphi(\lambda)$ be satisfied. Let also   
		 $\beta^t$, $\theta^t$ be such that 
		\begin{align}
		\label{eq:thm-wbb-algorithm-consistency-params}
		\beta^t/t \rightarrow 0, \, \theta^t/t \rightarrow 0 \, \text{ when } t\rightarrow +\infty.
		\end{align}
	 	Then, 
		\begin{equation}\label{eq:thm-wbb-algorithm-consistency-fmla}
		\widetilde{\lambda}^t_b - \lambda_* \xrightarrow{c.p.} w_{A,\lambda_*}(0) \text{ when } t\rightarrow +\infty, 
		\text{ a.s. } Y^t, \, t\in (0, +\infty),
		\end{equation}
		where $\widetilde{\lambda}^t_b$ is sampled in Algorithm~\ref{alg:npl-posterior-sampling:mri:binned}, 
		$w_{A,\lambda}(\cdot)$ is defined in \eqref{eq:proofs:lemma-existence-minimizer-ker-a}.
		
	\end{theorem}
	
	Conditional distribution of $\widetilde{\lambda}_{b}^t$  asymptotically concentrates at $\lambda_*$ in the subspace $\mathrm{Span}(A^T)$, where parameter $\lambda$ is identifiable through design $A$ and also regarding the positivity constraints. On the other hand, projection of $\lambda_*$ onto $\mathrm{ker}A$ is not identifiable in model \eqref{eq:poisson-model-pet} and it is defined solely by  penalty~$\varphi(\lambda)$ and positivity constraints at $\lambda_*$.

	There is also an extension of the above result for any generic bootstrap type procedure 
	provided that perturbation of $Y^t$ asymptotically is not too excessive.
	
	\begin{theorem}
		\label{thm:wbb-generic-consistency}
		Let conditions of Theorem~\ref{thm:wbb-algorithm-consistency} be satisfied  but Assumption~\ref{assump:theory:consistency:well-spec}. Assume also that 
		\begin{align}
		\begin{split}\label{eq:thm-wbb-generic-consistency-cond-mean}
		&\widetilde{\Lambda}_{b,i}^t \xrightarrow{c.p.} \Lambda_i^* = a_i^T\lambda_*, \,  i = 1, \dots,  \, d, \text{ for some }\lambda_* \in \R^p_+ \\
		&\text{when }t\rightarrow +\infty, \text{ a.s. } Y^t, \, t\in (0,+\infty).
		\end{split}
		\end{align}
		Then, formula \eqref{eq:thm-wbb-algorithm-consistency-fmla} remains valid.
	\end{theorem}

	\subsection{Tightness and the asymptotic distribution}  
	\label{subsect:theory:asymp-distribution}
	
	\begin{assumption}\label{assump:theory:distribution:inj-am}
		$A_{\mathcal{M}}\in \mathrm{Mat}(d, p_{\mathcal{M}})$ is injective.
	\end{assumption}
	
	\begin{assumption}[non-expansiveness condition]
		\label{assump:theory:distribution:non-expansive}
		Let $\Lambda^*\in \R^d_+$, $A_{\mathcal{M}}\in \mathrm{Mat}(d, p_{\mathcal{M}})$, $A_{\mathcal{M}}$ has only positive entries and satisfies the property in \eqref{eq:design-matrix-positivity-restr-2}.
		Consider set $\lambda_{\mathcal{M}, *}$ which is defined by the formula:
		\begin{align}\label{eq:assump:theory:asymp-distr:non-exp-cond}
		\begin{split}
		&\lambda_{\mathcal{M}, *} = \argmin_{\lambda_{\mathcal{M}} \succeq 0}
		L(\lambda_{\mathcal{M}} \mid \Lambda^*, A_{\mathcal{M}}, 1),
		\end{split}
		\end{align}
		where $L(\lambda_{\mathcal{M}} \mid \Lambda^*, A_{\mathcal{M}}, 1)$ is defined in \eqref{eq:prelim:poiss-prob-model}.
		There is at least one point in $\lambda_{\mathcal{M}, *}$ for which the following holds:
		\begin{equation}\label{eq:assump:theory:asymp-distr:non-exp-cond:ind-cond}
		I_0(\Lambda^*_{\mathcal{M}}) = I_0(\Lambda^*), \,  \Lambda^{*}_{\mathcal{M}} = A_{\mathcal{M}}\lambda_{\mathcal{M}, *},
		\end{equation}
		where $I_0(\cdot)$ is defined in \eqref{eq:ind-lors-pos-zeros}.
	\end{assumption}
	
	The proposition below states that Assumption~\ref{assump:theory:distribution:non-expansive} is always meaningful and not restrictive at~all. 
	
	\begin{proposition}
		\label{prop:existence-minimizers}
		Let $A_{\mathcal{M}} \in \mathrm{Mat}(d, p_{\mathcal{M}})$, $A_{\mathcal{M}}$ has only positive entries and the property in \eqref{eq:design-matrix-positivity-restr-2} holds. Then, for any $\Lambda^* \in \R^d_+$ set of minimizers $\lambda_{\mathcal{M},*}$ defined in \eqref{eq:assump:theory:asymp-distr:non-exp-cond} is always non-empty and constitutes an affine subset of $(p_{\mathcal{M}}-1)$-dimensional simplex $\Delta^p_{A_{\mathcal{M}}}(\Lambda^*)$ defined by the formula 
		\begin{equation}\label{eq:prop:existence-restritiveness-minimizers:symplex}
			\Delta^{p_{\mathcal{M}}}_{A_{\mathcal{M}}}(\Lambda^*) = \{ \lambda_{\mathcal{M}}\in \R^p_+ \mid \sum\limits_{j=1}^{p_{\mathcal{M}}} A_{\mathcal{M}, j}\lambda_{\mathcal{M}, j} = \sum\limits_{i=1}^{d}\Lambda_i^*\geq 0\}, \, A_{\mathcal{M},j} = \sum\limits_{i=1}^{d}a_{\mathcal{M}, ij} > 0.
		\end{equation}
		Moreover, it always holds that
		\begin{equation}\label{prop:assumptions-nonrestr:inclusion}
		I_1(\Lambda^*) \subset I_1(\Lambda^*_{\mathcal{M}}) \text{ or equivalently } I_0(\Lambda^*_{\mathcal{M}}) \subset I_0(\Lambda^*), 
		\end{equation}
		where $\Lambda^{*}_{\mathcal{M}} = A_{\mathcal{M}}\lambda_{\mathcal{M}, *}$.
	\end{proposition}

	The non-expansiveness condition is essential for us when we sample $\widetilde{\Lambda}_{\mathcal{M}}^t$ in Algorithm~\ref{alg:wbb-pet-bootstrap:mri:posterior-mixing-param} because we know that model $P_{A_{\mathcal{M}}, \lambda_{\mathcal{M}}}^t$ is strongly misspecified when we fit data $Y^t$ in it. The aim here is still to have a unique and stable KL-minimizer $\lambda_{\mathcal{M},*}$ so that identifiability holds for $\lambda_{\mathcal{M},*}$ and the prior effect of $\mathcal{M}$ on $\widetilde{\lambda}_b^t$ 
	is not spread ambiguously among different (but equivalent in terms of observations) combinations of tracer in segments of $\mathcal{M}$ (see Figure~\ref{fig:mri-ct-comparison} (c)). This is provided by the theorem below.

	\begin{theorem}[identifiability in the prior model]
		\label{thm:asympt-distr:well-spec:identifiability-cond}
		Let Assumptions~\ref{assump:theory:distribution:inj-am}-\ref{assump:theory:distribution:non-expansive} be satisfied. Then, $\lambda_{\mathcal{M}, *}$  defined in \eqref{eq:assump:theory:asymp-distr:non-exp-cond} is unique and 
		the following formula holds:
		\begin{align}\nonumber
		L(\lambda_{\mathcal{M}} \mid \Lambda^*, A_{\mathcal{M}}, 1) - 
		L(\lambda_{\mathcal{M}, *} \mid \Lambda^*, A_{\mathcal{M}}, 1) &= 
		\mu_{\mathcal{M},*}^T\lambda_{\mathcal{M}} + 
		\dfrac{1}{2}\sum\limits_{i\in I_1(\Lambda^*)}
		\Lambda_i^* \dfrac{(\Lambda_{\mathcal{M},i} - \Lambda_{\mathcal{M},i}^*)^2}{(\Lambda^*_{\mathcal{M},i})^2}\\
		\label{eq:asympt-distr:well-spec:identifiability-cond:approx}
		&+ o(\|\Pi_{A^T_{\mathcal{M},I_1(\Lambda^*)}}(\lambda_{\mathcal{M}} - \lambda_{\mathcal{M}, *})\|^2),
		\end{align}
		where $\Pi_{A^T_{\mathcal{M},I_1(\Lambda^*)}}$ denotes the orthogonal projector onto $\mathrm{Span}(A^T_{\mathcal{M},I_1(\Lambda^*)})$, 
		\begin{align}
		\label{eq:asympt-distr:well-spec:identifiability-cond:grad-formulas}
		\begin{split}
		&\mu_{\mathcal{M},*} = \sum\limits_{i\in I_1(\Lambda^*)}-\Lambda^*_i \dfrac{a_{\mathcal{M}, i}}{\Lambda_{\mathcal{M},i}^*} + \sum\limits_{i=1}^{d} a_{\mathcal{M}, i}, \\
		&\mu_{\mathcal{M},*} \succeq 0, \, \mu_{\mathcal{M}, *, j}\lambda_{\mathcal{M},*,j} = 0 \text{ for all }
		j\in \{1, \dots, p_{\mathcal{M}}\}.
		\end{split}
		\end{align}
		In particular, the function $L(\lambda_{\mathcal{M}} \mid \Lambda^*, A, 1)$ is locally strongly convex at $\lambda_{\mathcal{M}, *}$, that is, there exists an open ball $B_* = B(\lambda_{\mathcal{M}, *}, \delta_*)$, $\delta_{*} = \delta_{*}(A_{\mathcal{M}}, \Lambda_*) > 0$ and constant $C_* = C_*(A_{\mathcal{M}}, \Lambda_*) > 0$  such that 
		\begin{align}\label{eq:thm:asympt-distr:identif-cond:strong-convexity}
		L(\lambda_{\mathcal{M}} \mid \Lambda^*, A_{\mathcal{M}}, 1) - L(\lambda_{\mathcal{M}, *} \mid \Lambda^*, A_{\mathcal{M}}, 1) \geq C_*\|\lambda_{\mathcal{M}} - \lambda_{\mathcal{M},*}\|^2 \text{ for any } \lambda \in B_{*}\cap \R^{p_{\mathcal{M}}}_+.
		\end{align}
	\end{theorem}
	
	Result of Theorem~\ref{thm:asympt-distr:well-spec:identifiability-cond} is a  positive answer to the identification problem when model \eqref{eq:poisson-model-pet} is misspecified in the sense of wrong design.  Here, the non-expansiveness condition is essential and counterexamples are possible if it is removed. One such example is constructed in the proof of Theorem \ref{thm:misspecification:main-example} in , Subsection~\ref{subsect:theory:misspecification}. 
	
	Now we can turn to our main result on the tightness of the posterior.


	Let $\{e_j\}_{j=1}^{p}$ be the standard basis in $\R^p$ and define the following spaces:
	\begin{align}
	\label{eq:asymp-distr:subspace-v}
	&\mathcal{V} = \mathrm{Span}\{e_j \, \mid \, \exists\, i\in I_0(\Lambda^*) \text{ s.t. } a_{ij} > 0\}, \\
	\label{eq:asymp-distr:subspace-u}
	&\mathcal{U} =  \mathcal{V}^\perp\cap \mathrm{Span}\{A^T_{I_1(\Lambda^*)}\}, \\
	\label{eq:asymp-distr:subspace-w}
	&\mathcal{W} = (\mathcal{V} \oplus \mathcal{U})^\perp \cap \ker A.
	\end{align}

	Let 
	\begin{align}\label{eq:asymp-distr:projectors}
	\Pi_{\mathcal{V}}, \Pi_{\mathcal{V}}, \Pi_{\mathcal{W}} \text{ be the  orthogonal projectors on $\mathcal{V}, \mathcal{V}, \mathcal{W}$, respectively}.
	\end{align}

	\begin{theorem}[tightness of the asymptotic distribution]\label{thm:asympt-distr-main}
		Let assumptions~\ref{assump:theory:consistency:well-spec}-\ref{assump:theory:distribution:non-expansive} be satisfied and assume also that 
		\begin{align}\label{eq:thm:asympt-distr-main:varphi-assump}
		&\varphi \text{ satisfies } \eqref{eq:prelim:penalty-cond-convex}, \eqref{eq:prelim:penalty-cond-strict-conv} \text{ and $\varphi$ is locally Lipschitz continous}.
		\end{align}
		Let $\widetilde{\lambda}_{b}^t$  be defined as in Algorithm~\ref{alg:npl-posterior-sampling:mri:binned} and $\theta^t = o(\sqrt{t / \log\log t})$, $\beta^t = o(\sqrt{t})$ and assume that there exists a strongly consistent estimator  $\widehat{\lambda}_{sc}^t$ of $\lambda_*$ on $\mathcal{V}\oplus\mathcal{U}$ (i.e., $\Pi_{\mathcal{U}\oplus \mathcal{V}} \widehat{\lambda}_{sc}^t \xrightarrow{a.s.} \Pi_{\mathcal{U}\oplus \mathcal{V}}\lambda_*$) such that 
		\begin{align}
		\label{eq:thm:asympt-distr-main:strongly-consist-estim-positivity}
		&\widehat{\lambda}_{sc}^t  \succeq 0,\\
		\label{eq:thm:asympt-distr-main:strongly-consist-estim-i1}
		&\limsup\limits_{t\rightarrow +\infty}
		\left|
		\sum\limits_{i\in I_1(\Lambda^*)}\sqrt{t} \dfrac{Y_i^t/t - \widehat{\Lambda}_{sc,i}^t}{ \widehat{\Lambda}_{sc,i}^t} a_i \right| < +\infty \text{ a.s. } Y^t, \, t\in (0, +\infty),\\
		\label{eq:thm:asympt-distr-main:strongly-consist-estim-i0}
		&t\widehat{\Lambda}^{t}_{sc,i} \rightarrow 0 \text{ when }t\rightarrow +\infty, \text{ a.s. } Y^t, \, t\in (0, +\infty) \text{ for } i\in I_0(\Lambda^*),
		\end{align}
		where $\widehat{\Lambda}_{sc}^t = A\widehat{\lambda}_{sc}^t$.
		Then,
		\begin{itemize}
			\item[i)] 
			\begin{align}\label{eq:thm:asympt-distr-main:projection-v}
			t \Pi_{\mathcal{V}}(\widetilde{\lambda}_{b}^t - \widehat{\lambda}^t_{sc}) \xrightarrow{c.p.}0
			\text{ when }t\rightarrow +\infty, \text{ a.s. } Y^t, \, t\in (0, +\infty).
			\end{align}
			\item[ii)] Vector $\sqrt{t} \Pi_{\mathcal{U}}(\widetilde{\lambda}_{b}^t - \widehat{\lambda}_{sc}^t)$ is conditionally tight a.s. $Y^t$, $t\in (0, +\infty)$.

		\end{itemize} 
	\end{theorem}
	
	Statement in (i) claims that in pixels which are interested by LORs with zero intensities (i.e. $\Lambda_i^* = 0$) the posterior distribution contracts to zero with faster rate than for the ones intersected by LORs with positive intensities. Indeed, pixels in subspace $\mathcal{V}$ are strongly forced to be zeros by the positivity constraints (i.e., if $\Lambda_i^* = 0$ and $\lambda_*, \, a_i \in \R^p_+$, then necessarily $\lambda_{*,j} = 0$ where $a_{ij} > 0$).	
	Statement in (ii) claims that, in general, the posterior concentrates around $\widehat{\lambda}_{sc}^t$ in subspace $\mathcal{U}$ with standard scaling rate $\sqrt{t}$. This is not surprising since $\mathcal{U}$ does not contain projections on $\mathcal{V}$, so the positivity constraints do not give here extra information to achieve the faster contraction rate. Finally, requiring the non-expansiveness condition for the prior  (i.e., Assumption~\ref{assump:theory:distribution:non-expansive}) may seem surprising at first sight. The intuition behind is that it protects our sampler from creation of ``too many'' pseudo-photons in LORs where intensity is zero (i.e., $\Lambda_i^* = 0$ implies $Y^t_i \equiv 0$ for the well-specified model) and significantly simplifies the theoretical analysis.

	
	For $\widehat{\lambda}_{sc}^t$ we propose to take the penalized  MLE-estimate which is defined by the formula:
	
	\begin{align}\label{eq:asymp-distr:map-estimate-def}
	\widehat{\lambda}_{pMLE}^t = \argmin_{\lambda \succeq 0} L_p(\lambda \mid Y^t, A, t, \beta^t), 
	\end{align}
	where $L_p(\cdot)$ is defined in \eqref{eq:prelim:poiss-log-likelihood-penalized}. 

	\begin{conj}\label{conj:asympt-distr-mle-pen-estimator}
		Let assumptions of Theorem~\ref{thm:asympt-distr-main} be satisfied and   $\widehat{\lambda}_{sc}^t = \widehat{\lambda}_{pMLE}^t$, where the latter is defined by formula \eqref{eq:asymp-distr:map-estimate-def}. Then,  $\widehat{\lambda}_{sc}^t$ is a strongly consistent estimator of $\lambda_*$ on $\mathcal{V}\oplus \mathcal{V}$ and formulas
		\eqref{eq:thm:asympt-distr-main:strongly-consist-estim-positivity}-\eqref{eq:thm:asympt-distr-main:strongly-consist-estim-i0} hold.
	\end{conj}

	The requirement for existence of a strongly consistent estimator for weighted bootstrap is not new and already appears in \citet{ng2020random}. However, in that  case the sampling is performed via unconstrained optimization of quadratic functionals though with $\ell_1$-penalties for which existence of such estimators is trivial by taking the standard OLS estimator or LASSO estimator; see the discussion after Theorem~3.3 in \citet{ng2020random}. According to Kolmogorov's 0-1 Law the statements in \eqref{eq:thm:asympt-distr-main:strongly-consist-estim-i1},~\eqref{eq:thm:asympt-distr-main:strongly-consist-estim-i0} either hold with probability one (i.e., almost surely $Y^t$, $t\in (0, +\infty)$) or zero, and the case of zero probability would mean a very exotic and unexpected behavior of the constrained MLE estimate for such model because conditions \eqref{eq:thm:asympt-distr-main:strongly-consist-estim-positivity}-\eqref{eq:thm:asympt-distr-main:strongly-consist-estim-i0} are trivially satisfied, for example, if $A$ is diagonal. Finally, the asymptotic structure of Bayesian posterior from \citet{bochkina2014} gives a strong intuition that the conjecture above should hold: the asymptotic posterior projected on $\mathcal{V}$ has exponential distribution $\mathrm{Exp}(-c_A  t \Pi_{\mathcal{V}}\lambda)$ and $\sqrt{t}\Pi_{\mathcal{U}}\lambda$ is normal with mean equals $\sqrt{t}(A_{I_1(\Lambda^*)}^TD_{\Lambda^*}^{-1}A_{I_1(\Lambda^*)})^{-1}A_{I_1(\Lambda^*)}^TD_{\Lambda^*}^{-1}[Y^t/t]$ being also restricted to positivity cone (hence, half-Gaussian), therefore the corresponding MAP estimate asymptotically fits conditions \eqref{eq:thm:asympt-distr-main:strongly-consist-estim-positivity}-\eqref{eq:thm:asympt-distr-main:strongly-consist-estim-i0} being atom at zero for the exponential part and mean for the Gaussian one (up to higher order terms).
	Formal investigation of Conjecture~\ref{conj:asympt-distr-mle-pen-estimator} and of possible $\widehat{\lambda}_{sc}^t$'s are outside of the scope of this work and will be given in future. To our knowledge this is a completely new open problem and such result is necessary for further investigation of bootstrap procedures for the model of ET. 
	

	\subsection{Misspecification in design and identifiability}
	\label{subsect:theory:misspecification}
	Assumption~\ref{assump:theory:consistency:well-spec} in 
	Subsection~\ref{subsect:theory:consistency} reflects our belief that model \eqref{eq:poisson-model-pet} is correct. At the same time, for any practitioner in ET it is known that such   model is by far approximate: the tracer inside the human body surely does not respect locally constant behavior in each pixel on which our discretized model is based, also, in practice, matrix $A$ is known only approximately, with non-negligible errors, since it contains patient's attenuation map which is reconstructed via a separate MRI or CT scan; see e.g.,~\citet{stute2013practical}. There also are many other practical issues  which are not included in \eqref{eq:poisson-model-pet} such as non-stationarity of the process due to kinetics for the tracer, scattered photons, electronic noise in detectors, errors from multiple events etc.; see e.g., \citet{levin1995mcmc}, \citet{rahmim2009fourdim}. 
	
	Assuming temporal stationarity of the process
	we consider the following scenario for~ET:
	\begin{align}\label{eq:misspeicifcation:assump-data-1}
	&Y^t \sim P^t, \, Y^t \in (\mathbb{N}_0)^d,\\
	\label{eq:misspeicifcation:assump-data-2}
	&\mathbb{E}_{P^t}[Y^t] = \mathrm{var}_{P^t}[Y] = t \Lambda^* 
	\text{ for some }
	\Lambda^* = (\Lambda^*_1, \dots, \Lambda^*_d)\in \R^d_+.
	\end{align}
	Formulas \eqref{eq:misspeicifcation:assump-data-1}, \eqref{eq:misspeicifcation:assump-data-2} reflect our belief that $Y^t$ has Poisson-type behavior at least for its two first moments which is not far from truth in practice \citet{sitek2015limitations}.
	Most importantly, we do not assume that $\Lambda^* \in R_+(A)$. 
	
	The main question now is the identifiability of $\lambda$ which translated via  \eqref{eq:prelim:poiss-log-likelihood-model}, \eqref{eq:misspeicifcation:assump-data-1},   \eqref{eq:misspeicifcation:assump-data-2} to the problem of uniqueness in the following minimization problem:
	\begin{align}\label{eq:misspecification:kl-projector-prob}
	\begin{split}
	\lambda_{A,*}(P^t) = \arg\min_{\lambda \succeq 0} \mathcal{KL}(P^t, P_{A, \lambda}^t) 
	=\arg\min_{\lambda \succeq 0} L(\lambda \mid \Lambda^*, A, 1),
	\end{split}
	\end{align}
	where $P_{A, \lambda}^t$ is defined in \eqref{eq:prelim:poiss-prob-model}. 
	
	\begin{theorem}\label{thm:misspecification:main-example}
		There exist $\Lambda^* = (\Lambda_1^*, \dots, \Lambda_d^{*})\in \R^{d}_+, \, \Lambda^* \neq 0$, $A \in \mathrm{Mat}(d, p)$ which has only nonnegative entries, it is stochastic column-wise and injective such that
		solutions of the optimization problem \eqref{eq:misspecification:kl-projector-prob} constitute a non-empty affine subset of positive dimension of the $(p-1)$-simplex 
		$
		\Delta_p(\Lambda^*) = \left\{\lambda\in \R^p_+ : \sum_{j=1}^{p}\lambda_j = \sum_{i=1}^{d} \Lambda_i^*\right\}.
		$
	\end{theorem}
	
	\begin{proof}
		We construct $\Lambda^*$ and $A$ for $p=4, \, d=6$. 
		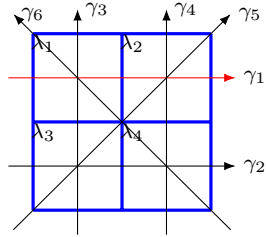
\begin{figure}[H]
			\centering
				\begin{tikzpicture}[scale=0.13]
				
				\begin{scope}
				\draw[very thick, step=9.0, blue] (0,0) grid (18,18);
				\node at (1, 8.0) {$\lambda_3$};
				\node at (1, 17.0) {$\lambda_1$};
				\node at (10, 8.0) {$\lambda_4$};
				\node at (10, 17.0) {$\lambda_2$};
				
				\draw [-{Latex[length=1.5mm]}] (-2.5, 4.5) -- (20.5, 4.5) node[right] {$\gamma_2$};
				\draw [-{Latex[length=1.5mm]}, red] (-2.5, 13.5) -- (20.5, 13.5) node[right, black] {$\gamma_1$};
				
				\draw [-{Latex[length=1.5mm]}] (4.5, -2.5) -- (4.5, 20.5) node[right] {$\gamma_3$};
				\draw [-{Latex[length=1.5mm]}] (13.5, -2.5) -- (13.5, 20.5) node[right] {$\gamma_4$};
				
				\draw [-{Latex[length=1.5mm]}] (-2, -2) -- (20, 20) node[right] {$\gamma_5$};
				\draw [-{Latex[length=1.5mm]}] (20, -2) -- (-2, 20) node[right] {$\gamma_6$};
				
				\end{scope}	
				\end{tikzpicture}
			\caption{$\mathcal{I}$}
			\label{fig:misspecification:image}
		\end{figure}
		Let $\mathcal{I}$ be the image consisting of four square pixels each with side length equal to $1$ as shown in Figure~\ref{fig:misspecification:image}, i.e.,  $\lambda = (\lambda_1, \dots, \lambda_4)\in \R^4_+$. Let $\Gamma = \{\gamma_1, \dots, \gamma_6\}$ be the family of rays as it is shown in the Figure~\ref{fig:misspecification:image} and matrix $A'$ corresponds to the the classical Radon transform on $\mathcal{I}$, i.e., $a'_{ij}$ is the length of intersection of ray $\gamma_i\in \Gamma$ with pixel $j$
		\begin{align*}
		A' &= \begin{pmatrix}
		1 & 1 & 0 & 0\\
		0 & 0 & 1 & 1\\
		1 & 0 & 1 & 0\\
		0 & 1 & 0 & 1\\
		0 & \sqrt{2} & \sqrt{2} & 0\\
		\sqrt{2} & 0 & 0 & \sqrt{2}
		\end{pmatrix}, \, \det(A'^TA') = 128 \neq 0.
		\end{align*}
		Let $A$ be a normalization of $A'$ with respect to columns such that $A$ is stochastic column-wise , i.e., $a_{ij} = a'_{ij} / (\sum_i a'_{ij})$. Such normalization obviously does not break the injectivity of~$A'$. Let $\Lambda^* = (1, 0, 0, 0, 0, 0)$. Then, the formula in \eqref{eq:misspecification:kl-projector-prob} has the following form
		\begin{equation}\label{eq:thm:misspecification:proof:opt-problem}
		\lambda_{A, *} = \arg\min_{\lambda \succeq 0} -\log\left(\frac{\lambda_1 + \lambda_2}{2 + \sqrt{2}}\right) + \lambda_1 + \lambda_2 + \lambda_3 + \lambda_4.
		\end{equation}
		Note that in \eqref{eq:thm:misspecification:proof:opt-problem} we have used the fact that $\sum_{i}a_{ij}=1$ for all $j\in \{1, \dots, 4\}$. It is obvious that the set of minimizers in \eqref{eq:thm:misspecification:proof:opt-problem} is an affine set of the following form:
		\begin{equation}\label{eq:thm:misspecification:proof:minimizers-family}
		\lambda_{A,*3} = \lambda_{A,*4} = 0, \, \lambda_{A,*1} + \lambda_{A,*2} = 1
		\end{equation}
		which gives the desired non-uniqueness. Theorem is proved.	
		\end{proof}
		

	At the same time,  Theorem~\ref{thm:asympt-distr:well-spec:identifiability-cond} provides identifiability and stability (via strong local convexity) for  $\lambda_{A,*}$ under the non-expansiveness condition and injectivity of $A$. The latter assumption can be also relaxed by simply restating the claim of Theorem~\ref{thm:asympt-distr:well-spec:identifiability-cond} with analogs \eqref{eq:asympt-distr:well-spec:identifiability-cond:approx}-\eqref{eq:thm:asympt-distr:identif-cond:strong-convexity} to hold but only in the subspace~$\mathrm{Span}(A^T)$.

	\section{Discussion}
	\label{sect:conclusion}

	To build the nonparametric posterior learning for the model of ET we have  used conjugacy between Poisson and Gamma processes which is analogous to the one in  \citet{lyddon2018npl}, \citet{fong2019scalable} between Dirichlet and Multinomial processes. 
	This explains why our main calibration parameter~$\rho$ ($\theta^t = t\rho $; see Remark~\ref{rem:parameter-theta-interpretation}) has physical interpretation as amount of pseudo-data (pseudo-photons for ET) generated from the posterior process. Possible future improvement of the method is to relax the independence of increments of the Gamma process in the prior and consider processes with correlations, for example, Gamma-weighted Polyà tree  priors. Such correlations can be used to smooth sinogram $Y^t$ (i.e., to project it approximately on the stable part of $\mathrm{Span}(A^T)$) using the MRI-based model and, in addition, remove completely regularizer $\varphi$ from the model. Note that in Algorithm~\ref{alg:npl-posterior-sampling:mri:binned} regularization of high frequencies is achieved via control of $\varphi$ and only low frequencies are regularized by $\mathcal{M}$. Our preliminary results show that new approach improves the resolution while retaining the interpretability of the calibration parameters. This is definitely a next goal for future work.
	
	From the theoretical side a very needed step is to demonstrate Conjecture~\ref{conj:asympt-distr-mle-pen-estimator}, which is also necessary for theoretical analysis of more complicated prior models discussed above. Work in this direction may also target studies of the first order asymptotics of the posterior (i.e., Edgeworth's expansions) which will be given elsewhere.
	
	Our numerical tests on synthetic data in the Supplementary Material show good coverage of the true signal even for large values of $\rho$ (empirical rule of thumb says that $\rho=1$ is satisfactory), so new tests on real patient data are needed in future.

	\section*{Supplementary material}
	Supplementary material includes the proof of Lemma~\ref{lem:consistency:lem-kernel-continuity}, numerical experiments for the Gibbs sampler in Section~\ref{sect:motivating-example-mcmc} and for Algorithm~\ref{alg:npl-posterior-sampling:mri:binned} (provided with links to the source code), proofs of all theoretical results in Section~\ref{sect:new-thero-results},  a remark on the intuition behind the non-expansiveness condition (Assumption~\ref{assump:theory:distribution:non-expansive}) and an additional remark on the choice of centering term in Theorem~\ref{thm:asympt-distr-main}.
	
	\section*{Acknowledgments}
	We are grateful to Zacharie Naulet from Université d'Orsay for many valuable comments on statistical side of the paper. We are also grateful to our colleagues from Service Hospitalier Frédéric Joliot (SHFJ) -- Marina Filipovi\'{c}, Claude Comtat and Simon Stute for many practical insights on the topic of PET-MRI reconstructions.
	This work is partly supported by the ‘MMIPROB’ project funded by ITMO Cancer (France).

	\numberwithin{equation}{section} 
	\numberwithin{theorem}{section} 
	\numberwithin{figure}{section} 
	\numberwithin{definition}{section} 
	\numberwithin{remark}{section} 
%
%
%
	
	\begin{appendix}
		\section{Construction of the common probability space.}
		\label{app:prob-space}
		Let $(\Omega', \mathcal{F}', P')$ be the probability space on which the stationary spatio-temporel Poisson point process $Z^t$ is defined ($Z^t$ has values in $Z \times (0, +\infty)$;  recall that $Z$ is the space of LORs). Sinogram data $Y^t$ is obtained from binning $Z^t$ to detector elements (see Section~\ref{subsect:new-algo:binning}), therefore process $Y^t$ is a well-defined random variable on $(\Omega', \mathcal{F}', P')$.
		Measure-theoretic construction of $Z^t$ and $(\Omega', \mathcal{F}', P')$ can be found, for example, in \citet{daley2007pointproc2}, Section 9.2, Example 9.2(b). 
		
		Algorithms~\ref{alg:wbb-pet-bootstrap:mri:posterior-mixing-param},~\ref{alg:npl-posterior-sampling:mri:binned} rely on perturbed intensities $\widetilde{\Lambda}^t_{\mathcal{M}}$ and $\widetilde{\Lambda}_{b}^t$ for which we show that they can be expressed as functions of random weighting of the list-mode data $$G^t = \{\delta_{(k,i)} : (k,i) \text{ -- $k^{\text{th}}$ photon was detected at detector $i$}\},$$ where $\delta_{(k,i)} \in \{0,1\}$. Indeed, from step~1 in Algorithm~\ref{alg:wbb-pet-bootstrap:mri:posterior-mixing-param} we can see that  $\widetilde{\Lambda}^t_{\mathcal{M}}$ is a function of $\widetilde{\Lambda}^t$ for which the following representation holds
		\begin{align}\label{eq:appendix:common-prob:int-lit-repr}
		&\widetilde{\Lambda}^t_i = t^{-1}\sum_{k=1}^{N^t}\delta_{(k, i)}\widetilde{w}_{k}, \, i\in \{1, \dots, d\},\\
		&\{\widetilde{w}_k\}_{k=1}^{N^t} \iid \Gamma(1,1),
		\end{align}
		where $N^t$ is the total number of photons.

		For $\widetilde{\Lambda}_b^t$ in step~2 of Algorithm~\ref{alg:npl-posterior-sampling:mri:binned} we have the following representation:
		\begin{align}
		\label{eq:appendix:common-prob:int-lbit-repr}
		&\widetilde{\Lambda}^t_{b,i} = (\theta^t + t)^{-1}\left(\sum_{k=1}^{N^t}\delta_{(k, i)}w_{k} + w_p\theta^t \Lambda_{\mathcal{M},i}^t\right), \, i\in \{1, \dots, d\},\\
		\label{eq:appendix:common-prob:weights-lbit-repr}
		&\{w_k\}_{k=1}^{N^t}, w_p \iid \Gamma(1,1).
		\end{align}
		\par From formulas \eqref{eq:appendix:common-prob:int-lit-repr}-\eqref{eq:appendix:common-prob:weights-lbit-repr} one can see that perturbations $\widetilde{\Lambda}^t_{\mathcal{M}}$ and $\widetilde{\Lambda}_{b}^t$ depend on data $Y^t$ and on infinite family of random mutually independent weights $(\{(w_k, \tilde{w}_k)\}_{k=1}^{\infty}, w_p)$ which are also independent of $Y^t$. Therefore, the common probability space can be defined as follows:
		\begin{equation}
		(\Omega', \mathcal{F}', P') = (\Omega' \times \Omega_w \times \Omega_{\widetilde{w}} \times \Omega_{w_p}, \mathcal{F}' \times \mathcal{F}_{w} \times \mathcal{F}_{\widetilde{w}} \times \mathcal{F}_{w_p}, P' \times P_{w} \times P_{\widetilde{w}} \times P_{w_p}),
		\end{equation}
		where $(\Omega_{w}, \mathcal{F}_w, P_w)$, $(\Omega_{\widetilde{w}}, \mathcal{F}_{\widetilde{w}}, P_{\widetilde{w}})$, $(\Omega_{w_p}, \mathcal{F}_{w_p}, P_{w_p})$ are the probability spaces for infinite sequences of i.i.d r.v.s $\{w_k\}_{k=1}^{\infty}$, $\{\widetilde{w}_k\}_{k=1}^{\infty}$, $w_{k} \sim \Gamma(1,1)$, $\widetilde{w}_k\sim \Gamma(1,1)$ and for $w_p\sim \Gamma(1,1)$, respectively. This construction originates to \citet{newton1994wbb}; similar ones have been recently used in \citet{ng2020random}.

		\section{Limit theorems for stationary Poisson processes.}
		\label{app:limit-thms}
		Let 
		\begin{equation}\label{eq:appendix:poisson-proc-def}
		Y^t \sim \mathrm{Po}(\Lambda t), \, \Lambda > 0, \, t\in [0, +\infty).
		\end{equation}
		
		The following result is a composition of theorems~9.3,~4.1 and 7.5 (pp. 306, 350, 417, respectively) from \citet{gut2013probability}.
		\begin{theorem}\label{thm:appendix:limits-poisson}
			Let $\{Y^t\}, \, t\in (0, +\infty)$ be the Poisson process defined in~\eqref{eq:appendix:poisson-proc-def}. Then,
			\begin{itemize}
				\item[i)] 
				\begin{equation}\label{eq:appendix:stlln}
				\dfrac{Y^t}{t} \xrightarrow{a.s.} \Lambda \text{ as }t\rightarrow +\infty.
				\end{equation}
				\item[ii)]
				\begin{equation}\label{eq:appendix:clt}
				\dfrac{Y^t - \Lambda t}{\sqrt{\Lambda t}} \xrightarrow{d} \mathcal{N}(0,1) \text{ as } t\rightarrow +\infty.
				\end{equation}
				\item[iii)]
				\begin{equation}\label{eq:appendix:lil}
				\liminf_{t\rightarrow +\infty} (\limsup_{t\rightarrow +\infty})\dfrac{Y^t-\Lambda t}{\sqrt{\Lambda t\log\log t}} = -\sqrt{2} \, (\sqrt{2}) \text{ a.s.},
				\end{equation}
			\end{itemize}
			where $\xrightarrow{a.s.}$, $\xrightarrow{d}$ denote the convergence almost surely and in distribution, respectively, a.s. denotes that  statement holds  for almost any trajectory $Y^t$, $t\in (0, +\infty)$.
		\end{theorem}

		\section{Binned NPL for emission tomographies with MRI data}
		\label{app:new-algo:wbb-algorithm-with-mri}
		
		First, we construct $P_{\mathcal{M}}$, then we proceed with construction of $P_{\mathcal{M}}(\widetilde{\Lambda}^t_\mathcal{M} \mid Y^t, t)$.
		
		\begin{enumerate}
			\item Recall that $\mathcal{M} = \{M_1, \dots, M_r\}$ are the segmented MRI images (see also Section~\ref{sect:prelim}), $p_k$ denotes the number of  disjoint segments in image $M_k\in \mathcal{M}$. Each segment is a subset of $\{1, \dots, p\}$, collection of segments in image $M_k$ is denoted by $S(M_k)\subset 2^{p}$.
			\item For each image $k\in \{1, \dots, r\}$ and segment $s\in S(M_k)$, we generate $\lambda_s^k\sim \Gamma(1, \infty)$ (uniform (improper) distribution on $\R_+$).
			\item  Compute random projections 
			\begin{equation}\label{eq:wbb-mri:npl-presentation:np-prior-data}
			\Lambda_{\mathcal{M}, i} = \sum_{k=1}^r\sum_{s=1}^{p_k}a_{is}^k\lambda_s^k,
			\text{ for each } i \in \{1, \dots, d\}.
			\end{equation} 
			where 
			\begin{equation}\label{eq:wbb-mri:npl-presentation:reduced-design-elem}
			a_{is}^{k} = \sum_{j=1}^{p}a_{ij}\mathds{1}\{\text{pixel }j \text{ belongs to segment }s\in S(M_k)\}, \, k\in \{1, \dots, r\}.
			\end{equation}
		\end{enumerate}
		
		Note that $\Lambda_{\mathcal{M}, i}$ in \eqref{eq:wbb-mri:npl-presentation:np-prior-data} is defined through the sum of projections over all images in~$\mathcal{M}$. This can be seen as concatenating $r$ models with segmentations :
		\begin{align}\label{eq:wbb-mri:npl-presentation:concat-design-def}
		&A_{\mathcal{M}} = (A_1, \dots, A_r)\in \mathrm{Mat}\left( d, p_{\mathcal{M}}\right), \, A_k = (a_{ij}^k)\in \mathrm{Mat}(d, p_k), \, p_{\mathcal{M}} = \sum_{k=1}^r p_k, \\
		\label{eq:wbb-mri:npl-presentation:concat-parameter-def}
		&\lambda_{\mathcal{M}} = (\lambda^1_1, \dots, \lambda^1_{p_1}, \dots, \lambda^r_{1}, \dots, \lambda^r_{p_r}),
		\end{align}
		Using notations from \eqref{eq:wbb-mri:npl-presentation:concat-design-def}, \eqref{eq:wbb-mri:npl-presentation:concat-parameter-def}, formula \eqref{eq:wbb-mri:npl-presentation:np-prior-data} can be rewritten as follows:
		\begin{align}\label{eq:wbb-mri:npl-presentation:concat-project-redef}
		\Lambda_{\mathcal{M}} = A_{\mathcal{M}}\lambda_{\mathcal{M}}, \, 
		\Lambda_{\mathcal{M}} = (\Lambda_{\mathcal{M}, 1}, \dots, \Lambda_{\mathcal{M}, d}).
		\end{align}
		For design matrix $A_\mathcal{M}$ we assume that it is injective and well-conditioned, that is 
		\begin{align}\label{eq:wbb-mri:npl-presentation:concat-inj-cond}
		\ker A_{\mathcal{M}} = \{0\}, \, \mathrm{cond}(A_{\mathcal{M}}) < c_{\mathcal{M}},
		\end{align}
		where $c_\mathcal{M}$ is some moderate constant.
		The latter assumption reflects the idea that images in $\mathcal{M}$ consist of low number of large segments. In practice, condition \eqref{eq:wbb-mri:npl-presentation:concat-inj-cond} can be checked via the singular values of $A_{\mathcal{M}}^TA_{\mathcal{M}}$ which, in turn, can be computed due to apriori moderate size of $A_{\mathcal{M}}$.
		In~principle, due to moderate size of $A_{\mathcal{M}}$ and good conditioning it is possible to use  MCMC-approach to sample from $P_{\mathcal{M}}(\widetilde{\Lambda}^t_{\mathcal{M}} \mid Y^t, t)$, however, in order to keep the overall implementation as simple as possible we turn to WLB from \cite{newton1994wbb} for approximate posterior sampling.


		\section{Remark on recent bootstrap algorithms for ET}
		
		A very recent and similar to ours sampling algorithm was proposed in \citet{filipovic2021reconstruction} provided with a very extensive experiment both on synthetic and real PET-MRI data.
		The algorithm there is also of boostrap-type, based on optimization of a randomized functional (the KL-distance) and in fact, it coincides up to minor details with Algorithm~\ref{alg:npl-posterior-sampling:mri:binned} for $\theta^t \equiv 0$ (i.e., without MRI). Instead, data $\mathcal{M}$ are used there to construct very special penalty $\varphi(\lambda) = \varphi_{\mathcal{M}}(\lambda)$ of Bowsher type (see Subsection~\ref{subsect:multimodal}). 
		This penalty satisfes the assumptions in \eqref{eq:prelim:penalty-cond-convex}, \eqref{eq:prelim:penalty-cond-strict-conv}, so  our theorems~\ref{thm:wbb-generic-consistency},~\ref{thm:asympt-distr-main}  serve as a theoretical foundation also for the algorithms presented there.
		A nice practical feature of Algorithm~\ref{alg:npl-posterior-sampling:mri:binned} is that $\theta^t$ has clear physical interpretation of the effect of MRI data on samples (see Remark~\ref{rem:parameter-theta-interpretation}), whereas large number of parameters in Bowsher-type penalties have no such easy interpretations making the problem of their calibration cumbersome for practice.
		
		The aforementioned minor differences between algorithms consist in the way data $Y^t$ (in \citet{filipovic2021reconstruction}) or intensities $\Lambda_i$ (in our work) are stochastically perturbed. From the first look this seems to be only a technical question, however, we think that it is not. From the above derivation of Algorithm~\ref{alg:npl-posterior-sampling:mri:binned} one can see that initially uncertainty propagates via the KL-projection in \eqref{eq:new-algo:npl-poisson:opt-kullback-projection} and not concerning at all the problem of limited data. Moreover, we retrieve version of WLB of  \citet{newton1994wbb} adapted for ET as a particular case of Algorithm~\ref{alg:npl-posterior-sampling:mri:binned} when choosing the scale parameter $\theta^t = 0$ in the nonparametric prior in \eqref{eq:new-algo:npl-poisson:mgp-intro} (each photon corresponds to multiplicative perturbation of the data term by $w\sim \Gamma(1,1)$).
		This is fully coherent with the derivation of NPL in \citet{lyddon2018npl} and nonparametric posterior bootstrap with MDP-prior in \citet{fong2019scalable}, where the classical WLB algorithm from \citet{newton1994wbb} is retrieved back as a particular case when choosing the concentration parameter $\alpha = 0$ ($c=0$ in \citet{fong2019scalable}) in the nonparametric Dirichlet process prior. On the other hand, the derivation in \citet{filipovic2021reconstruction} strongly relies on model with finite data and it is claimed that the resulting algorithm is also a version of WLB from \citet{newton1994wbb}, however, in this case for us is not clear which randomized functional  stands behind this procedure.

		\section{Practical interpreptation of slow mixing in MCMC}
		\label{app:mixing-example}
		In practice produced samples by the Markov chain are used to compute credible intervals for weighted means in certain subregions of reconstructed images. Let $h\in \R^p$ be a weighting mask which corresponds to subregion $\Omega\subset \{1, \dots, p\}$. For example, if  $h_j = \frac{1}{\# \Omega}$ for pixel $j\in \Omega$ and $h_j = 0$ otherwise, then $h^T\lambda$ gives the average tracer concentration in subregion $\Omega$. 
		Let $N$ be the number of generated samples which we denote by $\{\lambda_k^t\}_{k=1}^N$. Then, the posterior mean of  $h^T\lambda$ can be approximated by the following expression: 
		\begin{align}
		\widehat{f}_{h,N}^t = \dfrac{1}{N}\sum_{k=1}^{N}h^T\lambda_k^t,
		\end{align}
		
		The variance of estimator $\widehat{f}^t_{h,N}$ can be approximated as follows:
		\begin{align}
		\begin{split}
		\mathrm{var}(\widehat{f}^t_{h,N} \mid Y^t, t) &= \dfrac{1}{N^2}
		\sum\limits_{k=1}^{N}\sum\limits_{s=1}^{N}\mathrm{cov}(h(\lambda_k^t), h(\lambda_s^t) \mid Y^t, t)\\
		& \asymp \dfrac{\sigma^2}{N}(1 + 2\sum\limits_{k=1}^{\infty}\rho^t_k(h)),
		\end{split}
		\end{align}
		where
		\begin{align}
		\rho_k^t(h) &= \mathrm{corr}(h^T\lambda_1^t, h^T\lambda_{k+1}^t \mid Y^t, t), 
		\, \sigma^2 = \mathrm{var}(h^T\lambda).
		\end{align}
		
		In \citet{liu1994covariance} it was shown, in particular, that $\rho^t_k(h)\asymp (\gamma^t(h))^k$, so from this and the above formula we get the following expression for the variance of $\widehat{f}^t_{h,N}$ (modulo a universal multiplicative factor): 
		\begin{align}\label{eq:example:empirical-variance-vs-posterior}
		\mathrm{var}(\widehat{f}_{h,N} \mid Y^t, t) \asymp \dfrac{\sigma^2}{N}
		\left(\dfrac{1 + \gamma^t(h)}{1-\gamma^t(h)}\right) 
		\approx \dfrac{\sigma^2}{N} \left(
		\dfrac{1 + \gamma(h)}{1-\gamma(h)}
		\right), 
		\end{align}
		where $\gamma^t(h)$, $\gamma(h)$ are defined in \eqref{eq:example:corr-value},  \eqref{eq:example:asymptotic-fraction-missinfo-def}, respectively.
		The rule of thumb in \citet{green1991global} tells to choose $N$ such that empirical variance of $\widehat{f}_{h,N}$ does not exceed $1\%$ of $\sigma^2$, which is then translated to the following rule:
		\begin{equation}
		\dfrac{\mathrm{var}(\widehat{f}_{h,N} \mid Y^t, t)}{\sigma^2} < 0.01
		\Rightarrow N \gtrsim 100 \times \left(\dfrac{1 + \gamma(h)}{1-\gamma(h)}\right) \rightarrow +\infty
		\text{ for } h=h_m, \, m\gg 1.
		\end{equation}
		Therefore, to estimate reliably the average signal using mask $h\in \R^p$, one needs almost infinite number of samples if $h$ contains  a high-frequency component in terms of basis~$\{h_k\}_{k=1}^p$.
		
		\section{Numerical experiment for the Gibbs-type sampler in ET}
		\label{app:numerical-mixing-example}
		
		\begin{minipage}{0.51\textwidth}
			\begin{align*}
			\begin{split}
			&\lambda_* \text{ -- image of size } 64 \times 64 \text{ (see Figure~\ref{fig:example:example-true-point})},\\
			&A \text{ -- Radon transform matrix of size 4096 $\times$ 4096},\\
			&\text{prior } \pi_j = \Gamma(1,1),\\
			&\text{time } t=10^2, 10^{10} \text{ ($\sim$ photons per LOR)},\\
			&\text{initial point: }\lambda_*, \\
			&\text{burn-in samples: 1000}, \\
			&\text{number of samples for the output: 2000}
			\end{split}
			\end{align*}
		\end{minipage}
		\begin{minipage}[center]{0.39\textwidth}
			\begin{figure}[H]
				\includegraphics[width=57mm, height=45mm, trim= 10 1 1 1 ]{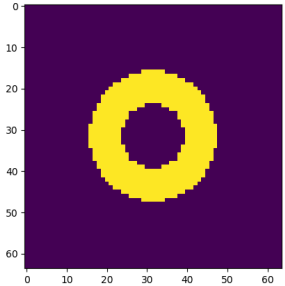}
				\caption{true distribution $\lambda_*$}
				\label{fig:example:example-true-point}
			\end{figure}
		\end{minipage}\\
		According to \eqref{eq:example:positivity-true-point-assump} we choose $\lambda_* \succ 0$, where $\lambda_*|_{\text{circle}} = 2$, $\lambda_*|_{\text{background}} = 1$ (see Figure~\ref{fig:example:example-true-point}), where radius of the inner circle $r_{\text{in}} = 0.25$ and of the outer $r_{\text{out}} = 0.5$, the image corresponds to domain $[-1, 1]^2$. Design $A$ is constructed using our implementation of Siddon's algorithm  (\cite{siddon1985fastcalculation}) for parallel beam geometry with 64 projections and 64 parallel lines per projection. Source code in Python of the experiment can be found at https://gitlab.com/eric.barat/npl-pet.
		
		\section{GEM-type algorithm derivation}
		\label{sect:gem}
		We mainly follow \citet{wang15ot} for the derivation of the minimization algorithm based on optimization transfer. Our aim is to build a majoring surrogate of $L_p(\lambda\mid\,\widetilde \Lambda_b^t, A,1, \beta^t / t)$. Using the fact that $L_p(\lambda\mid\,\widetilde \Lambda_b^t, A,1,\beta^t/t)=L(\lambda\mid\,\widetilde \Lambda_b^t, A,1)+\dfrac{\beta^t}{t} \varphi(\lambda)$, we proceed by finding a surrogate for each of both terms in the right hand-side.
		
		\subsection{GEM-type algorithm}
		\label{subsect:gem}
		The attractiveness of Algorithm~\ref{alg:npl-posterior-sampling:mri:binned} relies on having an efficient procedure for minimizing $L_p(\lambda\mid\,\widetilde{\Lambda}_b^t, A, 1, \beta^t/t)$ and $L(\lambda \mid \widetilde{\Lambda}^{t}, A_{\mathcal{M}}, 1)$. For integer-valued data $Y^t\in \mathbb{N}_0^d$ the $L_p(\lambda \mid Y^t, t)$ coincides with the penalized negative log-likelihood for Poisson-type sample and in this situation, provided penalty $\varphi(\lambda)$ satisfies elementary conditions (convex, $C^2$ -- smooth), fast monotonic GEM algorithms  \citet{fessler1995sagepet}, \citet{wang15ot} can be used.
		
		In our setting intensities $\widetilde{\Lambda}^t$, $\widetilde \Lambda_b^t$ are not integer-valued anymore, hence the GEM derivation machinery must be re-verified. We claim that the same so-called ``GEM-type'' iterative algorithms can be derived outside the context of a Poisson model and missing data. First, notice that EM belongs to the class of optimization transfer algorithms \citet{lange00surrogate} also denoted as MM (Majoration Minimization). In this context, the $E$-step is interpreted as the construction of a majorizing surrogate for the objective function, $M$-step corresponds to its consequent minimization (negative log-likelihood). Using the convexity argument from  \citet{dePierro93algem} we construct the same majoring surrogate for $L(\lambda \mid\,\widetilde \Lambda_b^{t}, A, 1)$ as in \citet{fessler1995sagepet} in a completely algebraic way but now for arbitrary nonnegative term~$\widetilde{\Lambda}^t_b$. Further extension to $L_p(\lambda \mid\,\widetilde \Lambda_b^{t}, A, 1, \beta^t/t)$ is straightforward by considering a separate surrogate for $\varphi(\lambda)$. 
		
		An immediate and substantial consequence for practitioners is that all celebrated GEM algorithms for MLE and MAP reconstructions can be used in the bootstrap context by simply replacing Poisson data term by $\widetilde \Lambda_b^{t}$.

		\label{app:gem}
		\subsection{Majoring surrogate of $L(\lambda \mid\,\widetilde \Lambda_b^{t}, A, 1)$}
		In \citet{dePierro93algem} authors propose a purely algebraic derivation of the surrogate outside the context of latent variables and evidence lower bound (ELBO) computation.
		
		Let $f_i(x)\triangleq x-\widetilde \Lambda_{b,i}^{t}\log(x)$,  $\lambda_j^{(r)}\succeq 0$, $j=1,\ldots,p$, be the $r^\text{th}$ iterate of the optimization algorithm minimizing $L(\lambda \mid \widetilde{\Lambda}_b^t, A, 1)$, and denote also $\Lambda_i^{(r)} = a_i^T \lambda^{(r)}$.
		
		Consider the formula
		\begin{align*}
		L(\lambda \mid\,\widetilde \Lambda_b^{t}, A, 1)&=\sum_{i=1}^{d}f_i(\Lambda_i)\\
		&=\sum_{i=1}^{d}f_i\left(\sum_{j=1}^p a_{ij}\lambda_j\right) \\
		&=\sum_{i=1}^{d}f_i\left(\sum_{j=1}^p \left[\frac{a_{ij}\lambda_{j}^{(r)}}{\Lambda_i^{(r)} }\right]\left[\frac{\lambda_j}{\lambda_{j}^{(r)}}\Lambda_i^{(r)} \right] \right) \\
		\end{align*}
		Since $f_i$ is convex for $\widetilde \Lambda_{b,i}^{t} \ge 0$ and using the fact that $\sum_{j=1}^p \frac{a_{ij}\lambda_{j}^{(r)}}{\Lambda_i^{(r)} }=1$ together with the Jensen's inequality we obtain
		$$
		L(\lambda \mid\,\widetilde \Lambda_b^{t}, A, 1)\leq Q(\lambda,\lambda^{(r)})
		$$
		where 
		$$
		Q(\lambda,\lambda^{(r)}) = \sum_{i=1}^{d}\sum_{j=1}^p \left[\frac{a_{ij}\lambda_{j}^{(r)}}{\Lambda_i^{(r)} }\right]f_i\left(\frac{\lambda_j}{\lambda_{j}^{(r)}}\Lambda_i^{(r)} \right)
		$$
		Note also that $Q(\lambda^{(r)},\lambda^{(r)})=L(\lambda^{(r)} \mid\,\widetilde \Lambda_b^{t}, A, 1)$. Using the definition of $f_i$ we find that 
		\begin{align*}
		Q(\lambda,\lambda^{(r)}) &=\sum_{i=1}^{d}\sum_{j=1}^p\left[a_{ij}\lambda_j-\frac{a_{ij}\lambda_{j}^{(r)}}{\Lambda_i^{(r)} }\widetilde \Lambda_{b,i}^{t} \log\left(\frac{\lambda_j}{\lambda_{j}^{(r)}}\Lambda_i^{(r)} \right)\right]\\
		&= \sum_{j=1}^p A_j\left[ \lambda_j -\left(\frac{\lambda_{j}^{(r)}}{A_j} \sum_{i=1}^{d}\frac{a_{ij}\widetilde \Lambda_{b,i}^{t}}{\Lambda_i^{(r)}}\right)\log\lambda_j\right]+const.\\
		\end{align*}
		where $R$ denotes terms independent of $\lambda$.
		
		Function $Q(\lambda, \lambda^{(r)})$ can be rewritten as follows:
		\begin{equation}\label{app:surrogate_likelihood}
		Q(\lambda,\lambda^{(r)}) \triangleq \sum_{j=1}^p A_j\left( \lambda_j -\lambda_{j, L}^{(r+1)}\log\lambda_j\right)
		\end{equation}
		with
		\begin{equation}\label{eq:ljem}
		\lambda_{j}^{(r+1), L} \triangleq \frac{\lambda_{j}^{(r)}}{A_j} \sum_{i=1}^{d}\frac{a_{ij}\widetilde \Lambda_{b,i}^{t}}{\Lambda_i^{(r)} }
		\end{equation}

		\subsection{Majoring surrogate for $\varphi(\lambda)$}
		Let 
		\begin{equation*}
		\varphi(\lambda)= \sum_{j=1}^p\sum_{k\in\mathcal N_j} w_{jk} \,\psi(\lambda_j-\lambda_{k})
		\end{equation*}
		with $w_{jk}>0$, $w_{kj}=w_{jk}$ are the weights and $\mathcal N_j$ is the neighborhood of pixel $j$. 
		
		From \citet{erdogan99surrogate}, any potential function $\psi$ satisfying the conditions 
		\begin{enumerate}[i.]\label{enu:penalty_conditions}
			\item $\psi$ is symmetric.
			\item  $\psi$ is continuous and differentiable everywhere.
			\item $\psi$ is convex.
			\item $\omega_{\psi}(u)\triangleq \frac{1}{u}\frac{\mathrm d \psi(u)}{\mathrm d u}$ is non-increasing for $u\geqslant 0$. 
			\item $\lim_{u\rightarrow 0} \omega_{\psi}(u)$ is finite and positive.
		\end{enumerate}
		can be majorized by a parabolic curve.
		
		With these requirements satisfied, $\varphi(\lambda)$ is majorized by a separable quadratic penalty given below (see \citet{wang15ot} and references therein):
		
		\begin{equation*}
		\varphi(\lambda) \leq Q_\varphi(\lambda;\lambda^{(r)})
		\end{equation*} 
		where
		\begin{align}
		&Q_\varphi(\lambda;\lambda^{(r)})=\frac{1}{2}\sum_{j=1}^{p}p_{j,\varphi}^{(r+1)}(\lambda_j-\lambda_{j,\varphi}^{(r+1)})^{2}, \\
		\label{eq:pj}
		&p_{j,\varphi}^{(r+1)}=4\sum_{k\in \mathcal{N}_{j}}w_{jk}\,\omega_{\psi}(\lambda_j^{(r)}-\lambda_k^{(r)}), \\
		\label{eq:lbdjpen}
		&\lambda_{j,\varphi}^{(r+1)}=\frac{2}{p_{j,\varphi}^{(r+1)}}\sum_{k\in \mathcal{N}_{j}}w_{jk}\,\omega_{\psi}(\lambda_j^{(r)}-\lambda_k^{(r)})(\lambda_{j}^{(r)}+\lambda_{k}^{(r)}).
		\end{align}

		\subsection{Global surrogate minimization}
		At iteration $(r+1)$, solving the Karush-Kuhn-Tucker condition for minimizing the combined surrogate, we get
		$$
		\lambda^{(r+1)} = \argmin_{\lambda\succeq 0} \,Q_L(\lambda,\lambda^{(r)})+\frac{\beta^t}{t} Q_\varphi(\lambda,\lambda^{(r)})
		$$
		which gives a unique analytical solution
		\begin{equation}\label{eq:lj}
		\lambda_{j}^{(r+1)}=\frac{2\lambda_{j,L}^{(r+1)} }{\sqrt{(b_{j}^{(r+1)})^{2}+4\beta_j^{(r+1)}\lambda_{j,L}^{(r+1)}}+b_j^{(r+1)}}
		\end{equation} 
		with $\beta_j^{(r+1)}=\frac{\beta^t}{t\,A_j} p_{j,\varphi}^{(r+1)}$ and $b_j^{(r+1)}=1-\beta_j^{(n+1) } \lambda_{j,\varphi}^{(r+1)}$.\\
		
		\noindent The GEM-type algorithm is summarized in Algorithm \ref{alg:gem}.
			\begin{center}
				\begin{minipage}{0.87\textwidth}
					\begin{algorithm}[H]
						\KwData{intensities $\widetilde \Lambda_b^t$;}
						\KwIn{Initial image $\lambda^{(0)}$, number max. of iterations $R$, projector $A$, regularization parameter $\beta^t$, penalty $\varphi(\lambda)$}
						\For{ $r=0$ \KwTo $R-1$ }{
							\For{ $j=1$ \KwTo $p$ }
							{\text{compute $\lambda_{j,L}^{(r+1)}$ using formula \eqref{eq:ljem}}\;
								\text{compute $\lambda_{j,\varphi}^{(r+1)}$ using formula \eqref{eq:lbdjpen}}\;
								\text{compute $\lambda_{j}^{(r+1)}$ using formula \eqref{eq:lj}}\;
							}
						}
						\KwOut{$\lambda^{(R)}$}
						\caption{$\argmin\limits_{\lambda\succeq 0}L_p(\lambda\mid\,\widetilde \Lambda_b^t, A,1,\frac{\beta^t}{t})$ by optimization transfer}
						\label{alg:gem}
						
					\end{algorithm} 
				\end{minipage}
			\end{center}
		
		\begin{remark}
			By setting $\frac{\beta^t}{t} \rightarrow 0$ in \eqref{eq:lj}, we immediately check that $\lambda^{(r+1)} \rightarrow \lambda_L^{(r+1)}$. 
		\end{remark}
		\begin{remark}
			Parameter $\widetilde{\lambda}^{t}_{\mathcal{M}}$ in Algorithm \ref{alg:wbb-pet-bootstrap:mri:posterior-mixing-param} is easily obtained by iterating formula \eqref{eq:ljem} with projector $A_\mathcal M$ and random intensities $\widetilde{\Lambda}^{t}$
			\begin{equation}
			\lambda_{\mathcal M,s}^{(r+1)} = \frac{\lambda_{\mathcal M,s}^{(r)}}{A_s^{\mathcal M}} \sum_{i=1}^{d}\frac{a_{is}^{\mathcal M} \Lambda_{i}^{t}}{\Lambda_{\mathcal M,i}^{(r)}}
			\end{equation}
		\end{remark}
		
		\section{Numerical experiment for the NPL in ET}
		\label{subsect:numerical:design}
		
		Source code in Python of the experiment can be found at https://gitlab.com/eric.barat/npl-pet
		\subsection{Penalty $\varphi$}
		\label{subsect:numerical:log-cosh-prior}
		For our numerical tests in Section~\ref{subsect:numerical:design} we choose the well-known in PET imaging log cosh penalty \citet{green1990bayesian} coupled with $\ell_2$ convex pairwise difference penalty:
		\begin{equation}\label{eq:prelim:numerical:log-cosh-prior-def}
		\varphi(\lambda) =  \sum_{j=1}^p \sum_{j'\in \mathcal N_j} w_{jj'}\left((1-\nu)\zeta  
		\log\cosh \left(
		\frac{\lambda_j-\lambda_{j'}}{\zeta}
		\right)+\frac{\nu}{2} \left(\lambda_j-\lambda_{j'}\right)^2\right),
		\end{equation}
		where $w_{jj'}>0$, $w_{j'j}=w_{jj'}$ and $\mathcal N_j$ the neighborhood of pixel $j$. In practice, on a square image we consider a 8-adjacent pixels neighborhood with $w_{jj'}=1$ for horizontal/vertical neighbors and $w_{jj'}=\frac{\sqrt{2}}{2}$ for diagonal ones.
		
		Parameter $\zeta$ is chosen to be fixed. Penalty of form    \eqref{eq:prelim:numerical:log-cosh-prior-def} is attractive since it bridges together Gaussian prior for pairwise interactions ($\zeta \rightarrow +\infty$), and for $\nu = 0, \, \zeta = 0$, it corresponds to pairwise $\ell^1$-penalty (Laplace prior). It is easy to check that $\varphi(\lambda)$ in \eqref{eq:prelim:numerical:log-cosh-prior-def} is strictly convex except the only direction given by vector $e = \{c(1, \dots, 1), \, c\in \R\}$. From formula \eqref{eq:design-matrix-positivity-restr-3} it follows that $e\not \in \ker A$, therefore conditions \eqref{eq:prelim:penalty-cond-convex},  \eqref{eq:prelim:penalty-cond-strict-conv} are automatically satisfied. 
		
		\subsection{Design}
		
		We illustrate Algorithm~\ref{alg:npl-posterior-sampling:mri:binned} on synthetic PET data based on a realistic phantom from the BrainWeb database \citet{vunckx2011pet}. Typical activity concentrations have been assigned to annotated tissues (gray matter, white matter, skin, etc.) and we delineated a tumor lesion area, not present in the initial phantom with an uptake of 50\% compared to the gray matter activity; see Figure~\ref{fig:phantom}(a). The anatomical MRI (T1) phantom (Figure~\ref{fig:phantom}(b)) does not contain any information relative to the lesion. 
		For segmentation of MRI-images we used ddCRP \citet{blei2011distance} with a concentration parameter fixed to $10^{-5}$ leading to a few hundreds of random segments for a 2D brain slice. 
		\begin{figure}[H]
			\centering
			\begin{tabular}{cc}
				\includegraphics[height=40mm]{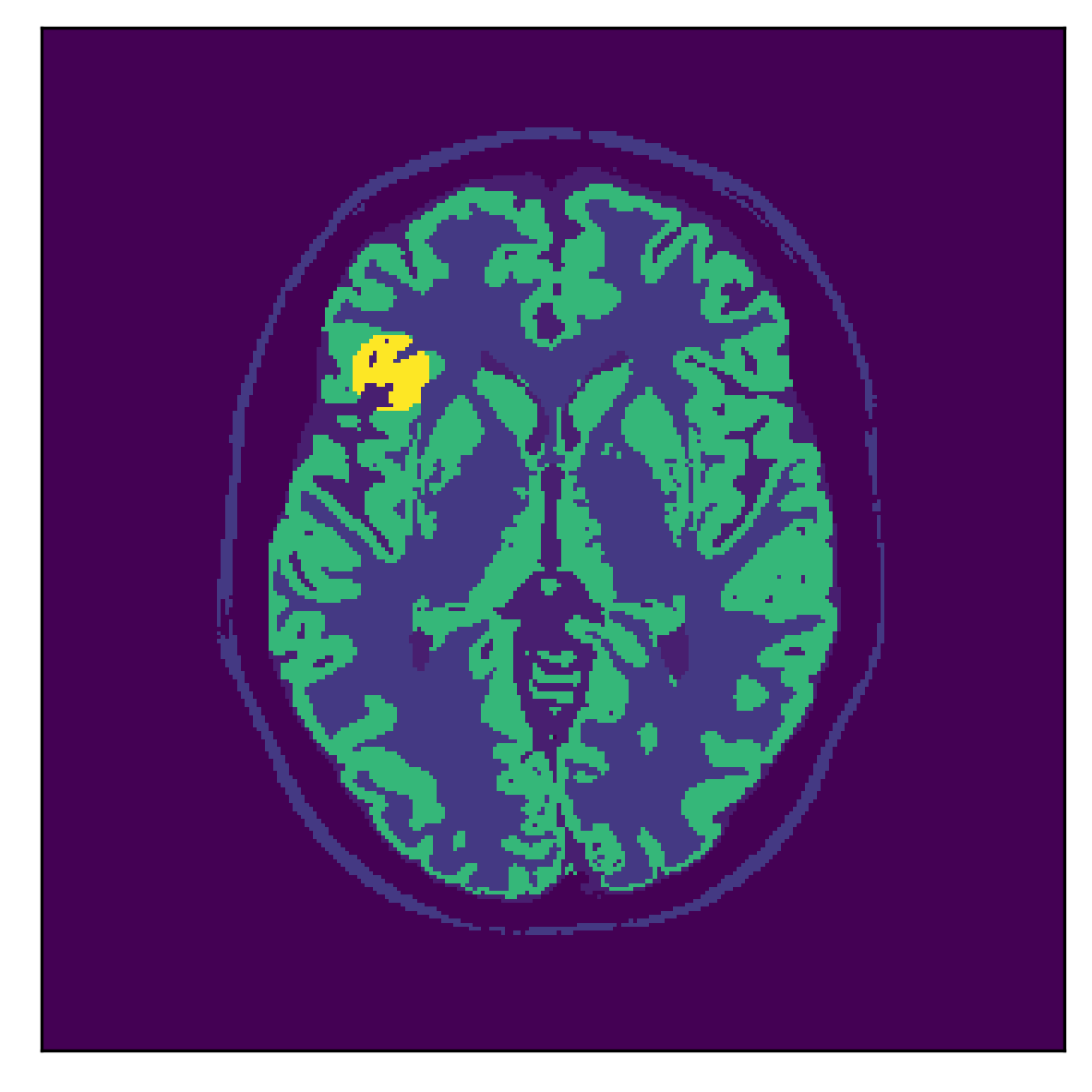} & \includegraphics[height=40mm]{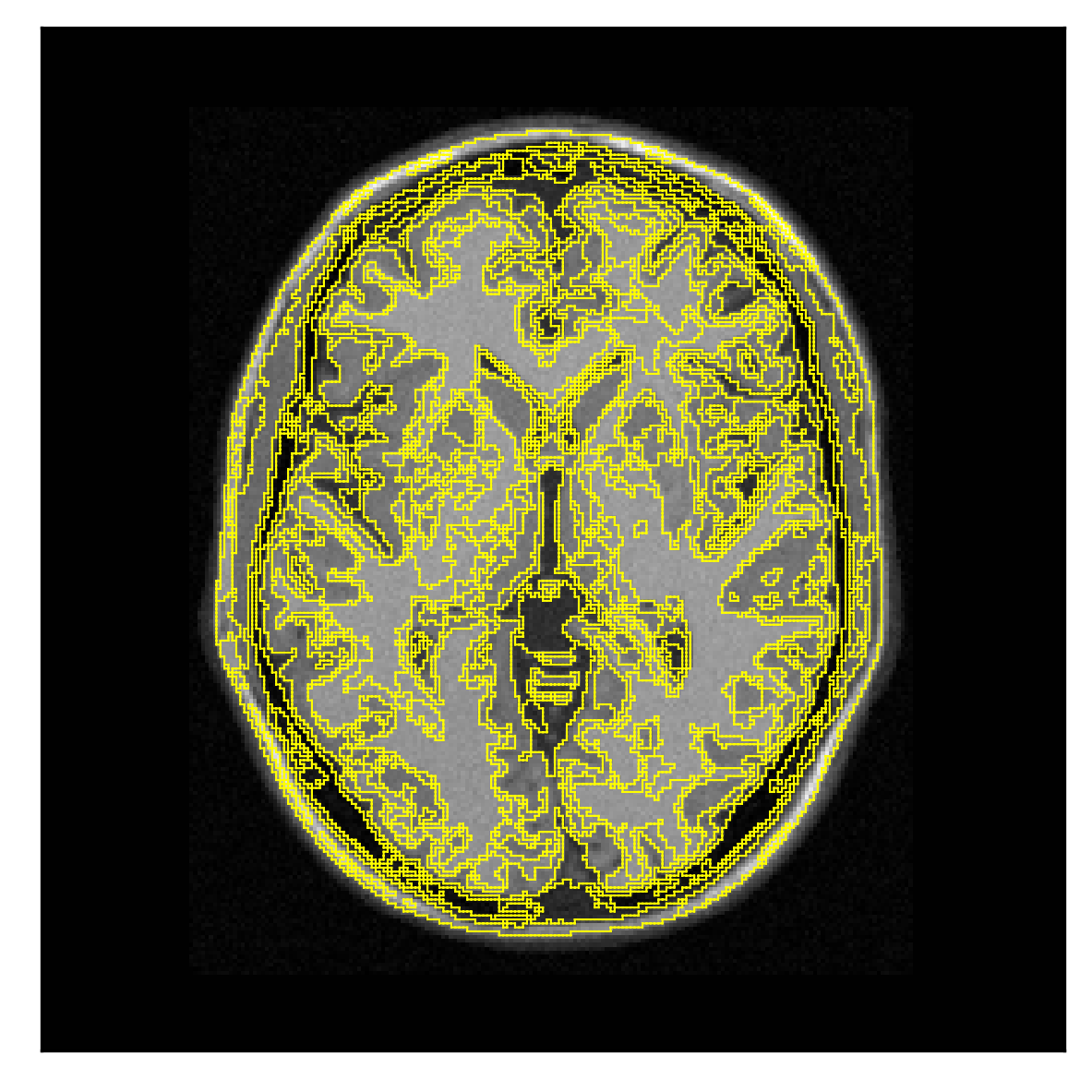}\\
				(a) $\lambda_*$ & (b) $M_1\in \mathcal{M}$
			\end{tabular}
			\caption{emission map with lesion hot spot at (a), segmented MRI at (b)}
			\label{fig:phantom}
		\end{figure}
		
		The reconstruction grid was taken $256\times 256$ pixels, i.e., $p = 2^{16}$, being  identical to the phantom's one. The  observation space consists of LORs derived from a ring of 512 detectors spaced uniformly on a circle. Design $A$ was generated using the Siddon's algorithm \citet{siddon1985fastcalculation} and $A_\mathcal M$ was  computed from $A$ and segmented image $M_1\in \mathcal M$ using formulas \eqref{eq:wbb-mri:npl-presentation:reduced-design-elem}, \eqref{eq:wbb-mri:npl-presentation:concat-design-def}.  The intensity $\lambda_*$ was set so that $\sum_{j=1}^p \lambda_{*j} = 5\cdot 10^5$ and for the experiment with mild $t$ time was set to $t_1=1$; for large $t$ (when asymptotic approximation is better) we set $t_2 = 100$. Sinograms for $t_1, t_2$ were generated via formula \eqref{eq:poisson-model-pet}. 
		
		Non-injectivity of $A$ results in the fact that $\lambda_*$ cannot be reconstructed in principle even from the noiseless sinogram $A\lambda_*$. Result of Theorem~\ref{thm:wbb-algorithm-consistency} in Subsection~\ref{subsect:theory:consistency} says that the optimal achievable reconstruction (i.e., in presence of infinite amount of data) using the KL-criterion with penalty $\varphi$ is the following one
		\begin{equation}
		\lambda_{*opt} = \lambda_* + w_{A,\lambda_*}(0),
		\end{equation}
		where $w_{A,\lambda_*}(\cdot)$ is defined in \eqref{eq:proofs:lemma-existence-minimizer-ker-a}; see Figures~\ref{fig:pseudo-phantom} (a), (b) below. Intuitively, $\ker A$ contains only high frequencies, therefore $\lambda_*$ coincides with $\lambda_{*opt}$ up to the smallest features on the image (e.g., up to boundaries).
		
		\begin{figure}[H]
			\centering
			\begin{tabular}{cc}
				\includegraphics[height=40mm]{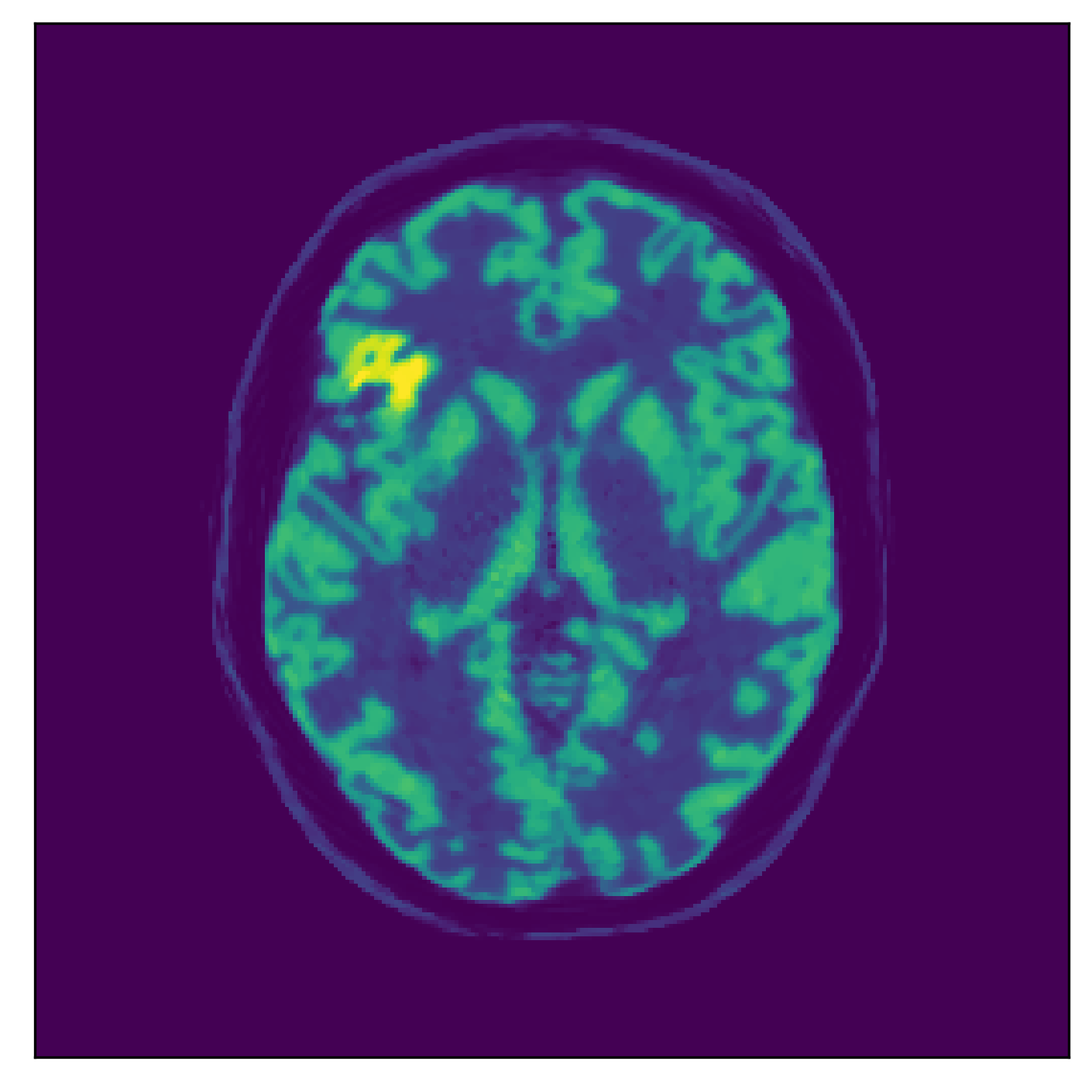} & \includegraphics[height=40mm]{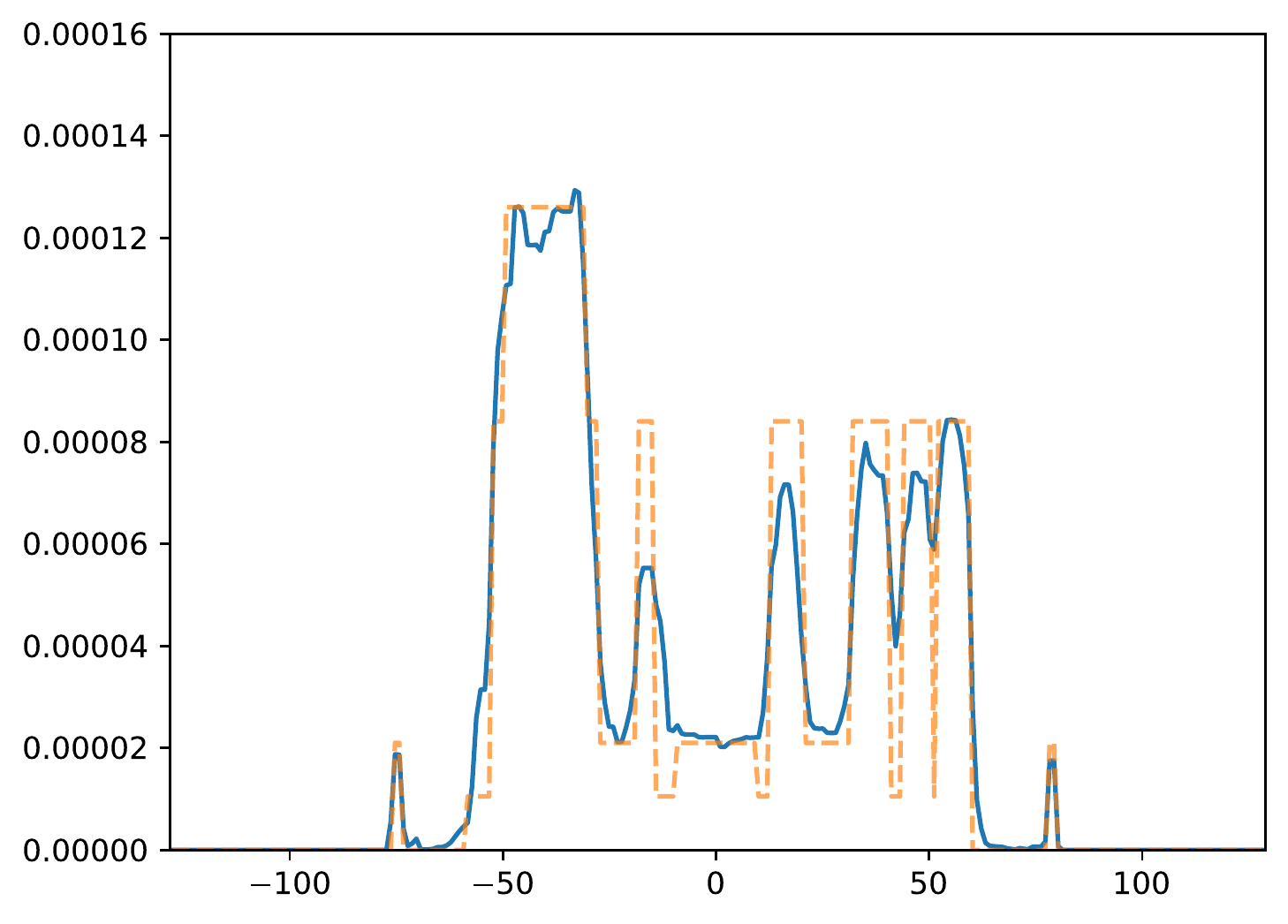}\\
				(a) $\lambda_{*opt}$ & (b) profile $\lambda_*$ vs. $\lambda_{*opt}$
			\end{tabular}
			\caption{$\lambda_{*opt}$ at (a), profile through the lesion in $\lambda_{*opt}$ (blue) vs. $\lambda_*$ (dotted orange) at~(b).}
			\label{fig:pseudo-phantom}
		\end{figure}
		
		In what follows empirical credible intervals are tested to cover $\lambda_{*opt}$ but not $\lambda_*$. In practice we computed $\lambda_{*opt}$ as a solution of the following minimization problem
		\begin{align}\label{eq:num-experiment:lambda-opt-def-practical}
		\lambda_{*opt} = \argmin_{\lambda\succeq 0} L_p(\lambda \mid A\lambda_*, A, 1, \beta_{min}),
		\end{align}
		where $\beta_{min}$ was chosen subjectively such that $\lambda_{*opt}$ does not contain visible numerical artifacts related to the implementation of Siddon's projector. As a result we choose $\beta_{min} = 10^{-3}$. The used minimization algorithm in \eqref{eq:num-experiment:lambda-opt-def-practical} was described in Appendix~\ref{subsect:gem}.	For $\varphi(\lambda)$ we use the function from \eqref{eq:prelim:numerical:log-cosh-prior-def}, where parameters are chosen 
		as follows: $\zeta = 0.05$, $\nu = 0.15$, $\beta^t = 2\times 10^{-3}$.
		For $t_1=1$, we present results for $\rho = \theta^t/t \in \{0, 0.25, 0.5, 1,2\}$ (see Remark \ref{rem:parameter-theta-interpretation}). For $t_2 = 100$ we choose only one value $\rho = 0.05$. For each combination of $(t,\rho)$, Algorithm \ref{alg:npl-posterior-sampling:mri:binned} was generating $B=1000$ bootstrap draws from which further statistics were computed (empirical mean, standard deviation, etc.).

		Finally, the misspecification in the nonparametric prior is mainly due to the fact that the lesion is not reflected in $\mathcal M$ and, more generally, to the mismatch between the actual emission map $\lambda_*$ and the segmentation in $\mathcal{M}$. In this sense our numerical test is the worst-case scenario of using the MRI data in ET.
		
		\subsection{Results}
		\newcommand\xs{2.75cm}
		\begin{figure}[H]
			\centering
			\vspace{-0.5cm}
			\begin{tabular}{lcccc}
				& NPL-mean & NPL-std $\times 3$ & Profile & Coverage\\
				
				\raisebox{1.3cm}{$\rho=0$} & \includegraphics[height=\xs]{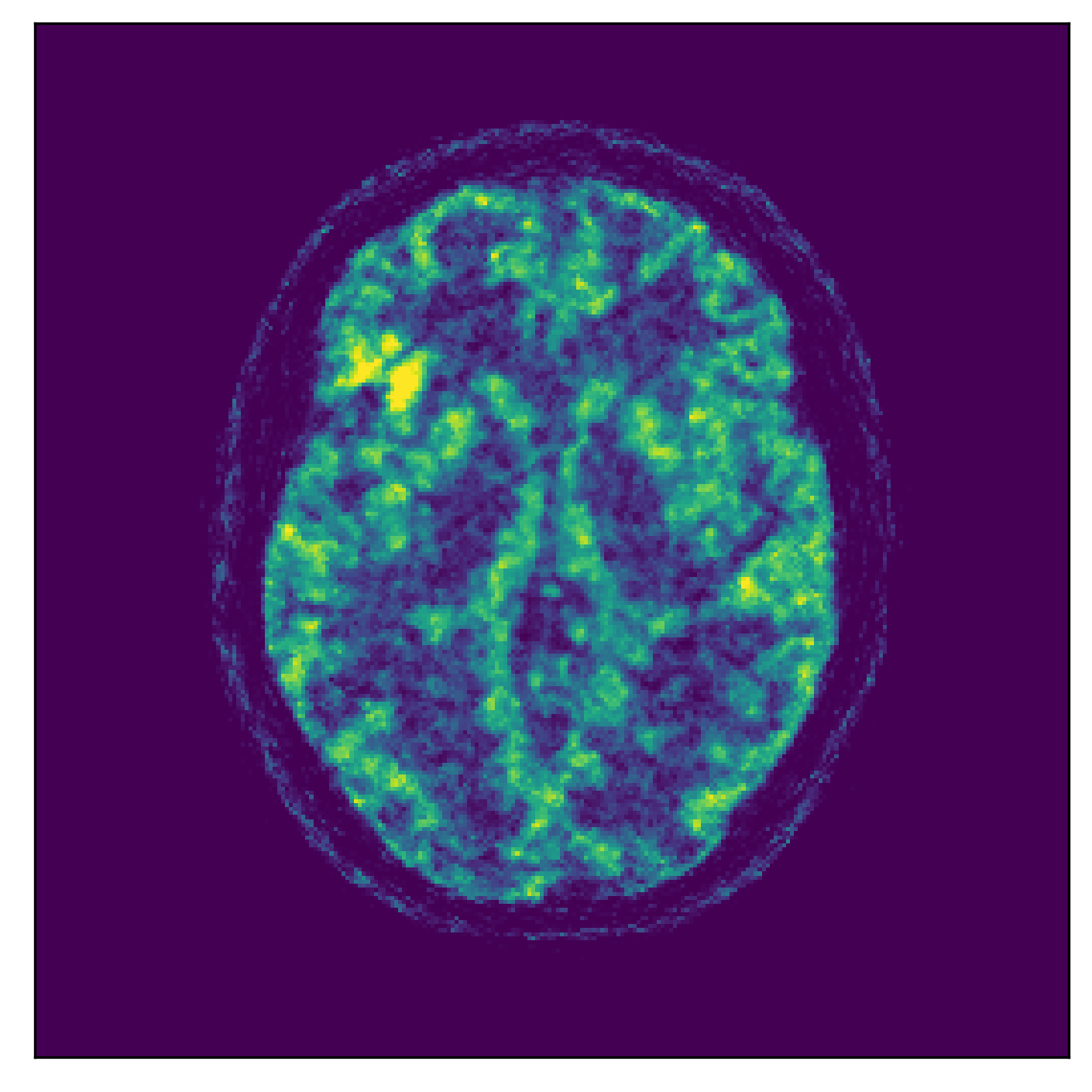} & \includegraphics[height=\xs]{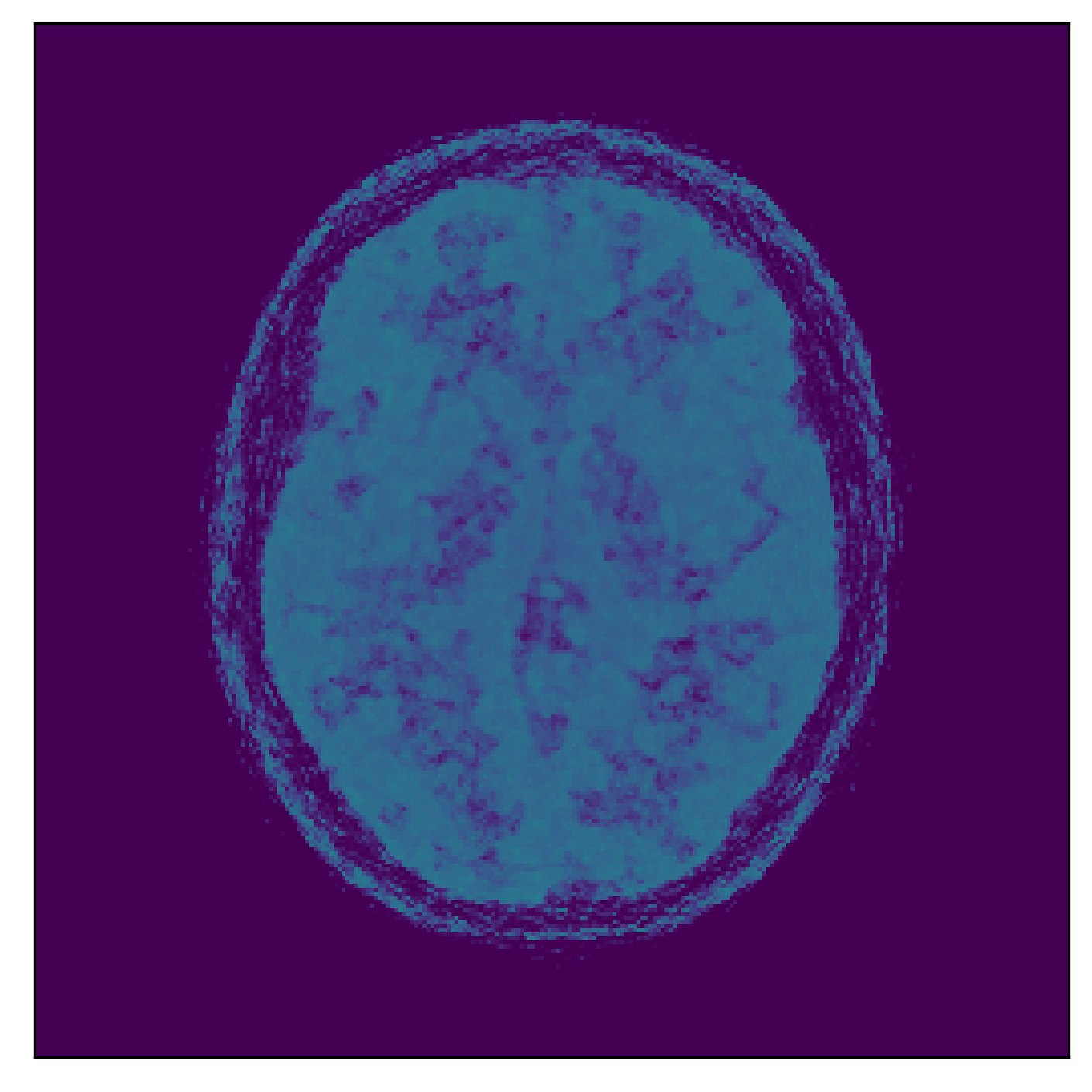} & \includegraphics[height=\xs]{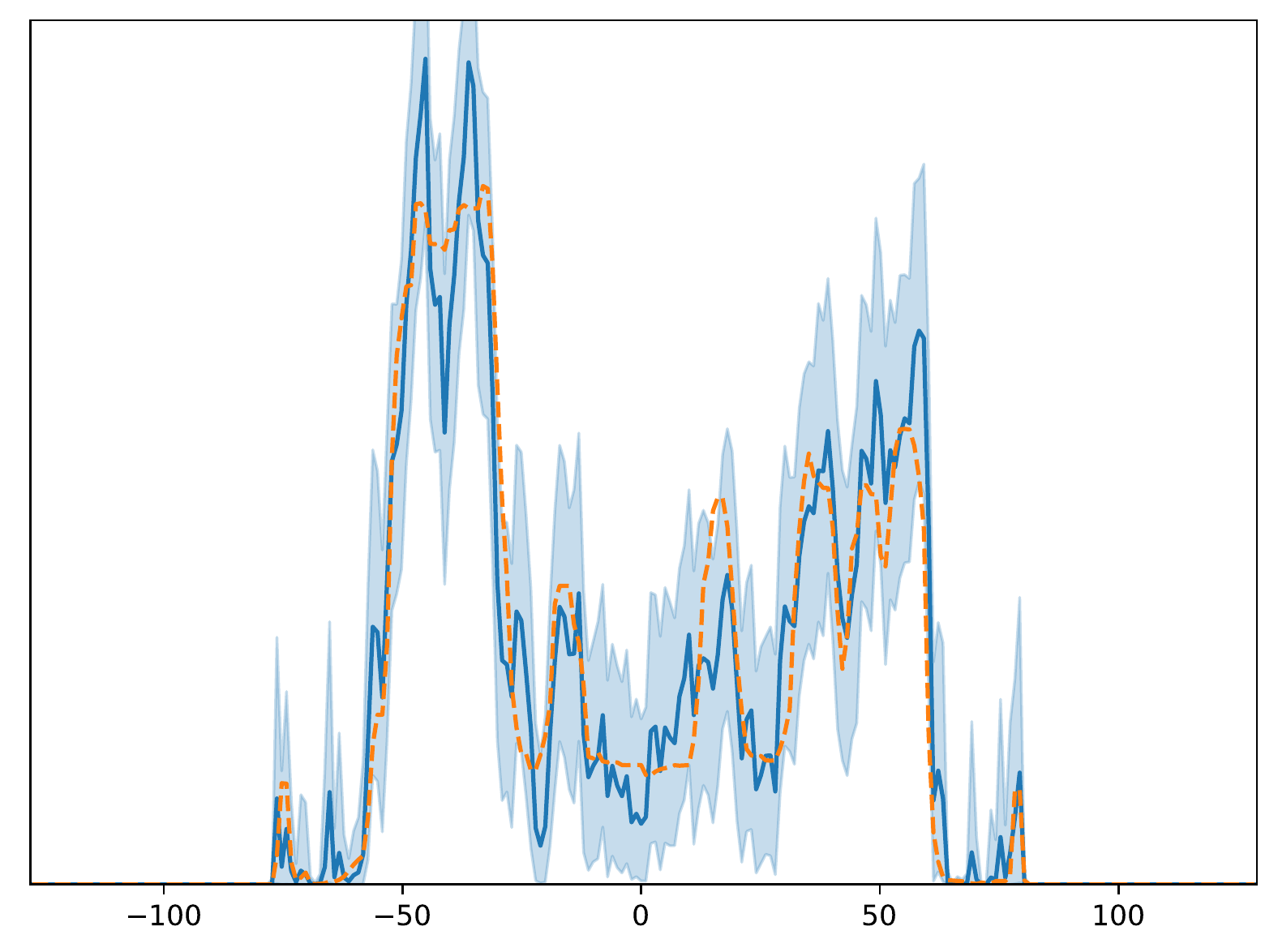} & \includegraphics[height=\xs]{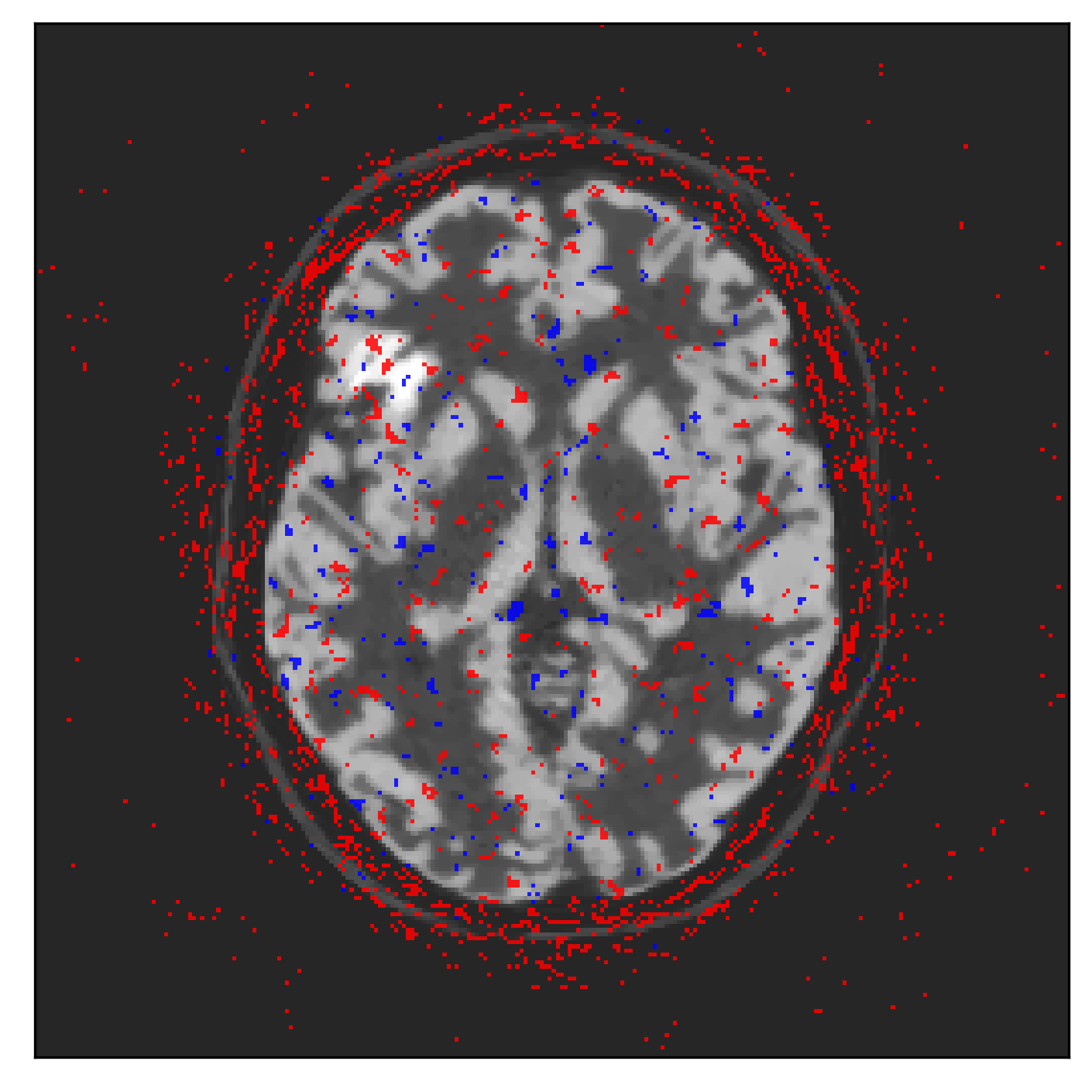}\\
				\raisebox{1.3cm}{$\rho=\frac 1 4$} & \includegraphics[height=\xs]{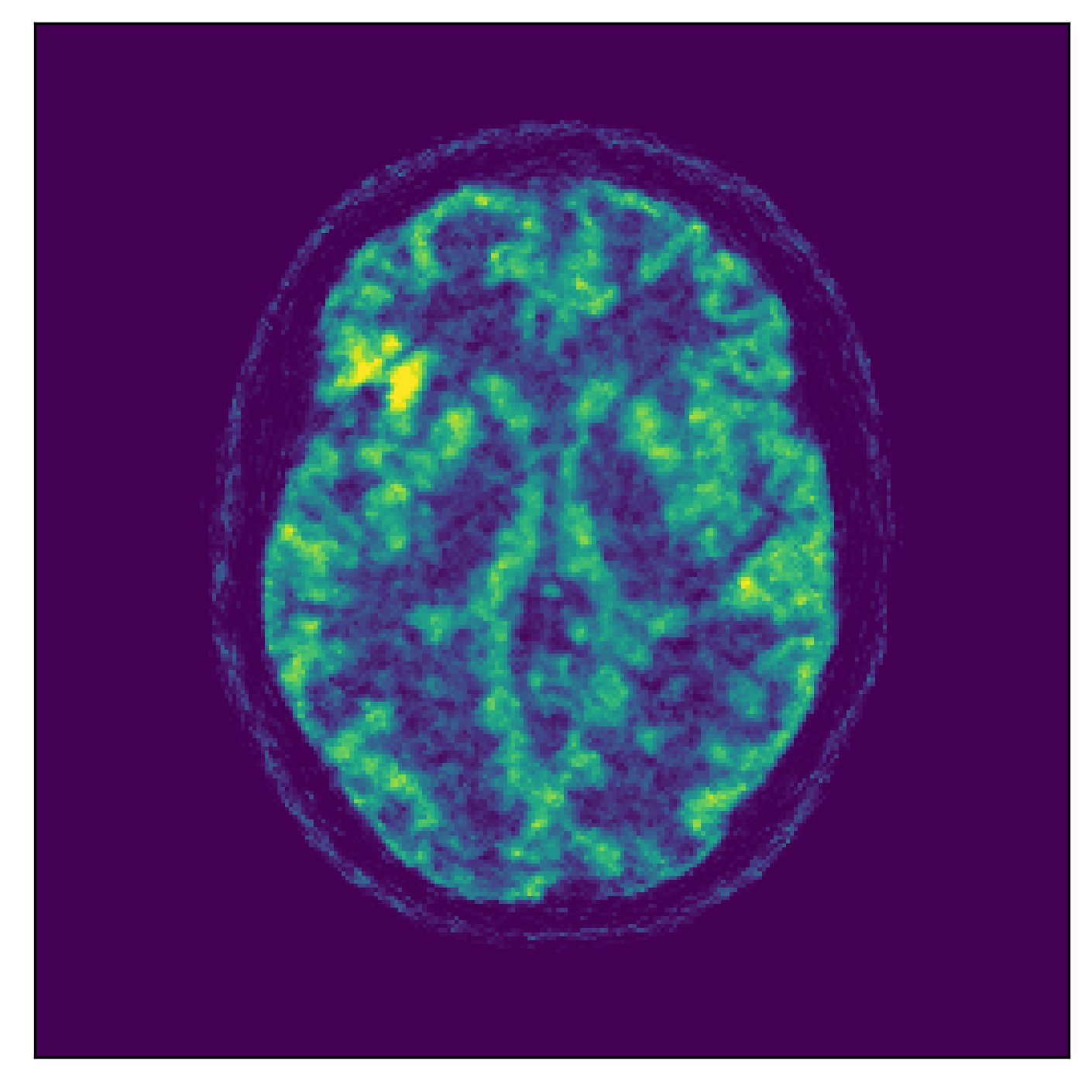} & \includegraphics[height=\xs]{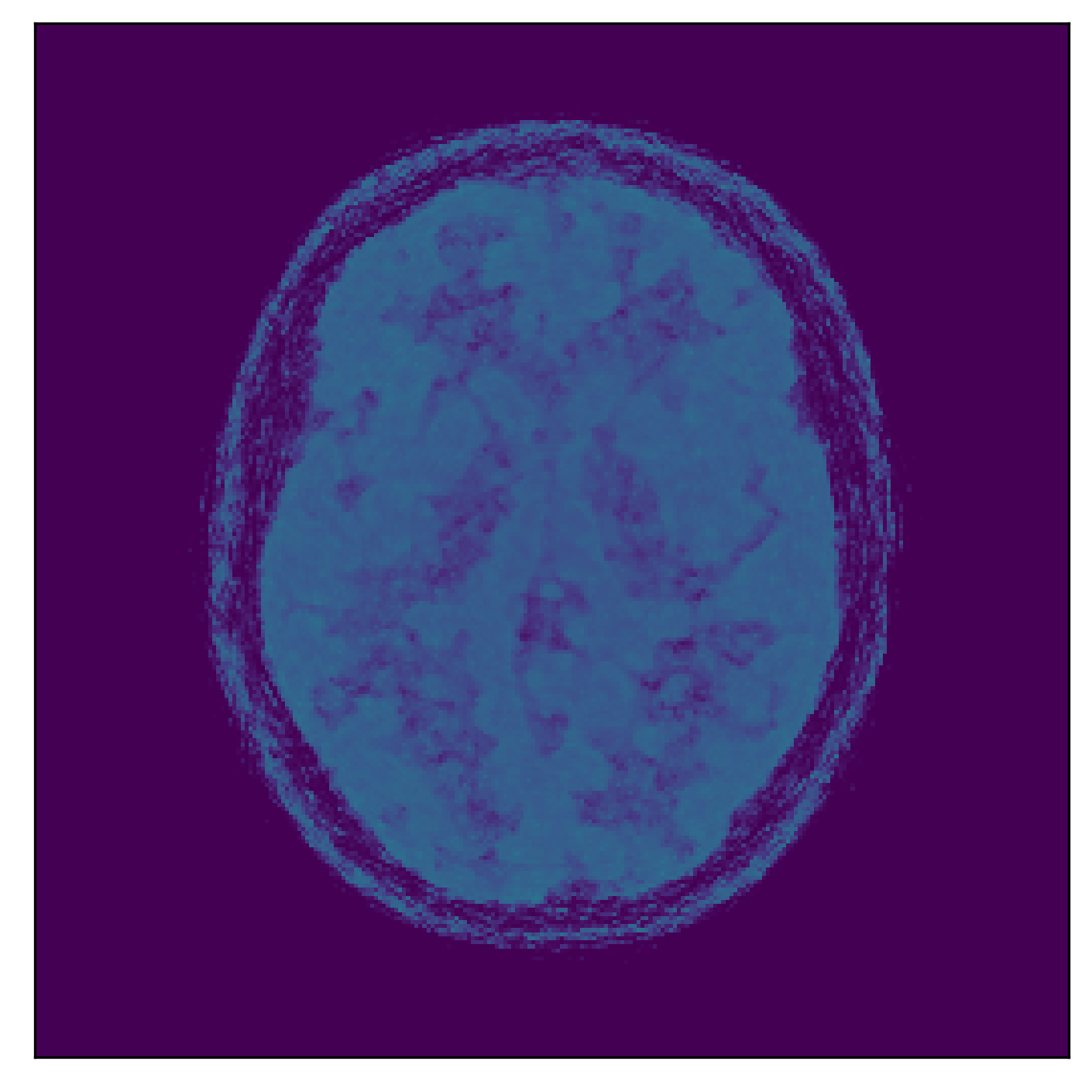} & \includegraphics[height=\xs]{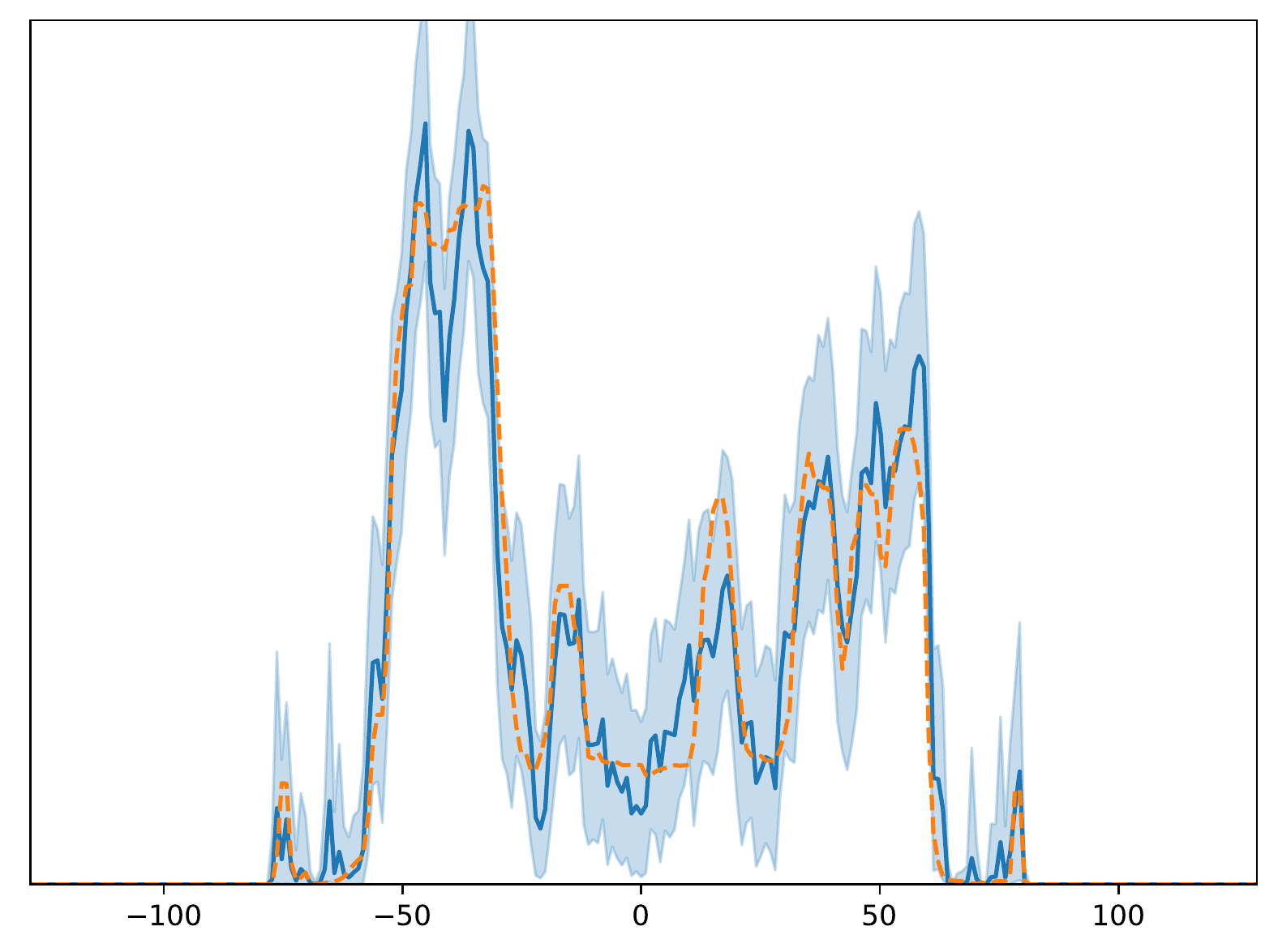} & \includegraphics[height=\xs]{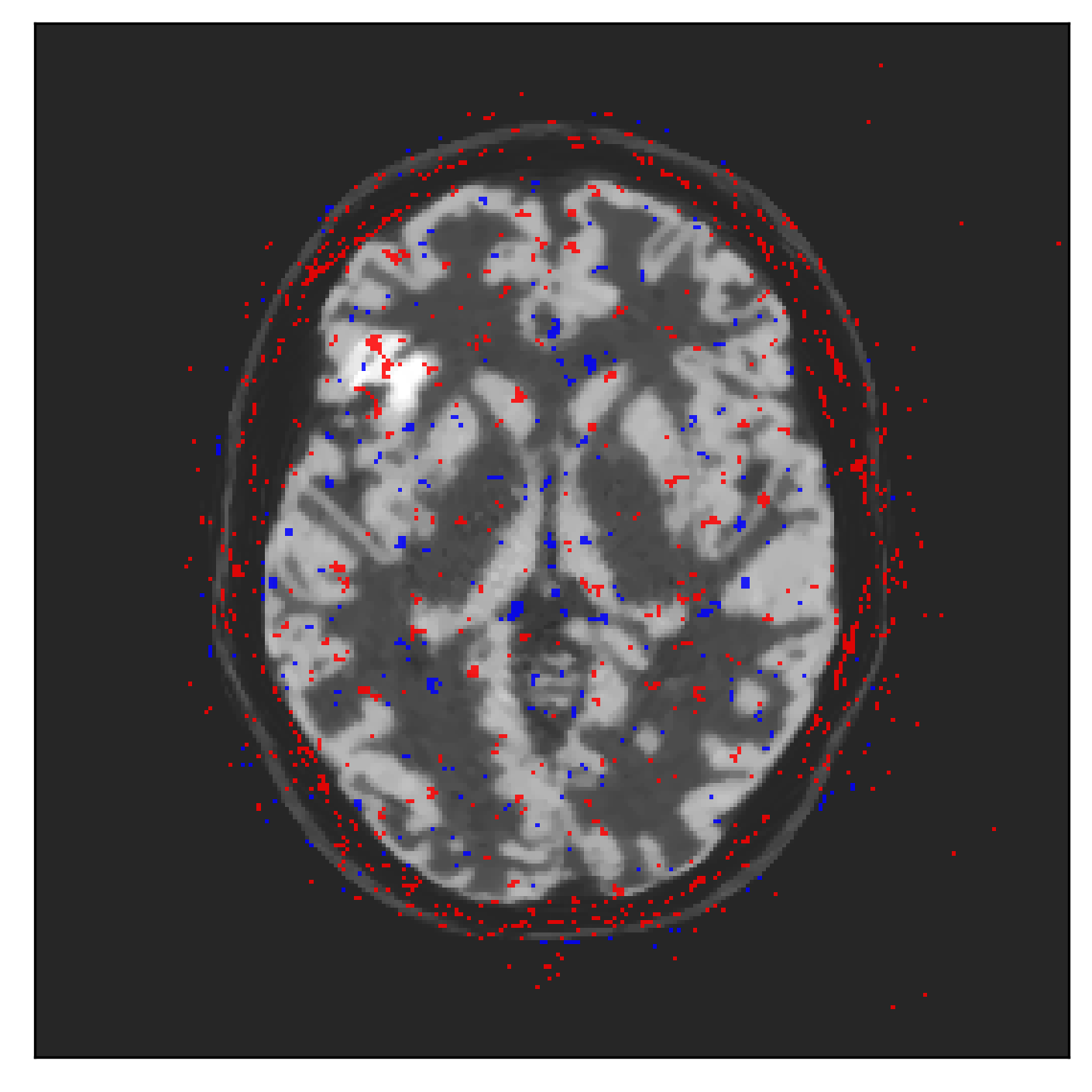}\\
				& (a) & (b) & (c) & (d)\\
			\end{tabular}
			\label{fig:NPL_results1}
		\end{figure}

		\begin{figure}[H]
			\centering
			\vspace{-0.5cm}
			\begin{tabular}{lcccc}
				& NPL-mean & NPL-std $\times 3$ & Profile & Coverage\\
				\raisebox{1.3cm}{$\rho=\frac 1 2$} & \includegraphics[height=\xs]{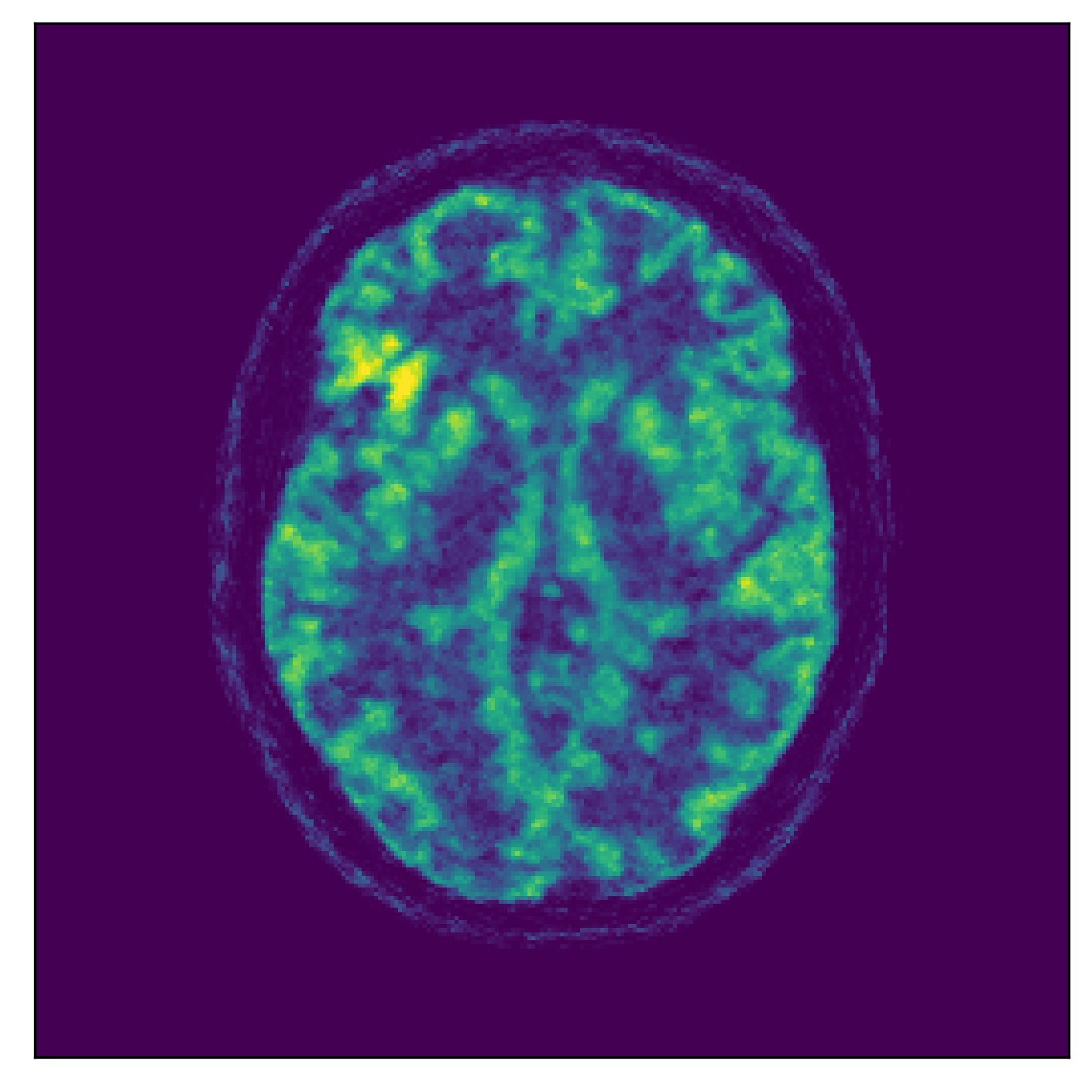} & \includegraphics[height=\xs]{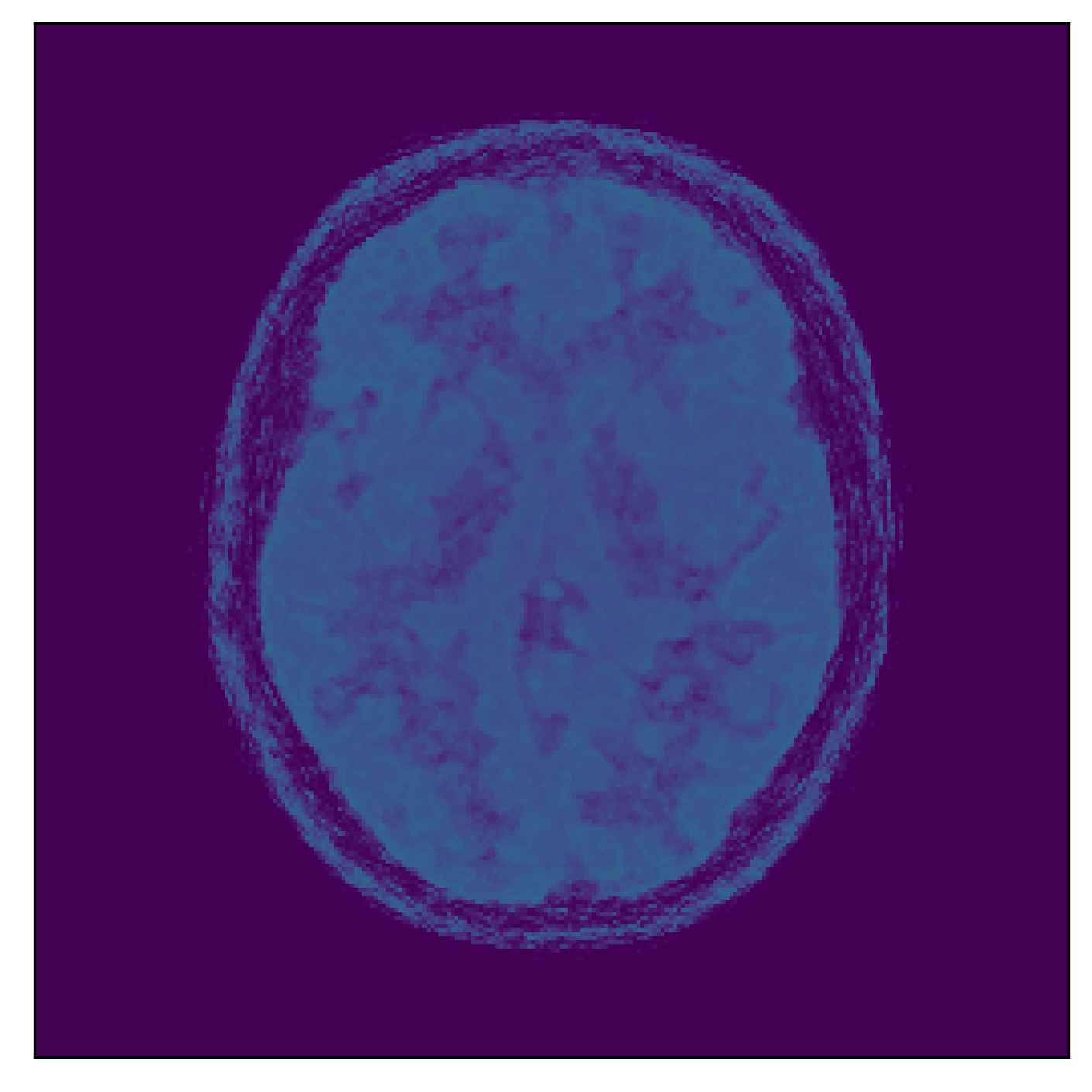} & \includegraphics[height=\xs]{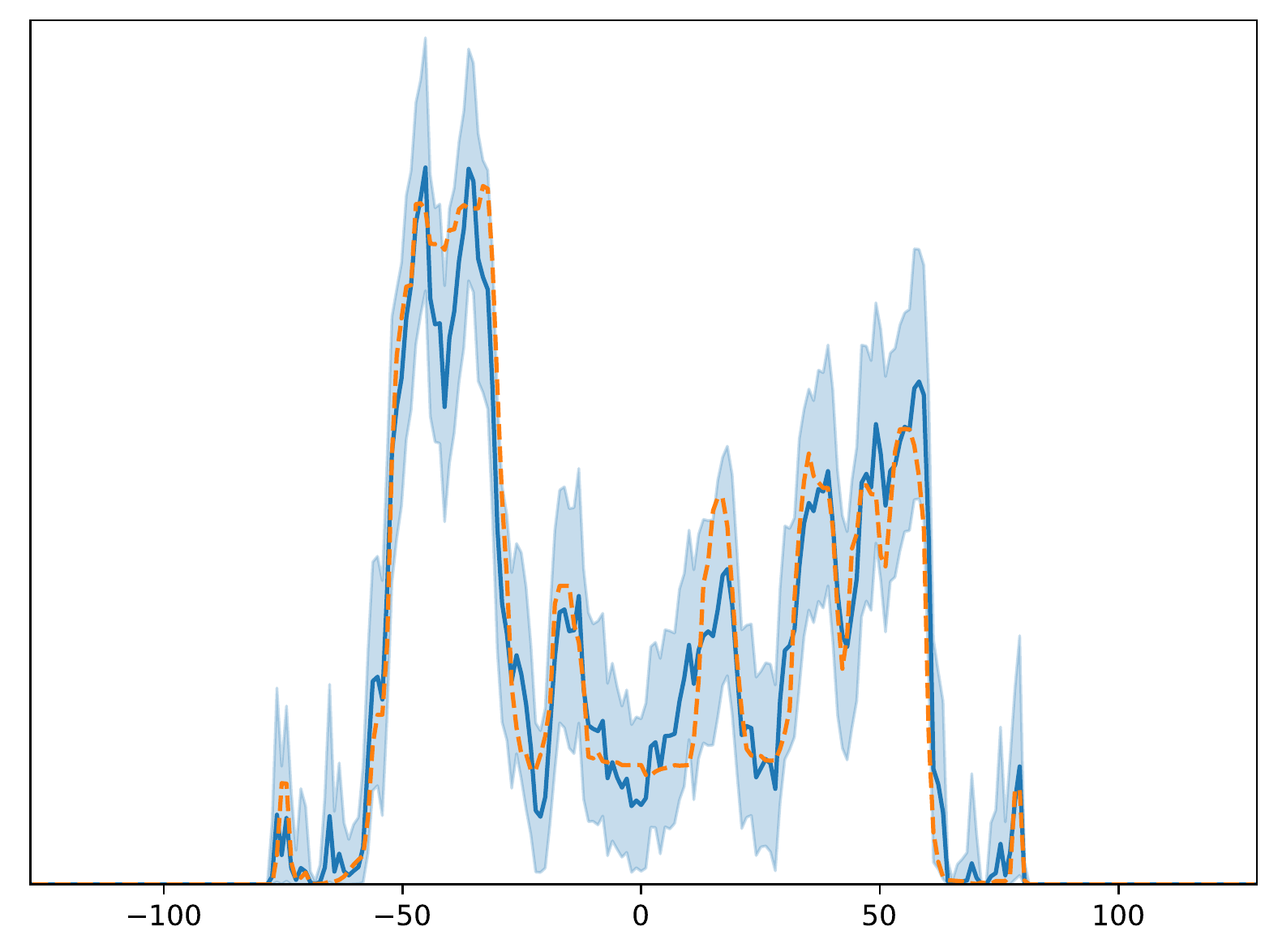} & \includegraphics[height=\xs]{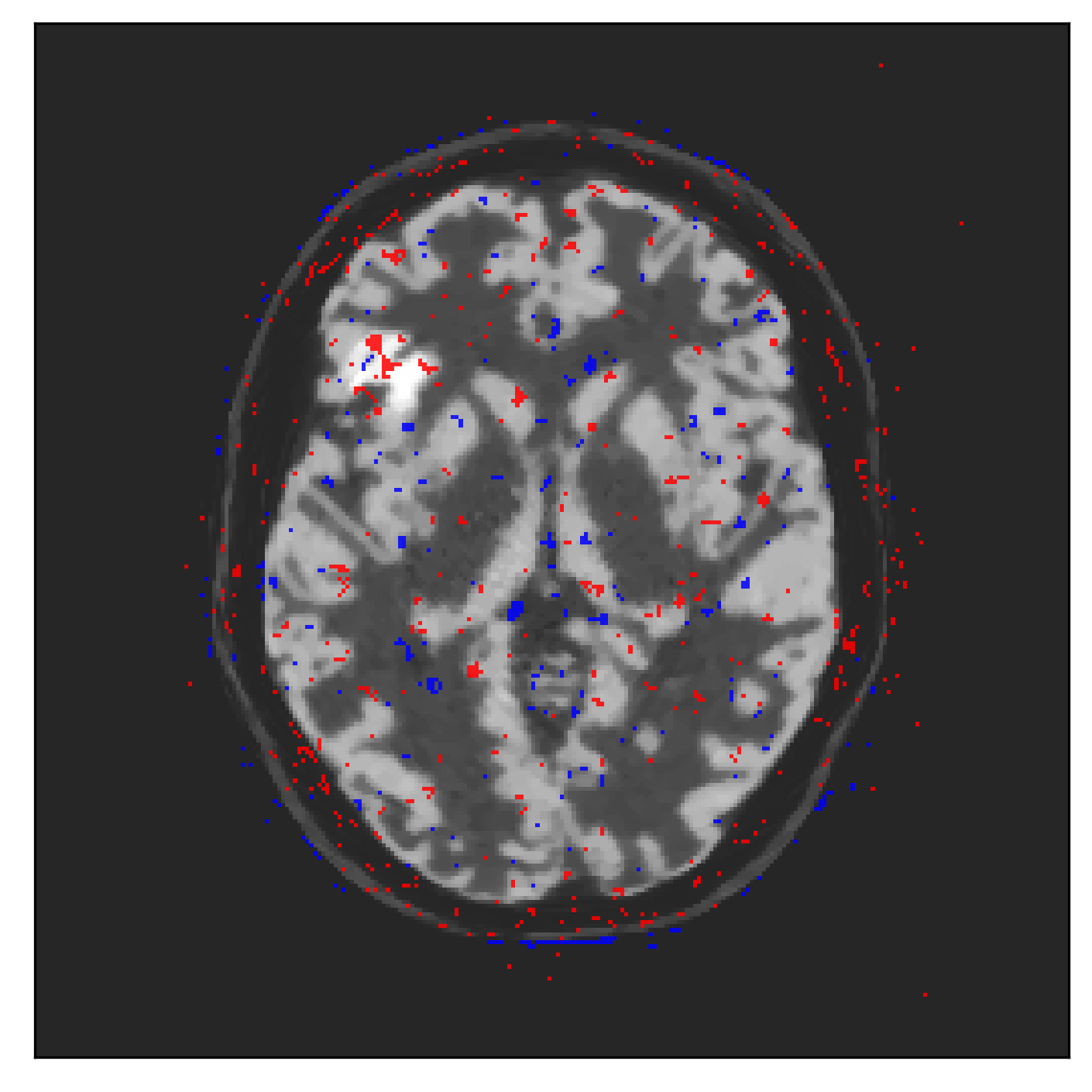}\\
				\raisebox{1.3cm}{$\rho=1$} & \includegraphics[height=\xs]{npl-mean-rho050.png} & \includegraphics[height=\xs]{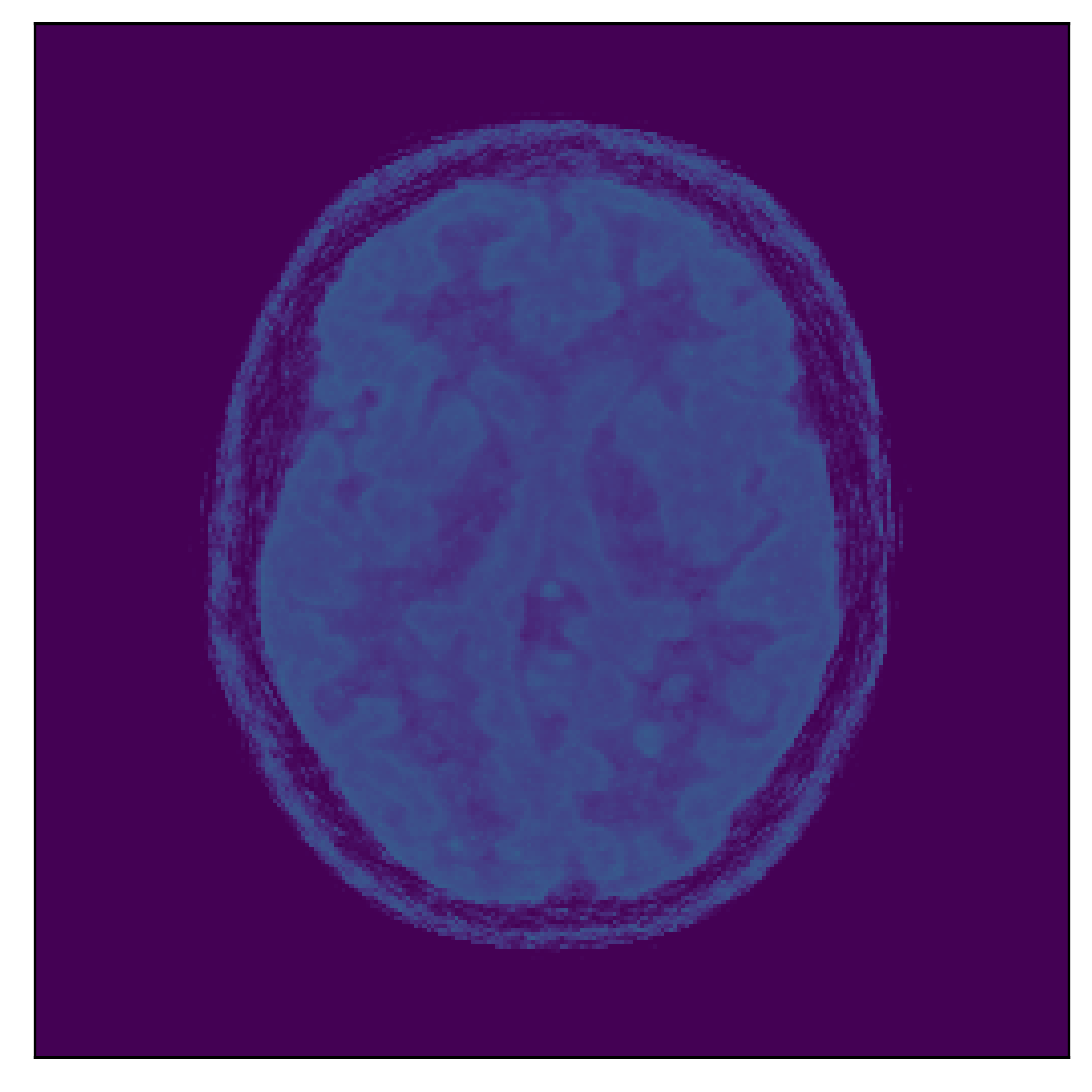} & \includegraphics[height=\xs]{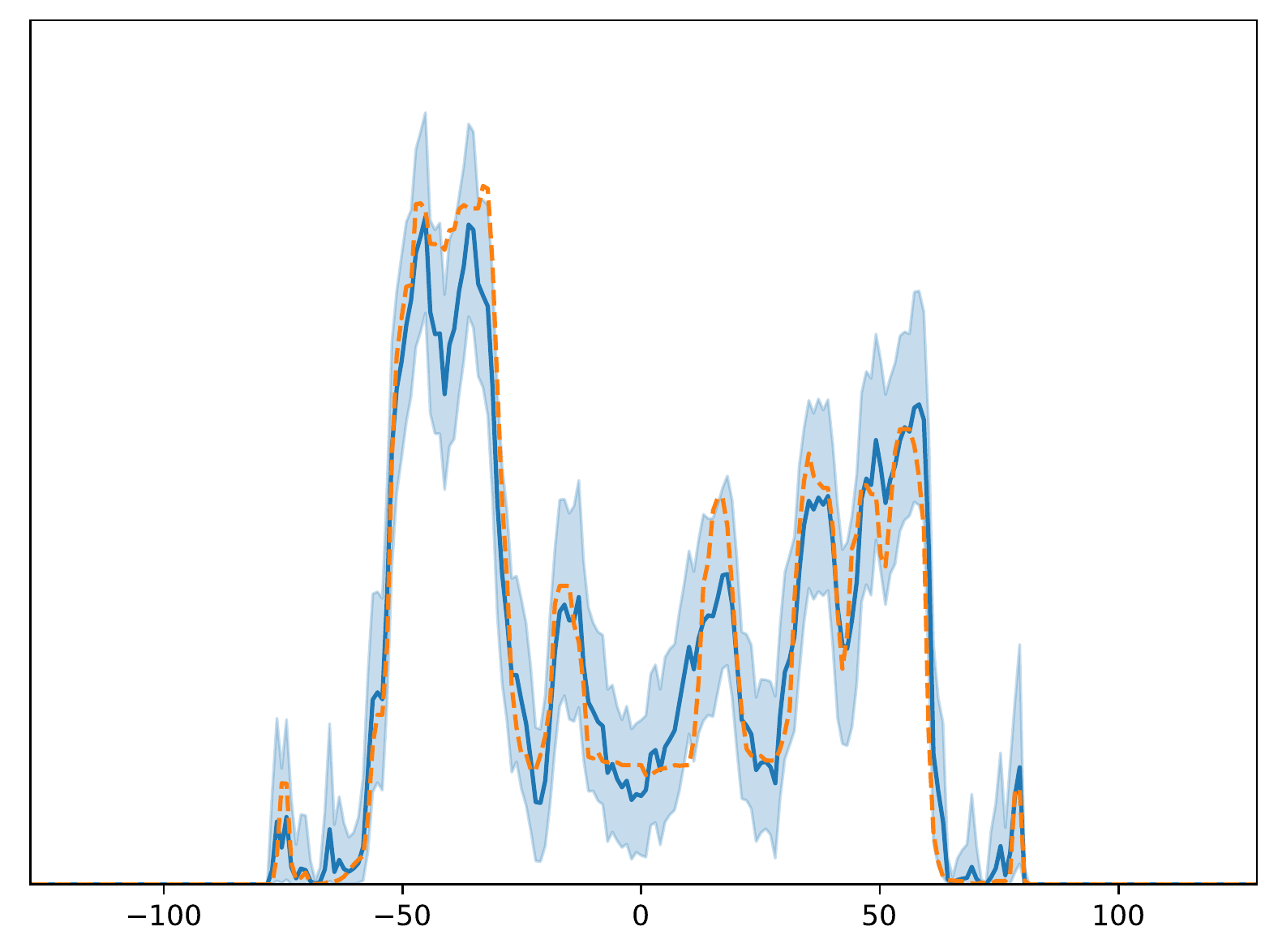} & \includegraphics[height=\xs]{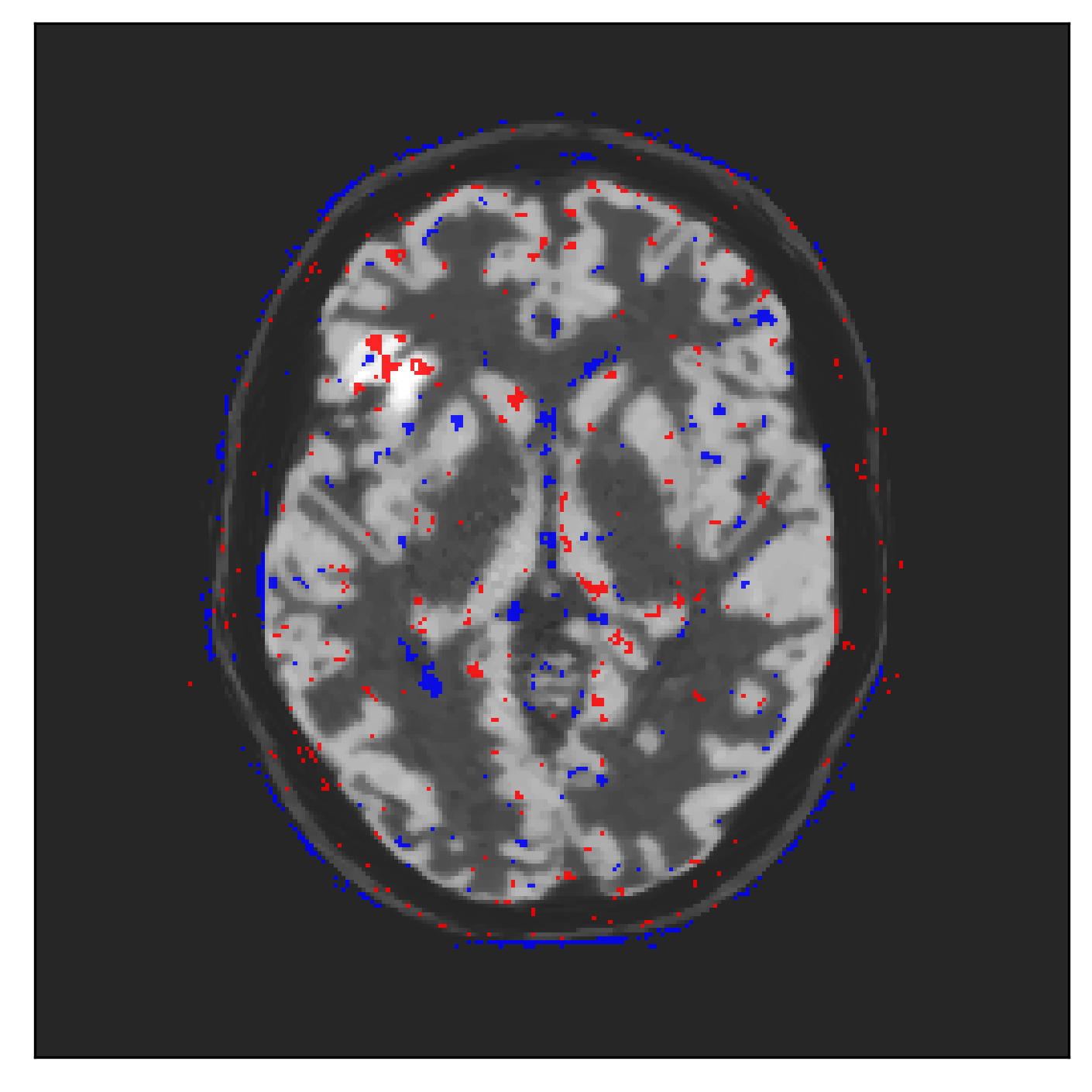}\\
				\raisebox{1.3cm}{$\rho=2$} & \includegraphics[height=\xs]{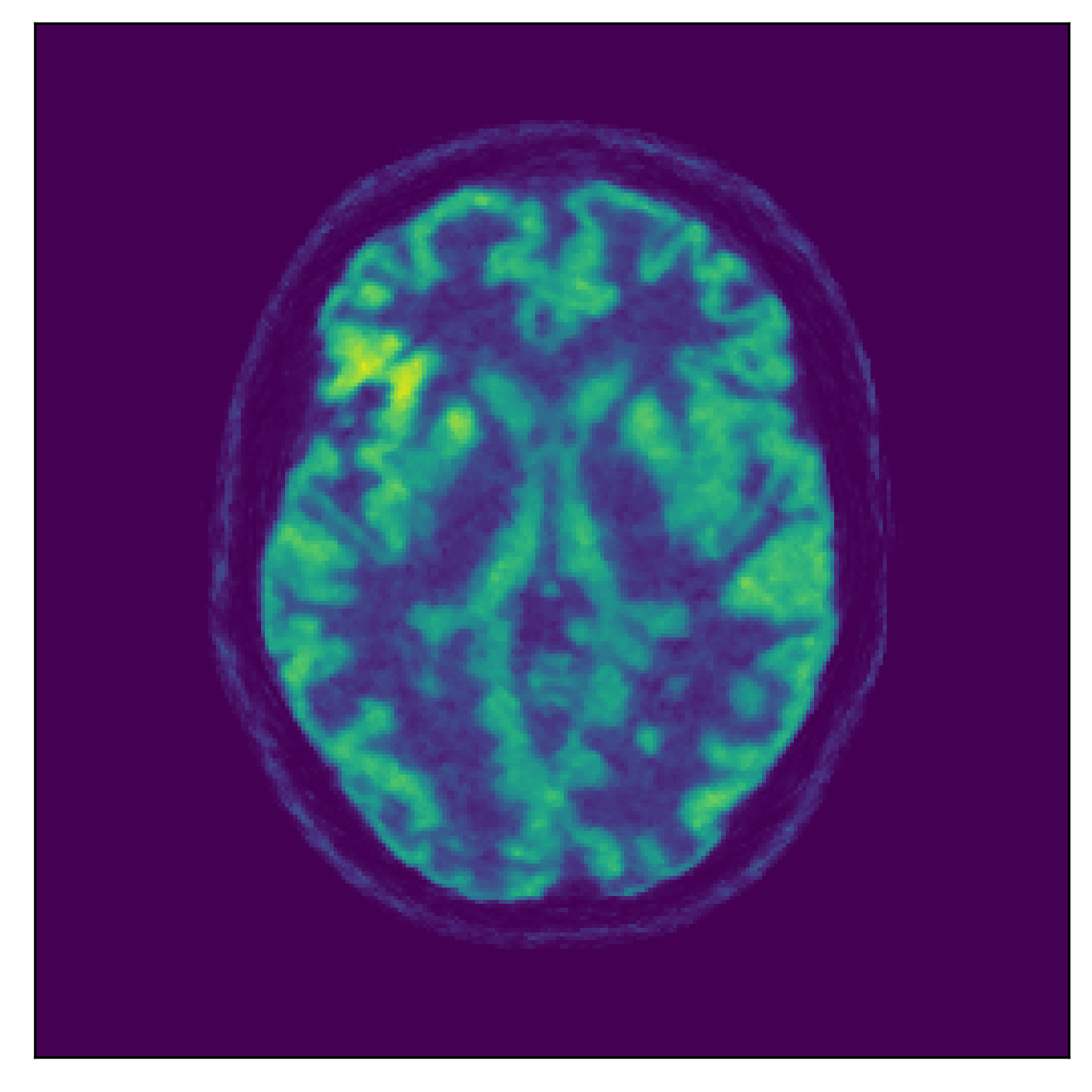} & \includegraphics[height=\xs]{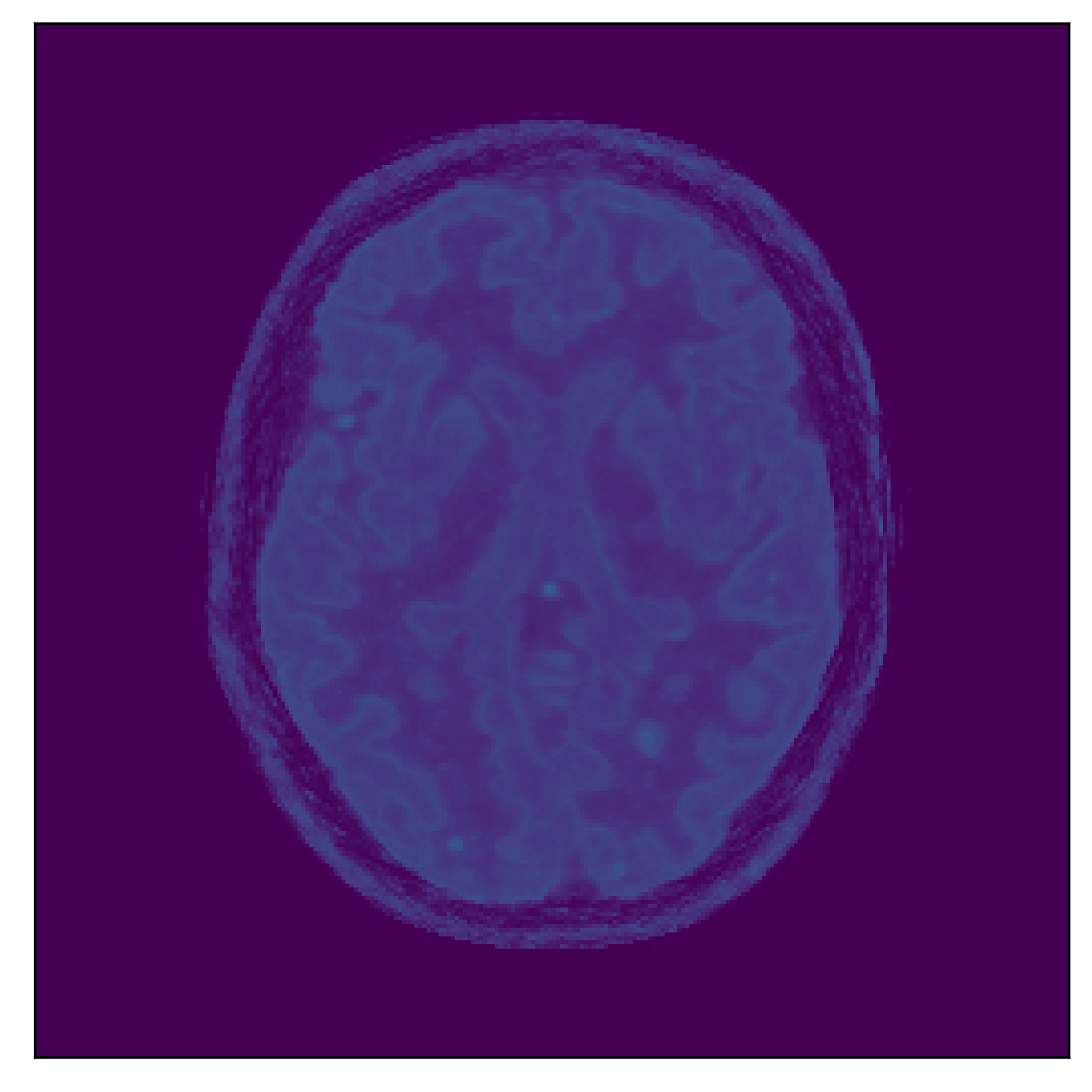} & \includegraphics[height=\xs]{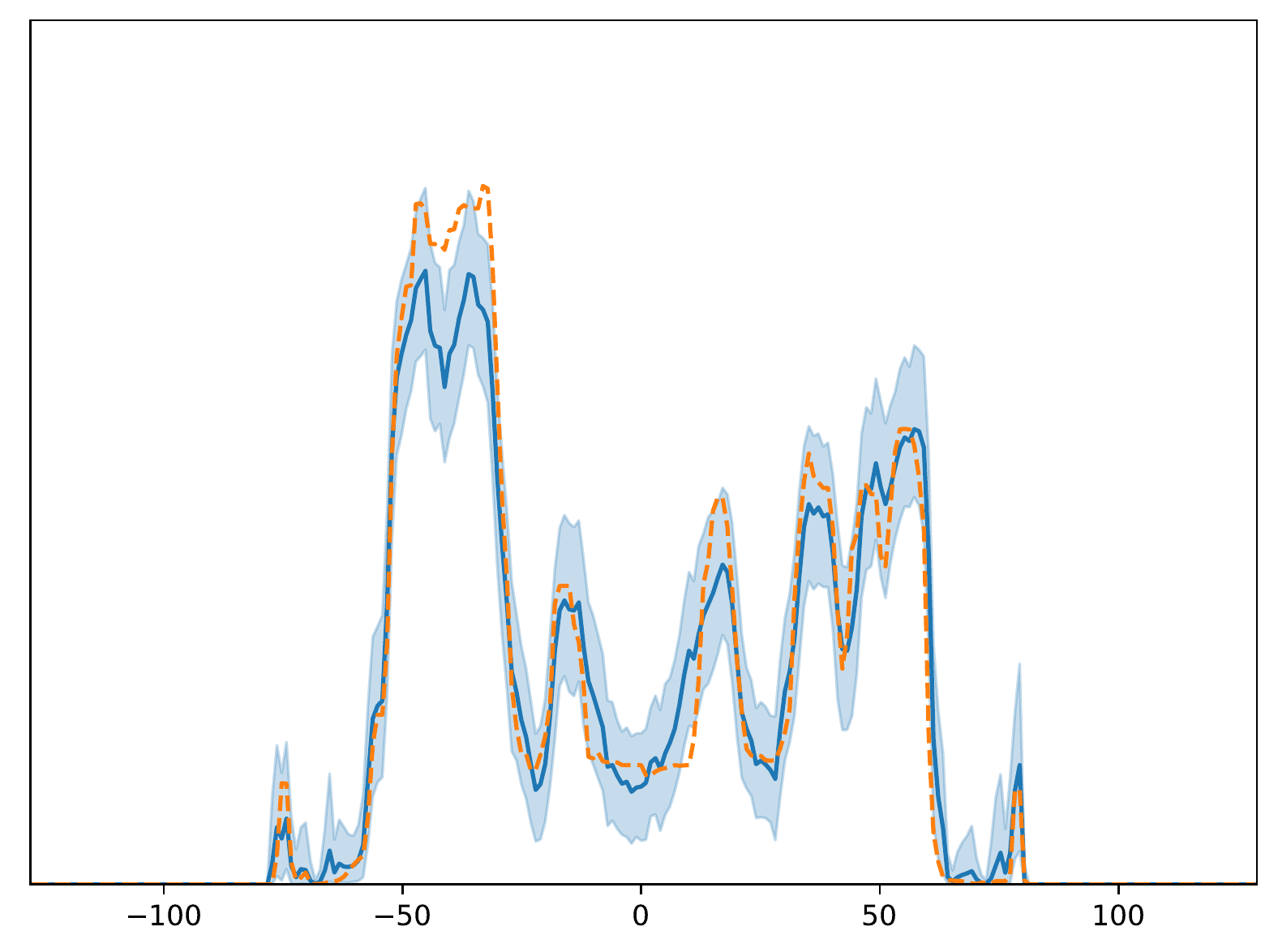} & \includegraphics[height=\xs]{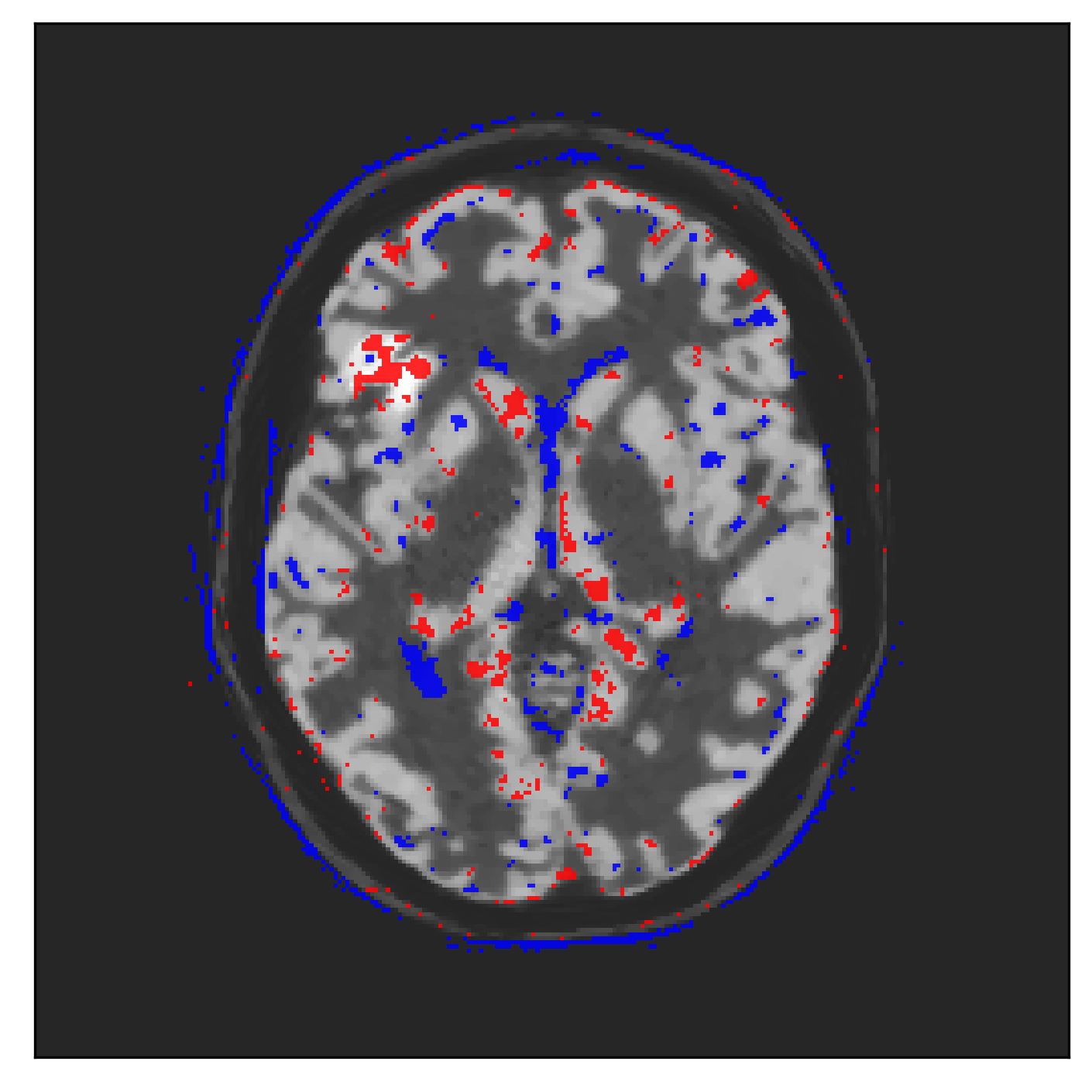}\\
				& (a) & (b) & (c) & (d)\\
			\end{tabular}
			\caption{NPL posterior mean (a), $3\times$ the posterior standard deviation with same color scale as mean (b), posterior 95\% band on an horizontal profile through the lesion; blue line -- posterior mean, orange dotted -- $\lambda_{*opt}$ (c),  pixel-wise coverage success: grayscale -- $\lambda_{*opt}$ is inside the credible band, red, blue colors -- $\lambda_{*opt}$ is above or below, respectively~(d).}
			\label{fig:NPL_results2}
		\end{figure}

		\vspace{-0.5cm}
		\begin{figure}[H]
			\centering
			\begin{tabular}{lcccc}
				& NPL-mean & NPL-std $\times 3$ & Profile & Coverage\\
				\raisebox{1.3cm}{\hphantom{$\rho=\frac 1 2$}}
				& \includegraphics[height=\xs]{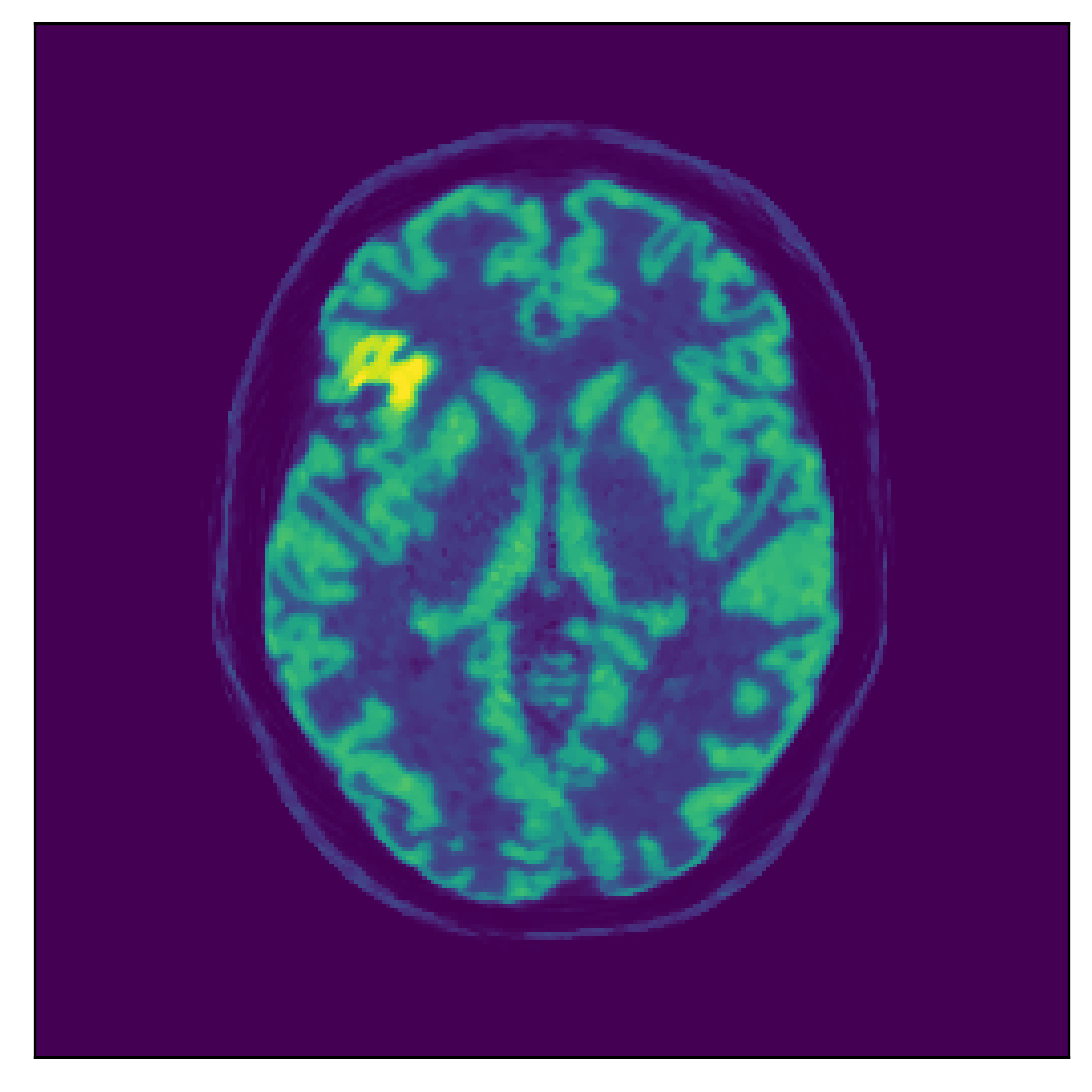} & \includegraphics[height=\xs]{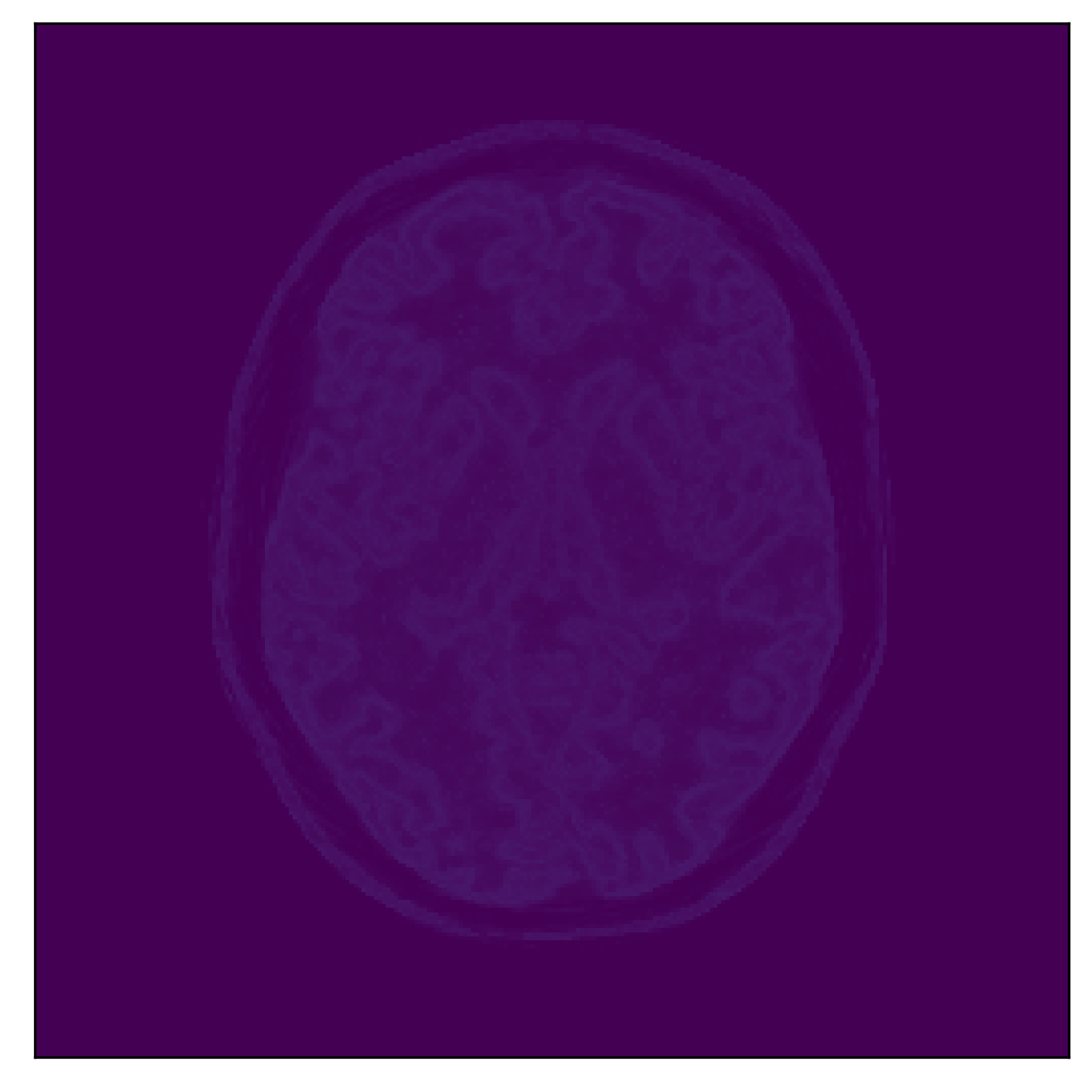} & \includegraphics[height=\xs]{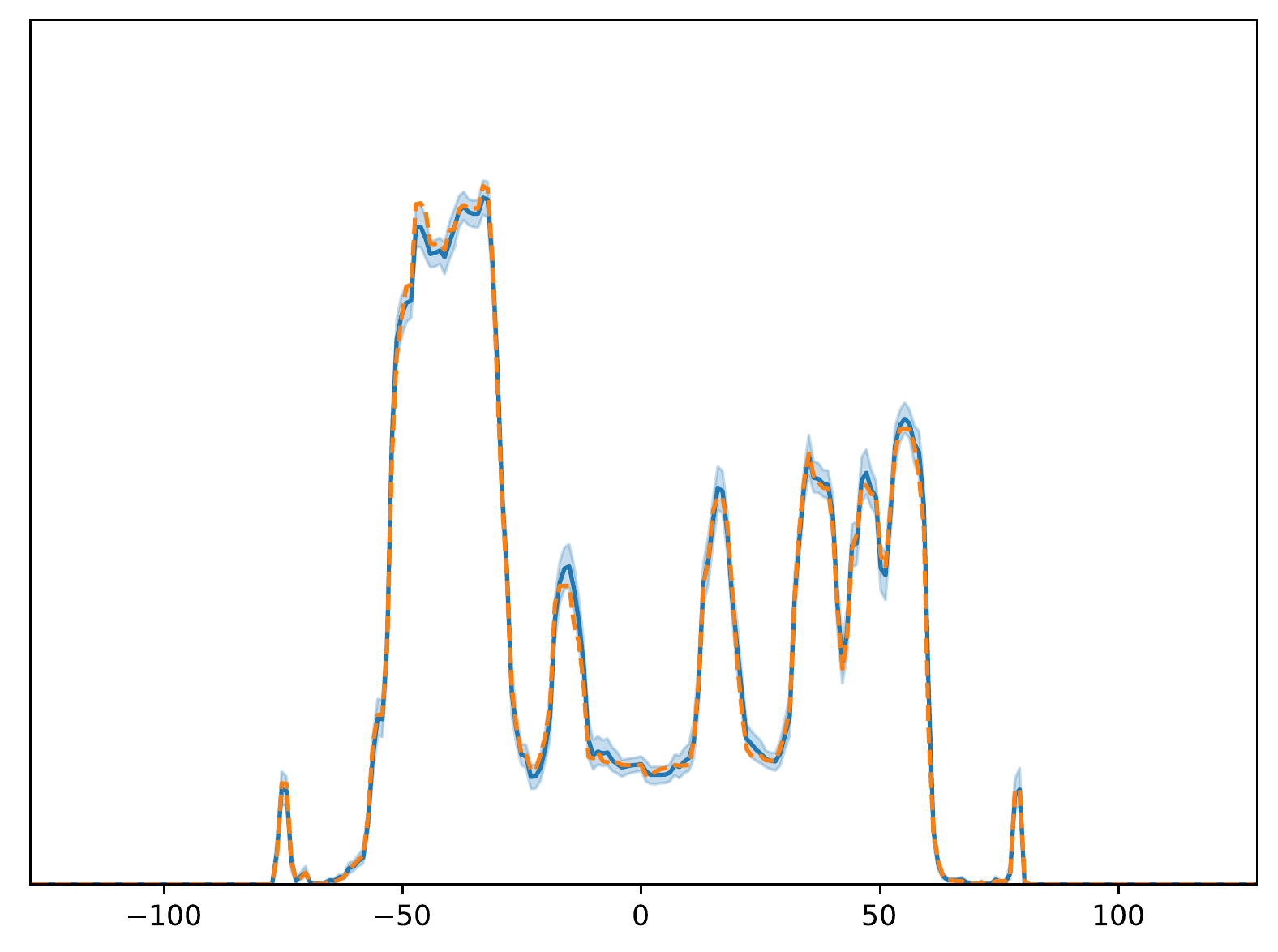} & \includegraphics[height=\xs]{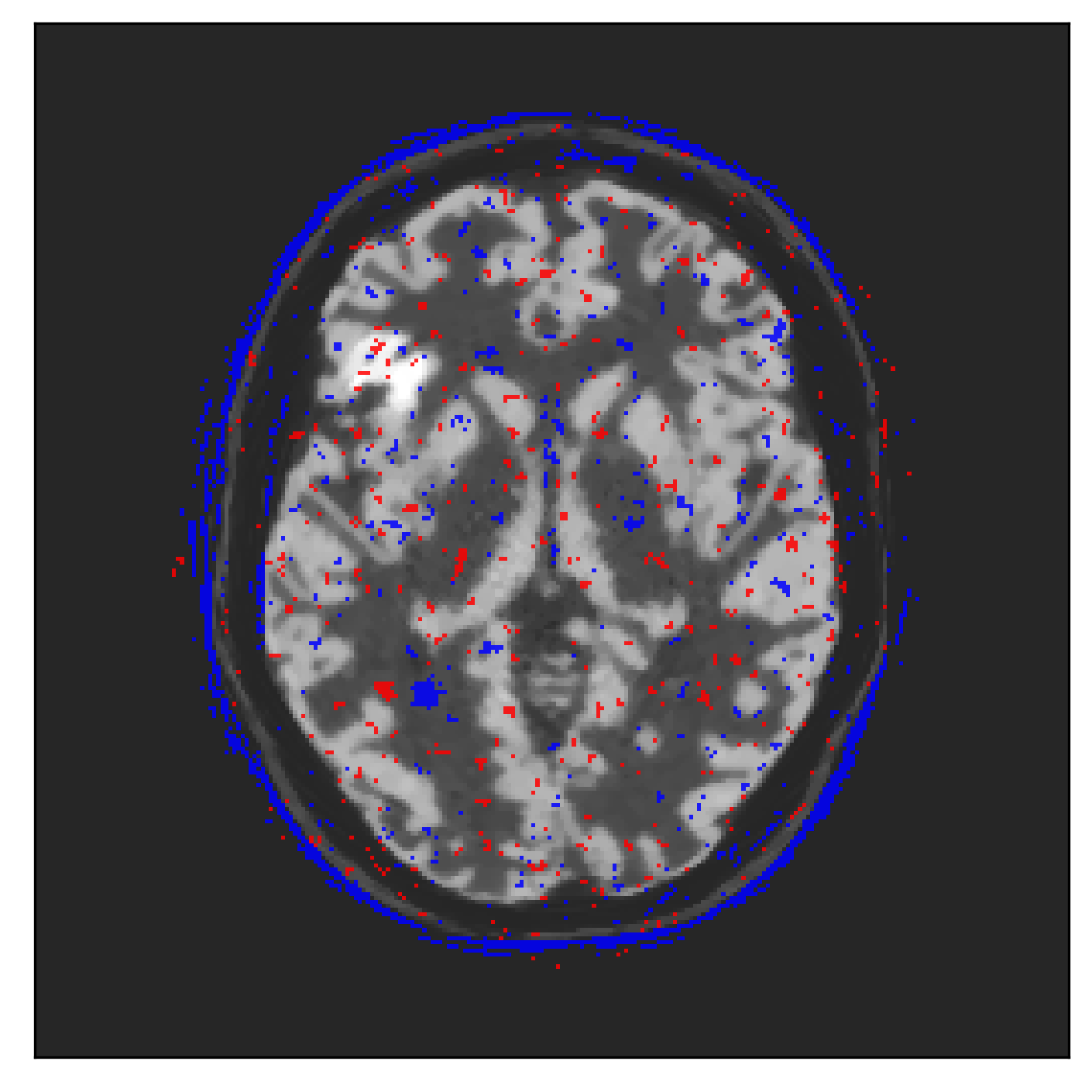}\\
				& (a) & (b) & (c) & (d)\\
			\end{tabular}
			\caption{NPL with $\rho=0.05$ for $t=100$}
			\label{fig:NPL_t=100}
		\end{figure}
		
		As expected, higher $\rho$ reinforce the effect of MRI on  reconstructions and posterior variance decreases with $\rho$ growing ($\mathrm{var}(\widetilde{\lambda}_b^t \, \mid \, Y^t, t) \sim (\rho)^{-1}$). As a rule of thumb, it seems reasonable not to exceed $\rho=1$ since the weight of pseudo-data from the misspecified model in the prior should not exceed the weight of observed data; see Remark~\ref{rem:parameter-theta-interpretation}.
		We also checked visually that NPL posterior mean with $\rho=0$ (no MRI) is indistinguishable from the MAP reconstruction with the same penalty tuning (see Section~\ref{app:NPLvsMAP}). This supports the claim in  Theorem~\ref{thm:asympt-distr-main} that the asymptotic distribution is concentrated not around $\lambda_*$ but a strongly consistent estimator for which we conjecture to coincide asymptotically with the MAP estimate.
		
		In Figure~\ref{fig:NPL_results2}(c) the coverage of $\lambda_{*opt}$ by pixel-wise 95\% credible bands is large almost for all pixels and all values of $\rho$ though the bias in the lesion dominates when $\rho > 1$. This can be explained by the choice of MRI images which do not contain at all of $\lambda_*$ in the lesion area; see Figures~\ref{fig:phantom} (a), (b).
		Visually it seems that $\rho=1$ is optimal for bias variance trade-off, however, this rule of thumb is reasonable only for moderate value of $t$ (hence, low number of counts in $Y^t$) and not in the asymptotic regime when $t\rightarrow +\infty$.  To highlight the latter we also consider the asymptotic behavior of NPL reconstruction by taking $t=100$ for the  regularization parameter $\beta^t/t = 10^{-3}$ (same as for $\beta_{min}$  in \eqref{eq:num-experiment:lambda-opt-def-practical}) and $\rho=0.05$ (see Figure~\ref{fig:NPL_t=100}). The point is that the  case of $t=100$ corresponds to almost noiseless data, so $\beta^t/t$ can be chosen in the ``optimal way''. Pixel-wise posterior bands capture most of the true signal  (see Figure~\ref{fig:NPL_t=100}(c)), except the blue region at the boundary of the cranium (see Figure~\ref{fig:NPL_t=100}(d)). This can be explained by the property of the GEM-algorithm (see Section~\ref{app:gem}) where the constructed parabolic majorizing surrogates which approximate poorly zero values at pixels.

		\section{Visual comparison between the NPL mean without MRI and the MAP reconstructions}
		\label{app:NPLvsMAP}
		\newcommand\xsa{3.45cm}
		\begin{figure}[H]
			\centering
			\begin{tabular}{lccc}
				\raisebox{1.6cm}{$t_1$} &
				\includegraphics[height=\xsa]{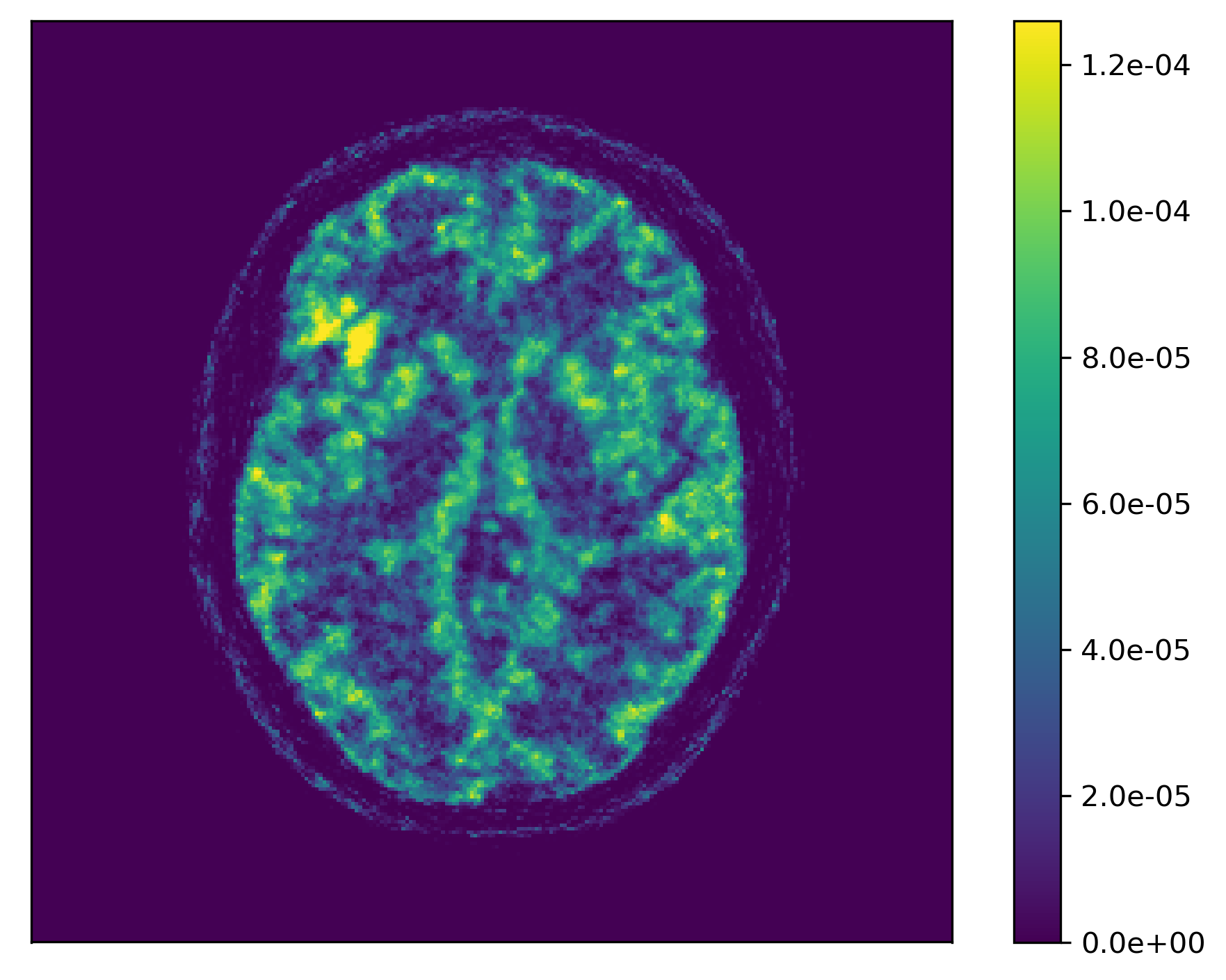} &
				\includegraphics[height=\xsa]{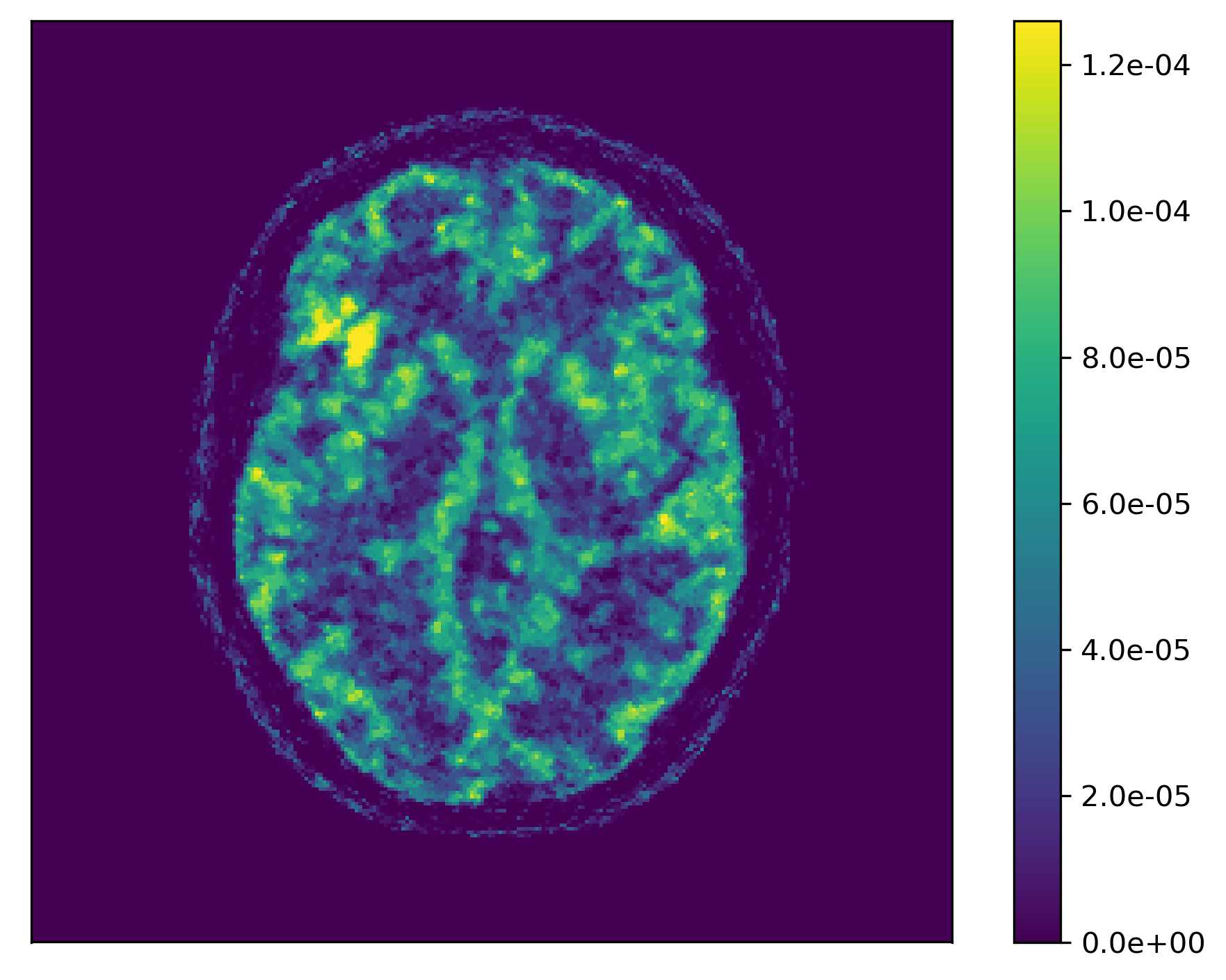} &
				\includegraphics[height=\xsa]{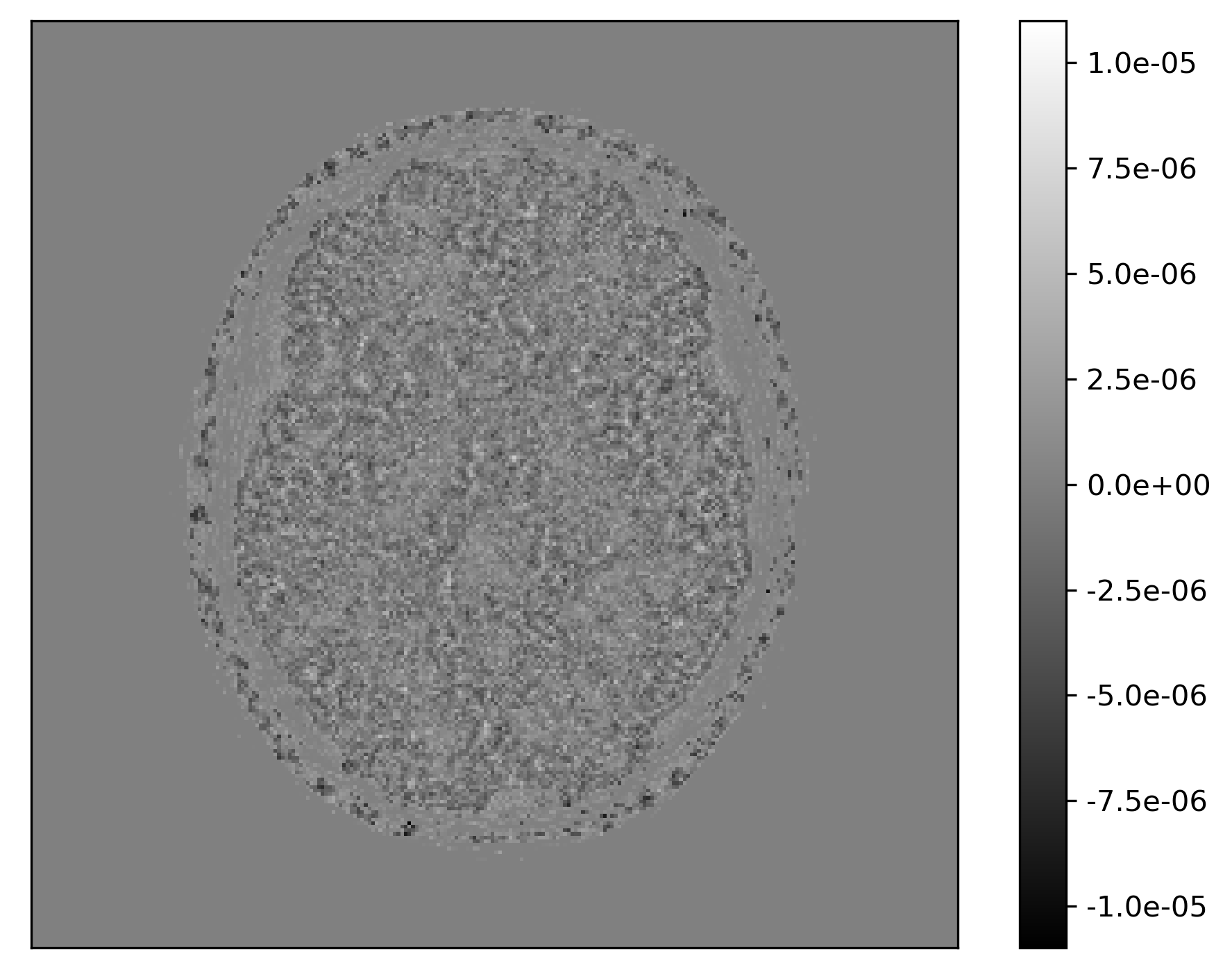} \\
				\raisebox{1.6cm}{$t_2$} &
				\includegraphics[height=\xsa]{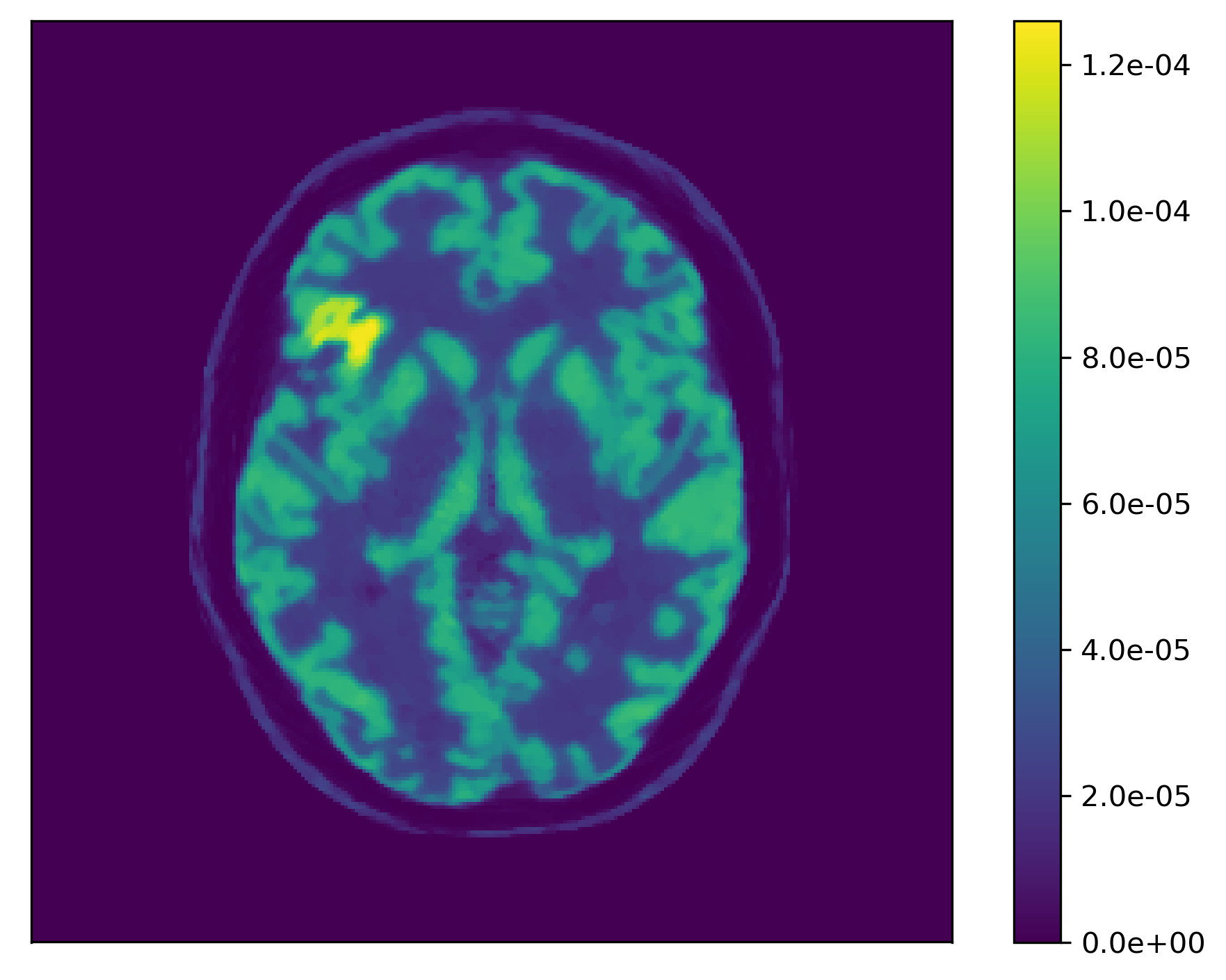} &
				\includegraphics[height=\xsa]{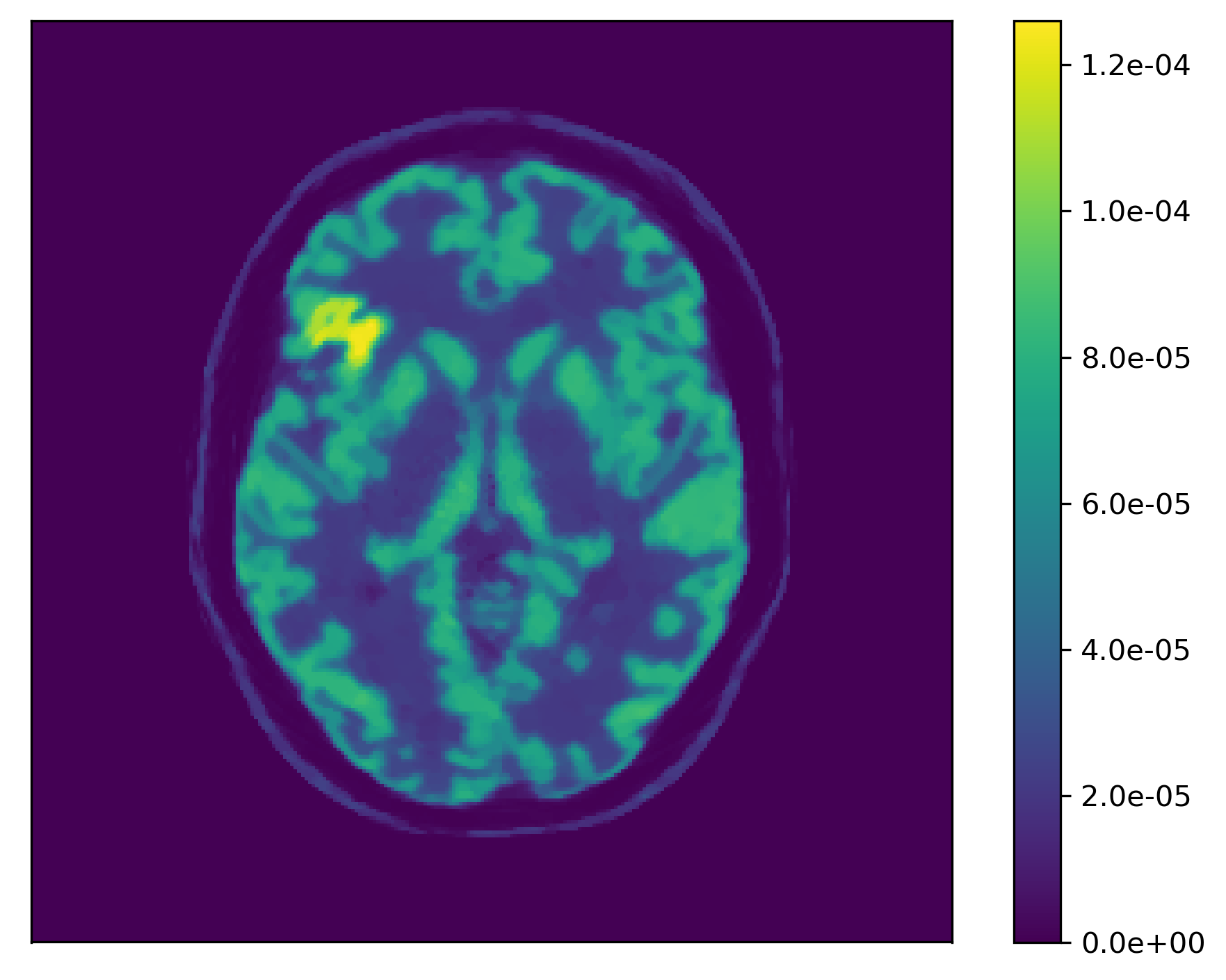} &
				\includegraphics[height=\xsa]{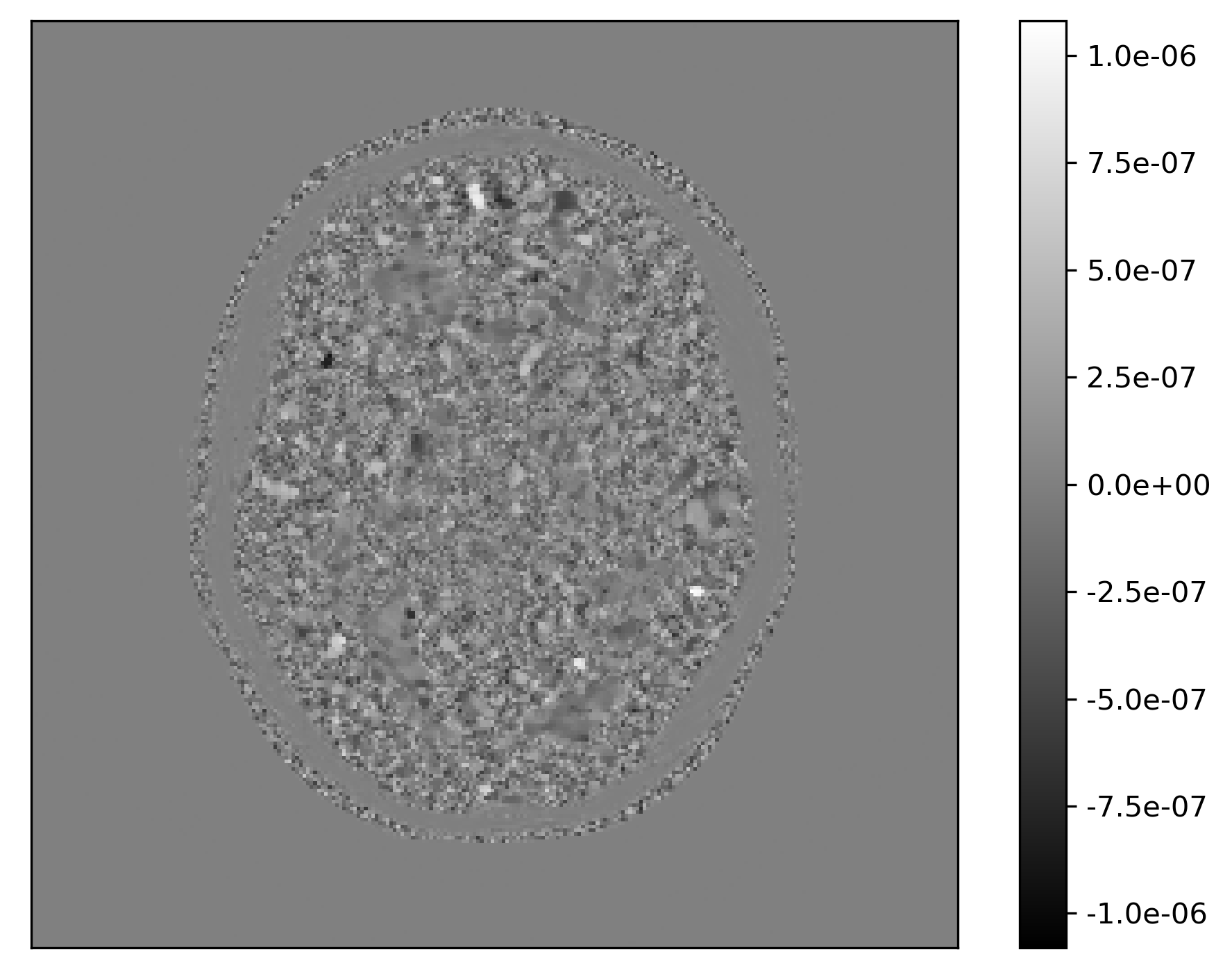} \\
				& NPL mean ($\rho=0$) & MAP & Difference
			\end{tabular}
			\caption{NPL mean without MRI ($\rho=0$) compared to MAP reconstruction for $t_1=1$ and $t_2=100$; $\zeta = 0.05$, $\nu = 0.15$, $\beta^t = 2\times 10^{-3}$; $B = 8192$ draws for NPL}
			\label{fig:npl_vs_map}
		\end{figure}
		In Figure~\ref{fig:npl_vs_map}, we contrast the empirical mean of NPL without MRI and the MAP estimate with same penalty tuning. For increasing  $t$ absolute differences between both images tend to zero (see scales in Figure~\ref{fig:npl_vs_map}) which is coherent with the result of Theorem~\ref{thm:asympt-distr-main} and also supports Conjecture~\ref{conj:asympt-distr-mle-pen-estimator} that MAP is the strongly consistent estimator for which \eqref{eq:thm:asympt-distr-main:strongly-consist-estim-positivity}-\eqref{eq:thm:asympt-distr-main:strongly-consist-estim-i0} hold. From the two simulations for $t_1 = 1$ and $t_2 = 100$ one may observe that the empirical  contraction rate of absolute differences is of order $t^{-1/2}$. This can be explained by the fact that for regular models with $n$ i.i.d observations (recall that model in \eqref{eq:poisson-model-pet} is regular for pixels intersected by LORs from $I_1(\Lambda^*)$), the next term beyond the normal approximation in the first order Edgeworth's expansion of the posterior decays with rate $n^{-1/2}$ (see \citet{pompe2021introducing}) which is equivalent to $t^{-1/2}$ in our case.
		
		\section{Remark on centering term of the posterior}
		\label{app:remark-centering}
		
		\begin{definition}
			We say that $U^t$ converges in conditional distribution to $V$ almost surely $Y^t$, $t\in (0, +\infty)$ if for every Borel set $A\in B(\R^n)$ the following holds:
			\begin{equation}\label{eq:convergence-cond-distr-as-def}
			P(U^t \in A \, \mid \, \mathcal{F}^t) \rightarrow P(V\in A) \text{ when }
			t\rightarrow +\infty, \text{ a.s. } Y^t, t\in (0, +\infty).
			\end{equation}
			This type of convergence will be denoted as follows:
			\begin{equation}
			U^t \xrightarrow{c.d.} U \text{ when } t\rightarrow +\infty, \text{ a.s. } Y^t, t\in (0, +\infty).
			\end{equation}
		\end{definition}
		
		Centering the distribution of $\widetilde{\lambda}_{b}^t$ at the true parameter $\lambda_*$ in (ii) does not allow to achieve conditional tightness almost surely $Y^t$, $t\in (0, +\infty)$ which we briefly explain below.

		As a part of the proof of Theorem~\ref{thm:asympt-distr-main} (see lemmas~\ref{lem:thm:asymp-disr:main-approximation-lemma},~\ref{lem:asymp-distr:quadratic-uv-approximation}) we show that 
		\begin{align}\label{eq:rem:asymp-distr:quadr-approximation}
		\Pi_{\mathcal{U}}(\widetilde{\lambda}_b^t - \widehat{\lambda}_{sc}^t) - u^t(\widetilde{\xi}^t) \xrightarrow{c.p.} 0
		\text{ when } t\rightarrow +\infty, \text{ a.s. } Y^t, t\in (0, +\infty),
		\end{align}
		where
		\begin{align}
		&\widetilde{\xi}^t = (\dots, \sqrt{t}\dfrac{\widetilde{\Lambda}_{b,i}^t - \widehat{\Lambda}_{sc,i}^t}{\sqrt{\widehat{\Lambda}_{sc,i}^t}}, \dots), \, i\in I_1(\Lambda^*), \\
		\begin{split}
		\label{eq:rem:asymp-distr:splitting-candidate-def}
		&u^t(\xi) = \argmin_{\substack{u:(1-\Pi_{\mathcal{V}})\widehat{\lambda}^t_{sc} + \frac{u}{\sqrt{t}}+ w\succeq 0, \\ u\in \mathcal{U}, \, w\in \mathcal{W}}} 
		-u^T (A_{I_1(\Lambda^*)})^T(\widehat{D}^t_{I_1(\Lambda^*)})^{-1/2}\xi + \frac{1}{2}u^T \widehat{F}^t_{I_1(\Lambda^*)}u,\\  
		\end{split}\\
		\label{eq:rem:asymp-distr:splitting-candidate-fisher-def}
		&\widehat{D}^t_{I_1(\Lambda^*)} = \mathrm{diag}(\dots, \widehat{\Lambda}^t_{sc, i}, \dots), \, i\in I_1(\Lambda^*), \\ 
		&\widehat{F}^t_{I_1(\Lambda^*)} = \sum\limits_{i\in I_1(\Lambda^*)}\dfrac{a_ia_i^T}{\widehat{\Lambda}_{sc,i}^t} = (A_{I_1(\Lambda^*)})^T (\widehat{D}^t_{I_1(\Lambda^*)})^{-1}A_{I_1(\Lambda^*)}.
		\end{align}
		That is the conditional tightness (and also the asymptotic distribution) of $\Pi_{\mathcal{U}}(\widetilde{\lambda}_b^t - \widehat{\lambda}_{sc}^t)$ asymptotically coincides with the one of $u^t(\widetilde{\xi}^t)$ being the minimizer of a quadratic function on a polyhedral set depending on $\widehat{\lambda}_{sc}^t$. In the proof we show that conditional tightness of $u^t(\widetilde{\xi}^t)$ is implied by tightness of $\widetilde{\xi}^t$ (this is especially obvious if the constraints in \eqref{eq:rem:asymp-distr:splitting-candidate-def} are not active for large $t$, e.g.,  when $\lambda_* \succ 0$) and that under the assumptions of the theorem it holds that 
		\begin{align}\label{eq:rem:asymp-distr:scale-part-normality}
		\begin{split}
		&(\dots, \sqrt{t}\dfrac{\widetilde{\Lambda}_{b,i}^t - \frac{Y_i^t}{t}}{\sqrt{\widehat{\Lambda}_{sc,i}^t}}, \dots) \xrightarrow{c.d.} \mathcal{N}(0,I) \text{ when } t\rightarrow +\infty, 
		\text{ a.s. } Y^t, t\in (0, +\infty), \\
		&I \text{ -- identity matrix of size } {\# I_1(\Lambda^*)} \times {\# I_1(\Lambda^*)}.
		\end{split}
		\end{align}
		From \eqref{eq:rem:asymp-distr:splitting-candidate-def}-\eqref{eq:rem:asymp-distr:scale-part-normality} and the Prohorov's theorem on tightness of weakly convergent sequences or r.v.s, 
		the asymptotic behavior (tightness, distribution) of $u^t(\widetilde{\xi}^t)$ is essentially depends on the term 
		$(\dots, \sqrt{t}\dfrac{\widehat{\Lambda}_{sc,i}^t - \frac{Y_i^t}{t}}{\sqrt{\widehat{\Lambda}_{sc,i}^t}}, \dots)$, $i\in I_1(\Lambda^*)$. For tightness this term needs to be asymptotically bounded for almost any trajectory $Y^t$, $t\in (0, +\infty)$, which is exactly asked in~\eqref{eq:thm:asympt-distr-main:strongly-consist-estim-i1} (in a slightly weakened form).

		Now, if we center $\widetilde{\lambda}_b^t$ on $\lambda_*$ one finds that $\widehat{\lambda}^t_{sc}$ must be replaced everywhere with $\lambda_*$ in formulas \eqref{eq:rem:asymp-distr:splitting-candidate-def}-\eqref{eq:rem:asymp-distr:scale-part-normality} and, most importantly, the latter term is now  equals  	
		$(\dots, \dfrac{Y_i^t - t\Lambda_i^*}{\sqrt{t\Lambda_i^*}}, \dots)$
		being asymptotically standard normal (see Section~\ref{app:limit-thms} in Appendix).
		Therefore, the mean of the asymptotic distribution of $\sqrt{t} \Pi_{\mathcal{U}}(\widetilde{\lambda}_b^t - \lambda_*)$ depends on the trajectory of $(Y^t_i-t\Lambda^*_i) / \sqrt{t\Lambda^*_i}$, $i\in I_1(\Lambda^*)$, which is almost surely unbounded infinitely often on $t\in (0, +\infty)$ in view of the Law of Iterated Logarithm for~$Y^t$ (see formula \eqref{eq:appendix:lil} in Appendix). So the tightness for $\sqrt{t} \Pi_{\mathcal{U}}(\widetilde{\lambda}_b^t - \lambda_*)$ almost surely for any trajectory $Y^t$, $t\in (0, +\infty)$ is impossible.
		A very similar behavior for centering of the posterior distribution for weighted bootstrap was also observed in Theorem~3.3 from~\citet{ng2020random}.

		\section{MRI data and the mask condition}
		\label{app:sect:MRI-mask-cond}
		
		Below we consider a geometrical interpretation of the non-expansiveness condition based on representation of designs $A$, $A_{\mathcal{M}}$ as weighted Radon transforms over the space of discrete images. We show that failure of this  condition implies presence of a segment in $M\in \mathcal{M}$ which is badly aligned with respect to the convex hull of the tracer support. To avoid such situations in practice, we propose to preprocess MRI images before using them in the context of ET which is explained in the end of this section.
		
		For simplicity, let $k=1$, i.e., MRI data consists of one segmented image $\mathcal{M} = \{M\}$,  and let  
		\begin{align}
		\label{eq:non-expansive:mask:proj-radon-1}
		\Gamma &= \{\gamma_i\}_{i=1}^{d} \text{ be the set of rays available in the acquisition geometry}.
		\end{align}
		
		Assume that $A = (a_{ij})$ is a discretized version of some weighted Radon transform  on set of rays $\Gamma$ with positive weight $W$. That is
		\begin{align}
		\label{eq:non-expansive:mask:proj-radon-2}
		a_{ij} &= \int_{\gamma_i} W(x, \gamma_i)\, \mathds{1}_{j}(x)\, dx, \, \gamma_i \in \Gamma, \, j\in \{1, \dots, p\}, \\
		\label{eq:non-expansive:mask:proj-radon-w}
		W &= W(x,\gamma), \, (x,\gamma) \in \R^2 \times T\Sp^1, \, 0 < c \leq W \leq C,
		\end{align}
		where $dx$ denotes the standard Lebesgue measure on ray $\gamma_i$,  $\mathds{1}_{j}(x)$ is the indicator function of pixel $j$ on the image. Weight $W(x,\gamma)$ is some known sufficiently regular function of spatial coordinates and oriented rays in $\R^2$ which are parameterized by $T\Sp^1$ (tangent bundle of the unit sphere, see e.g., \citet{natterer2001mathematics}). Projectors defined by the formulas of type \eqref{eq:non-expansive:mask:proj-radon-2}, \eqref{eq:non-expansive:mask:proj-radon-w} are common in CT and ET practice; see e.g., \citet{siddon1985fastcalculation}, \citet{han1999fastraytracing}. For example, in PET and SPECT weight $W$ is used to model  attenuation and nonuniform sensitivity of detectors; see e.g.,~\citet{quinto1983rotinv},~\citet{novikov2019nonabelian},~\citet{goncharov2019weighted}.

		From \eqref{eq:wbb-mri:npl-presentation:reduced-design-elem}, \eqref{eq:non-expansive:mask:proj-radon-2} it follows that 
		\begin{align}
		\label{eq:non-expansive:mask:proj-radon-3}
		A_{\mathcal{M}} = (a_{M, is}), \, a_{M,is} &= \int_{\gamma_i}W(x,\gamma_i)\mathds{1}_{M,s}(x)\, dx, \, s\in S(M),
		\end{align}
		where $\mathds{1}_{M,s}(x)$ is the indicator function of segment $s$ in image $M$.
		
		Recall that $\lambda_*\in \R^p_+$ is the discretized version of the real spatial distribution of the tracer and assume that $\lambda_*\in \R^p_+$ is pixel-wise connected (i.e., between two  arbitrary pixels with positive tracer uptake there is a path of pixels preserving the positivity of the signal; two pixels are neighbors if they share an edge (see Figure~\ref{fig:pixel-wise-path}(a))). This assumption is natural, for example, in the context of brain imaging when the tracer is distributed in the whole volume inside the cranium and only relative spatial variations are of practical interest. 

		\begin{figure}[H]
			\centering
			
			\subcaptionbox{}{
				\begin{tikzpicture}[scale=0.15]
				
				\begin{scope}
				\draw[fill, red!30] (8, 2) rectangle (10, 4);
				\draw[fill, red!30] (12, 14) rectangle (14, 16);
				\draw[fill, blue!30] (8, 4) rectangle (10, 16);
				\draw[fill, blue!30] (10, 14) rectangle (12, 16);
				\draw[step=2.0] (0,0) grid (18,18);

				\draw [-{Latex[length=1.5mm]}] (5, 20) -- (20, 5) node[right] { $\gamma_i$};
				\foreach \y in {0, 2, ..., 22}
				{
					\fill[red] (7+\y/2, 18-\y/2) circle (4pt);
					
				};
				
				\node at (21,16) (n1) {$\lambda_{*j}$};
				\path[-{Latex[length=1.5mm]}] (n1) edge [out=-90, in=20] (13, 15);
				
				\node at (21,3) (n2) {$\lambda_{*j'}$};
				\path[-{Latex[length=1.5mm]}] (n2) edge [out=180, in=0] (9, 3);
				\end{scope}
				
				\end{tikzpicture}
			} \hspace{1em}%
			\subcaptionbox{$\mathrm{DH}(\lambda_*; \gamma_i)$}{
				
				\begin{tikzpicture}[scale=.15]
				
				\begin{scope}
				
				\draw[fill, blue!30] (0, 0) rectangle (18, 10);
				\draw[fill, blue!30] (0, 8) rectangle (16, 12);
				\draw[fill, blue!30] (0, 12) rectangle (14, 14);
				\draw[fill, blue!30] (0, 12) rectangle (14, 14);
				\draw[fill, blue!30] (0, 14) rectangle (12, 16);
				\draw[fill, blue!30] (0, 16) rectangle (10, 18);
				
				\draw[fill, green!30, xshift=-80] plot [smooth cycle] coordinates {(8, 6) (14, 6) (15, 9) (9, 14) (5, 10) (6, 4)}; 
				\node at (-4, 9) (n) {$\lambda_*$};
				\path[-{Latex[length=1.5mm]}] (n) edge [out=0, in=180] (9, 9);

				\draw[step=2.0] (0,0) grid (18,18);

				\draw [-{Latex[length=1.5mm]}] (5, 20) -- (20, 5) node[right] { $\gamma_i$};
				\foreach \y in {0, 2, ..., 22}
				{
					\fill[red] (7+\y/2, 18-\y/2) circle (4pt);
					
				};
				
				
				\end{scope}

				\end{tikzpicture}
			} 
			\hspace{1em}%
			\subcaptionbox{$\mathrm{D\mathring{H}}(\lambda_*; \gamma_i)$}{
				
				\begin{tikzpicture}[scale=.15]
				
				\begin{scope}
				
				\draw[fill, blue!30] (0, 0) rectangle (18, 6);
				\draw[fill, blue!30] (0, 6) rectangle (16, 8);
				\draw[fill, blue!30] (0, 8) rectangle (14, 10);
				\draw[fill, blue!30] (0, 10) rectangle (12, 12);
				\draw[fill, blue!30] (0, 12) rectangle (10, 14);
				\draw[fill, blue!30] (0, 14) rectangle (8, 16);
				\draw[fill, blue!30] (0, 16) rectangle (6, 18);
				
				\draw[fill, green!30, xshift=-80] plot [smooth cycle] coordinates {(8, 6) (14, 6) (15, 9) (9, 14) (5, 10) (6, 4)}; 
				\node at (-4, 9) (n) {$\lambda_*$};
				\path[-{Latex[length=1.5mm]}] (n) edge [out=0, in=180] (9, 9);

				\draw[step=2.0] (0,0) grid (18,18);

				\draw [-{Latex[length=1.5mm]}] (5, 20) -- (20, 5) node[right] { $\gamma_i$};
				\foreach \y in {0, 2, ..., 22}
				{
					\fill[red] (7+\y/2, 18-\y/2) circle (4pt);
					
				};
				
				
				\end{scope}

				\end{tikzpicture}
			} 
			
			\caption{}
			\label{fig:pixel-wise-path}
		\end{figure}

		\begin{definition}
			\label{def:discrete-convex-hull}
			Let $\Gamma$ be the finite family of oriented rays in $\R^2$, $A$ be the projector defined by formulas \eqref{eq:non-expansive:mask:proj-radon-2}, \eqref{eq:non-expansive:mask:proj-radon-w}, $\lambda_*\in \R^p_+$, $\lambda_* \neq 0$ and $\lambda_*$ is pixel-wise connected. Consider  $\gamma_i\in \Gamma$ and assume that $i\in I_0(A\lambda_*)$. Then, support of $\lambda_*$ lies completely in one of the closed half-spaces in $\R^2$ separated from each other with ray $\gamma_i$. Let $H(\lambda_*, \gamma_i)$ be such a closed half-space.
			Consider the discrete version of $H(\lambda_*, \gamma_i)$ defined by the formula
			\begin{align}
			\begin{split}
			\mathrm{DH}(\lambda_*; \gamma_i) = \{j \in \{1,\dots, p\} \mid \, &\text{intersection between pixel }j \text{ and } H(\lambda_*, \gamma_i)\\ 
			&\text{is of non-zero Lebesgue measure on }\R^2 \}.
			\end{split}
			\end{align}
			Consider 
			\begin{align}
			\begin{split}
			\mathrm{D\mathring{H}}(\lambda_*; \gamma_i) = \{j \in \mathrm{DH}(\lambda_*, \gamma_i) \mid \, &\text{intersection between pixel }j \text{ and ray }  \gamma_i\\ 
			&\text{is of length zero}\}.
			\end{split}
			\end{align}
			Discrete convex hull of $\lambda_*$ for family $\Gamma$ is defined by the formula
			\begin{equation}
			\mathrm{DConv}(\lambda_*; \Gamma, A\lambda_*) = \bigcap_{\substack{\gamma_i\in \Gamma, \\ i\in I_0(\Lambda_*)}}
			D\mathring{H}(\lambda_*; \gamma_i).
			\end{equation}
			\qed
		\end{definition}
		For the geometrical intuition behind definitions $\mathrm{DH}(\cdot)$, $\mathrm{D\mathring{H}}(\cdot)$, $\mathrm{DConv}(\cdot)$, see examples (b), (c) in  Figure~\ref{fig:pixel-wise-path}.
		
		Now assume that the non-expansiveness condition fails in the following sense:
		\begin{equation}\label{eq:non-exp-cond-fail}
		\text{there exists } i\in I_0(\Lambda^*) \text{ such that } \Lambda^*_{\mathcal{M},i} > 0,
		\end{equation}
		where $\Lambda^*_{\mathcal{M}}$ is defined in \eqref{eq:assump:theory:asymp-distr:non-exp-cond:ind-cond}.
		From \eqref{eq:non-expansive:mask:proj-radon-1}-\eqref{eq:non-expansive:mask:proj-radon-3} and Definition~\ref{def:discrete-convex-hull} it follows that in the image for $\lambda_{\mathcal{M}, *}$ there is a segment $s\in S(M)$ which intersected by $\gamma_i\in \Gamma$ and such that~$\lambda_{\mathcal{M}, *, s}  > 0$ (see Figure~\ref{fig:segmentation-cond-fig}(a)), that is 
		\begin{equation}\label{eq:mask-cond:failed-alignment-segment}
		\bigcup_{\substack{M\in \mathcal{M}, \\ s\in S(M), \\ \lambda_{\mathcal{M},*,s} > 0}} \hspace{-0.3cm}s \not\subset \mathrm{DConv}(\lambda_*; \Gamma, \Lambda^*).
		\end{equation}
		\begin{figure}[H]
			\centering
			
			\subcaptionbox{$\Lambda_i^* = 0, \Lambda^*_{\mathcal{M},i} > 0$}{
				\begin{tikzpicture}[scale=0.17]
				
				\begin{scope}

				\draw[very thick, fill, red!30, xshift=-80, opacity=0.7] plot [smooth cycle] coordinates {(9, 12) (16, 14) (12, 16) (7, 15) (7, 14)};
				
				\draw[very thick, fill, green!30, xshift=-80, opacity=0.7] plot [smooth cycle] coordinates {(8, 6) (14, 6) (15, 9) (9, 14) (5, 10) (6, 4)}; 
				
				\node at (-2, 9) (n) {$\lambda_*$};
				\path[-{Latex[length=1.5mm]}] (n) edge [out=0, in=180] (7, 9);
				
				\node at (-10, 14) (ns) {$s\in S(M), \, \lambda_{\mathcal{M},*,s} > 0$};
				\path[-{Latex[length=1.5mm]}] (ns) edge [out=0, in=180] (7, 15);
				
				\draw [-{Latex[length=1.5mm]}] (5, 20) -- (20, 5) node[right] { $\gamma_i$};
				
				\end{scope}

				\end{tikzpicture}
			} \hspace{1em}%
			\subcaptionbox{$\Lambda^*_{i} > 0, \, \Lambda^*_{\mathcal{M}, i} = 0$}{
				\label{fig:line-a}
				\begin{tikzpicture}[scale=.17]
				
				\begin{scope}
				
				\draw[very thick, fill, green!30, xshift=-80, opacity=0.7] plot [smooth cycle] coordinates {(8, 6) (14, 6) (15, 9) (9, 14) (5, 10) (6, 4)}; 		
				
				\draw[very thick, fill, red!30, xshift=-80, yshift=-80, opacity=0.7] plot [smooth cycle] coordinates {(9, 12) (10, 14) (7, 15) (7, 14)};
				
				
				\node at (-2, 8) (n) {$\lambda_*$};
				\path[-{Latex[length=1.5mm]}] (n) edge [out=0, in=180] (7, 8);
				
				\node at (-3, 16) (ns) {$\bigcup\limits_{{\substack{M\in \mathcal{M}, \\ s\in S(M)}}}\hspace{-0.3cm}s$};
				\path[-{Latex[length=1.5mm]}] (ns) edge [out=0, in=90] (6, 11);

				\draw [-{Latex[length=1.5mm]}] (10, 18) -- (10, 3) node[right] { $\gamma_i$};
				
				\end{scope}

				\end{tikzpicture}
			}
			\caption{}
			\label{fig:segmentation-cond-fig}
		\end{figure}

		If we assume that $\lambda_{\mathcal{M},*}$ is also pixel-wise connected, then from \eqref{eq:mask-cond:failed-alignment-segment} it follows that 
		\begin{equation}
		\label{eq:mask-cond:failed-alignment}
		\mathrm{DConv}(\lambda_{\mathcal{M},*}; \Gamma, A_{\mathcal{M}}\lambda_{\mathcal{M},*}) \not\subset \mathrm{DConv}(\lambda_*; \Gamma, A\lambda_*).
		\end{equation}
		To conclude, we have just demonstrated the following statement.
		
		\begin{prop}\label{prop:misaglingment-mri-pet}
			Let $\lambda_*\in \R^p_+$, $\lambda_*\neq 0$,  $\lambda_*$ is pixel-wise connected and designs $A$, $A_{\mathcal{M}}$ be of type \eqref{eq:non-expansive:mask:proj-radon-1}-\eqref{eq:non-expansive:mask:proj-radon-3}. Let $\lambda_{\mathcal{M},*}$ be a solution of the minimization problem in \eqref{eq:assump:theory:asymp-distr:non-exp-cond} and  $\lambda_{\mathcal{M}, *}$ be also pixel-wise connected.
			Assume that the non-expansiveness condition (Assumption~\ref{assump:theory:distribution:non-expansive}) fails in the sense of \eqref{eq:non-exp-cond-fail}. Then, formula \eqref{eq:mask-cond:failed-alignment} holds.
		\end{prop}

		To avoid the situation in Proposition~\ref{prop:misaglingment-mri-pet} one may propose to use a significantly smaller segmentation area, for example, such that 
		\begin{equation}
		\bigcup_{\substack{M\in \mathcal{M}, \\ s\in S(M)}}\hspace{-0.3cm}s \subsetneq \mathrm{DConv}(\lambda_*; \Gamma, \Lambda^*),
		\end{equation}
		where $A\subsetneq B$ denotes the strict inclusion of sets. 
		In this case even a small misalignment may lead to a situation when $\mathcal{KL}(P^t_{A,\lambda_*}, P^t_{A_{\mathcal{M}, \lambda_{\mathcal{M}}}}) = +\infty$, so the KL-projection of $P^t_{A,\lambda_*}$ onto MRI-based model $P^t_{A_{\mathcal{M}}, \lambda_{\mathcal{M}}}$ is impossible; see Figure~\ref{fig:segmentation-cond-fig}(b). 
		
		In view of the latter an ideal choice for $S(M)$ would be such that 
		\begin{equation}\label{eq:good-segmentation-rule}
		\mathrm{DConv}(\lambda_{\mathcal{M},*}; \Gamma, A_{\mathcal{M}}\lambda_{\mathcal{M},*}) = \mathrm{DConv}(\lambda_*; \Gamma, A\lambda_*).
		\end{equation}
		The above arguments are can be easily extended to the case of $k > 1$ by simply checking the alignments for all images in $\mathcal{M}$. \\
		
		We conclude with a proposition to use the  following pipeline for preprocessing anatomical MRI-images:
		
		\begin{enumerate}
			\item Estimate $\mathrm{DConv}(\lambda_*; \Gamma, A\lambda^*)$ using any well-suited and fast algorithm. Let $D$ be such an estimate.
			\item In all MRI-images remove pixels lying outside of $D$ and perform segmentations only on those which are left inside of $D$.
		\end{enumerate}
		In view of step 2 we propose an alternative name for Assumption~\ref{assump:theory:distribution:non-expansive} -- the mask condition. The term `mask' is used in practical considerations of ET to denoted restrictions of support of the tracer (e.g., due to medical expertise), so the above procedure theoretically reflects well existing empirical practices.
		
		\section{Proofs}
		
		\subsection{Proof of Lemma~\ref{lem:consistency:lem-kernel-continuity}}
		\begin{proof}
			Proof is based on the two following lemmas.
			\begin{lem}\label{lem:proofs:compactness-domain-penalty}
				Let $\lambda\in \R^p_+$ and $A$ satisfies \eqref{eq:design-matrix-positivity-restr-1}, \eqref{eq:design-matrix-positivity-restr-2}. Then, for any compact $U\subset \mathrm{Span}(A^T)$ it holds that 
				
				\begin{equation}\label{eq:proofs:lem-geometric-positive-cone-statement}
				S_{A, \lambda}(U)=(\lambda + U + \ker A)\cap \R^p_+ \text{ is  convex and compact},
				\end{equation}
				where the summation sign denotes the Minkowski sum 
				\begin{equation*}
				A + B = \{w = u + v \subset \R^p : u\in A, \, v\in B\}, \, A \subset \R^p, \, B\subset \R^p.
				\end{equation*} 
			\end{lem}
			\begin{lem}\label{lem:proof:hausdorff-boundeness}
				Let assumptions of Lemma~\ref{lem:proofs:compactness-domain-penalty} be satisfied and $d_H(A, B)$ denote the Hausdorff distance between compact sets $A, B\subset \R^p$ being defined by the formula
				\begin{equation*}
				d_H(A,B) = \max \left(
				\sup_{x\in A} \inf_{y\in B} \|x-y\|, \, 
				\sup_{x\in B} \inf_{y\in A} \|x-y\|
				\right).
				\end{equation*}
				Let $U\subset \mathrm{Span}(A^T)$ be a compact such that $S_{A,\lambda}(U)\neq \emptyset$. Then, 
				\begin{equation}\label{eq:lem:hausdorff-boundeness:continuity}
				d_H(S_{A,\lambda}(\{u_0\}), S_{A,\lambda}(\{u\})) \rightarrow 0
				\text{ for }u\rightarrow u_0, \, u, u_0\in U,
				\end{equation}
				where $S_{A,\lambda}(\cdot )$ is defined in  \eqref{eq:proofs:lem-geometric-positive-cone-statement}.
			\end{lem}
			
			From the result of Lemma~\ref{lem:proofs:compactness-domain-penalty} and the assumption in \eqref{eq:prelim:penalty-cond-strict-conv} 
			it follows that for each $u\in U$ the following problem 
			\begin{align}
			\begin{split}\label{eq:proof:lem:consistency:lem-kernel-continuity-opt-pr}
			&\text{minimize } \varphi(\lambda + u + w) \text{ w.r.t } w, \\
			&\text{subject to: } \lambda + u + w\succeq 0, \, w\in \ker A.
			\end{split}
			\end{align}
			admits a unique solution $w(u)\in \ker A$. Indeed, the minimized function in  \eqref{eq:proof:lem:consistency:lem-kernel-continuity-opt-pr} 
			is strictly convex function in $w$ and the domain is compact and convex. This proves the first assertion of the lemma. 
			
			Now, we prove the continuity of $w(u)$ on its domain. Let $u_k$ be a sequence in $U$ such that $u_k\rightarrow u_0$ for some $u_0\in U$. Let $w_k = w(u_k)$, where the latter are minimizers in \eqref{eq:proof:lem:consistency:lem-kernel-continuity-opt-pr} for $u=u_k$, and $w_0 = w(u_0)$.
			We know that $\lambda_k = \lambda + u_k + w(u_k)\in S_{A,\lambda}(U)$, where the latter is a compact (by Lemma~\ref{lem:proofs:compactness-domain-penalty}). Since continuous mapping of a compact is again a compact, all $w_k$ belong to some compact $W_{A,\lambda}(U)$ being the orthogonal projection of $(S_{A,\lambda}(U)-\lambda)$ onto $\ker A$. 
			From compactness of $W_{A,\lambda}(U)$ it follows that $w_k$ contains a converging subsequence $w_m \rightarrow w_0$, $w_0\in W_{A,\lambda}(U)$, where $w_m = w(u_m)$, $m\in \mathbb{N}$.
			
			Since $w_m$ are the minimizers in \eqref{eq:proof:lem:consistency:lem-kernel-continuity-opt-pr}, 
			we know that 
			\begin{align}\label{eq:proof:lem:consistency:lem-kernel-main-ineq}
			\begin{split}
			&\varphi(\lambda + u_m + w_m) \leq \varphi(\lambda + u_m + w),\\
			&\text{for all }w\in \ker A, \text{ such that } \lambda + u_m + w \succeq 0.
			\end{split}
			\end{align}
			Taking the limit $m\rightarrow +\infty$, $u_m\rightarrow u_0$, $w_m \rightarrow w_0$ we aim to show that 
			\begin{align}\label{eq:proof:lem:consistency:lem-kernel-aim-continuity}
			\begin{split}
			&\varphi(\lambda + u_0 + w_0) \leq \varphi(\lambda + u_0 + w), \\
			&\text{for all }w\in \ker A, \text{ such that } \lambda + u_0 + w\succeq 0.
			\end{split}
			\end{align}
			Therefore, $w_0 = w(u_0)$ which is unique (by the strict convexity of $\varphi$ along $\ker A$) and proves the continuity of $w(u)$. The fact that any sequence has a convergent subsequence having the same limit $w(u_0)$ implies that $w_k = w(u_k)$ also converges to $w(u_0)$.
			However, taking the limit $m\rightarrow +\infty$ for each $w$ in \eqref{eq:proof:lem:consistency:lem-kernel-main-ineq} may not preserve the positivity constraint. To show  \eqref{eq:proof:lem:consistency:lem-kernel-aim-continuity}, for each $w$ satisfying the positivity constraint in \eqref{eq:proof:lem:consistency:lem-kernel-aim-continuity} we find another sequence $\{w_m'\}$ such that 
			\begin{equation}\label{eq:proof:lem:consistency:lem-kernel-auxiliary-seq-w}
			\lambda + u_m + w'_m\succeq 0, \, w_m' \rightarrow w
			\text{ for } m\rightarrow +\infty.
			\end{equation}
			In this case we can replace $w$ with $w_{m}'$ in \eqref{eq:proof:lem:consistency:lem-kernel-main-ineq} and take the limit $m\rightarrow\infty$ in order to obtain  \eqref{eq:proof:lem:consistency:lem-kernel-aim-continuity}.
			
			Now, it is left how to choose $w_m'$ so that \eqref{eq:proof:lem:consistency:lem-kernel-auxiliary-seq-w} holds. We choose $w_m'$ to be the solution in the following minimization problem
			\begin{align}\label{eq:proof:lem:consistency:w-projection}
			\begin{split}
			&\text{minimize } \|(\lambda + u_0 + w) - (\lambda + u_m + w_{m}')\| \text{ with respect to } w'_m,\\
			&\text{subject to: } w'_m\in \ker A, \, \lambda + u_m + w_m' \succeq 0.
			\end{split}
			\end{align}
			Solution $w_m'$ in \eqref{eq:proof:lem:consistency:w-projection} always exists and unique since it corresponds to the euclidean projection of $\lambda + u_0 + w$ onto convex set $S_{A,\lambda}(\{u_m\})$, that is 
			\begin{align}\label{eq:proof:lem:consistency:w-projection-euclidean}
			w_m' = \Pi_{\ker A} [\mathrm{Proj}(\lambda + u_0 + w, S_{A,\lambda}(\{u_m\}))-\lambda],
			\end{align}
			where $\Pi_{\ker A}$ is the orthogonal projector onto $\ker A$, $\mathrm{Proj}(x, X)$ denotes the euclidean projection of point $x$ onto $X$. From \eqref{eq:proof:lem:consistency:w-projection-euclidean} and the fact that $\lambda + u_0 + w\in S_{A,\lambda}(\{u_0\})$ it follows that 
			\begin{equation}\label{eq:proof:lem:consistency:w-projection-euclidean-diff}
			w_m' -w = \Pi_{\ker A}[\mathrm{Proj}(\lambda + u_0 + w, S_{A,\lambda}(\{u_m\})) - \mathrm{Proj}(\lambda + u_0 + w, S_{A,\lambda}(\{u_0\}))].
			\end{equation}		
			Using \eqref{eq:proof:lem:consistency:w-projection-euclidean-diff} and Proposition~5.3 from \citet{attouch1993quantitative} one can write the following estimate:
			\begin{align}\label{eq:proof:lem:consistency:w-hausdorff-bound}
			\|w_m' - w\| \leq \rho^{1/2}_m  d_{H,\rho_m}(S_{A,\lambda}(\{u_0\}), S_{A,\lambda}(\{u_m\}))^{1/2},
			\end{align}
			where $\rho_m = \|\lambda + u_0 + w\| + d(\lambda + u_0 + w, S_{A,\lambda}(\{u_m\}))$ ($d(x, y)$ denotes the standard euclidean distance between $x$, $y$, $d(x, X) = \inf_{x'\in X}d(x,x')$), $d_{H,\rho}(\cdot, \cdot)$ is the bounded Hausdorff distance (see the definition in Section~3 of  \citet{attouch1993quantitative}). In particular, for $d_{H,\rho}$ the following bound holds: 
			\begin{align}\label{eq:proof:lem:consistency:rw-hausdorff-bound}
			d_{H,\rho}(A,B) \leq d_{H}(A,B), 
			\end{align}
			for any sets $A$, $B$.
			
			First, note that $\sup_{m}\rho_m$ is finite. Indeed, this follows from the fact that $u_m\rightarrow u_0$ (hence $\{u_m\}$ is bounded) and following estimates:
			\begin{align}\label{eq:proof:lem:consistency:rw-hausdorff-triangle-ineq}
			\begin{split}
			&d(\lambda + u_0 +w, S_{A,\lambda}(\{u_m\})) \leq d(\lambda + u_0 + w, 0) + d(0, S_{A,\lambda}(\{u_m\}))\\		
			& \hspace{4.4cm} \leq \|\lambda + u_0 + w\| + d(0, S_{A,\lambda}(\{u_m\})), 
			\end{split}\\
			\label{eq:proof:lem:consistency:rw-hausdorff-simplex-estimate}
			&d(0, S_{A,\lambda_*}(\{u_m\})) \leq \max_{j\in \{1, \dots, p\}}
			\left(\sum\limits_{i=1}^d a_i^T(\lambda + u_m)\right) / A_j, \, 
			A_j = \sum\limits_{i=1}^{d}a_{ij}.
			\end{align}
			Formula \eqref{eq:proof:lem:consistency:rw-hausdorff-triangle-ineq} is a simple triangle inequality and the estimate in \eqref{eq:proof:lem:consistency:rw-hausdorff-simplex-estimate} follows from the fact that $S_{A,\lambda}(\{u\})$ is the affine subset of $(p-1)$ -- simplex defined by the formula
			\begin{align}
			\Delta_{A, \lambda}^p(u) = \{\lambda'\in \R^p_+ : \sum\limits_{j=1}^{p}\lambda_j' A_j = \sum\limits_{i=1}^{d}a_i^T(\lambda+ u) \geq 0\}, \, A_j = \sum\limits_{i=1}^{d}a_{ij} > 0.
			\end{align}
			So the inequality in \eqref{eq:proof:lem:consistency:rw-hausdorff-simplex-estimate} express the fact that the furtherst point from the origin to $\Delta_{A,\lambda}^p$ is one of its vertices. 
			From \eqref{eq:proof:lem:consistency:w-hausdorff-bound}, \eqref{eq:proof:lem:consistency:rw-hausdorff-bound} the fact that $\sup_m \rho_m < +\infty$ and the result of 
			Lemma~\ref{lem:proof:hausdorff-boundeness} it follows that $w'_m \rightarrow w$, where $\lambda + u_m + w_m \succeq 0$.  Therefore,  conditions in \eqref{eq:proof:lem:consistency:lem-kernel-auxiliary-seq-w} are satisfied which, in turn, proves \eqref{eq:proof:lem:consistency:lem-kernel-aim-continuity} and the second claim of the lemma.
			
			Lemma is proved.
		\end{proof}
		
		\subsection{Proof of Lemma~\ref{lem:proofs:compactness-domain-penalty}}
		\begin{proof}
			Closedness and convexity of $S_{A,\lambda}(U)$ follow directly from the fact that $(\lambda + U + \ker A)$, $\R^p_+$ are both closed and convex whereas their intersection preserves these properties.
			
			We prove boundedness of $S_{A, \lambda}(U)$ by the contradiction argument. 
			
			Assume that $S_{A,\lambda}(U)$ is not bounded, then there exists a sequence $\{(u_k, w_k)\}_{k=1}^{\infty}$, $u_k\in U$, $w_k\in \ker A$, such that 
			\begin{equation}\label{eq:cosnsitency:proof:lem:kernel-continuity:infty-lambda}
			\lambda + u_k + w_k \in \R^p_+,  \, \|\lambda + u_k + w_k\| \rightarrow \infty.
			\end{equation}
			From \eqref{eq:cosnsitency:proof:lem:kernel-continuity:infty-lambda} and compactness of $U$ it follows, in particular, that 
			\begin{equation}
			w_k \text{ in } \ker A, \|w_k\| \rightarrow +\infty.
			\end{equation}
			Also there exists a converging subsequence $\{u_{k_n}\}_{n=1}^{\infty}$ such that 
			\begin{equation}
			u_{k_n} \rightarrow u_0 \in U \text{ for some $u_0$, as } n\rightarrow +\infty.
			\end{equation}
			Consider the corresponding subsequence $\{w_{k_n}\}_{n=1}^{\infty}$ for which we know that 
			\begin{equation}
			w_{k_n}\in \ker A, \, \|w_{k_n}\| \rightarrow +\infty \text{ for }n\rightarrow +\infty.
			\end{equation}
			Let 
			\begin{equation}
			\theta_{n} = \dfrac{w_{k_n}}{\|w_{k_n}\|}, \, \theta_n \in \Sp^{p-1}\cap \ker A.
			\end{equation}
			Since $\Sp^{p-1}\cap \ker A$ is compact, $\{\theta_n\}_{n=1}^{\infty}$ has a converging subsequence $\{\theta_m\}_{m=1}^{\infty}$ such that 
			\begin{equation}\label{eq:cosnsitency:proof:lem:kernel-continuity:theta-zero}
			\theta_m \rightarrow \theta_0, \, \theta_0 \in \Sp^{p-1}\cap \ker A.
			\end{equation}
			Let $\{u_m\}_{m=1}^{\infty}$ be the corresponding subsequence of $\{u_{k_n}\}_{n=1}^{\infty}$ for index $m$ in formula \eqref{eq:cosnsitency:proof:lem:kernel-continuity:theta-zero}.
			From \eqref{eq:cosnsitency:proof:lem:kernel-continuity:infty-lambda}-\eqref{eq:cosnsitency:proof:lem:kernel-continuity:theta-zero} it follows that we have constructed a sequence $\{(u_m, w_m)\}_{m=1}^{\infty}$ such that 
			\begin{align}\label{eq:cosnsitency:proof:lem:kernel-continuity:um-wm-domain}
			&\lambda + u_m + w_m \in \R^p_+, u_m \in U, \, w_m \in \ker A, \\
			\label{eq:cosnsitency:proof:lem:kernel-continuity:w_m-infty}
			&u_m \rightarrow u_0, \, \|w_{m}\| \rightarrow +\infty, \\
			\label{eq:cosnsitency:proof:lem:kernel-continuity:theta-zero-limit-props}
			&\theta_m = \dfrac{w_m}{\|w_m\|} \rightarrow \theta_0\in \Sp^{p-1}\cap  \ker A.
			\end{align}
			Now we show that under our initial assumption we arrive to the fact that 
			\begin{equation}\label{eq:cosnsitency:proof:lem:kernel-continuity:contradiction-argument}
			\lambda + s \theta_0 \in \R^p_+ 
			\text{ for any } s > 0, 
			\end{equation}
			where $\theta_0$ is defined in \eqref{eq:cosnsitency:proof:lem:kernel-continuity:theta-zero-limit-props}.
			
			Indeed, from the fact that  $\lambda \in \R^p_+$ and that $\R^p_+$ is convex it follows that 
			\begin{equation}\label{eq:cosnsitency:proof:lem:kernel-continuity:positivity-ray}
			\lambda + t(u_m + w_m) = \lambda + t(u_m + \|w_m\|\theta_m) \in \R^p_+ \text{ for any } t\in [0,1].
			\end{equation}
			
			Let $s > 0$. By choosing $t = t_m(s) = s/ \|w_m\|$ in  \eqref{eq:cosnsitency:proof:lem:kernel-continuity:positivity-ray} ($t_m(s)\in [0,1]$ for large $m$; see \eqref{eq:cosnsitency:proof:lem:kernel-continuity:w_m-infty}) and using formulas \eqref{eq:cosnsitency:proof:lem:kernel-continuity:um-wm-domain}-\eqref{eq:cosnsitency:proof:lem:kernel-continuity:theta-zero-limit-props} we obtain
			\begin{align}\label{eq:cosnsitency:proof:lem:kernel-continuity:s-limiting-point}
			\begin{split}
			(\lambda + s\theta_0)& - (\lambda + t_m(s)u_m + t_m(s)\|w_m\|\theta_m)\\
			&=s(\theta_0 - \theta_m) - s \dfrac{u_m}{\|w_m\|} \rightarrow 0 
			\text{ for } m\rightarrow +\infty.
			\end{split}
			\end{align}
			From \eqref{eq:cosnsitency:proof:lem:kernel-continuity:s-limiting-point} it follows that $\lambda + s\theta_0$ is a limiting point in $\R^p_+$, and due to its closedness it follows that $\lambda + s\theta_0 \in \R^p_+$, $s \geq 0$.

			The statement in \eqref{eq:cosnsitency:proof:lem:kernel-continuity:contradiction-argument} cannot hold, because from \eqref{eq:design-matrix-positivity-restr-3} it follows that 
			\begin{equation}
			\text{ for any } \theta\in \ker A, \, \theta\neq 0 \, \exists j\in \{1, \dots, p\} 
			\text{ s.t. } \theta_j < 0.
			\end{equation}
			Since $\theta_0\in \ker A$, by taking $s > 0$ large enough in formula \eqref{eq:cosnsitency:proof:lem:kernel-continuity:contradiction-argument}, we will arrive to the case when $\lambda + s\theta_0\not \in \R^p_+$, which gives the desired contradiction.

			Lemma is proved.
		\end{proof}
		
		\subsection{Proof of Lemma~\ref{lem:proof:hausdorff-boundeness}}
		\begin{proof}

			The claim of the lemma makes part of Theorem~1 from \citet{walkup1969lipschitzian} which, informally says that a closed convex set $K\subset \R^p$ is a polyhedra iff the Hausdorff distance on the space sections by any family of parallel linear subspaces is Lipschitz continuous with respect to the shift vector.

			Using notations from \citet{walkup1969lipschitzian} we define the following  affine mapping 
			\begin{align}
			\tau_{A,\lambda}(u) = A\lambda + Au, \, u\in \R^p, 
			\end{align}		
			where $\lambda$ is a parameter, $A\in \mathrm{Mat}(d,p)$ is the design matrix satisfying \eqref{eq:design-matrix-positivity-restr-1}, \eqref{eq:design-matrix-positivity-restr-2}. 
			
			Let $K = \R^p_+$ which is obviously a polyhedra in $\R^p$. Next, we define family of sections of $K$ by the formula
			\begin{equation}
			k(\Lambda) = \tau_{A,\lambda}^{-1}(\Lambda) \cap K, \, \Lambda \in \R^d.
			\end{equation}
			Essentially, $k(\Lambda)$ is an section of $K$ by $\ker A$ which is  shifted by vector $u$ (in some cases $k(\Lambda)$ can be an empty set). 
			In particular, if $\Lambda = \Lambda(u) = A\lambda + Au$ for some $u\in \mathrm{Span}(A^T)$, then it is easy to see that 
			\begin{equation}\label{eq:lem:proof:hausdorff-boundeness:section-s-equiv}
			k(\Lambda(u)) = (\lambda + u + \ker A)\cap K = (\lambda + u + \ker A)\cap \R^p_+ = S_{A,\lambda}(\{u\}),
			\end{equation}
			where $S_{A,\lambda}$ is defined in \eqref{eq:proofs:lem-geometric-positive-cone-statement}. 
			
			The result of Theorem~1 from \citet{walkup1969lipschitzian} says, in particular, that 
			\begin{align}\label{eq:lem:proof:hausdorff-boundeness:hausdorff-lipschitz}
			d_H(k(\Lambda), k(\Lambda')) \leq C \|\Lambda - \Lambda'\|, 
			\end{align}
			where $C$ is some constant depending on $K$ and $A$, $d_H(\cdot, \cdot)$ is the standard Hausdorff distance being also extended for empty sets. However, this extension is not needed for us since we always consider parameters $\Lambda(u)$ for $u$ from  some $U\subset \mathrm{Span}(A^T)$ with apriori non-empty sets $S_{A,\lambda}(\{u\})$.
			
			From formulas \eqref{eq:lem:proof:hausdorff-boundeness:section-s-equiv}, \eqref{eq:lem:proof:hausdorff-boundeness:hausdorff-lipschitz} it follows that  
			\begin{align}
			d_H(S_{A,\lambda}(\{u\}), S_{A,\lambda}(\{u'\})) \leq C \|A(u-u')\|, 
			\end{align}
			which directly implies \eqref{eq:lem:hausdorff-boundeness:continuity}.
			
			Lemma is proved.
		\end{proof}

		
		

		\subsection{Proof of Theorem~\ref{thm:non-parametric-poisson-gamma-posterior}}
		\begin{proof}
			Claim follows directly from the result of Theorem 3.1 from  \citet{albertlo1982bayesiannon}. Indeed, having sample $N_1, \dots,  N_n$ of size $n$ from a Poisson point process with intensity $\nu$ is equivalent having sample $N_1 + \dots N_n$ of size $1$ for intensity $n \nu$. Therefore, parameter $n$ is a direct analog of $t$ in our considerations. Moreover, it is trivial to check that all results from Section~3 of \citet{albertlo1982bayesiannon} hold for $n$ being replaced with $t$.
			
			Theorem is proved.
		\end{proof}
		
		\subsection{Proofs of theorems~\ref{thm:wbb-algorithm-consistency} and  \ref{thm:wbb-generic-consistency}}
		
		First we prove Theorem~\ref{thm:wbb-generic-consistency}, then we show that if \eqref{eq:thm-wbb-algorithm-consistency-params} holds  
		conditions in \eqref{eq:thm-wbb-generic-consistency-cond-mean} for Theorem~\ref{thm:wbb-generic-consistency} are satisfied which, in turn,  automatically proves Theorem~\ref{thm:wbb-algorithm-consistency}.
		
		\begin{proof}[of Theorem~\ref{thm:wbb-generic-consistency}]
			Using \eqref{eq:prelim:poiss-log-likelihood-model}, \eqref{eq:prelim:poiss-log-likelihood-penalized}, the minimization problem in step~3 in  Algorithm~\ref{alg:npl-posterior-sampling:mri:binned} can be rewritten as  as follows:
			\begin{align}
			\begin{split}\label{eq:proof:consistency-generic:sample-def}
			\widetilde{\lambda}^t_b &= \argmin_{\lambda \succeq 0} L_{p}(\lambda \mid \widetilde{\Lambda}_b^t, A, 1, \beta^t/t) \\
			&= \argmin_{\lambda \succeq 0}\mathcal{L}^t(\lambda),
			\end{split}
			\end{align}
			where 
			\begin{align}
			\begin{split}\label{eq:proof:consistency-generic:normalized-likl-nice}
			\mathcal{L}^t(\lambda) &= \sum_{i\in I_1(\Lambda^*)} (-\widetilde{\Lambda}_{b,i}^t + \Lambda_i^*)\log\left(
			\dfrac{\Lambda_i}{\Lambda_i^*}
			\right) \\
			& + \sum_{i\in I_1(\Lambda^*)} -\Lambda_i^* \log\left(
			\dfrac{\Lambda_i}{\Lambda_i^*}
			\right) + (\Lambda_i - \Lambda_i^*) \\ 
			& + \sum\limits_{i\in I_0(\Lambda^*)} -\widetilde{\Lambda}_{b,i}^t \log(\Lambda_i) + \Lambda_i + \dfrac{\beta_t}{t}(\varphi(\lambda)-\varphi(\lambda_*)),
			\end{split}
			\end{align}
			where $I_0(\cdot)$, $I_1(\cdot)$ are defined in \eqref{eq:ind-lors-pos-zeros} and $\Lambda^* = A\lambda_*$. 
			
			Next, for the proof we use the following lemma.
			\begin{lem}\label{lem:consistency:strong-convexity-vicinity}
				Let $\mathcal{L}^t(\lambda)$ be defined in    \eqref{eq:proof:consistency-generic:normalized-likl-nice} and conditions of Theorem~\ref{thm:wbb-generic-consistency} be 
				satisfied. Let $C_{A,\delta}(\lambda')$, $\delta > 0, \, \lambda' \succeq 0$, be the cylinder set defined by the formula 
				\begin{align}\label{eq:proof:consistency-generic:cylinder-def}
				&C_{A, \delta}(\lambda') = \{\lambda\in \R^p_+, \, \lambda = \lambda' + \delta u + w \mid  (u, w) \in \mathrm{Span}(A^T) \times \ker A, \|u\|=1 \}.
				\end{align}
				
				Then,
				
				\begin{itemize}
					\item[i)] there exists $\delta_0 = \delta_0(A, \lambda_*) > 0$ such that for any $\delta < \delta_0$ it holds that 
					\begin{align}\label{eq:proof:consistency-generic:estimate-formulate}
					\inf_{\lambda\in C_{A, \delta}(\lambda_*)} \hspace{-0.4cm} \mathcal{L}^t(\lambda) \geq C \delta^2 + o_{cp}(1) \text{ when } t\rightarrow +\infty, \text{ a.s. } Y^t, t\in (0, +\infty). 
					\end{align}
					where $C$ is a positive constant independent of $\delta$.
					
					\item[ii)] there exists a family of random variables $\widetilde{\lambda}^t \in \R^p_+$, $t\in (0, +\infty)$, such that 
					\begin{align}\label{eq:eq:proof:consistency-generic:center-point-def}
					\widetilde{\lambda}^t \xrightarrow{c.p.} \lambda_* \text{ and }
					\mathcal{L}^t(\widetilde{\lambda}^t) \xrightarrow{c.p.} 0 \text{ when } t\rightarrow +\infty, \text{ a.s. } Y^t,\,  t\in (0, +\infty).
					\end{align}
				\end{itemize}
			\end{lem}
			
			
			From the result of Lemma~\ref{lem:consistency:strong-convexity-vicinity}(i) it follows that for all  $\lambda\succeq 0$
			at distance $\delta$ from $\lambda_*$ in the $\mathrm{Span}(A^T)$
			values of $\mathcal{L}^t(\lambda)$  are greater or equal than $C\delta^2$ with conditional probability tending to one a.s. $Y^t$, $t\in (0, +\infty)$. At the same time, result of  Lemma~\ref{lem:consistency:strong-convexity-vicinity}(ii) says that 
			there is $\widetilde{\lambda}^t\in \R^p_+$ which is arbitrarily close to $\lambda_*$ and $\mathcal{L}^t(\widetilde{\lambda}^t)$ is converges to zero for $t\rightarrow +\infty$ with conditional probability also tending to one. The fact that $\mathcal{L}^t(\lambda)$ is convex together with the above arguments and $\widetilde{\lambda}_b^t$ being the unique  minimizer of $\mathcal{L}^t(\lambda)$  imply that 
			\begin{align}\label{eq:proof:consistency-generic:span-convergence-1}
			P(\, \|\Pi_{A^T}(\widetilde{\lambda}_b^t - \lambda_*) \| < \delta \mid Y^t, t) \rightarrow 1 \text{ when } t\rightarrow +\infty, \text{ a.s. }Y^t, t\in (0, +\infty). 
			\end{align}
			where $\Pi_{A^T}$ is the orthogonal projector onto $\mathrm{Span}(A^T)$.
			Since $\delta$ can be chosen arbitrarily small in Lemma~\ref{lem:consistency:strong-convexity-vicinity} formula  \eqref{eq:proof:consistency-generic:span-convergence-1} implies that 
			\begin{equation}\label{eq:consistency:proof:projection-convergence}
			\Pi_{A^T}(\widetilde{\lambda}_b^t - \lambda_*) \xrightarrow{c.p.} 0 \text{ when }t\rightarrow +\infty, \text{ a.s. }Y^t, t\in (0,+\infty).
			\end{equation}
			Vector $\widetilde{\lambda}_{b}^t$ admits in a unique way the following representation
			\begin{equation}\label{eq:consistency:proof:random-decomposition}
			\widetilde{\lambda}_b^t = \lambda_* + \widetilde{u}_b^t + \widetilde{w}_b^t, \text{ where } (\widetilde{u}_b^t, \widetilde{w}^t_b) \in \mathrm{Span}(A^T)\times \ker A.
			\end{equation}
			Using \eqref{eq:proof:consistency-generic:sample-def}, \eqref{eq:proof:consistency-generic:normalized-likl-nice}, \eqref{eq:consistency:proof:random-decomposition} one can see that  
			\begin{equation}\label{eq:consistency:proof:kernel-part-problem}
			\widetilde{w}_b^t  =  \argmin_{\substack{w:\lambda_* + \widetilde{u}_b^t + w \succeq 0, \\ w\in \ker A}} \varphi(\lambda_* + \widetilde{u}^t_b + w) = w_{A,\lambda_*}(\widetilde{u}_b^t), 
			\end{equation}
			where $w_{A,\lambda}(\cdot)$ is defined in \eqref{eq:proofs:lemma-existence-minimizer-ker-a}. From \eqref{eq:consistency:proof:kernel-part-problem}, the fact that $\widetilde{u}_b^t \xrightarrow{c.p.}0$ (see formulas \eqref{eq:consistency:proof:projection-convergence}, \eqref{eq:consistency:proof:random-decomposition}), continuity of the map $w_{A,\lambda_*}(\cdot)$ (by the result of Lemma~\ref{lem:consistency:lem-kernel-continuity}) and the Continuous Mapping Theorem (see, e.g. \citet{vaart2000asymptotic}, Theorem~2.3, p.~7) it follows that 
			\begin{equation}\label{eq:consistency:proof:kernel-part-convergence}
			\widetilde{w}_b^t \xrightarrow{c.p.} w_{A,\lambda_*}(0) \text{ when }
			t\rightarrow +\infty, \text{ a.s. }Y^t, t\in (0, +\infty).
			\end{equation}
			Formula \eqref{eq:thm-wbb-algorithm-consistency-fmla} follows directly from \eqref{eq:consistency:proof:projection-convergence}-		\eqref{eq:consistency:proof:kernel-part-convergence}.
			
			Theorem is proved.	
		\end{proof}
		
		\begin{proof}[of Theorem~\ref{thm:wbb-algorithm-consistency}]
			To prove the theorem we use the following lemma.
			\begin{lem}\label{lem:consistency:proof:algorithm}
				Let $\widetilde{\lambda}_b^t$ be defined as in Algorithm~\ref{alg:npl-posterior-sampling:mri:binned} and let $\theta^t/t \rightarrow 0$ when $t\rightarrow +\infty$.
				Then, 
				\begin{equation}\label{eq:lem:consistency:proof:algorithm}
				\widetilde{\Lambda}_{b,i}^t \xrightarrow{c.p.} \Lambda_i^* = a_i^T\lambda_* \text{ when } t\rightarrow +\infty, \text{ a.s. } Y^t, t\in (0, +\infty).
				\end{equation}	
			\end{lem}
			
			In view of \eqref{eq:lem:consistency:proof:algorithm} in  Lemma~\ref{lem:consistency:proof:algorithm} all assumptions for Theorem~\ref{thm:wbb-generic-consistency} are satisfied, which implies  formula \eqref{eq:thm-wbb-algorithm-consistency-fmla}.
			
			Theorem is proved.
		\end{proof}
		
		\subsection{Proof of Lemma~\ref{lem:consistency:strong-convexity-vicinity}}
		\begin{proof}
			First we prove (i), then for  (ii) we give an explicit formula for $\widetilde{\lambda}^t$ for which  \eqref{eq:eq:proof:consistency-generic:center-point-def} holds.
			
			First, in formula \eqref{eq:proof:consistency-generic:normalized-likl-nice} one can see that 
			\begin{align}\label{eq:proof:consistency-generic:stochastic-small}
			\inf_{\lambda \in C_{A,\delta}(\lambda_*)} \sum\limits_{i\in I_1(\Lambda^*)}
			\left(
			-\widetilde{\Lambda}^t_{b,i} + \Lambda_i^* 
			\right)
			\log 
			\left(
			\dfrac{\Lambda_i}{\Lambda_i^*}
			\right)
			\xrightarrow{c.p.}0, \, \text{ when } t\rightarrow +\infty, \, 
			\text{ a.s. }Y^t, t\in (0, +\infty).
			\end{align}
			The above formula follows from the assumption that $\widetilde{\Lambda}^t_{b,i} \xrightarrow{c.p.}\Lambda_i^*$ and that $\log(\Lambda_i / \Lambda_i^*) = \log(1 + \delta a_i^Tu / \Lambda_i^*)$ is uniformly bounded for $\lambda\in C_{A,\delta}(\lambda_*)$ from above and below for $\delta$ small enough ($u\in \mathrm{Span}(A^T), \, \|u\|=1$). For example, to bound all of the logarithmic terms in \eqref{eq:proof:consistency-generic:stochastic-small} we may choose any $\delta$ such that 
			\begin{equation}\label{eq:proof:consistency-generic:choice-delta-1}
			0 < \delta < \min_{i\in I_1(\Lambda^*)}\left( \Lambda_i^* \|a_i\|^{-1}\right).
			\end{equation}
			Since $\varphi(\lambda)$ satisfies \eqref{eq:prelim:penalty-cond-convex}, \eqref{eq:prelim:penalty-cond-strict-conv}, there exists a constant $M = M(\lambda_*, \delta, A)$ such that 
			\begin{equation}\label{eq:proof:consistency-generic:penalty-sep}
			\inf_{\lambda \in C_{A,\delta}(\lambda_*)} \varphi(\lambda) \geq M.
			\end{equation}
			From \eqref{eq:thm-wbb-algorithm-consistency-params}, \eqref{eq:proof:consistency-generic:penalty-sep} it follows that 
			\begin{equation}
			\label{eq:proof:consistency-generic:penalty-small}
			(\beta^t /t )  \inf_{C_{A,\delta}(\lambda_*)}(\varphi(\lambda) - \varphi(\lambda_*)) \geq o(1), \text{ when } t\rightarrow +\infty.
			\end{equation}
			Using \eqref{eq:proof:consistency-generic:normalized-likl-nice}, \eqref{eq:proof:consistency-generic:stochastic-small}, \eqref{eq:proof:consistency-generic:penalty-small} we obtain the following estimate 
			\begin{align}\label{eq:proof:consistency-generic:inf-stochastic-penalty}
			\begin{split}
			\inf_{\lambda\in C_{A,\delta}(\Lambda_*)} \mathcal{L}^t(\lambda) 
			&\geq o_{cp}(1) + \inf_{\lambda\in C_{A, \delta}(\lambda_*)} 
			\sum_{i\in I_1(\Lambda^*)} -\Lambda_i^* \log\left(
			\dfrac{\Lambda_i}{\Lambda_i^*}
			\right) + (\Lambda_i - \Lambda_i^*) \\ 
			&+ \sum\limits_{i\in I_0(\Lambda^*)} -\widetilde{\Lambda}^t_{b,i} \log(\Lambda_i) + \Lambda_i.
			\end{split}
			\end{align}
			Note that 
			\begin{equation}\label{eq:proof:consistency-generic:log-zero-int}
			-\widetilde{\Lambda}_{b,i}^t\log (\Lambda_i) \geq 0 
			\text{ for }  \Lambda_i \leq 1, \, i\in I_0(\Lambda^*).
			\end{equation}
			From \eqref{eq:ind-lors-pos-zeros},  \eqref{eq:proof:consistency-generic:cylinder-def} it follows that we can choose $\delta$ sufficiently small so that 
			\begin{equation}\label{eq:proof:consistency-generic:zero-int-pos-cond}
			\Lambda_{i} \leq 1 \text{ for all } \lambda\in C_{A, \delta}(\lambda_*), i\in I_0(\Lambda^*).
			\end{equation}
			For example, it suffices to  choose $\delta$ as follows
			\begin{equation}\label{eq:proof:consistency-generic:delta-cond-2}
			0 < \delta \leq \min_{i\in \{1, \dots, d\}} ( \|a_i\|^{-1}).
			\end{equation}
			Using \eqref{eq:proof:consistency-generic:inf-stochastic-penalty}, \eqref{eq:proof:consistency-generic:log-zero-int}, for $\delta$ satisfying \eqref{eq:proof:consistency-generic:choice-delta-1},  \eqref{eq:proof:consistency-generic:delta-cond-2} we obtain
			\begin{align}\label{eq:proof:consistency-generic:full-stoch-estimate}
			\begin{split}
			\inf_{\lambda\in C_{A,\delta}(\Lambda_*)} \mathcal{L}^t(\lambda) 
			&\geq \inf_{\lambda\in C_{A,\delta}(\lambda_*)} 
			\sum_{i\in I_1(\Lambda^*)} -\Lambda_i^* \log\left(
			\dfrac{\Lambda_i}{\Lambda_i^*}
			\right) + (\Lambda_i - \Lambda_i^*) \\ 
			&+ \sum\limits_{i\in I_0(\Lambda^*)}\Lambda_i + o_{cp}(1).
			\end{split}
			\end{align}
			
			Now, consider 
			\begin{equation}
			\Phi_{s^*}(s) = -s^* \log(s) + s, \, 
			s > 0, \, s^* > 0.
			\end{equation}
			Function $\Phi_{s^*}(s)$ is convex, smooth, has positive non-vanishing second derivative $\Phi_{s^*}''(s)$ and at $s=s^*$ it has its global minimum. Therefore, for any $\varepsilon > 0$ small enough (for example, for   $\varepsilon < s^*$) there exists positive constant $C(\varepsilon, s^*)$ such that 
			\begin{equation}\label{eq:proof:consistency-generic:poisson-func-str-convexity}
			\Phi_{s^*}(s) - \Phi_{s^*}(s^*) \geq C(\varepsilon, s^*)|s-s^*|^2 \text{ for } |s - s^*| < \varepsilon.
			\end{equation}
			
			From \eqref{eq:proof:consistency-generic:poisson-func-str-convexity} it follows that one can choose $\delta_0 > 0$ such that 
			\begin{align}
			\begin{split}\label{eq:proof:consistency-generic:det-likl-estim-1}
			&\sum_{i\in I_1(\Lambda^*)} -\Lambda_i^* \log\left(
			\dfrac{\Lambda_i}{\Lambda_i^*}
			\right) + (\Lambda_i - \Lambda_i^*) \\ 
			&+ \sum\limits_{i\in I_0(\Lambda^*)}\Lambda_i \geq C(\delta_0, \Lambda^*)
			\sum_{i\in I_1(\Lambda^*)}(\Lambda_i - \Lambda_i^*)^2 + \sum\limits_{i\in I_0(\Lambda^*)}\Lambda_i\\
			&\text{for } \mid\Lambda_i-\Lambda^*_i\mid < \delta_0, \, i\in I_1(\Lambda^*).
			\end{split}
			\end{align}
			
			Value for $\delta_0$ is precised below. 
			Let $\lambda \in C_{A, \delta}(\lambda_*)$ and $\delta < \delta_0$, that is $\lambda = \lambda_* + \delta u + w$,  where $u\in \mathrm{Span}(A^T)$, $\|u\|=1$, $w\in \ker A$. For $\delta$ satisfying \eqref{eq:proof:consistency-generic:delta-cond-2} formula \eqref{eq:proof:consistency-generic:zero-int-pos-cond} holds and we get the following estimate:
			\begin{align}\label{eq:proof:consistency-generic:det-likl-estim-2}
			\Lambda_i = a_i^T\lambda = \delta a_i^Tu \geq \delta^2 (a_i^Tu)^2 \geq 0 \,\text{ for } i\in I_0(\Lambda^*).
			\end{align}
			In \eqref{eq:proof:consistency-generic:det-likl-estim-2} we used the fact that $\Lambda_i^* = a_i^T\lambda_* = 0$, $i\in I_0(\Lambda^*)$. 
			
			From \eqref{eq:proof:consistency-generic:det-likl-estim-1}, \eqref{eq:proof:consistency-generic:det-likl-estim-2} it follows that 
			\begin{align}\label{eq:proof:consistency-generic:det-likl-estim-3}
			\begin{split}
			\inf_{\lambda\in C_{A,\delta}(\lambda_*)}\sum_{i\in I_1(\Lambda^*)} -&\Lambda_i^* \log\left(
			\dfrac{\Lambda_i}{\Lambda_i^*}
			\right) + (\Lambda_i - \Lambda_i^*) + \sum\limits_{i\in I_0(\Lambda^*)}\Lambda_i\\ 
			& \geq \min(C(\delta_0, \Lambda^*), 1)\delta^2 \sum\limits_{i=1}^{d}(a_i^Tu)^2\\
			& \geq \min(C(\delta_0, \Lambda^*), 1)\delta^2 \sigma_{min}^{+}(A^TA),
			\end{split}
			\end{align}
			where $\sigma^+_{min}(A^TA)$ is the smallest non-zero eigenvalue of $A^TA$. In particular, in \eqref{eq:proof:consistency-generic:det-likl-estim-2},  \eqref{eq:proof:consistency-generic:det-likl-estim-3} we have used the property that $u \in \mathrm{Span}(A^T)$ which guarantees that
			\begin{equation}
			\sum\limits_{i=1}^{d}(a_i^Tu)^2 = u^TA^TAu \geq \sigma^+_{min}(A^TA) > 0 \text{ for } \|u\|=1.
			\end{equation}
			
			Formula \eqref{eq:proof:consistency-generic:estimate-formulate} follows directly from  \eqref{eq:proof:consistency-generic:full-stoch-estimate},  \eqref{eq:proof:consistency-generic:det-likl-estim-3}.
			
			Finally, we choose $\delta_0$ as follows
			\begin{equation}
			\delta_0 = \dfrac{1}{2}\min\left[\min_{i\in\{1, \dots, d\}}(
			\|a_i\|^{-1}), \min_{i\in I_1(\Lambda^*)}(\Lambda_i^*\|a_i\|^{-1}), \, 
			\min_{i\in I_1(\Lambda^*)} \Lambda_i^*
			\right],
			\end{equation}
			so that conditions \eqref{eq:proof:consistency-generic:choice-delta-1}, \eqref{eq:proof:consistency-generic:delta-cond-2} are simultaneously satisfied together with \eqref{eq:proof:consistency-generic:det-likl-estim-1}.

			Part (i) of Lemma~\ref{lem:consistency:strong-convexity-vicinity} is proved. Now we prove part (ii) of the lemma. 
			
			Let 
			\begin{equation}\label{eq:proof:consistency-generic:lambdat-def}
			\widetilde{\lambda}^t = \lambda_* + \sum\limits_{i\in I_0(\Lambda^*)} \widetilde{\Lambda}_{b,i}^t \dfrac{a_i}{\|a_i\|^2}.
			\end{equation}
			Note that $\widetilde{\lambda}^t\in \R^p_+$ because $a_i\in \R^p_+$ and $\widetilde{\Lambda}_{b,i}^t\geq 0$.
			Since $\widetilde{\Lambda}_{b,i}^t\xrightarrow{c.p.}0$ for $i\in I_0(\Lambda^*)$ (by the assumption) we immediately have that 
			\begin{equation}\label{eq:proof:consistency-generic:zero-point-convergence}
			\widetilde{\lambda}^t \xrightarrow{c.p.} \lambda_* \text{ for }
			t\rightarrow +\infty, \text{ a.s. }Y^t, t\in (0, +\infty).
			\end{equation}
			Note that in \eqref{eq:proof:consistency-generic:normalized-likl-nice} for $\mathcal{L}^t(\lambda)$ all summands are continuous and equal to zero at $\lambda = \lambda_*$ except the logarithmic part 
			\begin{equation}\label{eq:proof:consistency-generic:gt-def}
			g(\lambda) = \sum_{i\in I_0(\Lambda^*)} -\widetilde{\Lambda}_{b,i}^t\log(\Lambda_i), \, \Lambda_i = a_i^T\lambda.
			\end{equation}
			From the fact that $a_i\in \R^p_+$ (see formula \eqref{eq:design-matrix-positivity-restr-1}) it follows that $a_i^Ta_{i'} \geq 0$ for all $i$, $i'$. Using this property and monotonicity of  the logarithm ($\log(x + y) \geq \log(x)$ for $y\geq 0$) it follows that 
			\begin{align}\label{eq:proof:consistency-generic:zero-point-upp-bound}
			\begin{split}
			g(\widetilde{\lambda}^t) &= \sum\limits_{i\in I_0(\Lambda^*)}
			-\widetilde{\Lambda}_{b,i}^t
			\log\left(
			a_i^T\widetilde{\lambda}^t
			\right) \\
			&\leq 
			\sum\limits_{i\in I_0(\Lambda^*)}
			-\widetilde{\Lambda}_{b,i}^t
			\log\left(
			\widetilde{\Lambda}_{b,i}^t
			\right)  \xrightarrow{c.p.} 0
			\text{ when } t\rightarrow +\infty \text{ a.s. } Y^t, \, t\in (0, +\infty).
			\end{split}
			\end{align}
			Formula \eqref{eq:proof:consistency-generic:zero-point-upp-bound} gives an asymptotic upper bound on $g(\widetilde{\lambda}^t)$ which is equal to zero. For the lower bound we use formulas \eqref{eq:proof:consistency-generic:log-zero-int}, \eqref{eq:proof:consistency-generic:zero-point-convergence} and the fact that $a_i^T \widetilde{\lambda}^t \xrightarrow{c.p.} 0$ for $i\in I_0(\Lambda^*)$
			from which it follows that 
			\begin{align}\label{eq:proof:consistency-generic:zero-point-lower-bound}
			\begin{split}
			&g(\widetilde{\lambda}^t) \geq 0 \text{ with conditional probability tending to one for } t\rightarrow +\infty, \\
			&\text{ a.s. } Y^t, t\in (0, +\infty).
			\end{split}
			\end{align}
			From \eqref{eq:proof:consistency-generic:zero-point-upp-bound}, \eqref{eq:proof:consistency-generic:zero-point-lower-bound} it follows that 
			\begin{equation}\label{eq:proof:consistency-generic:exact-bound}
			g(\widetilde{\lambda}^t) \xrightarrow{c.p.} 0 \text{ when } t\rightarrow +\infty, \text{ a.s. } Y^t, t\in(0,+\infty).
			\end{equation}
			From  \eqref{eq:proof:consistency-generic:normalized-likl-nice},
			\eqref{eq:proof:consistency-generic:lambdat-def}, 
			\eqref{eq:proof:consistency-generic:gt-def}, 
			\eqref{eq:proof:consistency-generic:exact-bound} it follows that 
			\begin{equation}
			\mathcal{L}^t(\widetilde{\lambda}^t) \xrightarrow{c.p.}0 \text{ when } t\rightarrow +\infty, \text{ a.s. } Y^t, t\in (0, +\infty).
			\end{equation}
			This proves part (ii) of the lemma.
			
			Lemma is proved.
		\end{proof}

		\subsection{Proof of Lemma~\ref{lem:consistency:proof:algorithm}}
		\begin{proof}
			Recall that 
			\begin{align}
			\label{eq:proof:algorithm-consistency-lem:resampled-int}
			\widetilde{\Lambda}_{b,i}^{t} \mid Y^t, \widetilde{\Lambda}^t_{\mathcal{M}}, t \sim \Gamma(Y_i^t + \theta^t \widetilde{\Lambda}^t_{\mathcal{M},i}, (\theta^t + t)^{-1}),
			\, i\in \{1, \dots, d\},
			\end{align}
			where $\widetilde{\Lambda}_{\mathcal{M}}^{t} \mid Y^t, t$ is sampled in Algorithm~\ref{alg:wbb-pet-bootstrap:mri:posterior-mixing-param}.
			From the definition of $\widetilde{\Lambda}^t$ in step~1 of  Algorithm~\ref{alg:wbb-pet-bootstrap:mri:posterior-mixing-param} and necessary optimality conditions in step~2 (see also analogous  formula~\eqref{eq:theory:asymp-distr:non-exp-cond:simplex-constr}) it follows that 
			\begin{align}
			\label{eq:proof:algorithm-consistency-lem:em-prop}
			&\sum\limits_{i=1}^{d}\widetilde{\Lambda}^t_{\mathcal{M},i} = \sum\limits_{i=1}^{d}\widetilde{\Lambda}_i^t, \\
			\label{eq:proof:algorithm-consistency-lem:pos-prop}
			&\widetilde{\Lambda}^t_{\mathcal{M}}\succeq 0, \, 
			\widetilde{\Lambda}^t\succeq 0, 
			\, E[\widetilde{\Lambda}_i^t \mid Y^t, t] = 
			Y_i^t/t, \, i\in {1, \dots, d}.
			\end{align}
			Using \eqref{eq:proof:algorithm-consistency-lem:em-prop}, \eqref{eq:proof:algorithm-consistency-lem:pos-prop} we get the following estimate:
			\begin{equation}\label{eq:proof:algorithm-consistency-lem:upper-bound-photons}
			E[\widetilde{\Lambda}^t_{\mathcal{M},i} \mid Y^t, t] \leq \sum\limits_{i=1}^{d}\dfrac{Y_i^t}{t}, \, i\in \{1, \dots, d\}. 
			\end{equation}
			Let $\varepsilon > 0$. Using the Markov inequality we obtain 
			\begin{align}
			\label{eq:proof:algorithm-consistency-lem:markov-main}
			\begin{split}
			p(\mid\widetilde{\Lambda}_{b,i}^t - \Lambda_i^*\mid > \varepsilon \mid \, Y^t, t) &\leq \dfrac{E[\mid\widetilde{\Lambda}_{b,i}^t - \Lambda_i^*\mid\mid Y^t, t]}{\varepsilon} \\
			&\leq \dfrac{E[\mid\widetilde{\Lambda}_{b,i}^t - \frac{Y_i^t +\theta^t \widetilde{\Lambda}^t_{\mathcal{M},i}}{\theta^t +t}\mid\mid Y^t, t]}{\varepsilon} + 
			\dfrac{E[\mid\frac{Y_i^t +\theta^t \widetilde{\Lambda}^t_{\mathcal{M},i}}{\theta^t +t} - \Lambda_i^*\mid \mid Y^t, t]}{\varepsilon}.
			\end{split}
			\end{align}
			Using the Jensen's inequality $E|X|^2 \geq (E|X|)^2$, formulas \eqref{eq:proof:algorithm-consistency-lem:resampled-int}, \eqref{eq:proof:algorithm-consistency-lem:upper-bound-photons}, the Strong Law of Large Numbers for $Y^t$ (see Theorem~\ref{thm:appendix:limits-poisson}(i) in  Section~\ref{app:limit-thms}) and the fact that $\theta^t/t\rightarrow 0$, we get the following:
			\begin{align}
			\label{eq:proof:algorithm-consistency-lem:first-term-as-zero}
			\begin{split}
			E[\mid\widetilde{\Lambda}_{b,i}^t - \frac{Y_i^t +\theta^t \widetilde{\Lambda}^t_{\mathcal{M},i}}{\theta^t +t}\mid\mid Y^t, t] &\leq 
			\left(
			E[\mid\widetilde{\Lambda}_{b,i}^t - \frac{Y_i^t +\theta^t \widetilde{\Lambda}^t_{\mathcal{M},i}}{\theta^t +t}\mid^2\mid Y^t, t]
			\right)^{1/2}\\
			&=\left(
			E[\mathrm{var}[(\widetilde{\Lambda}^t_{b,i}) \mid Y^t, \widetilde{\Lambda}^t_{\mathcal{M}}, t] \mid Y^t, t]
			\right)^{1/2}\\
			&=\left(
			\dfrac{Y_i^t + \theta^tE[\widetilde{\Lambda}^t_{\mathcal{M},i} \mid Y^t, t]}{(t+\theta^t)^2}
			\right)^{1/2}\\
			&\leq \left(
			\dfrac{Y_i^t + (\theta^t/t) \sum_{i=1}^{d}Y_i^t}{(t+\theta^t)^2}
			\right)^{1/2} \rightarrow 0 \text{ a.s. } Y^t, \, t\in(0, +\infty).  
			\end{split}
			\end{align}
			
			For estimation of the second term in \eqref{eq:proof:algorithm-consistency-lem:markov-main} we use formula \eqref{eq:proof:algorithm-consistency-lem:upper-bound-photons}, the triangle inequality and again the property that $\theta^t/t\rightarrow 0$ to get the following:
			\begin{align}
			\label{eq:proof:algorithm-consistency-lem:second-term-as-zero}
			\begin{split}
			E\left[\left|\frac{Y_i^t +\theta^t \widetilde{\Lambda}^t_{\mathcal{M},i}}{\theta^t +t} - \Lambda_i^*\right| Y^t, t\right] &\leq 
			\left|\frac{Y_i^t}{\theta^t + t} - \Lambda_i^*\right| + 
			E\left[\dfrac{\theta^t \widetilde{\Lambda}_{\mathcal{M},i}^t}{\theta^t + t} \mid Y^t, t\right]\\
			&\leq \left|\frac{Y_i^t}{\theta^t + t} - \Lambda_i^*\right| + 
			\dfrac{\theta^t}{\theta^t + t}\sum\limits_{i=1}^{d}\frac{Y_i^t}{t}
			\rightarrow 0 \text{ a.s. } Y^t, t\in (0, +\infty).
			\end{split}
			\end{align}
			Formula \eqref{eq:lem:consistency:proof:algorithm} follows from formulas \eqref{eq:proof:algorithm-consistency-lem:markov-main}-\eqref{eq:proof:algorithm-consistency-lem:second-term-as-zero}.
			
			Lemma is proved.
		\end{proof}

		\subsection{Proof of Proposition~\ref{prop:existence-minimizers}}
		\begin{proof}
			First prove that the set of minimizers in \eqref{eq:assump:theory:asymp-distr:non-exp-cond} is always nonempty  and is a subset of the simplex in  \eqref{eq:prop:existence-restritiveness-minimizers:symplex}. From the Karush-Kuhn-Tucker optimality conditions (see e.g., \citet{bertsekas1997nonlinear}, Section 3.3) it follows that
			\begin{align}\nonumber
			&\exists (\lambda_{\mathcal{M}, *}, \mu_{\mathcal{M}, *}) \in \R^p_+\times \R^p_+ \text{ such that}\\
			\label{eq:theory:asymp-distr:non-exp-cond:kkt-1}
			& \sum_{i\in I_1(\Lambda^*)} -\Lambda_i^* \dfrac{a_{\mathcal{M},ij}}{\Lambda_{\mathcal{M}, i}^*} + \sum\limits_{i=1}^{d}a_{\mathcal{M}, ij} - \mu_{\mathcal{M},*,  j} = 0, \\
			\label{eq:theory:asymp-distr:non-exp-cond:kkt-2}
			&\mu_{\mathcal{M}, *, j} \lambda_{\mathcal{M}, *, j} \equiv 0, 
			\text{ for all }
			j \in \{1, \dots, p_{\mathcal{M}}\}.
			\end{align}
			By multiplying both sides of \eqref{eq:theory:asymp-distr:non-exp-cond:kkt-1} on $\lambda_{\mathcal{M}, *, j}$, summing up all equations with respect to $j$ and using \eqref{eq:theory:asymp-distr:non-exp-cond:kkt-2} we obtain the following necessary optimality condition:
			\begin{align}\label{eq:theory:asymp-distr:non-exp-cond:simplex-constr}
			\begin{split}
			&\langle
			\sum_{i=1}^{d} a_{\mathcal{M},i}, \lambda_{\mathcal{M}, *} \rangle
			= 
			\sum\limits_{j=1}^{p_{\mathcal{M}}}A_{\mathcal{M}, j}\lambda_{\mathcal{M}, *,j} 
			= \sum\limits_{i=1}^{d} \Lambda_{i}^{*}, \\
			&A_{\mathcal{M},j} = \sum\limits_{i=1}^{d}a_{\mathcal{M},ij}.
			\end{split}
			\end{align}
			Formula \eqref{eq:theory:asymp-distr:non-exp-cond:simplex-constr} proves \eqref{eq:prop:existence-restritiveness-minimizers:symplex}.
			The constraint in \eqref{eq:theory:asymp-distr:non-exp-cond:simplex-constr} can be added to the set of constraints in \eqref{eq:assump:theory:asymp-distr:non-exp-cond} without any effect since  it is necessary. Because the minimized functional in  \eqref{eq:assump:theory:asymp-distr:non-exp-cond} is convex and the domain  of constraints is now a convex compact there always exists at least one minimizer.
			
			Demonstration of \eqref{prop:assumptions-nonrestr:inclusion} is straightforward. 
			Indeed, if for some $i$ we have $\Lambda^*_i > 0$, then necessarily $\Lambda^*_{\mathcal{M},i} > 0$, otherwise the value of the target functional becomes $+\infty$ due to explosion of the logarithmic term.
			At the same time any interior point  $\lambda_{\mathcal{M}}\in \Delta^{p_{\mathcal{M}}}_{A_{\mathcal{M}}}(\Lambda^*)$ (i.e., $\lambda_{\mathcal{M}} \succ 0$) would result in the finite value of the target functional. Hence, inclusions \eqref{prop:assumptions-nonrestr:inclusion} always hold.
			
			Proposition is proved. 
		\end{proof}

		\subsection{Proof of Theorem~\ref{thm:asympt-distr:well-spec:identifiability-cond}}
		
		\begin{proof}
			First we prove \eqref{eq:asympt-distr:well-spec:identifiability-cond:approx}, then  \eqref{eq:thm:asympt-distr:identif-cond:strong-convexity} which also  implies uniqueness of the minimizer.
			
			Let $\lambda_{\mathcal{M}, *}\in \R^{p_\mathcal{M}}_+$ be a minimizer in  \eqref{eq:assump:theory:asymp-distr:non-exp-cond} (possibly not unique; see also Proposition~\ref{prop:existence-minimizers}).
			
			Let 
			\begin{equation}\label{eq:proof:misspecified-param-represnt}
			\lambda_{\mathcal{M}} = \lambda_{\mathcal{M},*} + u_{\mathcal{M}}, \, 
			\lambda_{\mathcal{M}}\in \R^{p_{\mathcal{M}}}_+.
			\end{equation}
			Consider the second order Taylor expansion of $L(\lambda \mid \Lambda^*, A_{\mathcal{M}}, 1)$ in \eqref{eq:assump:theory:asymp-distr:non-exp-cond} in a vicinity of $\lambda_{\mathcal{M},*}$:
			\begin{align}
			\label{eq:proof:misspecified-taylor-second-exp}
			\begin{split}
			L(\lambda_{\mathcal{M}}\mid \Lambda^*, A_{\mathcal{M}}, 1) &- L(\lambda_{\mathcal{M},*}\mid \Lambda^*, A_{\mathcal{M}}, 1) \\
			&=  u_{\mathcal{M}}^T \nabla L(\lambda_{\mathcal{M}, *} \mid \Lambda^*, A_{\mathcal{M}}, 1) 
			+ \dfrac{1}{2}\sum_{i\in I_1(\Lambda^*)} \Lambda^*_{i}\dfrac{(u^T_{\mathcal{M}} a_{\mathcal{M},i})^2}{(\Lambda_{\mathcal{M},i}^*)^2} \\ &+ 
			o(\|\Pi_{A^T_{\mathcal{M}, I_1(\Lambda^*)}}u_{\mathcal{M}}\|^2),
			\end{split}
			\end{align}
			where $\Lambda^*_{\mathcal{M}} = A_{\mathcal{M}}\lambda_{\mathcal{M}, *}$  and 
			\begin{align}\label{eq:proof:misspecified-kl-proj}
			\nabla L(\lambda_{\mathcal{M}, *} \mid \Lambda^*, A_{\mathcal{M}}, 1) = 
			\sum\limits_{i\in I_1(\Lambda^*)}-\Lambda^*_i \dfrac{a_{\mathcal{M}, i}}{\Lambda_{\mathcal{M},i}^*} + \sum\limits_{i=1}^{d} a_{\mathcal{M}, i}.
			\end{align}
			Karush-Kuhn-Tucker necessary optimality conditions for the problem in \eqref{eq:assump:theory:asymp-distr:non-exp-cond} imply that there exists $\mu_{\mathcal{M},*}$ such that 
			\begin{align}\label{eq:proof:misspecified-existence-lagrange-multiplier}
			\mu_{\mathcal{M},*} \succeq 0, \, 
			\nabla L(\lambda_{\mathcal{M},*} \mid \Lambda^*, A_{\mathcal{M}}, 1) = \mu_{\mathcal{M},*}, \, 
			\mu_{\mathcal{M},*,j} \lambda_{\mathcal{M},*,j} = 0, \, j = 1, \dots, p.
			\end{align}
			
			From formulas \eqref{eq:proof:misspecified-param-represnt}, \eqref{eq:proof:misspecified-existence-lagrange-multiplier} it follows that
			\begin{align}
			\label{eq:proof:misspecified-linear-positive}
			\begin{split}
			u_{\mathcal{M}}^T \nabla L(\lambda_{\mathcal{M}, *} \mid \Lambda^*, A_{\mathcal{M}}) &= u^T_{\mathcal{M}}\mu_{\mathcal{M},*} = 
			(\lambda_{\mathcal{M}} - \lambda_{\mathcal{M}, *})^T\mu_{\mathcal{M},*}\\
			&= \lambda_{\mathcal{M}}^T\mu_{\mathcal{M},*} \geq 0.
			\end{split}
			\end{align}
			Note also that $\mu_{\mathcal{M}, *}$ is the optimal Lagrangian multiplier for the problem in \eqref{eq:assump:theory:asymp-distr:non-exp-cond} for which the strong duality holds (e.g., by Slater's condition). 
			
			Formulas \eqref{eq:asympt-distr:well-spec:identifiability-cond:approx}, \eqref{eq:asympt-distr:well-spec:identifiability-cond:grad-formulas} follow from \eqref{eq:proof:misspecified-taylor-second-exp}-\eqref{eq:proof:misspecified-linear-positive}. Next, we prove that  \eqref{eq:thm:asympt-distr:identif-cond:strong-convexity} holds.
			
			Using \eqref{eq:proof:misspecified-linear-positive} we obtain the following estimate:
			\begin{equation}\label{eq:proof:misspecified:linear-to-quadratic}
			\begin{split}
			u_{\mathcal{M}}^T \nabla L(\lambda_{\mathcal{M}, *} \mid \Lambda^*, A_{\mathcal{M}}, 1) = u_{\mathcal{M}}^T \mu_{\mathcal{M},*} \geq  (u_{\mathcal{M}}^T \mu_{\mathcal{M},*})^2 \text{ if }
			\|u_{\mathcal{M}}\| \leq \|\mu_{\mathcal{M},*}\|^{-1}.
			\end{split}
			\end{equation}
			
			From \eqref{eq:proof:misspecified-taylor-second-exp},  \eqref{eq:proof:misspecified:linear-to-quadratic} it follows that 
			\begin{align}\label{eq:proof:misspecified-low-quadr-bound-1}
			\begin{split}
			L(\lambda_{\mathcal{M}}\mid \Lambda^*, A_{\mathcal{M}}, 1) &- L(\lambda_{\mathcal{M},*}\mid \Lambda^*, A_{\mathcal{M}}, 1)\\
			&\geq 
			u_{\mathcal{M}}^T
			C_{\mathcal{M},*}
			u_{\mathcal{M}} + o(\|u_{\mathcal{M}}\|^2),\\
			&\text{for }\|u_{\mathcal{M}}\| \leq \|\mu_{\mathcal{M},*}\|^{-1},
			\end{split}
			\end{align}
			where 
			\begin{equation}\label{eq:proof:misspecified-cm-quad-matr}
			C_{\mathcal{M},*} = \mu_{\mathcal{M},*}\mu_{\mathcal{M},*}^T + 
			\dfrac{1}{2}
			\sum_{i\in I_1(\Lambda^*)} \Lambda^*_{i}\dfrac{ a_{\mathcal{M},i}a_{\mathcal{M},i}^T}{(\Lambda_{\mathcal{M},i}^*)^2}.
			\end{equation}
			To finish the proof we use two following lemmas.
			\begin{lem}\label{lem:misspecified-inj-qform}
				Let assumptions of Theorem~\ref{thm:asympt-distr:well-spec:identifiability-cond} be satisfied.
				Let 
				\begin{align}\label{eq:proof:misspecified-cdelta}
				C_{\delta} = \inf_{\substack{u_{\mathcal{M}}:\lambda_{\mathcal{M}, *} + u_{\mathcal{M}}\succeq 0, \\ \|u_{\mathcal{M}}\| = \delta}}
				u^T_{\mathcal{M}} C_{\mathcal{M}, *}u_{\mathcal{M}}.
				\end{align}
				Then,
				\begin{equation}
				C_{\delta} > 0 \text{ for any }\delta > 0.
				\end{equation}
			\end{lem}
			
			\begin{lem}\label{lem:proofs:positive-cone-geom-lemma}
				Let $\lambda_{\mathcal{M}, *} \in \R^{p_{\mathcal{M}}}_+$. There exists $\delta_* > 0$ such that for any  $u_{\mathcal{M}}\in \R^{p_\mathcal{M}}$, $0< |u_{\mathcal{M}}|\leq \delta_*$, $\lambda_{\mathcal{M},*} + u_{\mathcal{M}} \succeq 0$ it also holds that 
				\begin{align}
				\lambda_{\mathcal{M},*} + \delta_{*}\dfrac{u_{\mathcal{M}}}{\|u_{\mathcal{M}}\|} \succeq 0.
				\end{align}
			\end{lem}
			Let $\delta_*$ be the one of Lemma~\ref{lem:proofs:positive-cone-geom-lemma} for chosen $\lambda_{\mathcal{M}, *}$. From  \eqref{eq:proof:misspecified-low-quadr-bound-1}, \eqref{eq:proof:misspecified-cm-quad-matr} and the results of Lemmas~\ref{lem:misspecified-inj-qform},~\ref{lem:proofs:positive-cone-geom-lemma}, it follows that 
			\begin{align}\label{eq:proof:misspecified-final-quadratic-lower-bnd}
			\begin{split}
			L(\lambda_{\mathcal{M}}\mid \Lambda^*, A_{\mathcal{M}}, 1) &- L(\lambda_{\mathcal{M},*}\mid \Lambda^*, A_{\mathcal{M}}, 1)\\
			&\geq 
			\dfrac{\delta_*u_{\mathcal{M}}^T}{\|u_{\mathcal{M}}\|}
			C_{\mathcal{M},*}
			\dfrac{\delta_{*}u_{\mathcal{M}}}{\|u_{\mathcal{M}}\|}  \dfrac{\|u_{\mathcal{M}}\|^2}{\delta_*^2}+ o(\|u_{\mathcal{M}}\|^2)\\
			&\geq C_{\delta_*}\dfrac{\|u_{\mathcal{M}}\|^2}{\delta_*^2}+ o(\|u_{\mathcal{M}}\|^2), \, C_{\delta_*} > 0, \\
			&\text{ for } \lambda_{\mathcal{M}} = \lambda_{\mathcal{M},*} + u_{\mathcal{M}} \succeq 0, \, |u_{\mathcal{M}}| \leq \min(\delta_*, |\mu_{\mathcal{M},*}|^{-1}).
			\end{split}
			\end{align}
			Formula \eqref{eq:proof:misspecified-final-quadratic-lower-bnd} proves the claim in \eqref{eq:thm:asympt-distr:identif-cond:strong-convexity}. 
			
			Theorem is proved.
		\end{proof}
		
		\begin{proof}[ of Lemma~\ref{lem:misspecified-inj-qform}]
			We use the contradiction argument.
			Assume that it exists $\delta > 0$ such that $C_\delta =0$, where $C_\delta$ is defined in \eqref{eq:proof:misspecified-cdelta}. Since the infimum in \eqref{eq:proof:misspecified-cdelta} is taken over a compact set, there should exist $u_{\mathcal{M}}$ such that 
			\begin{equation}\label{eq:proof:lem:misspecified:cdelta-contradiction-point}
			\|u_{\mathcal{M}}\| = \delta, \, \lambda_{\mathcal{M}, *} + u_{\mathcal{M}} \succeq 0, \, 
			u_{\mathcal{M}}^TC_{\mathcal{M},*}u_{\mathcal{M}} = 0.
			\end{equation}
			Formulas \eqref{eq:proof:misspecified-cm-quad-matr}, \eqref{eq:proof:lem:misspecified:cdelta-contradiction-point} imply that 
			\begin{align}\label{eq:proof:misspecified-u-kernel-props}
			u^{T}_{\mathcal{M}}a_{\mathcal{M},i} = 0, \, i\in I_1(\Lambda^*), \, u^{T}_{\mathcal{M}}\mu_{\mathcal{M},*} = 0.
			\end{align}
			Using formulas \eqref{eq:assump:theory:asymp-distr:non-exp-cond:ind-cond} in the non-expansiveness condition,
			\eqref{eq:proof:misspecified-kl-proj}, \eqref{eq:proof:misspecified-linear-positive},  \eqref{eq:proof:misspecified-u-kernel-props} we obtain the following:
			\begin{align}
			\label{eq:proof:misspecified-zero-cond-on-intensities-1}
			\begin{split}
			u^{T}_{\mathcal{M}}\mu_{\mathcal{M},*} &= \sum_{i\in I_0(\Lambda^*)}
			u_{\mathcal{M}}^Ta_{\mathcal{M},*,i} = 
			\sum_{i\in I_0(\Lambda^*)} (\lambda_{\mathcal{M},i}-\lambda_{\mathcal{M},*,i})^Ta_{\mathcal{M},*,i}\\
			&=\sum_{i\in I_0(\Lambda^*)} (\Lambda_{\mathcal{M},i} - \Lambda_{\mathcal{M},*,i}) = \sum_{i\in I_0(\Lambda^*)} \Lambda_{\mathcal{M},i} = 0, \, \Lambda_{\mathcal{M},i} = \lambda_{\mathcal{M}}^Ta_{\mathcal{M},i}.
			\end{split}
			\end{align}
			From \eqref{eq:proof:misspecified-zero-cond-on-intensities-1} and the fact that $\Lambda_{\mathcal{M}}\succeq 0$ it follows that 
			\begin{align}
			\label{eq:proof:misspecified-zero-cond-on-intensities-2}
			&\Lambda_{\mathcal{M},i} = u_{\mathcal{M}}^Ta_{\mathcal{M},i} = 0, \, 
			i \in I_0(\Lambda^*).
			\end{align}
			Putting formulas \eqref{eq:proof:misspecified-u-kernel-props}, \eqref{eq:proof:misspecified-zero-cond-on-intensities-2} together, we arrive to the following:
			\begin{equation}\label{eq:proof:misspecified-zero-proj-all}
			u^T_{\mathcal{M}}a_{\mathcal{M},i} = 0 \text{ for } i\in \{1, \dots, d\}.
			\end{equation}
			The injectivity of $A_{\mathcal{M}}$ and \eqref{eq:proof:misspecified-zero-proj-all} imply that $u_{\mathcal{M}} = 0$ which contradicts the initial assumption that $\|u_{\mathcal{M}}\| = \delta > 0$.
			
			Lemma is proved.
		\end{proof}

		\begin{proof}[ of Lemma~\ref{lem:proofs:positive-cone-geom-lemma}]
			We prove the claim by contradiction.
			
			The claim is obvious for $\lambda_{\mathcal{M},*} = 0$. 
			
			Let $\lambda_{\mathcal{M},*} \neq 0$ and 
			\begin{equation}\label{eq:proof:lem-positive-cone-geom-lemma:d-start-def}
			\delta_* = \dfrac{1}{2}\min \{\lambda_{\mathcal{M}, *, j} \mid \lambda_{\mathcal{M}, *, j} > 0\}, \, \delta_* > 0.
			\end{equation}
			Let $u_{\mathcal{M}}$ be such that 
			\begin{equation}\label{eq:proof:lem-positive-cone-geom-lemma:u-assmp}
			0 < \|u_{\mathcal{M}}\| \leq \delta_*,\, \lambda_{\mathcal{M},*} + u_{\mathcal{M}} \succeq 0
			\end{equation}
			and assume that 
			\begin{align}\label{eq:proof:lem-positive-cone-geom-lemma:contr-assump}
			\lambda_{\mathcal{M},*} + \delta_*\dfrac{u_{\mathcal{M}}}{\|u_{\mathcal{M}}\|} \not \succeq 0
			\Leftrightarrow \exists j\in \{1, \dots, p_{\mathcal{M}}\}\text{ such that } \lambda_{\mathcal{M},*,j} + \delta_* \dfrac{u_{\mathcal{M},j}}{\|u_{\mathcal{M}}\|} < 0.
			\end{align}
			From the fact that $\lambda_{\mathcal{M},*} \succeq 0$ and  \eqref{eq:proof:lem-positive-cone-geom-lemma:u-assmp}, \eqref{eq:proof:lem-positive-cone-geom-lemma:contr-assump} it follows that 
			\begin{align}\label{eq:proof:lem-positive-cone-geom-lemma:ind-negative}
			\text{ for }j\text{ from \eqref{eq:proof:lem-positive-cone-geom-lemma:contr-assump}} \text{ it holds that } \lambda_{\mathcal{M},*,j} > 0, \, u_{\mathcal{M},j} < 0.
			\end{align}
			Using \eqref{eq:proof:lem-positive-cone-geom-lemma:d-start-def}, \eqref{eq:proof:lem-positive-cone-geom-lemma:contr-assump}, \eqref{eq:proof:lem-positive-cone-geom-lemma:ind-negative} we get the following implication:
			\begin{align}\label{eq:proof:lem-positive-cone-geom-lemma:contradiction}
			\dfrac{\delta_*}{\|u_{\mathcal{M}}\|}(-u_{\mathcal{M}, j}) > \lambda_{\mathcal{M},*,j} \geq 2 \delta_* \Rightarrow (-u_{\mathcal{M},j}) > 2 \|u_{\mathcal{M}}\|.
			\end{align}
			The inequality in the right hand-side of \eqref{eq:proof:lem-positive-cone-geom-lemma:contradiction} gives the desired contradiction.
			
			Lemma is proved.
		\end{proof}

		\subsection{Proof of Theorem~\ref{thm:asympt-distr-main}}
		\begin{proof}
			In what follows we use the following auxiliary result. 
			
			\begin{theorem}[concentration rate for the mixing parameter]	
				\label{thm:asympt-distr:well-spec:concentr-mixing-param}
				Let Assumptions~\ref{assump:theory:consistency:well-spec}-\ref{assump:theory:distribution:non-expansive} be satisfied.
				Let $\widetilde{\lambda}_{\mathcal{M}}^t$ be sampled as in  Algorithm~\ref{alg:wbb-pet-bootstrap:mri:posterior-mixing-param} and $r(t) = o(\sqrt{t / \log \log t})$. Then,	
				\begin{align}\label{eq:thm:mixing-param:posterior-lambda-concentration}
				r(t)(\widetilde{\lambda}^t_{\mathcal{M}} - \lambda_{\mathcal{M}, *})
				\xrightarrow{c.p.} 0 
				\text{ when } t\rightarrow +\infty, \, \text{ a.s. } Y^t, t\in (0, +\infty),
				\end{align}
				where $\lambda_{\mathcal{M},*}$ is from  Theorem~\ref{thm:asympt-distr:well-spec:identifiability-cond}.
				Note that formula \eqref{eq:thm:mixing-param:posterior-lambda-concentration} also implies 
				\begin{align}\label{eq:thm:mixing-param:posterior-concentration}
				r(t)(\widetilde{\Lambda}^t_{\mathcal{M}} - \Lambda_{\mathcal{M}}^*) \xrightarrow{c.p.} 0 
				\text{ when } t\rightarrow +\infty, \, \text{ a.s. } Y^t, t\in (0, +\infty), 
				\end{align}
				where $\Lambda_{\mathcal{M}}^t = A\lambda_{\mathcal{M}}^t$,  $\Lambda^*_{\mathcal{M}} = A_\mathcal{M}\lambda_{\mathcal{M},*}$.
			\end{theorem}		
			
			\begin{remark}
				The log-factor for $r(t)$ in Theorem~\ref{thm:asympt-distr:well-spec:concentr-mixing-param} is necessary for the ``almost sure'' character of formula \eqref{eq:thm:mixing-param:posterior-concentration} and, in particular, it  is due to the Law of the Iterated Logarithm for trajectory $Y^t$ (see Section~\ref{app:limit-thms}).
				For our purposes it is sufficient to have the result for rate $r(t) = o(\sqrt{t/\log\log t})$ because $\widetilde{\Lambda}^t_{\mathcal{M}}$ is used in the  prior whose effect asymptotically disappears in view of the well-known Bernstein von-Mises phenomenon for Bayesian posteriors; see, e.g.~Section~10.2 in~\citet{vaart2000asymptotic}.\\
			\end{remark}

			The formula for $\widetilde{\lambda}_b^t$ in step~3 of Algorithm~\ref{alg:npl-posterior-sampling:mri:binned} can be rewritten as follows:
			\begin{align}
			\label{eq:proof:asympt-distr-sample-def}
			\widetilde{\lambda}_b^t& = \argmin_{\lambda \succeq 0}A^t(\lambda),\\
			\label{eq:proof:asympt-distr-target-func-def}
			\begin{split}
			A^t(\lambda)&= L_p(\lambda \mid t\widetilde{\Lambda}_{b}^t, A, t, \beta^t) - L_p(\widehat{\lambda}^t_{sc}\mid t\widehat{\Lambda}^t_{sc}, A, t, \beta^t)\\
			&=\sum\limits_{i\in I_1(\Lambda^*)}-t(\widetilde{\Lambda}_{b,i}^t - \widehat{\Lambda}_{sc,i}^t)\log\left(
			\dfrac{\Lambda_i}{\widehat{\Lambda}^t_{sc,i}}
			\right)\\
			&+\sum\limits_{i\in I_1(\Lambda^*)}-t\widehat{\Lambda}_{sc,i}^t \log\left(
			\dfrac{\Lambda_i}{\widehat{\Lambda}_{sc,i}^t}
			\right) + t(\Lambda_i - \widehat{\Lambda}_{sc,i}^t)\\
			& + \sum\limits_{i\in I_0(\Lambda^*)}-t\widetilde{\Lambda}^t_{b,i}\log(t\Lambda_i) + t\Lambda_i\\
			&-\left(\sum\limits_{i\in I_0(\Lambda^*)}-t\widehat{\Lambda}_{sc,i}^t\log(t\widehat{\Lambda}_{sc,i}^t) + t\widehat{\Lambda}_{sc,i}^t\right)\\
			& + \beta^t(\varphi(\lambda) - \varphi(\widehat{\lambda}_{sc}^t)), \, 
			\widehat{\Lambda}_{sc}^t = A\widehat{\lambda}_{sc}^t
			\end{split}
			\end{align}
			where $\widehat{\lambda}^t_{sc}$ is the strongly consistent estimator from \eqref{eq:thm:asympt-distr-main:strongly-consist-estim-positivity}-\eqref{eq:thm:asympt-distr-main:strongly-consist-estim-i0}.
			
			To prove the claim, first, we approximate $A^t(\lambda)$ with quadratic process $B^t(\lambda)$ for which its minimizers have the same asymptotic distribution in the $\mathrm{Span}(A^T)\cap \R^p_+$ as for $A^t(\lambda)$. Second, using this approximation we establish the statements in (i), (ii), but for minimizers of $B^t(\lambda)$  which together with the previous approximation argument completes the proof.
			
			Approximations $B^t(\lambda)$, $\widetilde{\lambda}_{b, app}^t$ of $A^t(\lambda)$, $\widetilde{\lambda}_b^t$ are defined by the formulas:
			\begin{align}\label{eq:conv-approx-bt}
			\begin{split}
			B^t(\lambda) &= \sum\limits_{i\in I_1(\Lambda^*)}-t(\widetilde{\Lambda}_{b,i}^t - \widehat{\Lambda}_{sc,i}^t)\dfrac{\Lambda_i - \widehat{\Lambda}_{sc,i}}{\widehat{\Lambda}_{sc,i}} + t 
			\dfrac{(\Lambda_i-\widehat{\Lambda}_{sc,i})^2}{2\widehat{\Lambda}_{sc,i}}\\
			& + \sum\limits_{i\in I_0(\Lambda^*)}t\Lambda_i, \, 
			\Lambda_i = a_i^T\lambda.
			\end{split}\\
			\label{lem:proofs:conv-approx-lambda-appr}
			\widetilde{\lambda}_{b,app}^t &= \argmin_{\lambda \succeq 0} B^t(\lambda).
			\end{align}
			
			Process $B^t(\lambda)$ is flat in directions from $\ker A$, therefore, though $\widetilde{\lambda}_{b,app}^t$ in \eqref{lem:proofs:conv-approx-lambda-appr} always exists, it may not be unique, and, in general, $\widetilde{\lambda}_{b,app}^t$ is set-valued. In what follows, if not said otherwise, for  $\widetilde{\lambda}_{b,app}^t$ one chooses any point from the set of minimizers (claims will automatically hold for all points in $\widetilde{\lambda}_{b,app}^t$).
			
			It may happen that $a_i^T\widetilde{\lambda}_{b,app}^t = 0$ for some $i\in I_0(\Lambda^*)$, so $A^t(\widetilde{\lambda}_{b,app}^t)$, in general, may not be defined due to the presence of logarithmic terms in \eqref{eq:proof:asympt-distr-target-func-def}. For this reason we approximate $\widetilde{\lambda}_{b,app}^t$ with another auxiliary point $\widetilde{\lambda}_{app}^t$ defined by the formula:
			
			\begin{equation}\label{eq:lambda-app-recentered}
			\widetilde{\lambda}_{app}^t = \widetilde{\lambda}_{b,app}^t + \sum\limits_{i\in I_0(\Lambda^*)} \widetilde{\Lambda}_{b,i}^t\dfrac{a_i}{\|a_i\|^2}, 
			\end{equation}
			where $\widetilde{\Lambda}_{b}^t$ is from step~2 of Algorithm~\ref{alg:npl-posterior-sampling:mri:binned}. It is easy to check that value $A^t(\widetilde{\lambda}^t_{app})$ is always well-defined (for $x=0$ we take convention that $x\log x = 0$).
			
			Let $\mathcal{V}$, $\mathcal{U}$ be the subspaces defined in  \eqref{eq:asymp-distr:subspace-v}, \eqref{eq:asymp-distr:subspace-u},  respectively. From \eqref{eq:lambda-app-recentered} and the definition of $\mathcal{V}$, $\mathcal{U}$ it follows that 
			\begin{equation}\label{eq:asymp-distr:app-abbp-proj-u}
			\Pi_{\mathcal{U}}(\widetilde{\lambda}_{app}^t - \widetilde{\lambda}^t_{b,app}) \equiv 0,
			\end{equation}
			where $\Pi_{\mathcal{U}}$ is defined in \eqref{eq:asymp-distr:subspace-u}.
			For the approximation on $\mathcal{V}$ the following result holds.
			
			\begin{lem}
				\label{lem:approximate-approximation}
				Let $\mathcal{V}$ be the subspace defined in \eqref{eq:asymp-distr:subspace-v}, $\Pi_{\mathcal{V}}$ be defined in \eqref{eq:asymp-distr:projectors}. Then,
				\begin{align}\label{eq:asymp-distr:lem:approximate-approximation:proj-v-convergece}
				t &\Pi_{\mathcal{V}}(\widetilde{\lambda}_{b,app}^t-\widetilde{\lambda}_{app}^t) \xrightarrow{c.p.}0 \text{ when }t\rightarrow +\infty, \text{ a.s. } Y^t, \, t\in (0, +\infty).
				\end{align}
			\end{lem}
			
			Let $\delta > 0$. Consider the two following sets:
			\begin{align}
			\label{eq:thm:asymp-distr:proof:cylinder-d-t-def}
			&D^t_{A,\delta}(\lambda) = \{\lambda'\in \R^p_+ : \lambda' = \lambda + \dfrac{u}{\sqrt{t}} + \dfrac{v}{t} + w, \, u\in \mathcal{U}, \, v\in \mathcal{V}, \, w\in \mathcal{W}, \, \|u\|_2 + \|v\|_1 \leq \delta\},\\
			\label{eq:thm:asymp-distr:proof:cylinder-c-t-def}
			&C^t_{A,\delta}(\lambda) = \{\lambda'\in \R^p_+ : \lambda' = \lambda + \dfrac{u}{\sqrt{t}} + \dfrac{v}{t} + w, \, u\in \mathcal{U}, \, v\in \mathcal{V}, \, w\in \mathcal{W}, \, \|u\|_2 + \|v\|_1 = \delta\},
			\end{align}
			where subspaces $\mathcal{V}, \mathcal{U}, \mathcal{W}$ are defined in \eqref{eq:asymp-distr:subspace-v}-\eqref{eq:asymp-distr:subspace-w}, respectively and $\|\cdot\|_2, \|\cdot\|_1$ denote the standard $\ell_2$ and $\ell_1$-norms in $\R^p$.
			
			The approximation argument for convex process $A^t(\lambda)$ is due to \citet{pollard2011argmin} and is based on the following implication:
			\begin{align}\label{eq:thm:asymp-distr:proof:convexity-arg}
			\widetilde{\lambda}_{app}^t \in \mathrm{int} D_{A,\delta}^t(\widetilde{\lambda}_{b,app}^t), \, 
			\inf_{\lambda \in C^t_{A, \delta}(\widetilde{\lambda}^t_{b, app})}\hspace{-0.5cm}(A^t(\lambda) - A^t(\widetilde{\lambda}^t_{app})) > 0 
			\Rightarrow \widetilde{\lambda}_b^t \in D^t_{A,\delta}(\widetilde{\lambda}_{b,app}^t).
			\end{align}
			From \eqref{eq:asymp-distr:app-abbp-proj-u}, \eqref{eq:asymp-distr:lem:approximate-approximation:proj-v-convergece} (in Lemma~\ref{lem:approximate-approximation}) and \eqref{eq:thm:asymp-distr:proof:cylinder-d-t-def} one can see that for any $\delta > 0$ it holds that
			\begin{equation}
			P(\widetilde{\lambda}_{app}^t \in \mathrm{int} D_{A,\delta}^t(\widetilde{\lambda}_{b,app}^t) \mid Y^t, t) \rightarrow 1 \text{ for } t\rightarrow +\infty, \, \text{a.s. }Y^t, t\in (0, +\infty).
			\end{equation}
			In view of this and \eqref{eq:thm:asymp-distr:proof:convexity-arg}, for the approximation it suffices to establish the following result.
			\begin{lem}\label{lem:thm:asymp-disr:main-approximation-lemma}
				Let $A^t(\lambda)$, $B^t(\lambda)$, $\widetilde{\lambda}^t_{b}$,  $\widetilde{\lambda}_{b,app}^t$, $\widetilde{\lambda}_{app}^t$ be defined in \eqref{eq:proof:asympt-distr-target-func-def}, \eqref{eq:conv-approx-bt}, \eqref{eq:proof:asympt-distr-sample-def}, \eqref{lem:proofs:conv-approx-lambda-appr}, \eqref{eq:lambda-app-recentered}, respectively.
				Then, for any $\delta > 0$ it holds that 
				\begin{equation}\label{lem:asymp-distr:main-approximation-lemma:positivity}
				P\left(\inf_{\lambda\in C^t_{A,\delta}(\widetilde{\lambda}_{b,app}^t)}
				\hspace{-0.5cm}
				[A^t(\lambda) - A^t(\widetilde{\lambda}_{app}^t)] > 0 \, \mid Y^t, t
				\right)\rightarrow 1 \text{ for } t\rightarrow +\infty, \text{ a.s. } Y^t, \, t\in (0, +\infty).
				\end{equation}
				From \eqref{eq:thm:asymp-distr:proof:convexity-arg}, \eqref{lem:asymp-distr:main-approximation-lemma:positivity} it follows that 
				\begin{align}\label{eq:asymp-distr:main-approximation-lemma:proj-u}
				&\sqrt{t} \Pi_{\mathcal{U}}(\widetilde{\lambda}^t_{b} - \widetilde{\lambda}_{b,app}^t) \xrightarrow{c.p.} 0 \text{ when } t\rightarrow +\infty, \text{ a.s. } Y^t, \, t\in (0, +\infty), \\
				\label{eq:asymp-distr:main-approximation-lemma:proj-v}
				&t \Pi_{\mathcal{V}}(\widetilde{\lambda}^t_{b} - \widetilde{\lambda}_{b,app}^t) \xrightarrow{c.p.} 0 \text{ when } t\rightarrow +\infty, \text{ a.s. } Y^t, \, t\in (0, +\infty).
				\end{align}
			\end{lem}
			
			Let 
			\begin{equation}\label{eq:asymp-distr:parametrization-uvw}
			\lambda = \widehat{\lambda}_{sc}^t + \dfrac{u}{\sqrt{t}} + \dfrac{v}{t} + w, \, 
			u\in \mathcal{U}, \, v\in \mathcal{V}, \, w\in \mathcal{W}.
			\end{equation} 
			Process $B^t(\cdot)$ defined in \eqref{eq:conv-approx-bt} has the following form in terms of variables $u,v$ (note that $B^t(\cdot)$ is independent of  $w\in \mathcal{W}$):
			\begin{align}\label{eq:asymp-distr:approximation-general}
			B^t(u,v) &= \widetilde{B}^t(u,v) + \widetilde{R}^t(u, v),\\
			\label{eq:asymp-distr:approximation-b-tilde}
			\widetilde{B}^t(u,v) &= \sum\limits_{i\in I_1(\Lambda^*)} -\sqrt{t}(\widetilde{\Lambda}_{b,i}^t - \widehat{\Lambda}^t_{sc,i})\dfrac{a_i^Tu}{\widehat{\Lambda}^t_{sc,i}} + 
			\dfrac{(a_i^Tu)^2}{2\widehat{\Lambda}^t_{sc,i}} + \sum\limits_{i\in I_0(\Lambda^*)}a_i^Tv, \\
			\label{eq:asymp-distr:remainder-r}
			\begin{split}
			\widetilde{R}^t(u, v) &= \sum\limits_{i\in I_1(\Lambda^*)} -(\widetilde{\Lambda}_{b,i}^t - \widehat{\Lambda}_{sc, i}^t)a_i^Tv + \dfrac{(a_i^Tv)^2}{2\widehat{\Lambda}^t_{sc,i} t} + \dfrac{ (a_i^Tu)(a_i^Tv)}{\sqrt{t}\widehat{\Lambda}^t_{sc,i}} \\
			& + \sum\limits_{i\in I_0(\Lambda^*)}t\widehat{\Lambda}_{sc,i}^t.
			\end{split}
			\end{align}
			Let 
			\begin{align}\label{eq:asymp-distr:quadr-minimizers-def}
			(\widetilde{u}^t, \widetilde{v}^t) = \hspace{-0.7cm}\argmin_{\substack{(u,v):\widehat{\lambda}^t_{sc} + \frac{u}{\sqrt{t}} + \frac{v}{t} + w\succeq 0 \\ 
					u\in \mathcal{U}, \, v\in \mathcal{V}, \, w\in \mathcal{W}}}\hspace{-0.7cm} \widetilde{B}^t(u,v)
			\end{align}

			In particular, from  the definition of $\mathcal{V}$ in \eqref{eq:asymp-distr:subspace-v} and from  \eqref{eq:asymp-distr:parametrization-uvw}, \eqref{eq:asymp-distr:approximation-b-tilde}, \eqref{eq:asymp-distr:quadr-minimizers-def} it follows that 
			\begin{equation}\label{eq:asymp-distr:quadr-minimizer-v-zero}
			\dfrac{\widetilde{v}^t_j}{t} = -\widehat{\lambda}_{sc,j}^t \text{ for }j \text{ s.t. } \exists a_{ij} > 0, \, i\in I_0(\Lambda^*) \Leftrightarrow 
			\Pi_{\mathcal{V}}(\widehat{\lambda}_{sc}^t + \dfrac{\widetilde{v}^t}{t}) = 0.
			\end{equation}
			Indeed, formulas \eqref{eq:asymp-distr:subspace-v}, \eqref{eq:asymp-distr:projectors}, \eqref{eq:asymp-distr:approximation-b-tilde} imply that the choice in \eqref{eq:asymp-distr:quadr-minimizer-v-zero} satisfies the positivity constraint in \eqref{eq:asymp-distr:quadr-minimizers-def} and at the same time minimizes the linear term $\sum_{i\in I_0(\Lambda^*)}a_i^Tv$ since all $a_{ij}$ are non-negative.
			
			\begin{lem}\label{lem:asymp-distr:quadratic-uv-approximation}
				Let $\widetilde{u}^t_{b,app}$, $\widetilde{v}^t_{b,app}$ be defined by \eqref{lem:proofs:conv-approx-lambda-appr}  for parametrization in \eqref{eq:asymp-distr:parametrization-uvw} and 
				$\widetilde{u}^t$, $\widetilde{v}^t$ be defined by \eqref{eq:asymp-distr:quadr-minimizers-def}, respectively. Then,  
				\begin{align}\label{eq:asymp-distr:quadr-approximation-u}
				&\widetilde{u}^t - \widetilde{u}_{b,app}^t \xrightarrow{c.p.} 0 \text{ for } t\rightarrow +\infty, \text{ a.s. }Y^t, t\in (0, +\infty), \\
				\label{eq:asymp-distr:quadr-approximation-v}
				&\widetilde{v}^t - \widetilde{v}_{b,app}^t \xrightarrow{c.p.} 0 \text{ for } t\rightarrow +\infty, \text{ a.s. }Y^t, t\in (0, +\infty).
				\end{align}
			\end{lem}
			
			Hence, in view of \eqref{eq:asymp-distr:main-approximation-lemma:proj-u}, \eqref{eq:asymp-distr:main-approximation-lemma:proj-v} and  Lemma~\ref{lem:asymp-distr:quadratic-uv-approximation} it suffices to demonstrate conditional tightness of $(\widetilde{u}^t, \widetilde{v}^t)$.

			Statement in (i), that is formula \eqref{eq:thm:asympt-distr-main:projection-v}), follows from
			\eqref{eq:asymp-distr:main-approximation-lemma:proj-v},~ 
			\eqref{eq:asymp-distr:quadr-minimizer-v-zero},~ \eqref{eq:asymp-distr:quadr-approximation-v} and the assumption in \eqref{eq:thm:asympt-distr-main:strongly-consist-estim-i0}.\\		
			Now we demonstrate (ii). From \eqref{eq:asymp-distr:quadr-minimizers-def}, \eqref{eq:asymp-distr:quadr-minimizer-v-zero} it follows that 
			\begin{align}\label{eq:asymp-distr:quadr-minimizer-u}
			\widetilde{u}^t = \argmin_{\substack{u:(1-\Pi_{\mathcal{V}})\widehat{\lambda}^t_{sc} + \frac{u}{\sqrt{t}} + w \succeq 0\\ u\in \mathcal{U}, \, w\in \mathcal{W}}}
			\sum\limits_{i\in I_1(\Lambda^*)} -\sqrt{t}(\widetilde{\Lambda}_{b,i}^t - \widehat{\Lambda}^t_{sc,i}) \dfrac{a_i^Tu}{\widehat{\Lambda}^t_{sc,i}} + \dfrac{(a_i^Tu)^2}{2\widehat{\Lambda}^t_{sc,i}}.
			\end{align}
			Since the minmized functional in \eqref{eq:asymp-distr:quadr-minimizer-u} is strongly convex in $u\in \mathcal{U}$ and the set of constraints is also convex, the following mapping is well-defined:
			\begin{align}\label{eq:asymp-distr:quadr-minimizer-mapping-virt-def}
			\widetilde{u}^t(\xi) &= \widetilde{u}(\xi, t)\in \mathcal{U}, \, \xi \in \R^{\# I_1(\Lambda^*)}, \, t\in (0, +\infty), \\
			\label{eq:asymp-distr:quadr-minimizer-mapping-tech-def}
			\widetilde{u}(\xi, t) &= \argmin_{\substack{u:(1-\Pi_{\mathcal{V}})\widehat{\lambda}_{sc}^t + \frac{u}{\sqrt{t}} + w \succeq 0\\ u\in \mathcal{U}, \, w\in \mathcal{W}}} -\xi^T(\widehat{D}^t_{I_1(\Lambda^*)})^{-1/2}A_{I_1(\Lambda^*)}u + \frac{1}{2}u^T\widehat{F}^t_{I_1(\Lambda^*)}u, 
			\end{align}
			where
			\begin{align}
			\label{eq:asymp-distr:quadr-minimizer-mapping-tech-def:diagonal}
			\widehat{D}^t_{I_1(\Lambda^*)} &= \mathrm{diag}(\dots, \widehat{\Lambda}_{sc,i}^t, \dots), \, i\in I_1(\Lambda^*),\\
			\label{eq:asymp-distr:quadr-minimizer-mapping-tech-def:fisher}
			\widehat{F}^t_{I_1(\Lambda^*)} &= \sum\limits_{i\in I_1(\Lambda^*)}\dfrac{a_ia_i^T}{\widehat{\Lambda}_{sc,i}^t} = A^T_{I_1(\Lambda^*)} (\widehat{D}^t_{I_1(\Lambda^*)})^{-1}A_{I_1(\Lambda^*)}.
			\end{align}
			Note that for $\xi = (\dots, \sqrt{t}(\widetilde{\Lambda}_{b,i}^t - \widehat{\Lambda}^t_{sc,i}) / \sqrt{\widehat{\Lambda}^t_{sc,i}},\dots)$, 
			$i\in I_1(\Lambda^*)$, $\widetilde{u}^t(\xi)$ coincides with $\widetilde{u}^t$ from~\eqref{eq:asymp-distr:quadr-minimizer-u}.
			In addition, the minimized functional in \eqref{eq:asymp-distr:quadr-minimizer-mapping-tech-def} does not depend on $w\in \mathcal{W}$ which in turn affects only the set of constraints.

			\begin{lem}\label{lem:asymp-distr:quadr-minimizer-mapping-boundeness}
				Let $\widetilde{u}^t(\xi)$ be the mapping defined in  \eqref{eq:asymp-distr:quadr-minimizer-mapping-virt-def}-\eqref{eq:asymp-distr:quadr-minimizer-mapping-tech-def:fisher}. Then, 
				\begin{align}\label{eq:lem:asymp-distr:quadr-minimizer-mapping-boundeness:norm-bound}
				&\|\widetilde{u}^t(\xi)\| \leq \widehat{c}^t \|A_{I_1(\Lambda^*)}^T\widehat{D}_{I_1(\Lambda^*)}^t\xi\|, \, \xi \in \R^{\# I_1(\Lambda^*)}, \\
				\label{eq:lem:asymp-distr:quadr-minimizer-mapping-boundeness:norm-bound-coeff}
				\begin{split}
				&\widehat{c}^t = \|(\widehat{F}_{I_1(\Lambda^*)}^t)^{-1}\|_{\mathcal{U}}
				\cdot \|(\widehat{F}_{I_1(\Lambda^*)}^t)^{-1/2}\|
				\left(
				\|(\widehat{F}_{I_1(\Lambda^*)}^t)^{-1}\|_{\mathcal{U}}
				+ 2	\hspace{-0.5cm}\max\limits_{\sigma \in \sigma_{\mathcal{U}}(\widehat{F}_{I_1(\Lambda^*)}^t)}
				\hspace{-0.5cm}\sigma^{-1/2}
				\right),
				\end{split}
				\end{align}
				where $\|\cdot\|_{\mathcal{U}}$ denotes the norm of the operator being reduced to subspace $\mathcal{U}$, $\sigma_\mathcal{U}(\cdot)$ denotes the spectrum of the self-adjoint operator acting on $\mathcal{U}$. Moreover, 
				\begin{align}
				& \widehat{c}^t \rightarrow c^* \text{ for } t\rightarrow +\infty,  \text{ a.s. } Y^t, t\in (0, +\infty),\, c^* < +\infty,  \\
				&c_* = \|(F_{I_1(\Lambda^*)}^*)^{-1}\|_{\mathcal{U}}
				\cdot \|(F_{I_1(\Lambda^*)}^*)^{-1/2}\|
				\left(
				\|(F_{I_1(\Lambda^*)}^*)^{-1}\|_{\mathcal{U}}
				+ 2\hspace{-0.5cm}\max\limits_{\sigma \in \sigma_{\mathcal{U}}(F_{I_1(\Lambda^*)}^*)}
				\hspace{-0.5cm}\sigma^{-1/2}
				\right),
				\end{align}
				where
				\begin{align}
				&D_{I_1(\Lambda^*)} = \mathrm{diag}(\dots, \Lambda^*_{i}, \dots), \, i\in I_1(\Lambda^*),\\ 
				\label{eq:eq:lem:asymp-distr:quadr-minimizer-mapping-boundeness:f-star-def}
				&F^*_{I_1(\Lambda^*)} =  \sum\limits_{i\in I_1(\Lambda^*)}\dfrac{a_ia_i^T}{\Lambda_{i}^*} = A_{I_1(\Lambda^*)})^T D^{-1}_{I_1(\Lambda^*)}A_{I_1(\Lambda^*)}.
				\end{align}
			\end{lem}
			
			\begin{lem}\label{lem:asympt-distr:grad-stoch-asymp-normality}
				Let 
				\begin{align}\label{lem:grad-stoch-asymp-normality:xi-tilde}
				&\widetilde{\xi}^t = (\dots, \sqrt{t}(\widetilde{\Lambda}_{b,i}^t - \widehat{\Lambda}^t_{sc,i}) / \sqrt{\widehat{\Lambda}^t_{sc,i}},\dots), \,
				i\in I_1(\Lambda^*), \, \widetilde{\xi}^t\in \R^{\#I_1(\Lambda^*)}.
				\end{align}
				Then, under the assumptions of Theorem~\ref{thm:asympt-distr-main}, family $A_{I_1(\Lambda^*)}^T
				(\widehat{D}^t_{I_1(\Lambda^*)})^{-1/2}\widetilde{\xi}^t$ is conditionally tight. 
			\end{lem}
			
			The result of Lemma~\ref{lem:asympt-distr:grad-stoch-asymp-normality} together with formulas \eqref{eq:lem:asymp-distr:quadr-minimizer-mapping-boundeness:norm-bound}-\eqref{eq:eq:lem:asymp-distr:quadr-minimizer-mapping-boundeness:f-star-def} imply that $\widetilde{u}^t = \widetilde{u}^t(\widetilde{\xi}^t)$ is conditionally tight almost surely $Y^t$, $t\in (0, +\infty)$. 		Statement (ii) of the lemma follows directly from this and 
			formulas \eqref{eq:asymp-distr:main-approximation-lemma:proj-u}, \eqref{eq:asymp-distr:quadr-approximation-u} from  lemmas~\ref{lem:thm:asymp-disr:main-approximation-lemma}, \ref{lem:asymp-distr:quadratic-uv-approximation}, respectively.
			
			Theorem is proved.

			\subsection{Proof of Theorem~\ref{thm:asympt-distr:well-spec:concentr-mixing-param}}
			\begin{proof}
				Claim in \eqref{eq:thm:mixing-param:posterior-concentration} directly follows from \eqref{eq:thm:mixing-param:posterior-lambda-concentration} and the Continuous Mapping Theorem, so we prove only \eqref{eq:thm:mixing-param:posterior-lambda-concentration}.
				
				Step~2 in Algorithm~\ref{alg:wbb-pet-bootstrap:mri:posterior-mixing-param} can be  rewritten as follows:
				\begin{align}
				\label{eq:proof:posterior-concentration-argmin-def}
				\widetilde{\lambda}^{t}_{\mathcal{M}} &= \argmin_{\lambda_{\mathcal{M}}\succeq 0}
				L_{\mathcal{M}}(\lambda_{\mathcal{M}} \mid \widetilde{\Lambda}^t),\\
				\label{eq:proof:posterior-concentration-functional-def}
				\begin{split}
				L_{\mathcal{M}}(\lambda_{\mathcal{M}} \mid \widetilde{\Lambda}^t) &= \sum_{i\in I_1(\Lambda^*)} -\log\left(
				\dfrac{\Lambda_{\mathcal{M},i}}{\Lambda_{\mathcal{M},i}^*}
				\right) (\widetilde{\Lambda}^t_i - \Lambda_i^*) \\
				& + L(\lambda_{\mathcal{M}} \mid \Lambda^*, A_{\mathcal{M}}, 1) - 
				L(\lambda_{\mathcal{M},*} \mid \Lambda^*, A_{\mathcal{M}}, 1),
				\end{split}
				\end{align}
				where $\lambda_{\mathcal{M},*}$ is the point from Theorem~\ref{thm:asympt-distr:well-spec:identifiability-cond}, $\Lambda_{\mathcal{M}}^* = A_{\mathcal{M}}\lambda_{\mathcal{M},*}$, and 
				\begin{align}
				\label{eq:proof:posterior-concenctration-resample-stat-props}
				\begin{split}
				&\widetilde{\Lambda}_i^t \sim \Gamma(Y_i^t, t^{-1}), \, i=1, \dots, d, \text{ are mutually independent}, \\
				&E[\widetilde{\Lambda}_i^t \mid Y^t, t] = Y_i^t / t, \, 
				\mathrm{var}[\widetilde{\Lambda}_i^t \mid Y^t, t] = Y_i^t / t^2, \, i\in \{1, \dots, d\}.
				\end{split}
				\end{align}
				Note that 
				\begin{equation}\label{eq:proof:posterior-concentration-props-func}
				L_{\mathcal{M}}(\lambda_{\mathcal{M}} \mid \widetilde{\Lambda}^t) \text{ is convex on } \R^{p_{\mathcal{M}}}_+, \, L_{\mathcal{M}}(\lambda_{\mathcal{M},*} \mid \widetilde{\Lambda}^t) = 0.
				\end{equation}
				For fixed $t > 0$ consider the following parametrization
				\begin{align}\label{eq:proof:posterior-concentration-parmetrization}
				\lambda_{\mathcal{M}} = \lambda_{\mathcal{M},*} + \dfrac{u_{\mathcal{M}}}{r(t)}, \, \lambda_{\mathcal{M}}\in \R^{p_{\mathcal{M}}}_+, \, r(t) = o(\sqrt{t/\log\log t}).
				\end{align}
				
				Let $\delta > 0$. In view of  \eqref{eq:proof:posterior-concentration-argmin-def},  \eqref{eq:proof:posterior-concentration-props-func}, \eqref{eq:proof:posterior-concentration-parmetrization} the following implication holds
				
				\begin{equation}\label{eq:proof:posterior-concentration-conv-target}
				\inf_{\substack{\lambda_{\mathcal{M}}:\|u_{\mathcal{M}}\| = \delta, \\ \lambda_{\mathcal{M}} \succeq 0}} L_{\mathcal{M}}(\lambda_{\mathcal{M}} \mid \widetilde{\Lambda}^t) > 0 \Rightarrow 
				r(t)\|\widetilde{\lambda}^t_{\mathcal{M}}-\lambda_{\mathcal{M},*}\| < \delta.
				\end{equation}
				Therefore, to prove \eqref{eq:thm:mixing-param:posterior-lambda-concentration} it is sufficient to show that for any small $\delta > 0$ the conditional probability of the event in the left hand-side of  \eqref{eq:proof:posterior-concentration-conv-target} tends to one for $t\rightarrow +\infty$, a.s.  $Y^t, \, t\in (0,+\infty)$.
				
				Let $C_*, \delta_*$ be the values of \eqref{eq:thm:asympt-distr:identif-cond:strong-convexity} from  Theorem~\ref{thm:asympt-distr:well-spec:identifiability-cond} and let $\|u_{\mathcal{M}}\| = \delta, \, \delta < \delta_*$.

				Using  \eqref{eq:thm:asympt-distr:identif-cond:strong-convexity} and  \eqref{eq:proof:posterior-concentration-functional-def}, \eqref{eq:proof:posterior-concentration-parmetrization} we get the following estimate:
				\begin{align}
				\label{eq:proof:posterior-concentration-chain-log-inequality}
				\begin{split}
				L(\lambda_{\mathcal{M}} \mid \widetilde{\Lambda}^t) &\geq 
				\sum\limits_{i\in I_1(\Lambda^*)} -\log\left(
				1 + \dfrac{u_{\mathcal{M}}^Ta_{\mathcal{M},i}}{r(t)\Lambda_{\mathcal{M},i}^*}
				\right) (\widetilde{\Lambda}_i^t - \Lambda_i^*) + C_*\delta^2 / r^2(t) \\
				&\geq C_{*} \delta^2 / r^2(t) - \sum_{i\in I_1(\Lambda^*)} \dfrac{\mid u_{\mathcal{M}}^Ta_{\mathcal{M},i}\mid}{r(t)\Lambda^*_{\mathcal{M},i}}  \mid\widetilde{\Lambda}_i^t - \Lambda_i^*\mid \\
				&= r^{-2}(t) \left(
				C_* \delta^2 - \sum_{i\in I_1(\Lambda^*)} \dfrac{\mid u_{\mathcal{M}}^Ta_{\mathcal{M},i}\mid}{\Lambda^*_{\mathcal{M},i}}  r(t)\mid\widetilde{\Lambda}_i^t - \Lambda_i^*\mid
				\right)\\
				&\geq r^{-2}(t) \left(
				C_* \delta^2 - \sum_{i\in I_1(\Lambda^*)} \dfrac{\delta \|a_{\mathcal{M},i}\|}{\Lambda^*_{\mathcal{M},i}}  r(t)\mid\widetilde{\Lambda}_i^t - \Lambda_i^*\mid
				\right).
				\end{split}
				\end{align}
				Note that 
				in \eqref{eq:proof:posterior-concentration-chain-log-inequality}  we have used the property that $\log(1+x) \leq x$, $x\in (-1, +\infty)$.
				
				Estimate in \eqref{eq:proof:posterior-concentration-chain-log-inequality} implies the left hand-side of  \eqref{eq:proof:posterior-concentration-conv-target}, for example, if 
				\begin{equation}\label{eq:proof:posterior-concentration-conv-requirement}
				r(t) \mid\widetilde{\Lambda}_i^t - \Lambda_i^*\mid \xrightarrow{c.p.} 0 \text{ for }t\rightarrow +\infty, \text{ a.s. } Y^t, t\in (0, +\infty), \, 
				i\in I_1(\Lambda^*).
				\end{equation}
				To demonstrate \eqref{eq:proof:posterior-concentration-conv-requirement} we use Markov inequality together with  \eqref{eq:proof:posterior-concenctration-resample-stat-props} and arrive to the following estimate
				\begin{align}\label{eq:proof:posterior-concentration-markov-in}
				\begin{split}
				P(r(t) \mid\widetilde{\Lambda}_i^t - \Lambda^*_i\mid > \varepsilon \mid Y^t, t ) &\leq 
				\dfrac{r^2(t)E(\mid\widetilde{\Lambda}_i^t - \Lambda_i^*\mid^2 \mid Y^t, t)}{\varepsilon^2}\\
				&\leq \dfrac{2r^2(t)E(\mid\widetilde{\Lambda}_i^t - Y_i^t/t\mid^2 \mid Y^t, t) + 2r^2(t)\mid Y_i^t - \Lambda_i^*\mid^2}{\varepsilon^2} \\
				& = \dfrac{2r^2(t)/t^2 + 2\mid r(t)(Y^t_i/t - \Lambda_i^*)\mid^2}{\varepsilon^2},
				\end{split}
				\end{align}
				where $\varepsilon > 0$ is arbitrary. For $r(t) = o(\sqrt{t/\log\log t})$ it holds that  (see Section~\ref{app:limit-thms}):
				\begin{align}\label{eq:proof:posterior-concentration-rates-effect}
				r^2(t) / t^2 \rightarrow 0 \text{ and }  r(t)(Y_i^t/t - \Lambda_i^*) 
				\rightarrow 0 \text{ a.s. } Y^t, t\in (0,+\infty)\text{}.
				\end{align}
				Therefore, from \eqref{eq:proof:posterior-concentration-markov-in}, \eqref{eq:proof:posterior-concentration-rates-effect} it follows that formula \eqref{eq:proof:posterior-concentration-conv-requirement} holds which together with \eqref{eq:proof:posterior-concentration-chain-log-inequality} imply \eqref{eq:proof:posterior-concentration-conv-target}.
				
				Theorem is proved.
				
			\end{proof}

		\end{proof}
		
		\subsection{Proof of Lemma~\ref{lem:approximate-approximation}}
		
		\begin{proof}
			To prove the claim is suffices to show that
			\begin{equation}\label{eq:lem:approximate-approximation:sufficient-cond}
			t \widetilde{\Lambda}^t_{b,i} \xrightarrow{c.p.} 0 \text{ for }i\in I_0(\Lambda^*)\text{ for } t\rightarrow +\infty, \text{ a.s. } 
			Y^t, \, t\in (0, +\infty).
			\end{equation}
			Let $\delta > 0$. Using  step~2 in Algorithm~\ref{alg:npl-posterior-sampling:mri:binned} and   Assumption~\ref{assump:theory:consistency:well-spec} we obtain 
			\begin{align}\nonumber
			\label{eq:lem:approximate-approximation:markov-ineq}
			P(t \widetilde{\Lambda}_{b,i}^t > \delta \mid Y^t, t) &= 
			\int\limits_{0}^{+\infty} P(t \widetilde{\Lambda}_{b,i}^t > \delta \mid \widetilde{\Lambda}_{\mathcal{M},i}^t = \Lambda, Y^t, t) \, P(\widetilde{\Lambda}_{\mathcal{M},i}^t = \Lambda \mid Y^t, t) d\Lambda \\ 
			& \leq \int\limits_{0}^{+\infty}\min\left(\dfrac{t\theta^t\Lambda}{(\theta^t + t)\delta}, 1\right) P(\widetilde{\Lambda}_{\mathcal{M},i}^t = \Lambda \mid Y^t, t) d\Lambda \\ \nonumber
			&\leq \int\limits_{0}^{\frac{(\theta^t + t)\delta}{t\theta^t}}
			\dfrac{t\theta^t\Lambda}{(\theta^t + t)\delta} P(\widetilde{\Lambda}_{\mathcal{M},i}^t = \Lambda \mid Y^t, t) d\Lambda + 
			P\left(\dfrac{t\theta^t\widetilde{\Lambda}_{\mathcal{M},i}^t}{\theta^t + t} > \delta \mid Y^t, t\right).
			\end{align}
			In \eqref{eq:lem:approximate-approximation:markov-ineq} we have used the Markov inequality for $\Lambda_{b}^t \mid Y^t, t, \widetilde{\Lambda}_{\mathcal{M}}^t$, $i\in I_0(\Lambda^*)$ for which it is known that $\Lambda_{b,i}^t \mid Y^t, t, \widetilde{\Lambda}_{\mathcal{M},i}^t \sim \Gamma(\theta^t \widetilde{\Lambda}^t_{\mathcal{M},i}, (t + \theta^t)^{-1})$.
			
			The last term in \eqref{eq:lem:approximate-approximation:markov-ineq} tends to zero a.s. $Y^t, \, t\in (0, +\infty)$ due to \eqref{eq:thm:mixing-param:posterior-concentration} from Theorem~\ref{thm:asympt-distr:well-spec:concentr-mixing-param}.

			Next, we show that the first integral in \eqref{eq:lem:approximate-approximation:markov-ineq} it is arbitrarily small a.s. $Y^t$, $t\in (0, +\infty)$ and, hence, tends to zero a.s. $Y^t$, $t\in (0, +\infty)$.
			The integral in \eqref{eq:lem:approximate-approximation:markov-ineq} is rewritten as follows:
			\begin{align}
			\begin{split}\label{eq:lem:approximate-approximation:markov-ineq:first-term}
			&\int\limits_{0}^{\frac{(\theta^t + t)\delta}{t\theta^t}}
			\dfrac{t\theta^t\Lambda}{(\theta^t + t)\delta} P(\widetilde{\Lambda}_{\mathcal{M},i}^t = \Lambda \mid Y^t, t) d\Lambda = \\ 
			&= \dfrac{\delta(\theta^t + t)}{t\theta^t}\int\limits_{0}^{1}
			s P(\theta^t \widetilde{\Lambda}_{\mathcal{M},i}^t = s\delta(t + \theta^t)/t \, \vert \, Y^t, t)\, ds.
			\end{split}
			\end{align}
			
			Let $ 0 < \varepsilon < 1$. Then, by splitting the integral in \eqref{eq:lem:approximate-approximation:markov-ineq:first-term} we obtain the following estimate: 
			\begin{align}\label{eq:lem:approximate-approximation:markov-ineq:first-term-split}
			\begin{split}
			\dfrac{\delta(\theta^t + t)}{t\theta^t}\int\limits_{0}^{1}
			s &P(\theta^t \Lambda_{\mathcal{M},i}^t = s\delta(t + \theta^t)/t \, \mid \, Y^t, t)\, ds = \int\limits_{0}^{\varepsilon}\dots \, ds +  \int\limits_{\varepsilon}^{1}\dots \, ds  \\
			&\leq \varepsilon + P(\theta^t \Lambda_{\mathcal{M},i}^t > \varepsilon\delta(t + \theta^t) /t \,  \mid \, Y^t, t).
			\end{split}
			\end{align}
			For fixed $\varepsilon > 0$, $\delta > 0$, the second term in \eqref{eq:lem:approximate-approximation:markov-ineq:first-term-split} tends to zero for $t\rightarrow +\infty$, a.s. $Y^t, \, t\in (0, +\infty)$, again due to \eqref{eq:thm:mixing-param:posterior-concentration} from Theorem~\ref{thm:asympt-distr:well-spec:concentr-mixing-param}. Since $\varepsilon$ can be arbitrarily small, it follows that the integral in \eqref{eq:lem:approximate-approximation:markov-ineq:first-term-split} is also arbitrarily small for $t\rightarrow +\infty$, a.s. $Y^t, \, t\in (0, +\infty)$. Hence, the integral in  \eqref{eq:lem:approximate-approximation:markov-ineq:first-term}, and most importantly the right hand-side in \eqref{eq:lem:approximate-approximation:markov-ineq}  converge to zero when $t\rightarrow +\infty$, a.s. $Y^t, \, t\in (0, +\infty)$.
			Since initial $\delta$ was chosen arbitrarily, this proves the convergence in  \eqref{eq:lem:approximate-approximation:sufficient-cond}.
			
			Lemma is proved.
		\end{proof}
		
		\subsection{Proof of Lemma~\ref{lem:thm:asymp-disr:main-approximation-lemma}}
		\begin{proof}
			Let $\delta > 0$.
			The left hand-side of \eqref{eq:thm:asymp-distr:proof:convexity-arg} can be estimated as follows:
			\begin{align}\label{eq:thm:asymp-distr:proof:a-diff-representation}
			\begin{split}
			\inf_{\lambda\in C^t_{A,\delta}(\widetilde{\lambda}^t_{b,app})}\hspace{-0.5cm}[A^t(\lambda) - A^t(\widetilde{\lambda}_{app}^t)] &\geq \hspace{-0.3cm}\inf_{\lambda\in C^t_{A,\delta}(\widetilde{\lambda}^t_{b,app})}\hspace{-0.5cm}[A^t(\lambda) - B^t(\lambda)] \\
			& + \hspace{-0.3cm}\inf_{\lambda\in C^t_{A,\delta}(\widetilde{\lambda}^t_{b,app})}\hspace{-0.5cm}[B^t(\lambda) - B^t(\widetilde{\lambda}^t_{b,app})] \\
			& + [B^t(\widetilde{\lambda}^t_{b,app}) - B^t(\widetilde{\lambda}^t_{app})] \\
			& + [B^t(\widetilde{\lambda}^t_{app}) - A^t(\widetilde{\lambda}_{app}^t)].
			\end{split}
			\end{align}

			We will show that under the assumptions of Theorem~\ref{thm:asympt-distr-main} the following holds:
			\begin{align}\label{eq:asymp-distr:lemm-approximation:inf-estimate-main}
			\inf_{\lambda\in C^t_{A,\delta}(\widetilde{\lambda}^t_{b,app})}\hspace{-0.5cm}[A^t(\lambda) - A^t(\widetilde{\lambda}_{app}^t)] \geq \inf_{\lambda\in C^t_{A,\delta}(\widetilde{\lambda}^t_{b,app})} \hspace{-0.5cm}[B^t(\lambda) - B^t(\widetilde{\lambda}_{b,app}^t)] + o_{cp}(1).
			\end{align}
			The first term in right hand-side of \eqref{eq:asymp-distr:lemm-approximation:inf-estimate-main} is expected to be positively separated from zero in view of \eqref{lem:proofs:conv-approx-lambda-appr}, \eqref{eq:thm:asymp-distr:proof:cylinder-c-t-def}, and in fact, it gives the main contribution for \eqref{eq:thm:asymp-distr:proof:convexity-arg} to hold.
			This is described precisely by the following lemma.
			\begin{lem}
				\label{lem:positivity-bt}
				Let $B^t(\lambda)$, $\widetilde{\lambda}^t_{b, app}$ be defined in \eqref{eq:conv-approx-bt}, \eqref{lem:proofs:conv-approx-lambda-appr}, respectively. Then, the following formulas hold:
				\begin{align}\label{eq:lem:positivity-bt-full-range}
				\begin{split}
				&B^t(\lambda) - B^t(\widetilde{\lambda}^t_{b,app}) = 
				\sum\limits_{i\in I_1(\Lambda^*)} \dfrac{t(\Lambda_i - \widetilde{\Lambda}_{b,app,i}^t)^2}{2\widehat{\Lambda}_{sc,i}^t} + t\langle
				\widetilde{\mu}_{b,app}^t, \lambda
				\rangle, \\
				&\lambda\in \R^p_+, \, \widetilde{\Lambda}^t_{b, app} = A\widetilde{\lambda}^t_{b, app}, 
				\end{split}
				\end{align}
				where 
				\begin{align}\label{eq:lem:positivity-bt-full-range:mub-formula}
				&\widetilde{\mu}_{b, app}^t = \sum\limits_{i\in I_1(\Lambda^*)}
				\dfrac{\widetilde{\Lambda}^t_{b, app,i} - \widetilde{\Lambda}^t_{b,i}}{\widehat{\Lambda}_{sc,i}^t}a_i + \sum\limits_{i\in I_0(\Lambda^*)}a_i, \\
				\label{eq:lem:positivity-bt-full-range:mub-props}
				&\widetilde{\mu}_{b, app}^t \in \R^p_+, \, \widetilde{\mu}_{b, app,j}^t  \widetilde{\lambda}^t_{b,app, j} = 0 \text{ for all }j\in \{1, \dots, p\}.
				\end{align}
			\end{lem}
			We show that \eqref{eq:asymp-distr:lemm-approximation:inf-estimate-main} and the result of Lemma~\ref{lem:positivity-bt} imply the statement in~\eqref{lem:asymp-distr:main-approximation-lemma:positivity}.\\
			Let 
			\begin{equation}
			\label{eq:asymp-distr:lemm-approximation:bt-kt-param-recall}
			\lambda(u,v,w) = \widetilde{\lambda}^t_{b,app} + \dfrac{u}{\sqrt{t}} + \dfrac{v}{t} + w, \, u\in \mathcal{U}, \, v\in \mathcal{V}, \, w\in \mathcal{W}, \, \lambda(u, v, w)\in \R^p_+.
			\end{equation}
			
			Using the parametrization from \eqref{eq:asymp-distr:lemm-approximation:bt-kt-param-recall}, the definition of $C_{A,\delta}^t(\cdot)$ in   \eqref{eq:thm:asymp-distr:proof:cylinder-c-t-def} and  \eqref{eq:lem:positivity-bt-full-range}-\eqref{eq:lem:positivity-bt-full-range:mub-props} from Lemma~\ref{eq:asymp-distr:lemm-approximation:inf-estimate-main} we obtain
			\begin{align}\label{eq:asymp-distr:lemm-approximation:bt-diff-inf-estim}
			B^t(\lambda) - B^t(\widetilde{\lambda}_{b,app}^t) &= K^t(u, v, w) + R^t(u,v, w), \, \lambda = \lambda(u, v, w), \\
			\label{eq:asymp-distr:lemm-approximation:kt-def}
			\begin{split}
			K^t(u,v, w) &= \sum\limits_{i\in I_1(\Lambda^*)} \dfrac{(a_i^Tu)^2}{2\widehat{\Lambda}_{sc,i}^t} + t\langle \widetilde{\mu}_{b,app}^t, \, \lambda(u,v,w) \rangle\\
			&=\sum\limits_{i\in I_1(\Lambda^*)} \dfrac{(a_i^Tu)^2}{2\widehat{\Lambda}_{sc,i}^t} + t\langle \widetilde{\mu}_{b,app}^t, \, \lambda(u,v,w)-\widetilde{\lambda}_{b,app}^t \rangle\\
			&=\sum\limits_{i\in I_1(\Lambda^*)} \dfrac{(a_i^Tu)^2}{2\widehat{\Lambda}_{sc,i}^t} + t\langle \widetilde{\mu}_{b,app}^t, \frac{u}{\sqrt{t}} + \frac{v}{t} \rangle \\
			&= \sum\limits_{i\in I_1(\Lambda^*)} \dfrac{(a_i^Tu)^2}{2\widehat{\Lambda}_{sc,i}^t} + 
			\sum\limits_{i\in I_1(\Lambda^*)} \sqrt{t} \dfrac{\widetilde{\Lambda}_{b, app,i}^t - \widetilde{\Lambda}^t_{b,i}}{\widehat{\Lambda}_{sc,i}^t}a_i^Tu \\
			& + \sum\limits_{i\in I_1(\Lambda^*)}\dfrac{\widetilde{\Lambda}_{b,app,i}^t - \widetilde{\Lambda}_{b,i}^t}{\widehat{\Lambda}_{sc,i}^t} a_i^Tv + \sum\limits_{i\in I_0(\Lambda^*)} a_i^Tv,
			\end{split}\\
			\label{eq:asymp-distr:lemm-approximation:rt-def}
			R^t(u,v, w) &= \sum\limits_{i\in I_1(\Lambda^*)} 
			\dfrac{(a_i^Tu)(a_i^Tv)}{\sqrt{t}\widehat{\Lambda}^t_{sc,i}} + \dfrac{(a_i^Tv)^2}{2t\widehat{\Lambda}^t_{sc,i}}.
			\end{align}
			From the fact that $\widehat{\Lambda}_{sc,i}^t \rightarrow \Lambda_i^*$ a.s. $Y^t, t\in (0, +\infty)$ ($\widehat{\lambda}_{sc}^t$ is strongly consistent at $\lambda_*$ on $\mathcal{U}\oplus \mathcal{V}$ by the assumption), 
			the definition of $C^t_{A,\delta}(\cdot)$ in \eqref{eq:thm:asymp-distr:proof:cylinder-c-t-def} and \eqref{eq:asymp-distr:lemm-approximation:rt-def} it follows that 
			\begin{align}\label{eq:asymp-distr:lemm-approximation:bt-kt-rt-o-small}
			\sup_{\lambda(u,v,w)\in C^t_{A,\delta}(\widetilde{\lambda}_{b,app}^t)}
			\hspace{-0.8cm}\mid R^t(u,v, w)\mid &= o_{cp}(1).
			\end{align}
			In view of formulas \eqref{eq:asymp-distr:parametrization-uvw},  \eqref{eq:asymp-distr:quadr-minimizer-v-zero}, the results of Lemmas~\ref{lem:asymp-distr:quadratic-uv-approximation}-\ref{lem:asympt-distr:grad-stoch-asymp-normality} and again the fact that $\widehat{\Lambda}_{sc,i}^t \rightarrow \Lambda_i^*$, we find that 
			\begin{align}
			\label{eq:asymp-distr:lemm-approximation:bt-kt-rt-mu-r-o-small}
			\begin{split}
			\dfrac{\widetilde{\Lambda}^t_{b,app,i}-\widetilde{\Lambda}^t_{b,i}}{\widehat{\Lambda}_{sc,i}^t} &= \dfrac{\widetilde{\Lambda}^t_{b,app,i}-\widehat{\Lambda}_{sc,i}^t}{\widehat{\Lambda}_{sc,i}^t} + \dfrac{\widehat{\Lambda}_{sc,i}^t-\widetilde{\Lambda}^t_{b,i}}{\widehat{\Lambda}_{sc,i}^t} = 
			o_{cp}(1), \, i\in I_1(\Lambda^*).
			\end{split}
			\end{align}
			
			Formulas \eqref{eq:asymp-distr:lemm-approximation:bt-diff-inf-estim}-\eqref{eq:asymp-distr:lemm-approximation:bt-kt-rt-mu-r-o-small} imply that 
			\begin{align}\label{eq:asymp-distr:lemm-approximation:inf-estimate-bt-kt}
			\inf_{\lambda\in C^t_{A,\delta}(\widetilde{\lambda}_{b,app}^t)}\hspace{-0.5cm}[B^t(\lambda) - B^t(\widetilde{\lambda}_{b,app}^t)] \geq 
			\hspace{-0.5cm}\inf_{\lambda(u,v,w)\in C^t_{A,\delta}(\widetilde{\lambda}_{b,app}^t)}
			\hspace{-1cm}
			K^t(u,v, w) + o_{cp}(1).
			\end{align}
			Now, note that if $\lambda(u, v, w)\succeq 0$ (see formula~\eqref{eq:asymp-distr:lemm-approximation:bt-kt-param-recall}), then 
			\begin{equation}\label{eq:asymp-dist:lemm-approximation:positivity-grad-u-lambda}
			\lambda(u, 0, w)\succeq 0.
			\end{equation}
			Indeed, from the definition of $\mathcal{V}$, $\mathcal{U}$, $\mathcal{W}$ in \eqref{eq:asymp-distr:subspace-v}-\eqref{eq:asymp-distr:subspace-w} it follows that $u$ and $v$ have disjoint set of non-zero components, therefore, setting $v$ to zero for $\lambda(u, v, w)$ cannot break the positivity constraint.
			
			From \eqref{eq:lem:positivity-bt-full-range:mub-formula}, \eqref{eq:lem:positivity-bt-full-range:mub-props}, \eqref{eq:asymp-dist:lemm-approximation:positivity-grad-u-lambda} it follows that 
			\begin{align}\label{eq:asymp-distr:lemm-approximation:kt-grad-positivity}
			\langle \widetilde{\mu}_{b,app}^t, \lambda(u, 0, w)\rangle = 
			\sum\limits_{i\in I_1(\Lambda^*)} \dfrac{\widetilde{\Lambda}_{b, app,i}^t - \widetilde{\Lambda}^t_{b,i}}{\widehat{\Lambda}_{sc,i}^t}a_i^Tu \geq 0.
			\end{align}
			Note also that $K^t(u,v,w)$ in \eqref{eq:asymp-distr:lemm-approximation:kt-def} does not change when varying $w\in \mathcal{W}$, so, in what follows we write $K^t(u,v)$ instead. Using formulas \eqref{eq:asymp-distr:lemm-approximation:kt-def}, \eqref{eq:asymp-distr:lemm-approximation:bt-kt-rt-mu-r-o-small}, \eqref{eq:asymp-distr:lemm-approximation:kt-grad-positivity} and the definition of $C_{A,\delta}^t(\cdot)$ in \eqref{eq:thm:asymp-distr:proof:cylinder-c-t-def} 
			we find that 
			\begin{align}\label{eq:thm:asymp-distr:proof:kt-uv-lower-bound}
			K^t(u, v) \geq \sum\limits_{i\in I_1(\Lambda^*)} \dfrac{(a_i^Tu)^2}{2\widehat{\Lambda}_{sc,i}^t} + \sum\limits_{i\in I_0(\Lambda^*)}a_i^Tv + o_{cp}(1), \, \lambda = \lambda(u,v,w)\in C^t_{A,\delta}(\widetilde{\lambda}_{b,app}^t).
			\end{align}
			where the term $o_{cp}(1)$ tends to zero uniformly on $C^t_{A,\delta}(\widetilde{\lambda}_{b,app}^t)$ for $t\rightarrow +\infty$, a.s. $Y^t$, $t\in (0, +\infty)$.				
			From \eqref{eq:thm:asymp-distr:proof:kt-uv-lower-bound} and strong consistency of $\widehat{\lambda}_{sc}^t$ on $\mathcal{U}\oplus \mathcal{V}$ it follows that 
			\begin{align}\label{eq:thm:asymp-distr:proof:kt-uv-norm-bound}
			K^t(u,v) \geq c_1\|u\|_2^2 + c_2 \|v\|_1 + o_{cp}(1), \text{ if }
			\Pi_{\mathcal{V}}\widetilde{\lambda}_{b,app}^t = 0, \, 
			\lambda = \lambda(u,v,w)\in C^t_{A,\delta}(\widetilde{\lambda}_{b,app}^t),
			\end{align}
			where $c_1$, $c_2$ are some fixed positive constants which depend only on $\Lambda^*$ and $A$. The bound above holds for $t$ large enough a.s. $Y^t$, $t\in (0, +\infty)$.
			
			Recall that 
			\begin{align}\label{eq:thm:asymp-distr:proof:kt-uv-cadelta-restriction}
			\|u\|_2 + \|v\|_1 = \delta \text{ for } \lambda(u,v,w)\in C^t_{A,\delta}(\widetilde{\lambda}_{b,app}^t).
			\end{align}
			Using \eqref{eq:thm:asymp-distr:proof:kt-uv-norm-bound}, \eqref{eq:thm:asymp-distr:proof:kt-uv-cadelta-restriction} it is easy to see that 
			\begin{align}\label{eq:thm:asymp-distr:proof:kt-uv-delta-estim}
			K^t(u,v) \geq c \delta^2 + o_{cp}(1), \text{ if }
			\Pi_{\mathcal{V}}\widetilde{\lambda}_{b,app}^t = 0, \, 
			\lambda = \lambda(u,v,w)\in C^t_{A,\delta}(\widetilde{\lambda}_{b,app}^t),
			\end{align}
			for $\delta$ small enough (smaller than some universal constant depending on $c_1$, $c_2$), where $c$ is some fixed constant also depending on $c_1$, $c_2$ from \eqref{eq:thm:asymp-distr:proof:kt-uv-norm-bound}.
			Note that the Karush-Kuhn-Tucker optimality conditions in \eqref{eq:lem:positivity-bt-full-range:mub-formula}, \eqref{eq:lem:positivity-bt-full-range:mub-props}, formula \eqref{eq:asymp-distr:lemm-approximation:bt-kt-rt-mu-r-o-small} and the definition of space $\mathcal{V}$ in \eqref{eq:asymp-distr:subspace-v} imply that 
			\begin{align}\label{eq:asymp-distr:lemm-approximation:positivity-v-prob}
			P( \Pi_{\mathcal{V}} \widetilde{\lambda}_{b, app}^t = 0 \, \mid\, Y^t, t) \rightarrow 1
			\text{ when }t\rightarrow +\infty, \text{ a.s. }Y^t, t\in (0, +\infty).
			\end{align}
			Hence, the event in \eqref{eq:thm:asymp-distr:proof:kt-uv-delta-estim} is conditioned on $\{\Pi_{\mathcal{V}}\widetilde{\lambda}_{b,app}^t=0\}$ which has  asymptotic conditional probability tending to one a.s. $Y^t$, $t\in (0, +\infty)$, and it also holds 
			\begin{align}\label{eq:thm:asymp-distr:proof:kt-uv-delta-estim-final}
			K^t(u,v) \geq c \delta^2 + o_{cp}(1), \, 
			\lambda = \lambda(u,v,w)\in C^t_{A,\delta}(\widetilde{\lambda}_{b,app}^t).
			\end{align}

			From \eqref{eq:asymp-distr:lemm-approximation:inf-estimate-main}, \eqref{eq:asymp-distr:lemm-approximation:inf-estimate-bt-kt}, \eqref{eq:thm:asymp-distr:proof:kt-uv-delta-estim-final} it follows that 
			\begin{align}\label{eq:asymp-distr:lemm-approximation:zeta-t-implication-bound}
			P\left(\inf_{\lambda\in 	C^t_{A,\delta}(\widetilde{\lambda}^t_{b,app})}\hspace{-0.5cm}[A^t(\lambda) - A^t(\widetilde{\lambda}_{app}^t)] > 0 \, \mid \, Y^t, t
			\right) \rightarrow 1 \text{ for }t\rightarrow +\infty, \text{ a.s. } Y^t, t\in (0, +\infty). 
			\end{align}
			
			It is left to demonstrate the initial statement in  \eqref{eq:asymp-distr:lemm-approximation:inf-estimate-main}. 
			Consider the first term in the left hand-side of \eqref{eq:thm:asymp-distr:proof:a-diff-representation}.
			Using \eqref{eq:proof:asympt-distr-target-func-def}, 
			\eqref{eq:conv-approx-bt}, the definitions in 
			\eqref{lem:proofs:conv-approx-lambda-appr}, \eqref{eq:thm:asymp-distr:proof:cylinder-c-t-def} and the facts that 
			$\Pi_{\mathcal{V}\oplus \mathcal{U}}(\widehat{\lambda}^t_{sc}-\lambda_*) \xrightarrow{a.s.} 0$, $\Pi_{\mathcal{V}\oplus \mathcal{U}}(\widetilde{\lambda}_{b,app}^t - \lambda_*)\xrightarrow{c.p.}0$, and the Taylor expansion of $A(\lambda)$ at $\widehat{\lambda}_{sc}^t$ up to the second order one gets the following estimate
			\begin{align}\label{eq:thm:asympt-distr:proof:at-bt-approx-1}
			\nonumber
			A^t(\lambda) - B^t(\lambda) & \geq \sum\limits_{i\in I_1(\Lambda^*)} -t C_1 \mid\widetilde{\Lambda}_{b,i}^t - \widehat{\Lambda}^t_{sc,i}\mid  \dfrac{\mid\Lambda_i - \widehat{\Lambda}^t_{sc,i}\mid^2}{\mid\widehat{\Lambda}^t_{sc,i}\mid^2}
			+ \sum\limits_{i\in I_1(\Lambda^*)} -t C_2 \mid\Lambda_i - \widehat{\Lambda}^t_{sc,i}\mid^3\\
			&+ \sum\limits_{i\in I_0(\Lambda^*)} -t\widetilde{\Lambda}_{b,i}^t\log(t\Lambda_i) +
			\sum\limits_{i\in I_0(\Lambda^*)} t\widehat{\Lambda}_{sc,i}^t \log(t\widehat{\Lambda}_{sc,i}^t) - t\widehat{\Lambda}_{sc,i}^t\\ \nonumber
			& + \beta^t(\varphi(\lambda) - \varphi(\widehat{\lambda}_{sc}^t)), \, 
			\lambda\in C^t_{A,\delta}(\widetilde{\lambda}_{b,app}^t)).
			\end{align}
			where $C_1$, $C_2$ are some positive constants which depend only design $A$ and $\Lambda^*$.
			The above estimate holds with conditional probability tending to one for $t\rightarrow +\infty$ a.s. $Y^t$, $t\in (0, +\infty)$.
			In particular, in \eqref{eq:thm:asympt-distr:proof:at-bt-approx-1} to bound uniformly the error-terms in the Taylor's expansion we have used the following estimates:
			\begin{align}
			\label{eq:thm:asympt-distr:proof:sup-big-lambda-ct-zero}
			&\sup_{\lambda \in C^t_{A,\delta}(\widetilde{\lambda}^t_{b,app})}
			\hspace{-0.5cm}
			\mid\Lambda_i - \widehat{\Lambda}_{sc,i}^t\mid/\mid\widehat{\Lambda}_{sc,i}^t\mid = o_{cp}(1), \, i\in I_1(\Lambda^*),\\
			\label{eq:thm:asympt-distr:proof:log-one-x-expansion}
			& \mid\log(1 + x) - x\mid \leq C_1  \mid x\mid^2, \text{ for some } C_1 > 0 \text{  for } \mid x\mid \leq 1/2, \\
			\label{eq:thm:asympt-distr:proof:slogs-estim-second-order}
			& \mid-\widehat{s} \log (s / \widehat{s}) + (s-\widehat{s}) - \frac{s^2}{2\widehat{s}} \mid \leq C_2\mid s-\widehat{s}\mid^3, \\ \nonumber
			& \hspace{1.0cm} \text{for some $C_2 = C_2(s_*, \varepsilon) > 0$ and } \mid s-\widehat{s}\mid < \widehat{s} / 2, \, \mid\widehat{s} - s_*\mid < \varepsilon \text{ for some fixed }\varepsilon, \, s_* > 0.
			\end{align} 
			Formulas \eqref{eq:thm:asympt-distr:proof:log-one-x-expansion}, \eqref{eq:thm:asympt-distr:proof:slogs-estim-second-order} describe the standard second order Taylor expansions of the logarithm in vicinity of $x =0$ and $\widehat{s} = s_*$, respectively. Formula \eqref{eq:thm:asympt-distr:proof:sup-big-lambda-ct-zero} can be proved via the following triangle-type inequality: 
			\begin{align}\label{eq:thm:asympt-distr:proof:sup-big-lambda-expl:triangle}
			\begin{split}
			&\mid\Lambda_i - \widehat{\Lambda}_{sc,i}^t\mid \leq \mid\Lambda_i - \widetilde{\Lambda}^t_{b,app,i}\mid + \mid\widetilde{\Lambda}^t_{b,app,i} - \widetilde{\Lambda}^t_{b,i}\mid + \mid\widetilde{\Lambda}^t_{b,i}-\Lambda_i^*\mid + \mid\Lambda_i^* + \widehat{\Lambda}_{sc,i}^t\mid, \\ 
			&\lambda\in C^t_{A,\delta}(\widetilde{\lambda}_{b,app}^t).
			\end{split}
			\end{align}
			The first term in the right hand-side of \eqref{eq:thm:asympt-distr:proof:sup-big-lambda-expl:triangle} is of order $o_{cp}(1)$ in view of the definition in \eqref{eq:thm:asymp-distr:proof:cylinder-c-t-def} and the fact that $\lambda\in C_{A,\delta}^t(\widetilde{\lambda}^t_{b,app})$ for some fixed $\delta > 0$. 
			The last two terms are also $o_{cp}(1)$ in view of Lemma~\ref{lem:consistency:proof:algorithm} and the fact that 
			$\widehat{\Lambda}_{sc,i}^t \rightarrow \Lambda_i^*$ a.s. $Y^t, t\in (0, +\infty)$. 
			Finally, from \eqref{eq:asymp-distr:lemm-approximation:bt-kt-rt-mu-r-o-small} and again the fact that $\widehat{\Lambda}_{sc,i}^t \rightarrow \Lambda_i^*$ a.s. $Y^t$, $t\in (0, +\infty)$, it follows that the second term in \eqref{eq:thm:asympt-distr:proof:sup-big-lambda-expl:triangle} is also of order $o_{cp}(1)$. This completes the proof of \eqref{eq:thm:asympt-distr:proof:sup-big-lambda-ct-zero}.


			Using the restriction that $\lambda\in C^t_{A,\delta}(\widetilde{\lambda}_{b,app,t}^t)$ two first sums in \eqref{eq:thm:asympt-distr:proof:at-bt-approx-1} can be estimated as follows:
			\begin{align}\label{eq:thm:asympt-distr:proof:at-first-sum-lwr-bnd}
			\begin{split}
			\sum\limits_{i\in I_1(\Lambda^*)} -t &C_1 \mid\widetilde{\Lambda}_{b,i}^t - \widehat{\Lambda}^t_{sc,i}\mid  \dfrac{\mid\Lambda_i - \widehat{\Lambda}^t_{sc,i}\mid^2}{|\widehat{\Lambda}^t_{sc,i}|^2} 
			\geq \sum\limits_{i\in I_1(\Lambda^*)} -t C_1 \mid\widetilde{\Lambda}_{b,i}^t - \widehat{\Lambda}^t_{sc,i}\mid \\
			& \times \left(
			\dfrac{2\mid\Lambda_i - \widetilde{\Lambda}_{b,app,i}^t\mid^2}{\mid\widehat{\Lambda}^t_{sc,i}\mid^2}
			+ \dfrac{2\mid\widetilde{\Lambda}_{b,app,i}^t - \widehat{\Lambda}_{sc,i}^t \mid^2}{\mid\widehat{\Lambda}^t_{sc,i}\mid^2}
			\right)\\
			&\geq \sum\limits_{i\in I_1(\Lambda^*)} -t C_1 \mid\widetilde{\Lambda}_{b,i}^t - \widehat{\Lambda}^t_{sc,i}\mid
			\left(
			\dfrac{c\delta^2}{t\mid\widehat{\Lambda}^t_{sc,i}\mid^2}
			+ \dfrac{2\mid\widetilde{\Lambda}_{b,app,i}^t - \widehat{\Lambda}_{sc,i}^t \mid^2}{\mid\widehat{\Lambda}^t_{sc,i}\mid^2}
			\right),
			\end{split}
			\end{align} 
			where $c$ depends only $A$. Using same argument for the second sum in \eqref{eq:thm:asympt-distr:proof:at-bt-approx-1} we obtain the following:
			\begin{align}\label{eq:thm:asympt-distr:proof:at-second-sum-lwr-bnd}
			\sum\limits_{i\in I_1(\Lambda^*)}-t C_2 \mid\Lambda_i - \widehat{\Lambda}^t_{sc,i}\mid^3 &\geq \sum\limits_{i\in I_1(\Lambda^*)}-8t C_2 \left(\mid\Lambda_i - \widetilde{\Lambda}_{b,app,i}^t\mid^3 
			+ \mid\widetilde{\Lambda}_{b,app,i}^t - \widehat{\Lambda}^t_{sc,i}\mid^3
			\right) \\ \nonumber
			&\geq \sum\limits_{i\in I_1(\Lambda^*)}-8t C_2 \left(\dfrac{c\delta^3}{t^{3/2}}
			+ \mid\widetilde{\Lambda}_{b,app,i}^t - \widehat{\Lambda}^t_{sc,i}\mid^3
			\right),
			\end{align}
			for  $\lambda\in C^t_{A,\delta}(\widetilde{\Lambda}_{b,app,t}^t)$,
			where $c$ depends only on $A$.
			
			From  \eqref{eq:thm:asympt-distr-main:strongly-consist-estim-i0}, \eqref{eq:asymp-distr:parametrization-uvw}, \eqref{eq:asymp-distr:quadr-minimizer-v-zero}, the results of lemmas~\ref{lem:asymp-distr:quadratic-uv-approximation},~\ref{lem:asympt-distr:grad-stoch-asymp-normality} it follows that 
			\begin{align}\label{eq:thm:asympt-distr:proof:at-sums-ocp-1}
			t\mid\widetilde{\Lambda}_{b,i}^t - \widehat{\Lambda}^t_{sc,i}\mid \cdot  \mid\widetilde{\Lambda}_{b,app,i}^t - \widehat{\Lambda}_{sc,i}^t \mid^2 &= o_{cp}(1), \\
			\label{eq:thm:asympt-distr:proof:at-sums-ocp-2}
			t\mid\widetilde{\Lambda}_{b,app,i}^t - \widehat{\Lambda}^t_{sc,i}\mid^3 &= o_{cp}(1).
			\end{align}
			The above formulas imply that sums in \eqref{eq:thm:asympt-distr:proof:at-first-sum-lwr-bnd},  \eqref{eq:thm:asympt-distr:proof:at-second-sum-lwr-bnd} are bounded from below and of order~$o_{cp}(1)$.
			
			The logarithmic term in \eqref{eq:thm:asympt-distr:proof:at-bt-approx-1} can be estimated as follows:
			\begin{align}
			\begin{split}\label{eq:thm:asympt-distr:proof:at-log-term-positivity-1}
			\sum\limits_{i\in I_0(\Lambda^*)} -t\widetilde{\Lambda}_{b,i}^t\log(t\Lambda_i) &= 
			\sum\limits_{i\in I_0(\Lambda^*)} -t\widetilde{\Lambda}_{b,i}^t\log(t(\Lambda_i - \widetilde{\Lambda}_{b,app,i}^t) + t\widetilde{\Lambda}_{b,app,i}^t)\\
			&\geq \sum\limits_{i\in I_0(\Lambda^*)} -t\widetilde{\Lambda}_{b,i}^t\log(t\mid\Lambda_i - \widetilde{\Lambda}_{b,app,i}^t\mid + t\widetilde{\Lambda}_{b,app,i}^t)\\
			&\geq \sum\limits_{i\in I_0(\Lambda^*)}-t\widetilde{\Lambda}_{b,i}^t\log(c\delta + t\widetilde{\Lambda}_{b,app,i}^t), \, \lambda \in C_{A,\delta}^t(\widetilde{\lambda}_{b,app}^t).
			\end{split}
			\end{align}
			where $c$ is some positive constant depending on $A$. Using \eqref{eq:asymp-distr:parametrization-uvw}, \eqref{eq:asymp-distr:quadr-minimizer-v-zero} and \eqref{eq:asymp-distr:quadr-approximation-v} from Lemma~\ref{lem:asymp-distr:quadratic-uv-approximation} we obtain 
			\begin{align}\label{eq:thm:asympt-distr:proof:Lambda-b-app-i-zero-limit}
			\begin{split}
			t\widetilde{\Lambda}_{b,app,i}^t = &t(\widetilde{\Lambda}_{b,app,i}^t - \widehat{\Lambda}_{sc,i}^t) + t\widehat{\Lambda}_{sc,i}^t \\
			&= a_i^T\widetilde{v}_{b,app}^t + ta_i^T\widehat{\lambda}_{sc}^t \\
			&= a_i^T(\widetilde{v}_{b,app}^t - \widetilde{v}^t) = o_{cp}(1), \, I_0(\Lambda^*)
			\end{split}
			\end{align}
			
			Formulas \eqref{eq:thm:asympt-distr:proof:at-log-term-positivity-1}, \eqref{eq:thm:asympt-distr:proof:Lambda-b-app-i-zero-limit} imply that 
			\begin{align}\label{eq:thm:asympt-distr:proof:at-log-term-positivity-2}
			\sum\limits_{i\in I_0(\Lambda^*)} -t\widetilde{\Lambda}_{b,i}^t \log(t\Lambda_i) \geq 
			\sum\limits_{i\in I_0(\Lambda^*)}-t\widetilde{\Lambda}_{b,i}^t\log(c\delta + o_{cp}(1)). 
			\end{align}
			By choosing $\delta$ smaller than some fixed constant (e.g., $\delta < c/2$) in \eqref{eq:thm:asympt-distr:proof:at-log-term-positivity-2} we find that the right hand-side in \eqref{eq:thm:asympt-distr:proof:at-log-term-positivity-2} becomes positive with conditional probability tending to one a.s. $Y^t$, $t\in (0, +\infty)$. Therefore, 
			\begin{equation}\label{eq:thm:asympt-distr:proof:at-log-term-positivity-final}
			\sum\limits_{i\in I_0(\Lambda^*)} -t\widetilde{\Lambda}_{b,i}^t \log(t\Lambda_i) \geq o_{cp}(1), \, \lambda\in C^t_{A,\delta}(\widetilde{\lambda}_{b,app}^t) \text{ for } \delta < c/2.
			\end{equation}
			
			In addition, from the initial assumption in   \eqref{eq:thm:asympt-distr-main:strongly-consist-estim-i0} it directly follows that 
			\begin{equation}\label{eq:thm:asympt-distr:proof:at-sc-log-term-positivity-final}
			\sum\limits_{i\in I_0(\Lambda^*)} t\widehat{\Lambda}_{sc,i}^t \log(t\widehat{\Lambda}_{sc,i}^t) - t\widehat{\Lambda}_{sc,i}^t = o_{cp}(1).
			\end{equation}

			Using \eqref{eq:thm:asympt-distr:proof:at-bt-approx-1},
			\eqref{eq:thm:asympt-distr:proof:at-first-sum-lwr-bnd}-\eqref{eq:thm:asympt-distr:proof:at-sums-ocp-2},
			\eqref{eq:thm:asympt-distr:proof:at-log-term-positivity-final}, \eqref{eq:thm:asympt-distr:proof:at-sc-log-term-positivity-final} we finally obtain:
			\begin{align}\label{eq:thm:asympt-distr:proof:inf-a-term-lower-bound}
			\inf_{\lambda\in C^t_{A,\delta}(\widetilde{\lambda}_{b,app}^t)}\hspace{-0.4cm} [A^t(\lambda) - B^t(\lambda)] \geq o_{cp}(1) + \beta^t   \hspace{-0.4cm}\inf_{\lambda\in C^t_{A,\delta}(\widetilde{\lambda}_{b,app}^t)}\hspace{-0.4cm} (\varphi(\lambda) - \varphi(\widehat{\lambda}_{sc}^t)).
			\end{align}
			
			Now, let us consider the third term in the left-hand side of \eqref{eq:thm:asymp-distr:proof:a-diff-representation}. Using  \eqref{eq:asymp-distr:approximation-general}-\eqref{eq:asymp-distr:remainder-r} we rewrite it as follows:
			\begin{align}\label{eq:thm:asympt-distr:proof:b-diff-expansion}
			\begin{split}
			B^t(\widetilde{\lambda}_{b,app}^t) - B^t(\widetilde{\lambda}_{app}^t) &= 
			\widetilde{B}^t(\widetilde{\lambda}_{b,app}^t) - \widetilde{B}^t(\widetilde{\lambda}_{app}^t) \\
			& + \widetilde{R}^t(\widetilde{\lambda}_{b,app}^t) - \widetilde{R}^t(\widetilde{\lambda}_{app}^t).
			\end{split}
			\end{align}
			From \eqref{eq:asymp-distr:app-abbp-proj-u}, the result of Lemma~\ref{lem:approximate-approximation},  \eqref{eq:asymp-distr:approximation-general}-\eqref{eq:asymp-distr:remainder-r}, \eqref{eq:asymp-distr:quadr-minimizer-v-zero}, the result of lemmas~\ref{lem:asymp-distr:quadratic-uv-approximation}, \ref{lem:asympt-distr:grad-stoch-asymp-normality} and formula  \eqref{eq:thm:asympt-distr:proof:b-diff-expansion} it follows directly that 
			\begin{equation}\label{eq:thm:asympt-distr:proof:inf-b-min-min-term-lower-bound}
			B^t(\widetilde{\lambda}_{b,app}^t) - B^t(\widetilde{\lambda}_{app}^t) = o_{cp}(1).
			\end{equation}
			
			Now we estimate the last term in the right-hand side of \eqref{eq:thm:asymp-distr:proof:a-diff-representation}. Using the same argument as in \eqref{eq:thm:asympt-distr:proof:at-bt-approx-1}-\eqref{eq:thm:asympt-distr:proof:at-sc-log-term-positivity-final} one gets the following estimate:
			
			\begin{align}
			\label{eq:thm:asympt-distr:proof:bt-at-approx-1}
			\begin{split}
			B^t(\widetilde{\lambda}_{app}^t) - A^t(\widetilde{\lambda}^t_{app}) &\geq \sum\limits_{i\in I_1(\Lambda^*)} -t C_1 \mid\widetilde{\Lambda}_{b,i}^t - \widehat{\Lambda}_{sc,i}^t\mid   \dfrac{\mid\widetilde{\Lambda}^t_{app,i} - \widehat{\Lambda}_{sc,i}^t\mid^2}{\mid\widehat{\Lambda}_{sc,i}^t\mid^2} \\
			& + \sum\limits_{i\in I_1(\Lambda^*)} -t C_2 \mid\widetilde{\Lambda}_{app,i}^t - \widehat{\Lambda}_{sc,i}^t\mid^3 \\
			& + \sum\limits_{i\in I_0(\Lambda^*)} t \widetilde{\Lambda}_{b,i}^t \log(\widetilde{\Lambda}_{app,i}^t)\\
			&+\sum\limits_{i\in I_0(\Lambda^*)} -t\widehat{\Lambda}_{sc,i}^t \log(t\widehat{\Lambda}_{sc,i}^t) + t\widehat{\Lambda}_{sc,i}^t\\
			&-\beta^t(\varphi(\widetilde{\lambda}^t_{app}) - \varphi(\widehat{\lambda}_{sc}^t)).
			\end{split}\\
			\label{eq:thm:asympt-distr:proof:bt-at-approx-2}
			\begin{split}
			&\geq \sum\limits_{i\in I_1(\Lambda^*)} -t C_1 \mid\widetilde{\Lambda}_{b,i}^t - \widehat{\Lambda}_{sc,i}^t\mid   \dfrac{\mid\widetilde{\Lambda}^t_{app,i} - \widehat{\Lambda}_{sc,i}^t\mid^2}{\mid\widehat{\Lambda}_{sc,i}^t\mid^2} \\
			& + \sum\limits_{i\in I_1(\Lambda^*)} -t C_2 \mid\widetilde{\Lambda}_{app,i}^t - \widehat{\Lambda}_{sc,i}^t\mid^3 \\
			& + \sum\limits_{i\in I_0(\Lambda^*)} t \widetilde{\Lambda}_{b,i}^t \log(t\widetilde{\Lambda}_{b,i}^t)\\
			&+\sum\limits_{i\in I_0(\Lambda^*)} -t\widehat{\Lambda}_{sc,i}^t \log(t\widehat{\Lambda}_{sc,i}^t) + t\widehat{\Lambda}_{sc,i}^t\\
			&-\beta^t(\varphi(\widetilde{\lambda}^t_{app}) - \varphi(\widehat{\lambda}_{sc,i}^t)),
			\end{split}
			\end{align}
			where constants $C_1, C_2$ depend only on $A$. To pass from \eqref{eq:thm:asympt-distr:proof:bt-at-approx-1} to \eqref{eq:thm:asympt-distr:proof:bt-at-approx-2} we have used the monotonicity of the logarithm (i.e., $\log(x + y) \geq \log(x)$, for any $y > 0$). 
			The above estimate holds with conditional probability tending to one a.s. $Y^t$, $t\in (0, +\infty)$.
			
			From formulas \eqref{eq:lambda-app-recentered}, \eqref{eq:asymp-distr:lem:approximate-approximation:proj-v-convergece},  \eqref{eq:asymp-distr:parametrization-uvw}, \eqref{eq:asymp-distr:quadr-minimizer-v-zero}, the results of lemmas~\ref{lem:asymp-distr:quadratic-uv-approximation},~\ref{lem:asympt-distr:grad-stoch-asymp-normality} it follows that 
			\begin{align}
			\label{eq:thm:asympt-distr:proof:bt-at-stoch-errors}
			&t \mid\widetilde{\Lambda}_{b,i}^t - \widehat{\Lambda}_{sc,i}^t\mid  \cdot \mid\widetilde{\Lambda}^t_{app,i}-\widehat{\Lambda}_{sc,i}^t\mid^2 = o_{cp}(1), \\
			&t \mid\widetilde{\Lambda}_{app,i}^t - \widehat{\Lambda}_{sc,i}^t\mid^3 = o_{cp}(1).
			\end{align}
			In addition, using \eqref{eq:lem:approximate-approximation:sufficient-cond} in the proof of Lemma~\ref{lem:approximate-approximation} we find that 
			\begin{align}\label{eq:thm:asympt-distr:proof:bt-at-log-positivity}
			\sum\limits_{i\in I_0(\Lambda^*)} t \widetilde{\Lambda}_{b,i}^t \log(t\widetilde{\Lambda}_{b,i}^t) = o_{cp}(1).
			\end{align}
			
			Putting together \eqref{eq:thm:asympt-distr:proof:bt-at-approx-2}-\eqref{eq:thm:asympt-distr:proof:bt-at-log-positivity} and using again \eqref{eq:thm:asympt-distr:proof:at-sc-log-term-positivity-final} we obtain
			\begin{align}\label{eq:thm:asympt-distr:proof:bt-at-lower-bound}
			B^t(\widetilde{\lambda}_{app}^t) - A^t(\widetilde{\lambda}^t_{app}) \geq o_{cp}(1) - \beta^t (\varphi(\widetilde{\lambda}_{app}^t) - \varphi(\widehat{\lambda}_{sc}^t)).
			\end{align}	
			Formulas \eqref{eq:thm:asymp-distr:proof:a-diff-representation}, \eqref{eq:thm:asympt-distr:proof:inf-a-term-lower-bound}, \eqref{eq:thm:asympt-distr:proof:inf-b-min-min-term-lower-bound} \eqref{eq:thm:asympt-distr:proof:bt-at-lower-bound} imply that 
			\begin{align}\label{eq:thm:asympt-distr:proof:inf-at-centered-pen-bound}
			\begin{split}
			\inf_{\lambda\in C^t_{A,\delta}(\widetilde{\lambda}^t_{b,app})}\hspace{-0.5cm}[A^t(\lambda) - A^t(\widetilde{\lambda}_{app}^t)] &= \hspace{-0.3cm}\inf_{\lambda\in C^t_{A,\delta}(\widetilde{\lambda}^t_{b,app})}\hspace{-0.5cm}[B^t(\lambda) - B^t(\widetilde{\lambda}^t_{b,app})] \\
			& + \hspace{-0.3cm}\inf_{\lambda\in C^t_{A,\delta}(\widetilde{\lambda}^t_{b,app})}\hspace{-0.5cm}\beta^t(\varphi(\lambda) - \varphi(\widetilde{\lambda}_{app}^t))\\
			& + o_{cp}(1).
			\end{split}
			\end{align}
			\begin{lem}\label{lem:thm:asympt-distr:proof:inf-varphi-diff}
				Let $\beta^t$, $\varphi(\cdot)$ satisfy the assumptions of Theorem~\ref{thm:asympt-distr-main} and $\widetilde{\lambda}_{b, app}^t$, $\widetilde{\lambda}_{app}^t$ be defined in  \eqref{lem:proofs:conv-approx-lambda-appr}, \eqref{eq:lambda-app-recentered}, respectively. Then, 
				\begin{align}\label{eq:lem:thm:asympt-distr:proof:inf-varphi-diff}
				\hspace{-0.3cm}\inf_{\lambda\in C^t_{A,\delta}(\widetilde{\lambda}^t_{b,app})}\hspace{-0.5cm}\beta^t (\varphi(\lambda) - \varphi(\widetilde{\lambda}_{app}^t)) = o_{cp}(1), \text{ a.s. } Y^t, t\in (0, +\infty).
				\end{align}
			\end{lem}
			Formula \eqref{eq:asymp-distr:lemm-approximation:inf-estimate-main} directly follows from \eqref{eq:thm:asympt-distr:proof:inf-at-centered-pen-bound} and the result of Lemma~\ref{lem:thm:asympt-distr:proof:inf-varphi-diff}.
			
			Lemma is proved.
		\end{proof}
		
		\subsection{Proof of Lemma~\ref{lem:asymp-distr:quadratic-uv-approximation}}
		\begin{proof}
			
			To prove the claim we use essentially the same convexity argument as before, for example in  Lemma~\ref{lem:thm:asymp-disr:main-approximation-lemma}.
			
			Let $\delta > 0$ and 
			\begin{align}\label{lem:asymp-distr:quadratic-uv-approximation:lambda-tilde-t}
			&\widetilde{\lambda}^t = \widehat{\lambda}_{sc}^t + \frac{\widetilde{u}^t}{\sqrt{t}} + \frac{\widetilde{v}^t}{t} + \widetilde{w}^t, \, \widetilde{\lambda}^t \succeq 0, \\
			\label{lem:asymp-distr:quadratic-uv-approximation:lambda-uvw-def}
			&\lambda(u,v,w) = \widehat{\lambda}_{sc}^t + \frac{u}{\sqrt{t}} + \frac{v}{t} + w, \, (u, v, w)\in \mathcal{U}\times \mathcal{V}  \times \mathcal{W}, \, \lambda(u,v,w) \succeq 0.
			\end{align}
			
			Recall that
			\begin{align}
			\label{lem:asymp-distr:quadratic-uv-approximation:uv-diff-delta}
			&\|u-\widetilde{u}^t\|_2 + \|v-\widetilde{v}^t\|_1 = \delta \text{ for } \lambda(u,v,w)\in C^t_{A,\delta}(\widetilde{\lambda}^t).
			\end{align}
			where $C^t_{A,\delta}(\cdot)$ is defined in \eqref{eq:thm:asymp-distr:proof:cylinder-c-t-def}.

			Next we show that 
			\begin{align}\label{lem:asymp-distr:quadratic-uv-approximation:cond-proba-one}
			P(\hspace{-1.0cm}\inf_{(u,v, w) : \lambda(u,v,w)\in C^t_{A,\delta}(\widetilde{\lambda}^t)} \hspace{-1cm}[B^t(u,v) - B^t(\widetilde{u}^t, \widetilde{v}^t)] > 0 \, \mid \, Y^t, t) \rightarrow 1 \text{ when }t\rightarrow +\infty, \text{ a.s. } Y^t, t\in (0, +\infty)
			\end{align}
			which together with the fact that $\delta$ can be arbitrarily small and convexity of $B^t(u,v)$, implies the claim of the lemma. Using formulas \eqref{eq:asymp-distr:approximation-general}-\eqref{eq:asymp-distr:remainder-r} we obtain 
			\begin{align}\label{lem:asymp-distr:quadratic-uv-approximation:Bt-diff-cylinder}
			\begin{split}
			B^t(u,v) - B^t(\widetilde{u}^t, \widetilde{v}^t) &= [\widetilde{B}^t(u,v) - \widetilde{B}^t(\widetilde{u}^t, \widetilde{v}^t)] \\
			& + [\widetilde{R}^t(u,v) - \widetilde{R}^t(\widetilde{u}^t, \widetilde{v}^t)],\\ 
			(u,v) \text{ s.t. } &\exists w\in \mathcal{W}, \, \lambda(u,v, w) \in C^t_{A,\delta}(\widetilde{\lambda}^t).
			\end{split}
			\end{align}
			From the facts that $\widetilde{\Lambda}_{b,i}^t \xrightarrow{c.p.} \Lambda_i^*$ (by Lemma~\ref{lem:consistency:proof:algorithm}), 		
			$\widehat{\Lambda}_{sc,i}^t \xrightarrow{a.s.} \Lambda_i^*$  for $i\in \{1, \dots, d\}$ (see \eqref{eq:thm:asympt-distr-main:strongly-consist-estim-i1}, \eqref{eq:thm:asympt-distr-main:strongly-consist-estim-i0} and \eqref{eq:appendix:stlln} in Appendix~\ref{app:limit-thms}),
			the conditional tightness of $\widetilde{u}^t$ (by Lemma~\ref{lem:asympt-distr:grad-stoch-asymp-normality})
			and 
			formulas \eqref{eq:thm:asympt-distr-main:strongly-consist-estim-i0}, \eqref{eq:asymp-distr:quadr-minimizer-v-zero}, \eqref{lem:asymp-distr:quadratic-uv-approximation:uv-diff-delta} it follows that 
			\begin{align}\label{lem:asymp-distr:quadratic-uv-approximation:sup-remainder-bounded}
			\hspace{-1.0cm}\sup_{(u,v, w) : \lambda(u,v, w)\in C^t_{A,\delta}(\widetilde{\lambda}^t)} \hspace{-1cm}
			\mid\widetilde{R}^t(u,v) - \widetilde{R}^t(\widetilde{u}^t, \widetilde{v}^t)\mid = o_{cp}(1),
			\end{align}
			where $\widetilde{R}^t(\cdot)$ is defined in \eqref{eq:asymp-distr:remainder-r}.
			
			Formulas \eqref{lem:asymp-distr:quadratic-uv-approximation:Bt-diff-cylinder}, \eqref{lem:asymp-distr:quadratic-uv-approximation:sup-remainder-bounded} imply that 
			\begin{align}\label{lem:asymp-distr:quadratic-uv-approximation:Bt-diff-positive-part}
			\begin{split}
			\inf_{(u,v, w) : \lambda(u,v,w)\in C^t_{A,\delta}(\widetilde{\lambda}^t)} \hspace{-1cm}[B^t(u,v) - B^t(\widetilde{u}^t, \widetilde{v}^t)] &\geq
			\hspace{-1.0cm}
			\inf_{(u,v, w) : \lambda(u,v,w)\in C^t_{A,\delta}(\widetilde{\lambda}^t)} \hspace{-1cm} [\widetilde{B}^t(u,v) - \widetilde{B}^t(\widetilde{u}^t, \widetilde{v}^t)] + o_{cp}(1).
			\end{split}
			\end{align}
			
			
			Since the positivity constraints in \eqref{eq:asymp-distr:quadr-minimizer-u} include restrictions on $u\in\mathcal{U}$ and also depend on $w\in \mathcal{W}$, for simplicity, we include $w$ in the minimization problem as an independent variable
			
			\begin{align}\label{lem:asymp-distr:quadratic-uv-approximation:argmin-expanded-w}
			(\widetilde{u}^t, \widetilde{w}^t) = \argmin_{\substack{(u, w):(1-\Pi_{\mathcal{V}})\widehat{\lambda}^t_{sc} + \frac{u}{\sqrt{t}} + w \succeq 0\\ u\in \mathcal{U}, \, w\in \mathcal{W}}}
			\sum\limits_{i\in I_1(\Lambda^*)} -\sqrt{t}(\widetilde{\Lambda}_{b,i}^t - \widehat{\Lambda}^t_{sc,i}) \dfrac{a_i^Tu}{\widehat{\Lambda}^t_{sc,i}} + \dfrac{(a_i^Tu)^2}{2\widehat{\Lambda}^t_{sc,i}}.
			\end{align}
			Note that minimizer $\widetilde{u}^t$ in \eqref{lem:asymp-distr:quadratic-uv-approximation:argmin-expanded-w} coincides with the original solution from \eqref{eq:asymp-distr:quadr-minimizer-u}. The problem in \eqref{lem:asymp-distr:quadratic-uv-approximation:argmin-expanded-w} is convex and the strong duality is satisfied (e.g., by Slater's condition). From the Karush-Kuhn-Tucker necessary optimality conditions (see e.g., \citet{bertsekas1997nonlinear}, Section~3.3) for the optimization problem in \eqref{lem:asymp-distr:quadratic-uv-approximation:argmin-expanded-w} and the strong duality it follows that 
			\begin{align}\label{lem:asymp-distr:quadratic-uv-approximation:kkt-mu-pos-ortho}
			&\exists \, \widetilde{\mu}^t \succeq 0, \, \widetilde{\mu}^t \in \mathcal{W}^\perp,\\
			\label{lem:asymp-distr:quadratic-uv-approximation:kkt-grad-u}
			& \sum\limits_{i\in I_1(\Lambda^*)}-\sqrt{t} 
			(\widetilde{\Lambda}_{b,i}^t - \widehat{\Lambda}_{sc,i}^t)
			\dfrac{\Pi_{\mathcal{U}}a_i}{\widehat{\Lambda}_{sc,i}^t} + 
			\dfrac{\Pi_{\mathcal{U}}a_ia_i^T\widetilde{u}^t}{\widehat{\Lambda}_{sc,i}^t} = \dfrac{\widetilde{\mu}_{\mathcal{U}}^t}{\sqrt{t}}, \, \widetilde{\mu}_{\mathcal{U}}^t = \Pi_{\mathcal{U}}\widetilde{\mu}^t, \\
			\label{lem:asymp-distr:quadratic-uv-approximation:mu-compl-slackness}
			&\widetilde{\mu}^t_{j}\left([(I - \Pi_{\mathcal{V}})\widehat{\lambda}_{sc}^t]_j + \frac{\widetilde{u}^t_j}{\sqrt{t}} + \widetilde{w}^t_j
			\right) = 0, \, j\in \{1, \dots, p\},
			\end{align}
			where $(\widetilde{u}^t, \widetilde{w}^t)$ are defined in \eqref{lem:asymp-distr:quadratic-uv-approximation:argmin-expanded-w}. 
			Strong duality implies, in particular, that $\widetilde{\mu}^t$ is a solution for the dual problem and $\widetilde{\mu}^t \in \mathcal{W}^\perp$ (dual functional equals $-\infty$ for $\widetilde{\mu}^t \not \in \mathcal{W}^\perp$). Note also that the optimized functional in \eqref{lem:asymp-distr:quadratic-uv-approximation:argmin-expanded-w} is strongly convex in $u$, so $\widetilde{u}^t$ is always unique, whereas at least one $\widetilde{w}^t$ always exists, however, may not be unique. The latter fact does not pose any problem since the target functional is flat for $w\in \mathcal{W}$, so if not said otherwise, we choose any solution $\widetilde{w}^t$ in \eqref{lem:asymp-distr:quadratic-uv-approximation:argmin-expanded-w} so that positivity constraints are satisfied.

			From \eqref{eq:asymp-distr:approximation-general}, \eqref{lem:asymp-distr:quadratic-uv-approximation:kkt-mu-pos-ortho}-\eqref{lem:asymp-distr:quadratic-uv-approximation:mu-compl-slackness} it follows that 
			\begin{align}\label{lem:asymp-distr:quadratic-uv-approximation:bt-tilde-positive}
			\nonumber
			\widetilde{B}^t(u,v) &- \widetilde{B}^t(\widetilde{u}^t, \widetilde{v}^t) =  \sum\limits_{i\in I_1(\Lambda^*)} -\sqrt{t}
			\dfrac{\widetilde{\Lambda}_{b,i} - \widehat{\Lambda}_{sc,i}^t}{\widehat{\Lambda}_{sc,i}^t}
			a_i^T(u-\widetilde{u}^t) + \dfrac{1}{2} \dfrac{(a_i^Tu)^2 - (a_i^T\widetilde{u}^t)^2}{\widehat{\Lambda}_{sc,i}^t}\\
			\nonumber
			& + \sum\limits_{i\in I_0(\Lambda^*)}a_i^T(v-\widetilde{v}^t)\\
			\nonumber
			&=\sum\limits_{i\in I_1(\Lambda^*)} -\sqrt{t}
			\dfrac{\widetilde{\Lambda}_{b,i} - \widehat{\Lambda}_{sc,i}^t}{\widehat{\Lambda}_{sc,i}^t}
			a_i^T(u-\widetilde{u}^t) + \dfrac{1}{2}\dfrac{(a_i^T(u-\widetilde{u}^t))^2}{\widehat{\Lambda}_{sc,i}^t} \\ \nonumber
			&+ \dfrac{(\widetilde{u}^t)^T a_ia_i^T(u-\widetilde{u}^t)}{\widehat{\Lambda}_{sc,i}^t} + 
			\sum\limits_{i\in I_0(\Lambda^*)}a_i^T(v-\widetilde{v}^t)\\
			&=\langle \frac{\widetilde{\mu}_{\mathcal{U}}^t}{\sqrt{t}}, u -\widetilde{u}^t
			\rangle + \dfrac{1}{2}\sum\limits_{i\in I_1(\Lambda^*)} 
			\dfrac{\mid a_i^T(u-\widetilde{u}^t)\mid^2}{\widehat{\Lambda}_{sc,i}^t} + 
			\sum\limits_{i\in I_0(\Lambda^*)}a_i^T(v-\widetilde{v}^t).
			\end{align}
			Note that 
			\begin{align}\label{lem:asymp-distr:quadratic-uv-approximation:mu-linear-positive}
			&\langle \frac{\widetilde{\mu}^t_{\mathcal{U}}}{\sqrt{t}}, u -\widetilde{u}^t
			\rangle \geq 0, \\
			\label{lem:asymp-distr:quadratic-uv-approximation:v-linear-positive}
			& v - \widetilde{v}^t \succeq 0, \\ \nonumber
			&\text{ for } (u,v)\in \mathcal{U}\times \mathcal{V} \text{ s.t. } \lambda(u,v, w) = \widehat{\lambda}_{sc}^t + \frac{u}{\sqrt{t}} + \frac{v}{t} + w \succeq 0 \text{ for some } w\in \mathcal{W}.
			\end{align}
			Indeed, in view of \eqref{lem:asymp-distr:quadratic-uv-approximation:kkt-mu-pos-ortho}, \eqref{lem:asymp-distr:quadratic-uv-approximation:mu-compl-slackness} the left hand-side in \eqref{lem:asymp-distr:quadratic-uv-approximation:mu-linear-positive} can be rewritten as follows:
			\begin{align}\label{lem:asymp-distr:quadratic-uv-approximation:mu-u-linear-rewr}
			\begin{split}
			\langle \frac{\widetilde{\mu}^t_{\mathcal{U}}}{\sqrt{t}}, u -\widetilde{u}^t
			\rangle &= \langle \widetilde{\mu}^t_{\mathcal{U}}, \frac{u}{\sqrt{t}} -\frac{\widetilde{u}^t}{\sqrt{t}}
			\rangle \\
			&= \langle (I-\Pi_{\mathcal{V}})\widetilde{\mu}^t, (I-\Pi_{\mathcal{V}})\widehat{\lambda}_{sc}^t + \frac{u}{\sqrt{t}}\rangle\\
			& = \langle (I-\Pi_{\mathcal{V}})\widetilde{\mu}^t, \widehat{\lambda}_{sc}^t + \frac{u}{\sqrt{t}} + \frac{v}{t} + w\rangle \\
			& = \langle (I-\Pi_{\mathcal{V}})\widetilde{\mu}^t, \lambda(u,v,w)\rangle
			\end{split}
			\end{align}
			Note also that from \eqref{lem:asymp-distr:quadratic-uv-approximation:kkt-mu-pos-ortho} and the definition of $\mathcal{V}$ in \eqref{eq:asymp-distr:subspace-v} it follows that  
			\begin{equation}\label{lem:asymp-distr:quadratic-uv-approximation:mu-u-positive}
			\mu^t_\mathcal{U} = (I-\Pi_{\mathcal{V}})\mu^t \succeq 0.
			\end{equation}
			Formula \eqref{lem:asymp-distr:quadratic-uv-approximation:mu-linear-positive} follows directly from \eqref{lem:asymp-distr:quadratic-uv-approximation:mu-u-linear-rewr}, \eqref{lem:asymp-distr:quadratic-uv-approximation:mu-u-positive} and the fact that $\lambda(u,v,w)\succeq 0$.
			
			In turn, formula \eqref{lem:asymp-distr:quadratic-uv-approximation:v-linear-positive} follows from \eqref{eq:asymp-distr:quadr-minimizer-v-zero}.
			
			Formulas \eqref{lem:asymp-distr:quadratic-uv-approximation:uv-diff-delta},  \eqref{lem:asymp-distr:quadratic-uv-approximation:bt-tilde-positive}-\eqref{lem:asymp-distr:quadratic-uv-approximation:v-linear-positive} and the fact that $\widehat{\Lambda}_{sc,i}\rightarrow \Lambda_{i}^*$ for $i\in \{1, \dots, d\}$ a.s. $Y^t$, $t\in (0, +\infty)$ (as a strongly consistent estimator), imply that 
			with conditional probability tending to one a.s. $Y^t$, $t\in (0, +\infty)$ the following estimate holds:
			\begin{align}\label{lem:asymp-distr:quadratic-uv-approximation:b-tilde-positivity}
			\hspace{-1.0cm}\inf_{(u,v, w) : \lambda(u,v, w)\in C^t_{A,\delta}(\widetilde{\lambda}^t)} \hspace{-1cm}[\widetilde{B}^t(u,v) - \widetilde{B}^t(\widetilde{u}^t, \widetilde{v}^t)] \geq c\delta^2,
			\end{align}
			where $c$ is some fixed positive constant depending only on $\Lambda^*$ and $A$.
			
			Formula \eqref{lem:asymp-distr:quadratic-uv-approximation:cond-proba-one} follows directly from \eqref{lem:asymp-distr:quadratic-uv-approximation:Bt-diff-positive-part}, \eqref{lem:asymp-distr:quadratic-uv-approximation:b-tilde-positivity}.
			
			Lemma is proved.	
		\end{proof}

		\subsection{Proof of Lemma~\ref{lem:asymp-distr:quadr-minimizer-mapping-boundeness}}
		
		\par Let $\xi\in \R^{\# I_1(\Lambda^*)}$ be a parameter and consider $\widetilde{u}^t(\xi)$ defined in \eqref{eq:asymp-distr:quadr-minimizer-mapping-tech-def}. 
		
		Since the positivity constraints in \eqref{eq:asymp-distr:quadr-minimizer-mapping-tech-def} include restrictions on $u\in\mathcal{U}$ and $w\in \mathcal{W}$, for simplicity, we include $w$ in the minimization problem as an independent variable
		\begin{align}\label{lem:eq:asymp-dist:lem:primal-problem-reformulated}
		(\widetilde{u}^t, \widetilde{w}^t) = \argmin_{\substack{(u,w) : (1-\Pi_{\mathcal{V}})\widehat{\lambda}_{sc}^t + \frac{u}{\sqrt{t}} + w\succeq 0\\ u \in\mathcal{U}, w\in\mathcal{W}}} -\xi^T C^t u + \frac{1}{2}u^TF^tu,
		\end{align}
		where 
		\begin{align}
		\begin{split}
		&C^t = (\widehat{D}^t_{I_1(\Lambda^*)})^{-1/2}A_{I_1(\Lambda^*)},\, F^t = \widehat{F}^t_{I_1(\Lambda^*)}, \\
		&\widehat{D}^t_{I_1(\Lambda^*)}, \, 
		\widehat{F}^t_{I_1(\Lambda^*)} \text{ are defined in } \eqref{eq:asymp-distr:quadr-minimizer-mapping-tech-def:diagonal}, \,  \eqref{eq:asymp-distr:quadr-minimizer-mapping-tech-def:fisher}. 
		\end{split}
		\end{align}
		The Lagrangian function for the primal problem in \eqref{lem:eq:asymp-dist:lem:primal-problem-reformulated} is defined by the formula:
		\begin{align}\label{lem:eq:asymp-dist:lem:lagrangian-primal-dual}
		&\mathcal{L}^t(u, w; \mu) = -\xi^T C^tu + \frac{1}{2}u^TF^tu - \mu^T((1-\Pi_{\mathcal{V}})\widehat{\lambda}^t_{sc} + \frac{u}{\sqrt{t}} + w),\\ 
		&u\in \mathcal{U}, \, w\in \mathcal{W}, \, \mu \succeq 0.
		\end{align}
		The dual function for $G^t(\mu)$ and solution $\mu^t$ for the dual problem are defined by the formulas:
		\begin{align}\label{lem:eq:asymp-dist:lem:dual-problem-reformulated}
		G^t(\mu) = \inf_{u\in\mathcal{U}, \, w\in\mathcal{W}} \mathcal{L}^t(u,w; \mu), \, \mu^t = \arg\max_{\mu\succeq 0}G^t(\mu).
		\end{align}
		
		From the Karush-Kuhn-Tucker necessary optimality conditions, the fact that the primal problem is strongly convex in $u\in \mathcal{U}$ and the strong duality it follows that 
		\begin{align}
		&\exists (u^t, w^t)\in \mathcal{U}\times \mathcal{W}, \, \mu^t \succeq 0, \, \mu^t \in \mathcal{W}^\perp 
		\text{ s.t. }\\
		&(u^t, w^t) \text{ is a solution for the primal problem in \eqref{lem:eq:asymp-dist:lem:primal-problem-reformulated}}, \\
		\label{lem:eq:asymp-dist:lem:mu-t-dual-solution}
		&\mu^t = \mu^t(\xi) \text{ is a solution for the dual problem in \eqref{lem:eq:asymp-dist:lem:dual-problem-reformulated}}, \\
		\label{lem:eq:asymp-dist:lem:kkt-gradient}
		& \nabla_{u, w}\mathcal{L}^t(u^t, w^t; \mu^t) = 0, \\
		& ((1-\Pi_{\mathcal{V}})\widehat{\lambda}_{sc,j}^t + \frac{u^t_j}{\sqrt{t}} + w^t_j)
		\mu^t_j = 0, \, j\in \{1, \dots, p\}.
		\end{align}
		Using \eqref{lem:eq:asymp-dist:lem:lagrangian-primal-dual}, \eqref{lem:eq:asymp-dist:lem:kkt-gradient} we obtain the following:
		\begin{align}
		\label{lem:eq:asymp-dist:lem:kkt-gradient-u}
		-&\Pi_{\mathcal{U}}(C^t)^T\xi + (\Pi_{\mathcal{U}}F^t\Pi_{\mathcal{U}})u^t - \dfrac{\Pi_{\mathcal{U}}\mu^t(\xi)}{\sqrt{t}} = 0, \, \\
		\label{lem:eq:asymp-dist:lem:kkt-gradient-w}
		&\Pi_{\mathcal{W}}\mu^t = 0, 
		\end{align}
		where $\Pi_{\mathcal{U}}$, $\Pi_{\mathcal{W}}$ are defined in \eqref{eq:asymp-distr:projectors}. 
		In what follows we use the following notations
		\begin{align}\label{lem:eq:asymp-dist:lem:new-notations-u-t}
		C^t_\mathcal{U} = C^t\Pi_{\mathcal{U}}, \, F_{\mathcal{U}}^t = (\Pi_{\mathcal{U}}F^t\Pi_{\mathcal{U}}), \, \mu^t_{\mathcal{U}} = \Pi_{\mathcal{U}}\mu^t.
		\end{align}
		Strong consistency of $\widehat{\lambda}_{sc}^t$ on $\mathcal{U} \oplus \mathcal{V}$ and the Continuous Mapping Theorem imply that 
		\begin{align}\label{lem:eq:asymp-dist:lem:new-notations-u-t-as-convergence}
		C^t_{\mathcal{U}} \rightarrow C^*_{\mathcal{U}}, \, F^t_{\mathcal{U}} \rightarrow F^*_{\mathcal{U}} \text{ when } t\rightarrow +\infty, \, 
		\text{ a.s. } Y^t, \, t\in (0, +\infty),
		\end{align}
		where
		\begin{align}
		&C^*_{\mathcal{U}} = \Pi_{\mathcal{U}}C^*, \, F^*_{\mathcal{U}} = \Pi_{\mathcal{U}}F^*\Pi_{\mathcal{U}},\\
		&C^* =  (D^*_{I_1(\Lambda^*)})^{-1/2}A_{I_1(\Lambda^*)}, \, 
		D^*_{I_1(\Lambda^*)} = \mathrm{diag}(\dots, \Lambda^*_{i}, \dots), \, i\in I_1(\Lambda^*),\\
		\label{lem:eq:asymp-dist:lem:new-notations-u-t-fisher-limit}
		&F^* =  \sum\limits_{i\in I_1(\Lambda^*)}\dfrac{a_ia_i^T}{\Lambda_{i}^*} = A_{I_1(\Lambda^*)})^T (D_{I_1(\Lambda^*)}^*)^{-1}A_{I_1(\Lambda^*)}.
		\end{align}
		Using the notations from \eqref{lem:eq:asymp-dist:lem:new-notations-u-t} formula \eqref{lem:eq:asymp-dist:lem:kkt-gradient-u} can be rewritten as follows:
		\begin{align}\label{lem:eq:asymp-dist:primal-solution-via-mu-t}
		u^t(\xi) = (F^t_\mathcal{U})^{-1}(C_{\mathcal{U}}^t)^T\xi + (F^t_\mathcal{U})^{-1}\frac{\mu^t_{\mathcal{U}}(\xi)}{\sqrt{t}}.
		\end{align}
		Note that $F_{\mathcal{U}}^t$ is continuously invertible on $\mathcal{U}$, therefore $(F_{\mathcal{U}}^t)^{-1}$ is well-defined. Moreover, $(F_{\mathcal{U}}^t)^{-1}\rightarrow (F_\mathcal{U}^*)^{-1}$ for $t\rightarrow +\infty$ a.s. $Y^t$, $t\in (0, +\infty)$.
		Next, we show that the following estimate always holds:
		\begin{align}\label{lem:eq:asymp-dist:lem:bounded-mu-t-estimate}
		\left|\frac{\mu_{\mathcal{U}}^t(\xi)}{\sqrt{t}}\right| \leq 
		2\hspace{-0.2cm}\max_{\sigma\in \sigma_{\mathcal{U}}(F_{\mathcal{U}}^t)} \hspace{-0.2cm}\sigma^{-1/2}  \|(F_{\mathcal{U}}^t)^{-1/2}\|\|(C^t_{\mathcal{U}})^T\xi\|,
		\end{align}
		where $\sigma_{\mathcal{U}}(F^t_{\mathcal{U}})$ denotes the spectrum of $F^t_{\mathcal{U}}$ on $\mathcal{U}$ (which in view of \eqref{lem:eq:asymp-dist:lem:new-notations-u-t-as-convergence}, \eqref{lem:eq:asymp-dist:lem:new-notations-u-t-fisher-limit} contains only non-zero positive elements starting from some $t\geq t_0$).
		
		
		We begin with characterization of mapping $\mu^t_{\mathcal{U}}(\xi)$ via the dual problem in \eqref{lem:eq:asymp-dist:lem:dual-problem-reformulated}.
		
		First, from \eqref{lem:eq:asymp-dist:lem:lagrangian-primal-dual}, \eqref{lem:eq:asymp-dist:lem:dual-problem-reformulated} it follows that 
		\begin{align}
		G^t(\mu) = -\infty \text{ if } \mu \not\in \mathcal{W}^\perp.
		\end{align}
		That is for $\mu \not \in \mathcal{W}^\perp$ the dual problem is unfeasible. In view of this and the strong duality, formulas in \eqref{lem:eq:asymp-dist:lem:dual-problem-reformulated} can be rewritten as follows:
		\begin{align}
		\label{lem:eq:asymp-dist:lem:dual-reformulated-w-perp}
		&G^t(\mu) = \inf_{u\in\mathcal{U}} \mathcal{L}^t(u,0; \mu), \, \mu \succeq 0, \, \mu\in \mathcal{W}^\perp, \\
		\label{lem:eq:asymp-dist:lem:argmax-dual-reformulated-w-perp}
		&\mu^t = \argmax_{\mu\succeq 0, \, \mu\in \mathcal{W}^\perp}G^t(\mu).
		\end{align}
		Using \eqref{lem:eq:asymp-dist:lem:lagrangian-primal-dual}, \eqref{lem:eq:asymp-dist:lem:new-notations-u-t}, the first order optimality condition in \eqref{lem:eq:asymp-dist:lem:dual-reformulated-w-perp} has the following form:
		\begin{align}\label{lem:eq:asymp-dist:dual-deriv-first-order-optimality}
		\begin{split}
		&u^t_{min}(\mu) = (F^t_\mathcal{U})^{-1}(C_{\mathcal{U}}^t)^T\xi + 	(F^t_\mathcal{U})^{-1}\frac{\mu_{\mathcal{U}}}{\sqrt{t}}, \\ &\mu_{\mathcal{U}} = \Pi_{\mathcal{U}}\mu, \, \mu \succeq 0, \, \mu\in \mathcal{W}^\perp.
		\end{split}
		\end{align}
		From \eqref{lem:eq:asymp-dist:lem:lagrangian-primal-dual}, \eqref{lem:eq:asymp-dist:lem:dual-problem-reformulated},  \eqref{lem:eq:asymp-dist:lem:dual-reformulated-w-perp}, \eqref{lem:eq:asymp-dist:dual-deriv-first-order-optimality} it follows that 
		\begin{align}\label{lem:eq:asymp-dist:dual-deriv-lagrangian-on-mu}
		\begin{split}
		G^t(\mu) = \mathcal{L}^t(u_{min}^t(\mu), 0; \mu) &= 
		-\xi^T C^t_{\mathcal{U}}u_{min}^t(\mu) + \frac{1}{2}[u^t_{min}(\mu)]^TF^t_{\mathcal{U}}u^t_{min}(\mu) \\
		&- \mu^T((1-\Pi_{\mathcal{V}})\widehat{\lambda}^t_{sc} + \frac{u^t_{min}(\mu)}{\sqrt{t}}), \\
		&\mu_{\mathcal{U}} = \Pi_{\mathcal{U}}\mu, \, \mu \succeq 0, \, \mu\in \mathcal{W}^\perp.
		\end{split}
		\end{align}
		
		Formulas \eqref{lem:eq:asymp-dist:dual-deriv-first-order-optimality}, \eqref{lem:eq:asymp-dist:dual-deriv-lagrangian-on-mu} imply that 
		\begin{align}\label{lem:eq:asymp-dist:lem:dual-gt-mu-rewritten-constrained}
		\begin{split}
		&G^t(\mu) = -\frac{1}{2}\frac{\mu_{\mathcal{U}}^T}{\sqrt{t}}(F^t_{\mathcal{U}})^{-1}\frac{\mu_{\mathcal{U}}}{\sqrt{t}} - \xi^TC_{\mathcal{U}}^t(F_{\mathcal{U}}^t)^{-1}\frac{\mu_{\mathcal{U}}}{\sqrt{t}} - \mu^T(I-\Pi_{\mathcal{V}})\widehat{\lambda}_{sc}^t, \\
		&\mu \succeq 0, \, \mu\in \mathcal{W}^\perp.
		\end{split}
		\end{align}
		From the facts that $\mu\in \mathcal{W}^\perp$, $\mu \succeq 0$ and the definition of $\mathcal{V}$ in \eqref{eq:asymp-distr:subspace-v}  it follows that 
		\begin{equation}\label{lem:eq:asymp-dist:lem:mu-u-proj-positivity}
		\mu_{\mathcal{U}} = (I-\Pi_{\mathcal{V}})\mu = 
		\begin{cases}
		\mu_{j}, \text{ if } \sum\limits_{i\in I_0(\Lambda^*)}a_{ij} = 0, \\
		0, \text{ otherwise},
		\end{cases} \Rightarrow
		\mu_{\mathcal{U}} = (I-\Pi_{\mathcal{V}})\mu \succeq 0.
		\end{equation}
		From \eqref{lem:eq:asymp-dist:lem:mu-u-proj-positivity} and the fact that $\widehat{\lambda}_{sc}^t \succeq 0$ it follows that 
		\begin{align}\label{lem:eq:asymp-dist:lem:mu-lambdas-sc-positive}
		\mu^T(I-\Pi_{\mathcal{V}})\widehat{\lambda}_{sc}^t = [(I - \Pi_{\mathcal{V}})\mu]^T\widehat{\lambda}_{sc}^t =  \mu_{\mathcal{U}}^T\widehat{\lambda}_{sc}^t \geq 0.
		\end{align}
		From \eqref{lem:eq:asymp-dist:lem:dual-gt-mu-rewritten-constrained} one can see that minimizer $\mu^t$ in \eqref{lem:eq:asymp-dist:lem:argmax-dual-reformulated-w-perp} may not be unique, however, its projection $\mu_{\mathcal{U}}^t$ is unique since functional $G^t(\mu)$ is strongly convex in $\mu_{\mathcal{U}}$. At the same time, from \eqref{lem:eq:asymp-dist:primal-solution-via-mu-t} it follows that only $\mu^t_{\mathcal{U}}$ is essential for $\widetilde{u}^t(\xi)$. In view of  \eqref{lem:eq:asymp-dist:primal-solution-via-mu-t}, \eqref{lem:eq:asymp-dist:lem:dual-gt-mu-rewritten-constrained}, the optimization problem in \eqref{lem:eq:asymp-dist:lem:argmax-dual-reformulated-w-perp} can be rewritten as follows: 
		\begin{align}\label{lem:eq:asymp-dist:lem:mu-t-argming-norm-lin-def}
		\dfrac{\mu_{\mathcal{U}}^t}{\sqrt{t}} = \widetilde{\mu}_{\mathcal{U}}^t
		= \argmin_{\mu_{\mathcal{U}}\in \Pi_{\mathcal{U}}(\R^p_+ \cap \mathcal{W}^\perp)} \dfrac{1}{2}\|(F_{\mathcal{U}}^t)^{-1/2}\mu_{\mathcal{U}} + (F_{\mathcal{U}}^t)^{-1/2}(C^t_{\mathcal{U}})^T\xi\|^2 + \sqrt{t} \mu_{\mathcal{U}}^T\widehat{\lambda}_{sc}^t.
		\end{align}
		From \eqref{lem:eq:asymp-dist:lem:mu-t-argming-norm-lin-def} and the fact that $0\in \Pi_{\mathcal{U}}(\R^p_+ \cap \mathcal{W}^\perp)$ it follows that 
		\begin{align}\label{lem:eq:asymp-dist:lem:functional-at-zero-estimate}
		\dfrac{1}{2}\|(F_{\mathcal{U}}^t)^{-1/2}\widetilde{\mu}_{\mathcal{U}}^t + (F_{\mathcal{U}}^t)^{-1/2}(C^t_{\mathcal{U}})^T\xi\|^2 + \sqrt{t} \mu_{\mathcal{U}}^t\widehat{\lambda}_{sc}^t \leq \|(F_{\mathcal{U}}^t)^{-1/2}(C^t_{\mathcal{U}})^T\xi\|^2,
		\end{align}
		where $\widetilde{\mu}_{\mathcal{U}}^t$ is the solution in \eqref{lem:eq:asymp-dist:lem:mu-t-argming-norm-lin-def}.
		Formulas \eqref{lem:eq:asymp-dist:lem:mu-lambdas-sc-positive}, \eqref{lem:eq:asymp-dist:lem:functional-at-zero-estimate} imply that 
		\begin{align}
		|(F_{\mathcal{U}}^t)^{-1/2}\widetilde{\mu}_{\mathcal{U}}^t + (F_{\mathcal{U}}^t)^{-1/2}(C^t_{\mathcal{U}})^T\xi| \leq \|(F_{\mathcal{U}}^t)^{-1/2}(C^t_{\mathcal{U}})^T\xi\|.
		\end{align}
		which together with inequality $|a + b| \geq |a|-|b|$ imply the following estimate
		\begin{align}\label{lem:eq:asymp-dist:lem:ft-mu-t-upperbound}
		\|(F_{\mathcal{U}}^t)^{-1/2}\widetilde{\mu}_{\mathcal{U}}^t\| \leq 2\|(F_{\mathcal{U}}^t)^{-1/2}(C^t_{\mathcal{U}})^T\xi\|.
		\end{align}
		From \eqref{eq:asymp-distr:subspace-u}, \eqref{lem:eq:asymp-dist:lem:new-notations-u-t},  \eqref{lem:eq:asymp-dist:lem:new-notations-u-t-as-convergence}, \eqref{lem:eq:asymp-dist:lem:new-notations-u-t-fisher-limit} it follows that $F_{\mathcal{U}}^t$ is of full rank on $\mathcal{U}$ (starting from some $t\geq t_0$ a.s. $Y^t$, $t\in (0, +\infty)$), therefore, for large $t$ matrix $(F^t_{\mathcal{U}})^{-1/2}$ is positive definite, injective on $\mathcal{U}$ and, hence,  $|(F^t_{\mathcal{U}})^{-1/2} \widetilde{\mu}_{\mathcal{U}}^t| \geq \min_{\sigma\in \sigma_{\mathcal{U}}(F_{\mathcal{U}}^t)} \sigma^{1/2}  \|\widetilde{\mu}_{\mathcal{U}}^t\|$, where $\sigma_{\mathcal{U}}(\cdot)$ denotes the spectrum of an operator acting on $\mathcal{U}$.  
		
		The above argument with formula  \eqref{lem:eq:asymp-dist:lem:ft-mu-t-upperbound} directly imply \eqref{lem:eq:asymp-dist:lem:bounded-mu-t-estimate}.
		
		Formulas \eqref{eq:lem:asymp-distr:quadr-minimizer-mapping-boundeness:norm-bound}-\eqref{eq:eq:lem:asymp-distr:quadr-minimizer-mapping-boundeness:f-star-def} follow from \eqref{lem:eq:asymp-dist:lem:new-notations-u-t}-\eqref{lem:eq:asymp-dist:lem:new-notations-u-t-fisher-limit},  \eqref{lem:eq:asymp-dist:primal-solution-via-mu-t}, \eqref{lem:eq:asymp-dist:lem:bounded-mu-t-estimate}.

		Lemma is proved.

		\subsection{Proof of Lemma~\ref{lem:asympt-distr:grad-stoch-asymp-normality}}
		\begin{proof}
			In view of step~2 in Algorithm~\ref{alg:npl-posterior-sampling:mri:binned} intensities $\widetilde{\Lambda}_{b,i}^t$ can be represented as follows:
			\begin{align}\label{eq:asymp-dist:lem:normality-representation}
			&\widetilde{\Lambda}_{b,i}^t = \dfrac{1}{\theta^t + t}\sum\limits_{k=1}^{Y_i^t} w_{ik} + \widetilde{r}^t_{b,\mathcal{M},i}, \, i\in I_1(\Lambda^*), \\
			\label{eq:asymp-dist:lem:normality-representation-weights}
			&\{w_{ik}\}_{k=1, \, i = 1}^{\infty, \, d} \text{ are mutually independent}, \, w_{ik}\sim \Gamma(1,1),
			\end{align}
			where
			\begin{align}\label{eq:asymp-dist:lem:remainder-m-representation}
			\begin{split}
			&\widetilde{r}^t_{b, \mathcal{M}, i} \mid \widetilde{\Lambda}_{\mathcal{M},i}^t, Y^t, t \sim \Gamma(\theta^t \Lambda_{\mathcal{M},i}^t, (\theta^t + t)^{-1}), \\
			&\widetilde{\Lambda}_{\mathcal{M},i}^t \text{ are sampled in Algorithm~\ref{alg:wbb-pet-bootstrap:mri:posterior-mixing-param}}.
			\end{split}
			\end{align}
			In particular,
			\begin{align}\label{lem:eq:asymp-dist:lem:normality-remainder}
			\sqrt{t}  r_{b,\mathcal{M}, i}^t  = o_{cp}(1). 
			\end{align}
			Indeed, from \eqref{eq:proof:algorithm-consistency-lem:upper-bound-photons}, \eqref{eq:asymp-dist:lem:remainder-m-representation} and the Markov inequality it holds that
			\begin{align}
			P(\sqrt{t}  r_{b,\mathcal{M},i}^t > \delta\, \mid \, Y^t, t) &\leq 
			\dfrac{\sqrt{t}  \theta^t}{\delta(\theta^t + t)}E[\widetilde{\Lambda}_{\mathcal{M},i}^t \mid Y^t, t]\\ \nonumber
			&\leq \dfrac{\sqrt{t}  \theta^t}{\delta(\theta^t + t)} \sum\limits_{i\in I_1(\Lambda^*)} \dfrac{Y_i^t}{t} \rightarrow 0 \text{ for } t\rightarrow +\infty, \text{ a.s. } Y^t, t\in (0, +\infty), 
			\end{align}
			where $\delta$ is arbitrary positive value.
			
			Using the Central Limit Theorem for sums of $w_{ik}$ in \eqref{eq:asymp-dist:lem:normality-representation}, \eqref{eq:asymp-dist:lem:normality-representation-weights} and the Strong Law of Large Numbers for $Y^t$ (see Theorem~\ref{thm:appendix:limits-poisson}, formula \eqref{eq:appendix:stlln}) and the fact that $\theta^t = o(\sqrt{t})$, we obtain:
			\begin{align}\label{lem:eq:asymp-dist:lem:normality-sums-weights}
			\dfrac{\sqrt{t}}{(\theta^t + t)\sqrt{Y_i^t/t}}\sum\limits_{k=1}^{Y_i^t} (w_{ik}-1) \xrightarrow{c.d.} \mathcal{N}(0, 1) \text{ for } t\rightarrow +\infty, \text{ a.s. }Y^t, t\in (0, +\infty).
			\end{align}
			Due to mutual independence between $w_{ik}$, the above convergence holds for all components $i\in I_1(\Lambda^*)$, hence, as for the vector in $\R^{\#I_1(\Lambda^*)}$.
			
			Using formula \eqref{lem:grad-stoch-asymp-normality:xi-tilde} we obtain:
			\begin{align}\label{lem:eq:asymp-dist:lem:normality-one-index}
			\begin{split}
			A_{I_1(\Lambda^*)}^T (\widehat{D}^t_{I_1(\Lambda^*)})^{-1/2}\widetilde{\xi}^t &= 
			\sum\limits_{i\in I_1(\Lambda^*)}\sqrt{t}   \dfrac{\widetilde{\Lambda}_{b,i}^t - \widehat{\Lambda}_{sc,i}^t}{\widehat{\Lambda}_{sc,i}^t}a_i \\
			&= 
			\sum\limits_{i\in I_1(\Lambda^*)}\sqrt{t}  \dfrac{\widetilde{\Lambda}_{b,i}^t - Y_i^t/t}{\widehat{\Lambda}_{sc,i}^t}a_i + \sum\limits_{i\in I_1(\Lambda^*)}
			\sqrt{t} \dfrac{Y_i^t/t-\widehat{\Lambda}_{sc,i}^t}{\widehat{\Lambda}_{sc,i}^t}a_i.
			\end{split}
			\end{align}
			The first sum is conditionally tight in view of the Prokhorov theorem on tightness of weakly convergence sequences and the result in \eqref{lem:eq:asymp-dist:lem:normality-sums-weights}. Due to \eqref{eq:thm:asympt-distr-main:strongly-consist-estim-i1} the second sum is simply bounded for large $t$ for almost any trajectory $Y^t$, $t\in (0, +\infty)$. These arguments directly imply conditional tightness of 
			$A_{I_1(\Lambda^*)}^T (\widehat{D}^t_{I_1(\Lambda^*)})^{-1/2}\widetilde{\xi}^t$ for almost any trajectory $Y^t$, $t\in (0, +\infty)$.

			Lemma is proved.
		\end{proof}
		
		\subsection{Proof of Lemma~\ref{lem:positivity-bt}}
		\begin{proof}
			Since $B^t(\lambda)$ is proportional to $t$ in  \eqref{eq:conv-approx-bt}, it suffices to prove formula \eqref{eq:lem:positivity-bt-full-range} for normalized process $B^t(\lambda) / t$ which we denote here by $G^t(\lambda)$, that is
			\begin{equation}
			\label{eq:lem:positivity-bt-redef}
			\begin{split}
			G^t(\lambda) &= \sum\limits_{i\in I_1(\Lambda^*)}-(\widetilde{\Lambda}_{b,i}^t - \widehat{\Lambda}^t_{sc,i})\dfrac{\Lambda_i - \widehat{\Lambda}^t_{sc,i}}{\widehat{\Lambda}^t_{sc,i}} + \dfrac{1}{2}
			\sum\limits_{i\in I_1(\Lambda^*)} \dfrac{ (\Lambda_i-\widehat{\Lambda}^t_{sc,i})^2}{\widehat{\Lambda}^t_{sc,i}}\\
			& + \sum\limits_{i\in I_0(\Lambda^*)}\Lambda_i, \, 
			\Lambda_i = a_i^T\lambda, \, i\in \{1, \dots, d\}.
			\end{split}
			\end{equation}
			Note also that minimizers of $B^t$ and of $G^t$ coincide.
			
			From the necessary Karush-Kuhn-Tucker optimality conditions in \eqref{lem:proofs:conv-approx-lambda-appr} (see e.g., \citet{bertsekas1997nonlinear}, Section 3.3) it follows that 
			\begin{align}\nonumber
			&\exists \widetilde{\lambda}_{b,app}^t, \widetilde{\mu}_{b,app}^t \in \R^p_+ \text{ such that }\\
			\label{eq:lem:positivity-bt-lagrangian-gradient}
			&-\sum\limits_{i\in I_1(\Lambda^*)}
			\dfrac{\widetilde{\Lambda}_{b,i}^t - \widehat{\Lambda}^t_{sc,i}}{\widehat{\Lambda}^t_{sc,i}} a_i + \sum\limits_{i\in I_1(\Lambda^*)}
			\dfrac{\widetilde{\Lambda}^t_{b,app,i} - \widehat{\Lambda}^t_{sc,i}}{\widehat{\Lambda}^t_{sc,i}}a_i + 
			\sum\limits_{i\in I_0(\Lambda^*)}a_i - \widetilde{\mu}_{b, app}^t = 0, \\ \nonumber
			&\widetilde{\Lambda}_{b,app}^t = A\widetilde{\lambda}_{b,app}^t,\\
			\label{eq:lem:positivity-bt-lagrangian-slackness}
			& \widetilde{\mu}_{b,app, j}^t  \widetilde{\lambda}_{b,app, j}^t = 0
			\text{ for all } j\in \{1, \dots, p\}.
			\end{align}
			Multiplying both sides of \eqref{eq:lem:positivity-bt-lagrangian-gradient} on $(\widetilde{\lambda}_{b,app}^t - \widehat{\lambda}^t_{sc})$ and using formula \eqref{eq:lem:positivity-bt-lagrangian-slackness} we obtain  following formulas:
			\begin{align}\label{eq:lem:positivity-bt-grad-diff}
			\nonumber
			&-\langle \widetilde{\mu}_{b,app}^t, \widehat{\lambda}_{sc}^t \rangle = -\sum\limits_{i\in I_1(\Lambda^*)}
			\dfrac{(\widetilde{\Lambda}_{b,i}^t - \widehat{\Lambda}^t_{sc,i})(\widetilde{\Lambda}_{b,app,i}^t - \widehat{\Lambda}^t_{sc,i})}{\widehat{\Lambda}^t_{sc,i}} + \sum\limits_{i\in I_1(\Lambda^*)}
			\dfrac{(\widetilde{\Lambda}^t_{b,app,i}-\widehat{\Lambda}^t_{sc,i})^2}{\widehat{\Lambda}^t_{sc,i}} \\
			&\qquad \qquad \qquad + \sum\limits_{i\in I_0(\Lambda^*)} \widetilde{\Lambda}_{b,app,i}^t - \widehat{\Lambda}_{sc,i}^t, \\
			&-\langle \widetilde{\mu}_{b,app}^t, \widehat{\lambda}_{sc}^t \rangle = \sum\limits_{i\in I_1(\Lambda^*)}\widetilde{\Lambda}_{b,i}^t - \widetilde{\Lambda}_{b,app,i}^t - \sum\limits_{i\in I_0(\Lambda^*)}\widehat{\Lambda}^t_{sc,i}.
			\end{align}
			From formulas \eqref{eq:lem:positivity-bt-redef}, \eqref{eq:lem:positivity-bt-lagrangian-gradient},  \eqref{eq:lem:positivity-bt-grad-diff} it follows that
			\begin{align}\label{eq:lem:positivity-bt-gt-at-min}
			G^t(\widetilde{\lambda}^t_{b, app}) &= -\langle \widetilde{\mu}_{b,app}^t, 
			\widehat{\lambda}_{sc}^t \rangle - \dfrac{1}{2} \sum\limits_{i\in I_1(\Lambda^*)}
			\dfrac{(\widetilde{\Lambda}^t_{b,app,i}-\widehat{\Lambda}^t_{sc,i})^2}{\widehat{\Lambda}^t_{sc,i}} + \sum\limits_{i\in I_0(\Lambda^*)}\widehat{\Lambda}^t_{sc,i}. 
			\end{align}
			Using \eqref{eq:lem:positivity-bt-redef}-\eqref{eq:lem:positivity-bt-gt-at-min} we get the following identity:
			\begin{align}\label{eq:lem:positivity-bt-gt-diff-min}
			\begin{split}
			G^t(\lambda) - G^t(\widetilde{\lambda}_{b,app}^t) &= 
			\sum\limits_{i\in I_1(\Lambda^*)}-(\widetilde{\Lambda}_{b,i}^t - \widehat{\Lambda}^t_{sc,i})\dfrac{\Lambda_i - \widehat{\Lambda}^t_{sc,i}}{\widehat{\Lambda}^t_{sc,i}}
			+ \sum_{i\in I_0(\Lambda^*)}\Lambda_i - \widehat{\Lambda}^t_{sc,i}\\
			& + \dfrac{1}{2}\sum\limits_{i\in I_1(\Lambda^*)} \dfrac{(\Lambda_i - \widehat{\Lambda}^t_{sc,i})^2 + (\widetilde{\Lambda}_{b,app,i}^t - \widehat{\Lambda}^t_{sc,i})^2}{\widehat{\Lambda}^t_{sc,i}} + \langle \widetilde{\mu}_{b,app}^t, 
			\widehat{\lambda}_{sc}^t \rangle \\
			& = \sum\limits_{i\in I_1(\Lambda^*)}\dfrac{(\Lambda_i-\widetilde{\Lambda}^t_{b,app,i})^2}{2\widehat{\Lambda}^t_{sc,i}} + 
			\sum\limits_{i\in I_1(\Lambda^*)} \dfrac{(\widetilde{\Lambda}^t_{b,app,i} - \widehat{\Lambda}^t_{sc,i})(\Lambda_i - \widehat{\Lambda}^t_{sc,i})}{\widehat{\Lambda}^t_{sc,i}}\\
			&+ \sum\limits_{i\in I_0(\Lambda^*)}\Lambda_i - \widehat{\Lambda}^t_{sc,i} +
			\sum\limits_{i\in I_1(\Lambda^*)}\widetilde{\Lambda}_{b,app,i}^t-\widetilde{\Lambda}_{b,i}^t + \sum\limits_{i\in I_0(\Lambda^*)} \widehat{\Lambda}_{sc,i}^t \\
			&-\sum\limits_{i\in I_1(\Lambda^*)}(\widetilde{\Lambda}_{b,i}^t - \widehat{\Lambda}^t_{sc,i})\dfrac{\Lambda_i-\widehat{\Lambda}^t_{sc,i}}{\widehat{\Lambda}^t_{sc,i}} \\
			& = \sum\limits_{i\in I_1(\Lambda^*)}\dfrac{(\Lambda_i-\widetilde{\Lambda}^t_{b,app,i})^2}{2\Lambda_i^*} + \sum\limits_{i\in I_0(\Lambda^*)}\Lambda_i + 
			\sum\limits_{i\in I_1(\Lambda^*)}\dfrac{\widetilde{\Lambda}_{b,app,i}^t-\widetilde{\Lambda}_{b,i}^t}{\widehat{\Lambda}_{sc,i}^t}\Lambda_i.
			\end{split}
			\end{align}
			Formulas \eqref{eq:lem:positivity-bt-full-range}-\eqref{eq:lem:positivity-bt-full-range:mub-props} follow from \eqref{eq:lem:positivity-bt-redef} \eqref{eq:lem:positivity-bt-lagrangian-gradient}, \eqref{eq:lem:positivity-bt-lagrangian-slackness}, \eqref{eq:lem:positivity-bt-gt-diff-min}.
			
			Lemma is proved.
		\end{proof}

		\subsection{Proof of Lemma~\ref{lem:thm:asympt-distr:proof:inf-varphi-diff}}
		\begin{proof}
			Consider the following formula 
			\begin{align}
			\label{eq:lem:inf-varphi-diff:diff-partitioning}
			\inf_{\lambda\in C^t_{A,\delta}(\widetilde{\lambda}_{b,app}^t)}\hspace{-0.4cm}[\varphi(\lambda) - \varphi(\widetilde{\lambda}_{app}^t)] = 
			\inf_{\lambda\in C^t_{A,\delta}(\widetilde{\lambda}_{b,app}^t)}\hspace{-0.4cm}[\varphi(\lambda - \varphi(\widetilde{\lambda}_{b,app}^t)] + [\varphi(\widetilde{\lambda}_{b,app}^t) - \varphi(\widetilde{\lambda}_{app}^t)].
			\end{align}
			Recall that $\widetilde{\lambda}_{b,app}^t$ may not be chosen uniquely since the functional $B^t(\lambda)$ is strongly convex only in directions from $\mathrm{Span}\{a_i : i\in I_1(\Lambda^*)\}$ (see formula \eqref{eq:conv-approx-bt}) and it is flat in directions from $\ker A$. 
			From the strong convexity of $B^t(\lambda)$ on $\mathrm{Span}\{a_i : i\in I_1(\Lambda^*)\}$ and formulas \eqref{eq:conv-approx-bt}, \eqref{lem:proofs:conv-approx-lambda-appr},  \eqref{eq:asymp-distr:parametrization-uvw} it follows that 
			$\widetilde{u}_{b,app}^t = \sqrt{t} \Pi_{\mathcal{U}}(\widetilde{\lambda}_{b,app}^t - \widehat{\lambda}_{sc}^t)$ \text{ is unique}.
			At the same time, from \eqref{eq:thm:asympt-distr-main:strongly-consist-estim-i0}, \eqref{eq:asymp-distr:quadr-minimizer-v-zero} and the result of Lemma~\ref{lem:asymp-distr:quadratic-uv-approximation} it follows that 
			\begin{align}
			\widetilde{v}^t_{b,app} = t\Pi_{\mathcal{V}}(\widetilde{\lambda}_{b,app}^t - \widehat{\lambda}_{sc}^t) = o_{cp}(1), 
			\end{align}
			where the above formula is understood as a uniform bound on the set of all possible minimizers $\widetilde{\lambda}_{b,app}^t$. 
			We may assume that for each $t$ there is some unique $\widetilde{v}_{b,app}^t$. 
			
			Then, to choose uniquely $\widetilde{\lambda}_{b,app}^t$ one has to fix its projection onto $\mathcal{W}$ regarding the positivity constraints.
			Consider the following mapping
			\begin{align}\label{eq:lem:inf-varphi-diff:wuv-mapping}
			\begin{split}
			&w(u,v) = \argmin_{\substack{w: \lambda_* + u + v + w\succeq 0 \\ w\in \mathcal{W}}} \varphi(\lambda_* + u + v + w), \\
			&u\in \mathcal{U}, \, v\in \mathcal{V} : (\lambda_* + u + v + \mathcal{W})\cap \R^p_+ \neq \emptyset,
			\end{split}
			\end{align}
			where $\lambda_*$ is the true parameter.
			From the strict convexity of $\varphi(\cdot)$ along $\ker A$ (by the assumption in~\eqref{eq:prelim:penalty-cond-strict-conv}), the definition of $\mathcal{W}$ in \eqref{eq:asymp-distr:subspace-w} and 
			the result of Lemma~\ref{lem:consistency:lem-kernel-continuity} it follows that $w(u,v)$ is one-to-one and continuous in $(u,v)$ on its domain of definition.
			
			Note that 
			\begin{align}\label{eq:lem:inf-varphi-diff:w00-wstart-equiv}
			w_* = w(0, 0) = w_{A,\lambda_*}(0,0),
			\end{align}
			where $w_{A,\lambda}(\cdot, \cdot)$ is defined in \eqref{eq:lem:inf-varphi-diff:wuv-mapping} ($w_{A,\lambda_*}(0,0)$ appears in   Theorems~\ref{thm:wbb-algorithm-consistency},~\ref{thm:wbb-generic-consistency}). 
			The property that $w_*\in \mathcal{W}$ can be proved by the contradiction argument. Assume that $w_*\in \ker A$ but $w_* \not \in \mathcal{W}$ and $w_* \neq 0$. Then, from the definition of $\mathcal{V}$,  $\mathcal{U}$, $\mathcal{W}$ it follows that 
			\begin{align}\label{eq:lem:inf-varphi-diff:w00-wstart-proof-1}
			\exists i\in I_0(\Lambda^*), \, j\in \{1,\dots, p\} : a_{ij} > 0, \, w_{*j} > 0.
			\end{align}
			At the same time from the fact that $w_*\in \ker A$ it follows that 
			\begin{align}\label{eq:lem:inf-varphi-diff:w00-wstart-proof-2}
			0 = \sum_{i\in I_0(\Lambda^*)} a_i^Tw_{*} = 
			\sum\limits_{j=1}^p\left(
			\sum\limits_{i\in I_0(\Lambda^*)}a_{ij}
			\right)w_{*j}
			\end{align}
			Formulas \eqref{eq:lem:inf-varphi-diff:w00-wstart-proof-1}, \eqref{eq:lem:inf-varphi-diff:w00-wstart-proof-2} imply that 
			\begin{align}\label{eq:lem:inf-varphi-diff:w00-wstart-proof-3}
			\exists i'\in I_0(\Lambda^*), \, j'\in \{1,\dots, p\} : a_{i'j'} > 0, \, w_{*j'} < 0.
			\end{align}
			At the same time, from the definition of $I_0(\Lambda^*)$ in \eqref{eq:ind-lors-pos-zeros} it follows that $\lambda_{*j'} = 0$ which together with the results from \eqref{eq:lem:inf-varphi-diff:w00-wstart-proof-3} contradicts the positivity constraint in \eqref{eq:lem:inf-varphi-diff:w00-wstart-equiv}.
			Thus, $w_*\in \mathcal{W}$.

			Let 
			\begin{align}\label{eq:lem:thm:asympt-distr:proof:inf-varphi-diff:wbb-app-diff}
			\widetilde{w}_{b,app}^t = w\left(\Pi_{\mathcal{U}}(\widehat{\lambda}_{sc}^t -\lambda_*) + \dfrac{\widetilde{u}_{b,app}^t}{\sqrt{t}}, 
			\Pi_{\mathcal{V}}(\widehat{\lambda}_{sc}^t -\lambda_*) +  \dfrac{\widetilde{v}_{b,app}^t}{t}\right),
			\end{align}
			where $\widetilde{u}_{b,app}^t$, $\widetilde{v}_{b,app}^t$ are defined in \eqref{lem:proofs:conv-approx-lambda-appr}, \eqref{eq:asymp-distr:parametrization-uvw}, $w$ is the mapping from \eqref{eq:lem:inf-varphi-diff:wuv-mapping}. Recall that   $\widetilde{\lambda}_{b,app}^t$ from  \eqref{lem:proofs:conv-approx-lambda-appr} can be rewritten via the parametrization in \eqref{eq:asymp-distr:parametrization-uvw} as follows
			\begin{align}
			\label{eq:lem:inf-varphi-diff:lambda-b-app-t-w-chosen}
			\widetilde{\lambda}_{b,app}^t = \widehat{\lambda}_{sc}^t + \dfrac{\widetilde{u}_{b,app}^t}{\sqrt{t}} + \dfrac{\widetilde{v}_{b,app}^t}{t} + \widetilde{w}_{b,app}^t,
			\end{align}
			where $\widetilde{w}_{b,app}^t$ is chosen in \eqref{eq:lem:thm:asympt-distr:proof:inf-varphi-diff:wbb-app-diff}.
			For $\widetilde{\lambda}_{b,app}^t$ from \eqref{eq:lem:inf-varphi-diff:lambda-b-app-t-w-chosen} it holds that 
			\begin{align}\label{eq:lem:inf-varphi-diff:labbp-lstar-limit}
			&\widetilde{\lambda}_{b,app}^t \xrightarrow{c.p.} \lambda_* + w_* \text{ for } t\rightarrow +\infty, \text{ a.s. }Y^t, \, t\in (0, +\infty),
			\end{align}
			where $w_*$ is defined in \eqref{eq:lem:inf-varphi-diff:w00-wstart-equiv}.
			
			Indeed, formula \eqref{eq:lem:inf-varphi-diff:labbp-lstar-limit} follows  from the fact that $\Pi_{\mathcal{U}\oplus \mathcal{V}}\widehat{\lambda}_{sc}^t \xrightarrow{c.p.} \Pi_{\mathcal{U}\oplus\mathcal{V}}\lambda_*$, the fact that $\widetilde{u}_{b,app}^t / \sqrt{t} = o_{cp}(1)$, $\widetilde{v}_{b,app}^t / t = o_{cp}(1)$ (see formula \eqref{eq:asymp-distr:quadr-minimizer-v-zero} and results of Lemma~\ref{lem:asympt-distr:grad-stoch-asymp-normality}) and the continuity of mapping $w$.
			
			From the local Lipschitz continuity of $\varphi$ and \eqref{eq:asymp-distr:app-abbp-proj-u}, \eqref{eq:asymp-distr:lem:approximate-approximation:proj-v-convergece},  \eqref{eq:lem:inf-varphi-diff:labbp-lstar-limit} it follows that there exits some universal constant $L > 0$ such that with conditional probability tending to one a.s. $Y^t$, $t\in (0, +\infty)$ it holds that:
			\begin{align}\label{eq:lem:inf-varphi-diff:fixed-remainder-ocp-1}
			\varphi(\widetilde{\lambda}_{b,app}^t) - \varphi(\widetilde{\lambda}_{app}^t) \leq L \|\widetilde{\lambda}_{b,app}^t - \widetilde{\lambda}_{app}^t\|.
			\end{align}
			In particular, from \eqref{eq:asymp-distr:app-abbp-proj-u}, \eqref{eq:asymp-distr:lem:approximate-approximation:proj-v-convergece}, \eqref{eq:lem:inf-varphi-diff:fixed-remainder-ocp-1} it follows that 
			\begin{equation}\label{eq:lem:inf-varphi-diff:second-term-ocp-1}
			\beta^t  (\varphi(\widetilde{\lambda}_{b,app}^t) - \varphi(\widetilde{\lambda}_{app}^t)) = o_{cp}(1).
			\end{equation}
			
			It is left to show that the first term in \eqref{eq:lem:inf-varphi-diff:diff-partitioning} is also of order $o_{cp}(1)$. For this we use extensively the results from~\citet{wets2003lipschitz} on the lipshitz-continuity of inf-projections.
			
			The first term in \eqref{eq:lem:inf-varphi-diff:diff-partitioning} can be rewritten as taking the infimum two times:
			\begin{align}\label{eq:lem:inf-varphi-diff:inf-inf-representation}
			\begin{split}
			\inf_{\lambda\in C^t_{A,\delta}(\widetilde{\lambda}_{b,app}^t)}\hspace{-0.4cm}(\varphi(\lambda) - \varphi(\widetilde{\lambda}_{b,app}^t)) &= \hspace{-0.6cm}
			\inf_{\substack{(u,v)\in C^t_{A,\delta}(\widetilde{\lambda}_{b,app}^t) \\ (u,v)\in \mathcal{U}\times \mathcal{V}}}\hspace{-0.4cm}
			[\varphi_* (\Pi_{\mathcal{U}}(\widetilde{\lambda}_{b,app}^t - \lambda_*) + \frac{u}{\sqrt{t}}, \Pi_{\mathcal{V}}(\widetilde{\lambda}_{b,app}^t - \lambda_*) + \frac{v}{t}) \\
			& - 
			\varphi_* (\Pi_{\mathcal{U}}(\widetilde{\lambda}_{b,app}^t - \lambda_*), \Pi_{\mathcal{V}}(\widetilde{\lambda}_{b,app}^t - \lambda_*))],
			\end{split}
			\end{align}
			where 
			\begin{align}\label{eq:lem:inf-varphi-diff:phi-star-def}
			\begin{split}
			&\varphi_*(u,v) = \inf_{\substack{w : \lambda_* + u + v + w \succeq 0, \\ w\in \mathcal{W}}} \hspace{-0.4cm} \varphi(\lambda_* + u + v + w),\\
			&u\in \mathcal{U}, \, v\in \mathcal{V} : (\lambda_* + u + v + \mathcal{W})\cap \R^p_+ \neq \emptyset.
			\end{split}
			\end{align}
			The expression in the square brackets in \eqref{eq:lem:inf-varphi-diff:inf-inf-representation} is essentially the variation of the inf-projection for $\varphi_*(u, v)$ for parameter $(u,v)\in \mathcal{U} \times \mathcal{V}$ in the vicinity of zero along $\mathcal{U}\oplus \mathcal{V}$. Indeed, this follows from the facts that $\Pi_{\mathcal{U}}(\widetilde{\lambda}_{b,app}^t - \lambda_*)$ and $\Pi_{\mathcal{V}}(\widetilde{\lambda}_{b,app}^t - \lambda_*)$ are both of order $o_{cp}(1)$ and $u/\sqrt{t}$, $v/t$ are also $o_{cp}(1)$ in view of the fact that $(u,v)\in C^t_{A,\delta}(\widetilde{\lambda}^t_{b,app})$.
			
			Using Theorem~3.4 and examples in Section~4 (pp. 278-282) of  \citet{wets2003lipschitz} we find that $\varphi_*(u,v)$ is locally Lipschitz continuous.
			
			Indeed, consider the optimization problem in \eqref{eq:lem:inf-varphi-diff:phi-star-def}, where $(u,v)\in \mathcal{U} \times \mathcal{V}$ is a parameter. Then, the problem can be rewritten as follows:
			\begin{align}\label{eq:lem:inf-varphi-diff:phi0-def}
			&\inf_{w} \varphi_0((u,v); w), \, \varphi_0 : (\mathcal{U}\times \mathcal{V}) \times \mathcal{W} \rightarrow \overline{\R},\\ 
			&\varphi_0((u,v); w) = \begin{cases}
			\varphi(\lambda_* + u + v + w), \, \text{ if } \lambda_* + u + v + w \succeq 0, \\
			+\infty, \text{ otherwise}, 
			\end{cases}
			\end{align}
			where $\overline{\R}$ denotes the extended real line.
			From the fact that $\varphi(\cdot)$ is locally Lipschitz continuous it is easy to see that $\varphi_0$ is locally Lipschitz continuous on $D = \{(u, v, w)\in \mathcal{U}\times \mathcal{V} \times \mathcal{W} : \lambda_* + u + v + w \succeq 0\}$, where the latter is a polyhedral subset of~$\mathcal{U}\times \mathcal{V} \times \mathcal{W}$.
			
			Consider the feasibility mapping 
			\begin{align}\label{eq:lem:inf-varphi-diff:phi0-feasibility-mapping-def}
			S : \mathcal{U}\times \mathcal{V} \rightrightarrows \mathcal{W} \text{ with } S(u,v) = \{w \in \mathcal{W} : \lambda_* + u + v + w \succeq 0\},
			\end{align}
			where $\rightrightarrows$ denotes the property to be a set-valued mapping.
			From \eqref{eq:lem:inf-varphi-diff:phi0-feasibility-mapping-def} one can see that $\mathrm{gph}\, S = D$ ($\mathrm{gph}$ denotes the graph of a  mapping). Therefore, $\mathrm{gph} \, S$ is polyhedral and, hence, the Proposition~4.1 from \citet{wets2003lipschitz} applies to our case (see also Example 9.35 in \citet{rockafellar2009variational}), so mapping $S$ in \eqref{eq:lem:inf-varphi-diff:phi0-feasibility-mapping-def} is Lipschitz continuous on $\mathrm{dom} \, S$ (as set-valued mapping). At the same time, the result of Lemma~\ref{lem:proofs:compactness-domain-penalty} implies that feasibility mapping $S$ is locally bounded which yields level boundedness in $w$ locally uniformly in $(u,v)$ of $\varphi_0(\cdot, \cdot)$. The above properties are exactly the same is in Section~4 of \citet{wets2003lipschitz}, so Theorem~3.4 therein applies to the case of $\varphi_0$ from \eqref{eq:lem:inf-varphi-diff:phi0-def} and $\varphi_*(u,v) = \inf_{w}\varphi((u,v); w)$ is locally Lipschitz continuous. 
			
			Hence, there exists a constant $L > 0$ such that with conditional probability tending to one a.s. $Y^t$, $t\in (0, +\infty)$ the following holds
			
			\begin{align}
			\nonumber
			\mid\varphi_* &(\Pi_{\mathcal{U}}(\widetilde{\lambda}_{b,app}^t - \lambda_*) + \frac{u}{\sqrt{t}}, \Pi_{\mathcal{V}}(\widetilde{\lambda}_{b,app}^t - \lambda_*) + \frac{v}{t}) - 
			\varphi_* (\Pi_{\mathcal{U}}(\widetilde{\lambda}_{b,app}^t - \lambda_*), \Pi_{\mathcal{V}}(\widetilde{\lambda}_{b,app}^t - \lambda_*))\mid \\
			\label{eq:lem:inf-varphi-diff:phi-star-lipschitz-estim}
			& \leq L\left(\frac{\|u\|}{\sqrt{t}} + \frac{\|v\|}{t}\right) 
			\leq L\left(\frac{\delta}{\sqrt{t}} + c\frac{\delta}{t}\right)
			\text{ for any } (u,v)\in C^t_{A,\delta}(\widetilde{\lambda}_{b,app}^t),
			\end{align}
			where $c$ is a positive constant depending only on dimension $p$.
			
			Using formulas \eqref{eq:lem:inf-varphi-diff:inf-inf-representation},  \eqref{eq:lem:inf-varphi-diff:phi-star-lipschitz-estim} and the assumption that $\beta^t = o(\sqrt{t})$ we obtain
			\begin{align}\label{eq:lem:inf-varphi-diff:first-term-ocp-1}
			\beta^t \hspace{-0.5cm} \inf_{\lambda\in C^t_{A,\delta}(\widetilde{\lambda}_{b,app}^t)}\hspace{-0.4cm}(\varphi(\lambda) - \varphi(\widetilde{\lambda}_{b,app}^t)) = o_{cp}(1).
			\end{align}
			Formula \eqref{eq:lem:thm:asympt-distr:proof:inf-varphi-diff} directly follows from \eqref{eq:lem:inf-varphi-diff:second-term-ocp-1}, \eqref{eq:lem:inf-varphi-diff:first-term-ocp-1}.
			
			Lemma is proved.
		\end{proof}

		
		
	\end{appendix}
	
	\bibliographystyle{./babib/ba}
	\bibliography{wbb-pet-all.bib}
	
\end{document}